\documentclass[sn-mathphys-num]{sn-jnl}

\usepackage{graphicx}%
\usepackage{multirow}%
\usepackage{amsmath,amssymb,amsfonts}%
\usepackage{amsthm}%
\usepackage{mathrsfs}%
\usepackage[title]{appendix}%
\usepackage{xcolor}%
\usepackage{textcomp}%
\usepackage[textsize=tiny,color = none]{todonotes}%
\usepackage{manyfoot}%
\usepackage{booktabs}%
\usepackage{algorithm}%
\usepackage{algorithmicx}%
\usepackage{algpseudocode}%
\usepackage{listings}%
\usepackage{bbm}
\usepackage{bm}
\usepackage{dcolumn}
\usepackage[caption=false]{subfig}
\usepackage{mathtools}
\usepackage{siunitx}
\usepackage{booktabs}

\newtheorem{theorem}{Theorem}

\newtheorem{proposition}[theorem]{Proposition}
\newtheorem{lemma}[theorem]{Lemma}

\usepackage{dsfont}
\newcommand{\indicator}[1]{ \mathds{1} [ #1 ] }

\newcommand{\process}[2]{ \{ #1 \}_{ #2 } }

\newcommand{\pnorm}[2]{ \| #1 \|{}_{#2} }

\newcommand{\probability}[1]{ \mathbb{P} [ #1 ] }

\newcommand{\naturalNumbersPlus}{ \mathbb{N}_{+} }

\newcommand{\refProposition}[1]{{\textrm{Proposition~\ref{#1}}}}

\newcommand{\numClusters}{K}

\usepackage[acronym]{glossaries}
\usepackage{glossaries-extra}

\glsdisablehyper
\GlsXtrEnableEntryCounting{acronym}{3}
\setabbreviationstyle[acronym]{long-short}

\newacronym{AIC}{AIC}{Akaike Information Criterion}
\newacronym{BIC}{BIC}{Bayesian Information Criterion}
\newacronym{BMC}{BMC}{Block Markov Chain}
\newacronym{CAIC}{CAIC}{Consistent Akaike Information Criterion}
\newacronym{CPU}{CPU}{Central Processing Unit}
\newacronym{cf-idf}{cf-idf}{Cluster Frequency--Inverse Document Frequency}
\newacronym{DLL}{DLL}{Dynamic-link library}
\newacronym{DNA}{DNA}{Deoxyribonucleic Acid}
\newacronym{DC-BMC}{DC-BMC}{Degree-corrected Block Markov Chain}
\newacronym{GPS}{GPS}{Global Positioning System}
\newacronym{HMM}{HMM}{hidden Markov model}
\newacronym{KL}{KL}{Kullback--Leibler}
\newacronym{RL}{RL}{Reinforcement Learning}
\newacronym{MAB}{MAB}{Multi-Armed Bandit}
\newacronym{MC}{MC}{Markov Chain}
\newacronym[longplural={Markov Decision Processes}]{MDP}{MDP}{Markov Decision Process}
\newacronym{MLE}{MLE}{Maximum-Likelihood Estimator}
\newacronym{MSVC}{MSVC}{Microsoft Visual \texttt{C++}}
\newacronym{NLTK}{NLTK}{Natural Languages Toolkit}
\newacronym{SBM}{SBM}{Stochastic Block Model}
\newacronym{SP500}{S\&P500}{Standard and Poor's 500}
\newacronym{SPECTRA}{SPECTRA}{Sparse Eigenvalue Computation Toolkit as a Redesigned ARPACK}
\newacronym{SVD}{SVD}{Singular Value Decomposition}

\begin{document}

\title[Article Title]{Detection and Evaluation of Clusters within Sequential Data}

\author*[1]{\fnm{Alexander} \sur{Van Werde}}

\author[1,2]{\fnm{Albert} \sur{Senen--Cerda}}

\author[1,3]{\fnm{Gianluca} \sur{Kosmella}}

\author[1]{\fnm{Jaron} \sur{Sanders}}

\affil*[1]{\orgdiv{Dept.~of Mathematics \& Computer Science}, \orgname{TU/e}, \orgaddress{\country{The Netherlands}}}

\affil[2]{\orgname{LAAS--CNRS, IRIT--CNRS, and Universit\'e de Toulouse}, \orgaddress{\country{France}}}

\affil[3]{\orgdiv{Dept.~of Electrical Engineering}, \orgname{TU/e}, \orgaddress{\country{The Netherlands}}}

\abstract{%
	Sequential data is ubiquitous---it is routinely gathered to gain insights into complex processes such as behavioral, biological, or physical processes.
	Challengingly, such data not only has dependencies within the observed sequences, but the observations are also often high-dimensional, sparse, and noisy.
	These are all difficulties that obscure the inner workings of the complex process under study.
    \\ 
    \\ 
	One solution is to calculate a low-dimensional representation that describes (characteristics of) the complex process.
	This representation can then serve as a proxy to gain insight into the original process.
	However, uncovering such low-dimensional representation within sequential data is nontrivial due to the dependencies, and an algorithm specifically made for sequences is needed to guarantee estimator consistency.
	Fortunately, recent theoretical advancements on \glsentrylongpl{BMC} have resulted in new clustering algorithms that can provably do just this in synthetic sequential data.
    \\ 
    \\ 
	This paper presents a first field study of these new algorithms in real-world sequential data; a wide empirical study of clustering within a range of data sequences.
	We investigate broadly whether, when given sparse high-dimensional sequential data of real-life complex processes, useful low-dimensional representations can in fact be extracted using these algorithms.
	Concretely, we examine data sequences containing
	\glsentryshort{GPS} coordinates describing animal movement,
	strands of human \glsentryshort{DNA},
	texts from English writing,
	and daily yields in a financial market.
	The low-dimensional representations we uncover are shown to not only successfully encode the sequential structure of the data, but also to enable gaining new insights into the underlying complex processes.
}

\maketitle

\glsreset{BMC}

\section{Introduction}
Modern data often consists of observations that were obtained from some complex process, and that became available sequentially.
The specific order in which the observations occurred then often matters: future observations frequently correlate with past observations.
By identifying a relation between subsequent observations within the sequential data one may hope to gain insight into the underlying complex process.
The high-dimensional nature of modern data however can make understanding the sequential structure difficult.
For example, on high-dimensional data, many algorithms slow down to an infeasible degree, overfitting may occur, and human interpretation becomes problematic.

In view of the challenges associated with the high dimensionality of processes and/or data, it is desirable to identify a latent structure which respects the sequential structure but has reduced dimensions.
We therefore now focus on a popular class of methods for discovering latent structure in datasets: \emph{clustering algorithms}.
Clustering algorithms work by clustering together data points from a dataset that are ``similar'' in some sense.
Let us illustrate by considering clustering in \emph{nonsequential} data (i.e., data in which the order of the observations does not matter).
If such data has a geometric structure for which a notion of distance is applicable, then one may call two points similar if their distance is small.
This distance-based notion of similarity can then be leveraged with the well-known $K$-means algorithm for clustering point clouds \cite{macqueen1967some}.
Or, if the data instead has a graph structure, then it is natural to call two vertices of the graph similar if they connect to other vertices in similar ways.
This second connection-based notion is then made rigorous in e.g.\ the \gls{SBM} \cite{holland1983stochastic}.

A natural notion of similarity between \emph{sequential} observations---when the exact order of observations really does matter---may similarly be given.
Consider the following informal criterion: ``two observations are similar if and only if they follow after earlier observations in similar ways.''
A recent model which makes this transition-based notion formal are \glspl{BMC} \cite{sanders2020clustering}.
Specifically, the \gls{BMC} model assumes that the observations are the states of a \gls{MC} in which the state space can be partitioned in such a manner that the transition rate between two states only depends on the parts of the partition in which these two states lie.
Each part of the partition is also referred to as a \emph{cluster}.

To give an example, consider the sequence of songs which a user of a music platform listens to.
If they start with a song from the ``Metal" genre, then the next song is likely to be from the same genre.
Once they decide to switch genres, however, the user may be more likely to select the ``Rock" genre than the ``Disco" genre.
The \gls{BMC} model captures such information by allowing the transition probabilities to depend on the clusters---the music genres here---but not to depend on states within a cluster ---the songs of a genre---so that the sequential dependence is entirely captured by the clusters.
Actionable insight based on user data may then be derived from the \gls{BMC} model, for example, by attaching user-specific clusters to recommendation systems, by using the clusters to determine the favorite genre of the user or by categorizing new songs given a small amount of user data.
In a more general application area, algorithms for training agents with reinforcement learning have also recently appeared that use data to cluster the state space to improve the training sample efficiency \cite{zhu2021learning}; see also Section~\ref{sec:Literature__State_space_reduction_in_decision_theoretical_problems}.

The problem of clustering the observations in a single (possibly short) sequence of observations of a \gls{BMC} was recently investigated theoretically \cite{zhang2019spectral,sanders2020clustering}.
For example, given a sample path generated by a \gls{BMC}, an information-theoretic threshold below which exact clustering is impossible because insufficient data is available has been established in \cite[Theorem 1]{sanders2020clustering}.
Further, in \cite[Theorem 3]{sanders2020clustering}, a clustering algorithm for \glspl{BMC} was provided and shown to recover the underlying clusters whenever the implied conditions for recoverability are satisfied; so even when the sequence is short relative to the size of the state space.
The fact that this algorithm is explicitly designed to manage in sparse regimes where the amount of data is small is a favorable property for applications where gathering large volumes of data may be expensive and laborious.
Until now, however, a broad study on the performance of this clustering algorithm when applied to sequential data obtained from actual real-life processes was not provided.
The purpose of the current paper is to address this important gap in the literature. 

Let us remark that our goal is not to compare the performance of the \gls{BMC} clustering algorithm relative to other algorithms.
Indeed, the \gls{BMC} algorithm is explicitly designed to manage in sparse regimes where the amount of data is small.
Most model-free algorithms on the other hand, such as those based on deep learning, excel when one has access to large amounts of training data.
The outcome of a direct comparison would consequently be predetermined by the choice of the amount of training data.
Our goal is rather to study this new clustering algorithm's capabilities to provide meaningful insights into real-life complex processes, and to supplement the theoretical understanding of the \gls{BMC}-based algorithm with a practical viewpoint.
To achieve this goal, we focus on questions such as:
\begin{enumerate}[noitemsep]
	\item How can the \gls{BMC} model practically aid in data exploration of sequential data obtained from real-life data?
	\item How can one statistically decide whether the \gls{BMC} model is an appropriate model for the sequence of observations? How can it be detected that either a simpler model than a \gls{BMC} would suffice, or a richer model is required?
	\item Can the algorithm be expected to give meaningful results despite the sparsity and complexity of real-life data? Is the clustering algorithm robust to model violations?
\end{enumerate}
 
\subsection{Contributions}\label{sec: OurCont}

We investigate the performance of the \gls{BMC}-based algorithm using a diverse collection of datasets that come from the fields of ethology, microbiology, natural language processing, and finance.
Specifically, we investigate sequences of:\nopagebreak
\begin{enumerate}[noitemsep]
	\item[a.]  \gls{GPS} coordinates from animal movements.
	\item[b.] Codons in human \gls{DNA}.
	\item[c.] Words in Wikipedia articles.
	\item[d.] Companies in the \gls{SP500} with the highest daily returns.
\end{enumerate}
\nopagebreak
To each dataset we apply the \gls{BMC}-based clustering algorithm to uncover underlying clusters. 
Our findings are summarized in Section \ref{sec: SummaryDetected} and confirm that the algorithm can uncover relevant latent structure in practice.  

Evaluating the performance of a clustering algorithm and the appropriateness of the model in a real-life scenario can be nontrivial. 
For instance, unlike scenarios with synthetic data, one can not compare with a ground-truth cluster structure.
To answer the second and third research questions raised above, Section \ref{section:summary_methods_for_eval} hence explores a set of experimental tools that incorporating insights from statistics \cite{bozdogan1987model,kullback1951information}, machine learning \cite{lewis2004rcv1}, and random matrix theory \cite{sanders2021spectral,sanders2022singular}.
These tools are applied to the aforementioned real-life datasets in Section \ref{sec:Detected_clusters_within_the_data} and give us insights on the suitability of the model (both positive and negative, depending on the dataset). 

Finally, we programmed a \gls{DLL} in \texttt{C++} that allows efficient simulation and analysis of trajectories of a \gls{BMC}.  
Our source code can be found at \href{https://gitlab.tue.nl/acss/public/detection-and-evaluation-of-clusters-within-sequential-data}{https://gitlab.tue.nl/acss/public/detection-and-evaluation-of-clusters-within-sequential-data}.
We distributed this \gls{DLL} with an easy-to-use Python module called \emph{BMCToolkit} at \url{https://pypi.org/project/BMCToolkit/}.
This approach of interfacing with a \gls{DLL} written in \texttt{C++}, and careful parallelization and compilation, outperformed earlier versions of the module written entirely in Python considerably.
This enabled us to tackle larger sequences with more distinct observations. 

So, to summarize, we evaluated the BMC-based algorithm across diverse real-life datasets and demonstrate its practical applicability, filling a gap in the literature.
Moreover, along the way, we developed experimental evaluation tools and efficient implementation that are expected to be crucial for future practical applications.

\subsection{Summary of the detected clusters}\label{sec: SummaryDetected}
Our findings in the animal movement data are particularly striking.
There, a scatter plot of the data yields a picture which is difficult to interpret (Fig. \ref{fig:movebank_screenshot}).
After clustering, a picture can be displayed which provides significantly more insight (Fig. \ref{fig:bison_clustering}).

Specifically, the graph displayed by the white arrows in Fig. \ref{fig:bison_clustering} gives insight into the global topological structure of the latent dynamics of the animal movements.
Comparing to a satellite image of the area reveals that the boundaries between clusters often correspond to barriers, here rivers, which hinder animal movements.
We emphasize that the algorithm does not access the satellite image: the aforementioned features are found using solely the sequential structure of the data.
In other datasets, it could therefore also be possible to detect structures of different varieties such as breeding sites, human presence, territorial boundaries, roads, or pesticide-caused chemical barriers which may be relevant for animal behavioral studies \cite{belisle2005measuring,urban2001landscape,vuilleumier2006animal,keeley2021connectivity} or wild-life conservation \cite{robert2004effects,taylor2010roads,ruby1994behavioral}.
Let us finally note that this paper's model evaluation tools are found to be informative for this dataset, suggesting room for future methodological expansion.\footnote{Inspired by the current paper's findings, the first and final author have recently considered a different model for a stochastic process with a natural notion of barriers and studied the theory of the associated recovery problem as well as applications to animal movement data \cite{van2024recovering}. 
Notably, however, the barrier recovery algorithms from that paper would not as easily give insight into the global topological structure of the dynamics. 
For such global structure, a clustering-based approach is more insightful.
}

\begin{figure}[!h]
	\centering
	\includegraphics[width= 0.99\linewidth]{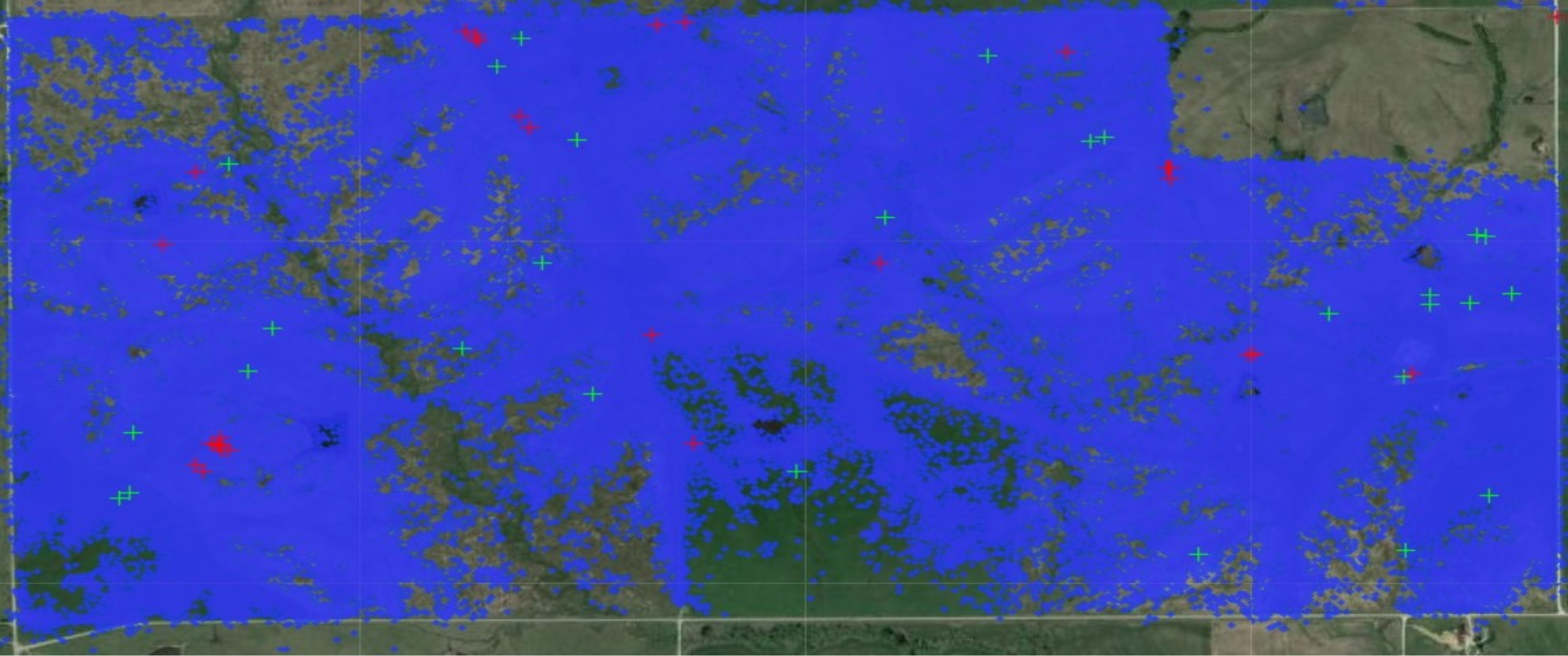}
	\caption{The raw \gls{GPS} data from the ``Dunn Ranch Bison Tracking Project'' (see \cite[\#{8019591}]{movebank}) projected onto a satellite image. Each blue point depicts a single recorded datapoint.
	Note that it is not easy to extract insight from this scatter plot, and one should really aggregate the data in some useful manner.
	The clustering techniques that we implement do this by taking sequential information into account, resulting in the much more insightful Fig.~\ref{fig:bison_clustering} below.}
	\label{fig:movebank_screenshot}
\end{figure}

\begin{figure*}[!h]
	\centering
	\includegraphics[width= 0.99\linewidth]{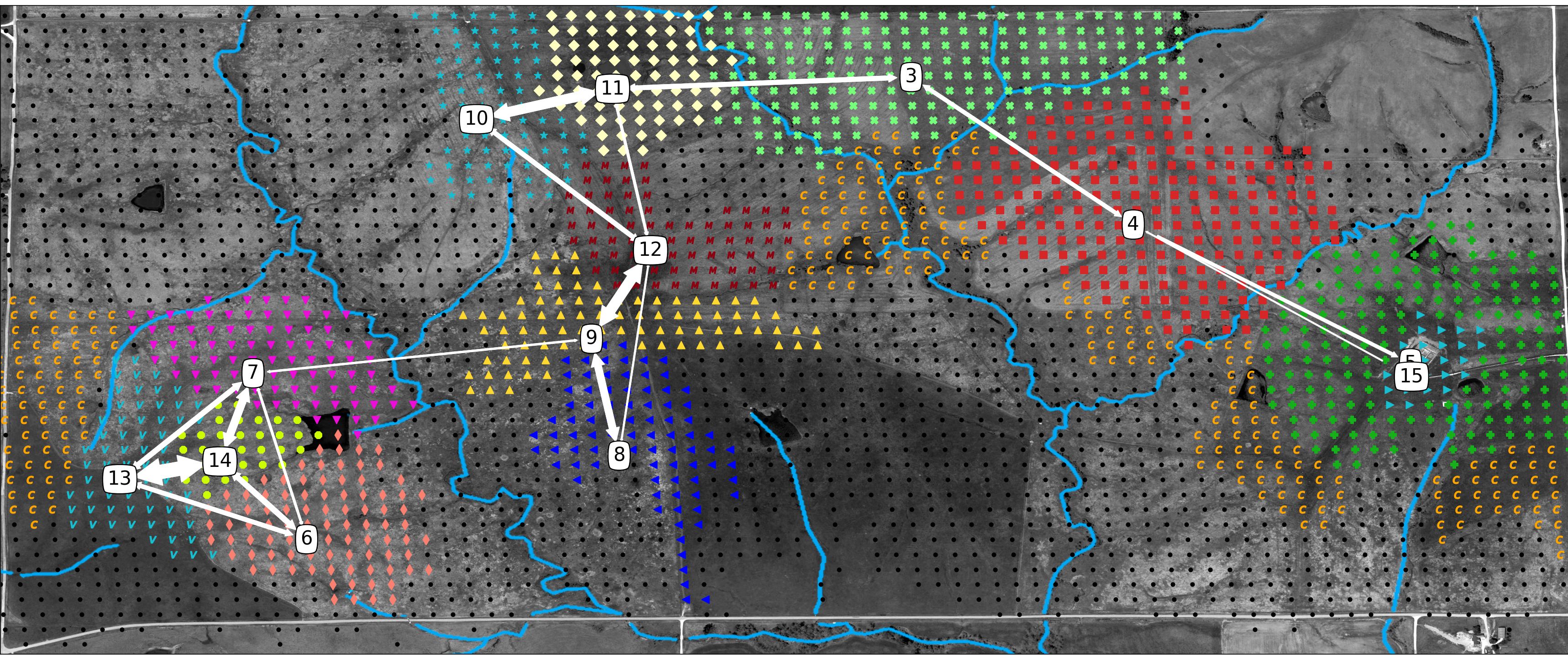}
	\caption{ In the background: a satellite image of Dunn Ranch with rivers highlighted in blue for visualization purposes. In the foreground: the detected clusters as colored bullets, cluster centers indicated by boxes containing the cluster number, and edges between the boxes indicating the transitions between clusters with probability of at least $1\%$. Thicker arrows correspond to higher transition probabilities. Self-transitions and the clusters $1$ and $2$ are omitted, because they are noninformative.}
	\label{fig:bison_clustering}
\end{figure*}

\newpage 
In \gls{DNA}, the algorithm leads us to rediscover phenomena that are known in the genomics community as \emph{codon--pair bias} and \emph{dinucleotide bias} \cite{gutman1989nonrandom,coleman2008virus, kunec2016codon}.
More precisely, in Table \ref{tab: DNA} it may be observed that cluster $k = 2$ mainly contains codons ending with the nucleotide $C$ whereas cluster $k = 3$ mainly contains codons starting with nucleotide $G$.
Closer inspection of the transition rates between these clusters reveals that we only rarely observe transitions from cluster $k = 2$ to cluster $k = 3$: see Fig. \ref{fig:DNA}(a).
In other words, there is a bias against a $C$--to--$G$ transition on the junction between two codons.
It is further interesting to note that our model evaluation tools suggest that, while not perfect, the \gls{BMC} assumption seems reasonable for this dataset; see Section \ref{sec:detected_clusters__seq_of_codons}.

In the text data we consider a document classification task and find that a \gls{BMC}-based cluster improvement algorithm performs better than plain spectral clustering; see Table \ref{tab: perf} for the results and Section \ref{sec:clustering_procedure} for the algorithms.
Recall that high performance here is not our main objective.
Rather, it serves as an evaluation tool (see Section \ref{section:summary_methods_for_eval}).
If one simply desires optimal performance, not an interpretable model, then neural machine learning methods \cite{minaee2021deep} will outperform \gls{BMC}-based methods on this task because large volumes of data are available in natural language processing.
Our point is that because the improvement algorithm exploits the model assumptions more aggressively than the spectral algorithm, the findings suggest that the model itself brings merit.
In Section \ref{sec: NLP_Histograms}, we again find that the evaluation tools are informative, uncovering model violations whose nature we can clarify.

Finally, the \gls{SP500} dataset is distinct as it gives the least clear conclusions.
The difficulty of this dataset is due to the combination of sparsity and a nuisance factor.
We discuss this dataset extensively in Sections \ref{sec:Companies_with_the_highest_daily_returns} and \ref{sec:Detected_orders_within_the_data} as an illustrative dataset for our evaluation tools in a difficult setting.
To summarize: we find that a simpler model called a $0$th-order \gls{BMC} (see Section \ref{sec: Models}) can describe its statistical aspects, while simultaneously that there are indications that a $1$st-order \gls{BMC} is also suitable.

\subsection{Related literature}
\label{sec:Related_literature}

\paragraph{Clustering in \texorpdfstring{\glspl{MC}}{MCs} and random graphs}

Algorithms for detection in \glspl{BMC} have been studied in \cite{zhang2019spectral,sanders2020clustering} including information-theoretic limits stating when it is impossible to recover clusters in \cite{sanders2020clustering}, and estimation of the number of clusters was recently studied in \cite{van2024estimating}.
Other clustering algorithms and models that use spectral decompositions to uncover clusters or low-rank structures based on trajectories of \glspl{MC} are studied in \cite{duan2019state,bi2022low,du2019mode,zhu2021learning}.

The clustering algorithm involves a spectral step that relies on random matrices constructed from sample paths of \glspl{BMC}.
This motivated further theoretical studies of random matrices constructed from Markovian data in \cite{sanders2021spectral,sanders2022singular,van2023matrix}.
In \cite{sanders2022singular}, convergence of singular value distributions in the \gls{BMC} model is established in the dense regime $\ell = \Theta(n^2)$.
We use and refine this result in our experiments.

Community detection in random graphs, such as those produced by the \gls{SBM}, is a closely related area of research.
The distinction with clustering in \glspl{BMC} is that the vertices within a single observation of a random graph are clustered, instead of the observations within sequential data.
We refer the reader to \cite{gao2017achieving} for an extensive overview on cluster recovery within the context of the \gls{SBM}, and to \cite{fortunato2010community} for an overview on community detection in graphs.

\paragraph{Different types of clustering for sequential data}

In the reviews \cite{zolhavarieh2014review, aghabozorgi2015time}, some further lines of research that relate to both clustering and sequential data are divided into three categories.
First, \emph{whole-time-series clustering} groups the trajectories of different time-series \cite{aghabozorgi2015time, liao2005clustering,driemel2016clustering}.
Second, \emph{clustering of subsequences of a time-series} where individual time-series are extracted via a sliding window \cite{lin2003clustering,rakthanmanon2011time, rodpongpun2012selective}.
Finally, there is \emph{time-point clustering} which includes problems like segmenting an $n$-element sequence into $k$ segments, that can come from $h$ different sources; see e.g. \cite{gionis2003finding,morchen2005extracting}.
These three categories are all distinct from the notion
which we employ, but the final category is closest. 

\paragraph{State space reduction in decision theoretical problems}
\label{sec:Literature__State_space_reduction_in_decision_theoretical_problems}

Studying clustering in \glspl{MC} may also be motivated by the necessity for effective state space reduction techniques in decision theoretical problems.
For example, in \gls{RL}, \glspl{MDP}, and \gls{MAB} problems it is known that learning a latent space reduces regret in \gls{MAB} problems \cite{maillard2014latent, lazaric2013sequential}.
State aggregation and low-rank approximation methods have been studied for \glspl{MDP} as well as \gls{RL}, see \cite{li2006towards} and \cite{ong2015value, azizzadenesheli2016reinforcement, yang2019harnessing}, respectively.
The idea to cluster states in \gls{RL} based on the process' trajectory was first explored in \cite{singh1994reinforcement,ortner2013adaptive}.

\paragraph{Some related experiments in microbiology, natural language processing, ethology, and finance}
\label{sec:Literature__A_few_related_experimental_approaches}

Using similar means as in the animal movement data in this paper, \gls{GPS} coordinate sequences for New York City taxi trips are investigated in \cite{zhang2019spectral,bi2022low,sanders2022singular}.
The found low-dimensional representation of the taxi data also gives insight into taxi customer behavior, just as it does in this paper for the animal movement behavior.
The taxi data is however quite different from the animal movement data: taxi transitions tend to be between far away entrance and drop-off locations.

\gls{MC} models for the sequence of nucleotides or codons in \gls{DNA} are considered in \cite{almagor1983markov,jorre1976model,robin2005dna}.
The current paper is the first time that a \gls{BMC} was used for this task.
\glspl{MC} and \glspl{HMM} are often used in natural language processing; see \cite{manning1999foundations}.
In \cite{GIALAMPOUKIDIS2014141} the transition between the Dow Jones closing prices are described as a \gls{MC} close to equilibrium.
Other references for \gls{MC} models in finance include \cite{zhang2009study, van2012asset, mamon2007hidden}.

\subsection*{Structure of this paper}
We introduce the problem of clustering in sequential data in Section \ref{sec:Problem_formulation}.
We describe the \gls{BMC} as well as other models that appear in our experiments in Section \ref{sec: Models}, and briefly discuss the advantages of a model-based approach.
Next, we give an overview of related literature in Section \ref{sec:Related_literature}, and we introduce the clustering algorithm in Section \ref{sec:clustering_procedure}.
We describe there also our \texttt{C++} implementation of this clustering algorithm, which we have made publicly available as a Python library.
Section \ref{section:summary_methods_for_eval} describes practical tools to evaluate clusters found in datasets in the absence of knowledge on the underlying ground truth.
Section \ref{sec:Data_sets_and_preprocessing} introduces the datasets and explains our preprocessing procedures;
Sections \ref{sec:Detected_clusters_within_the_data}, \ref{sec:Detected_orders_within_the_data} then extensively evaluate the clusters detected within these datasets.
Finally, Section \ref{sec:Conclusions} concludes with a brief summary of our findings.

\newpage 
\section{Problem formulation}
\label{sec:Problem_formulation}

We suppose that we have obtained an ordered sequence of $\ell \in \naturalNumbersPlus$ discrete observations
\begin{equation}
	X_{1:\ell}
	:=
	X_1
	\to X_2
	\to \cdots
	\to X_\ell
	\label{eqn:Arbitrary_ordered_sequence_of_observations}
\end{equation}
from some complex process.
The observations can be real numbers or abstract system states; as long as the observations come from a finite set.
We assume specifically that there exists a number $n \in \naturalNumbersPlus$ such that $X_t \in [n] := \{ 1, \ldots, n \}$ for all $t \in [\ell]$.
Here, $n$ can be interpreted as the number of distinct, discrete observations that are possible.

Given such ordered sequence of observations,
we wonder whether there exists a map $\sigma_n : [n] \to [\numClusters]$ with $1 \leq \numClusters \leq n$ an integer,
such that the ordered sequence
\begin{equation}
	\sigma_n(X_{1:\ell})
	:=
	\sigma_n(X_1)
	\to \sigma_n(X_2)
	\to \cdots
	\to \sigma_n(X_\ell)
	\label{eqn:Arbitrary_ordered_sequence_of_clusters_of_observations}
\end{equation}
captures dynamics of the underlying complex process.
Observe that $\sigma_n$ defines \emph{clusters}:
\begin{equation}
	\mathcal{V}_k
	:=
	\bigl\{
	i \in [n]
	\mid
	\sigma_n(i) = k
	\bigr\}
	\label{eqn:Clusters_as_a_function_of_the_map_sigma}
\end{equation}
for $k \in [\numClusters]$.
Furthermore, $\mathcal{V}_k \cap \mathcal{V}_l = \emptyset$ whenever $k \neq l$ and $\cup_{k=1}^\numClusters \mathcal{V}_k = [n]$.

The clusters $\mathcal{V}_1, \ldots, \mathcal{V}_\numClusters$ are particularly interesting when $\numClusters \ll n$.
In such a case the \emph{clustered} process $\process{\sigma_n(X_t)}{t}$ lives in a much smaller observation space than the original process $\process{X_t}{t}$.
The reduction may then prove to be beneficial for computational tasks since the time complexity of some algorithms depends on the size of the observation space.
If \eqref{eqn:Arbitrary_ordered_sequence_of_clusters_of_observations} furthermore indeed captures the dynamics of the complex process, then it is not unreasonable to expect that the clusters $\mathcal{V}_k$ could themselves be meaningful thus allowing for human interpretation of the data.

\section{Preliminaries}

\subsection{Models}\label{sec: Models}

\subsubsection{Main model: \texorpdfstring{\gls{BMC}}{BMC}}
\label{sec:Benchmark_model:BMCs}

Formally, a \emph{$1$st-order \gls{BMC}} is a discrete-time stochastic process $\process{X_t}{t \geq 0}$ on a state space $\mathcal{V}:= [n]$ that satisfies not only the \gls{MC} property
\begin{align*}
	\probability{
		X_{t+1} = j
		\mid
		X_{t} = i,
		\ldots ,
		X_0 = i_0
	}
	=
	\probability{
		X_{t+1} = j
		\mid X_t = i
	}
	\ \forall j,i,i_{t-1},\ldots,i_0 \in [n];
\end{align*}
but also that there exists a cluster assignment map $\sigma_n: [n] \to [\numClusters]$ and a stochastic matrix $p\in \mathbb{R}^{\numClusters\times \numClusters}$ with
\begin{equation}
	P_{i,j}
	:=
	\probability{
		X_{t+1} = j
		\mid
		X_t = i
	}
	=
	\frac{p_{\sigma_n(i), \sigma_n(j)}}{ \# \mathcal{V}_{\sigma_n(j)} }
\end{equation}
with $\mathcal{V}_k$ defined as in
\eqref{eqn:Clusters_as_a_function_of_the_map_sigma}.
Fig. \ref{fig:Schematic_depiction_of_a_BMC} depicts a \gls{BMC} on $\numClusters = 3$ clusters.

The \gls{BMC} model can be viewed as an ideal case for the setup of \eqref{eqn:Arbitrary_ordered_sequence_of_clusters_of_observations}.
The reduced process $\process{\sigma_n(X_t)}{t}$ not only captures \emph{some} part of the dynamics of the true process but rather \emph{all} the order-dependent dynamics.
Indeed, for any $t>1$ it holds that conditional on $\sigma_n(X_t) = k$ the observation $X_t$ is chosen uniformly at random in the cluster $\mathcal{V}_k$.
The previous state $X_{t-1}$ hence influences the next cluster $\sigma_n(X_t)$ but does not provide any further information about the precise element in $\mathcal{V}_{\sigma_n(X_t)}$.

If $p$ defines an ergodic \gls{MC}, then the \gls{BMC} has a unique \emph{state equilibrium distribution} $\Pi\in [0,1]^n$. 
This distribution has the symmetry property that $\Pi_j$ only depends on the cluster assignment $\sigma_n(j)$: 
\begin{align}
	\Pi_j
	 & :=
	\lim_{t \to \infty}
	\probability{ X_t = j \mid X_0 = i_0 }                    \\
	 & = \frac{1}{\# \mathcal{V}_{\sigma_n(j)}} \lim_{t\to \infty} \probability{\sigma_n(X_t) = \sigma_n(j) \mid \sigma_n(X_0) = \sigma_n(i_0)}\nonumber  =: \frac{ \pi_{\sigma_n(j)} }{ \# \mathcal{V}_{\sigma_n(j)} }.\nonumber
\end{align}
Here, $\pi\in [0,1]^\numClusters$ is the equilibrium distribution of the \gls{MC} with transition matrix $p$.

\begin{figure}[hbtp]
	\centering
	\includegraphics[width=0.7\linewidth]{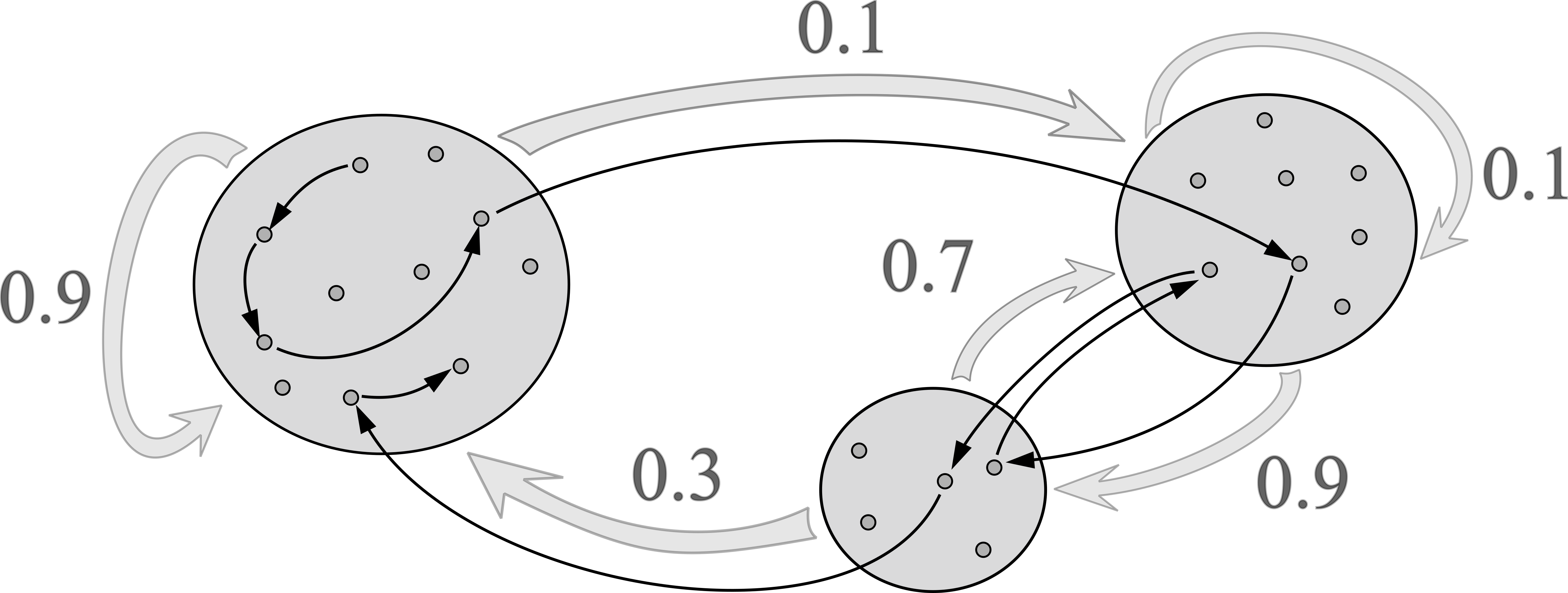}
	\caption{
	A visualization of a \gls{BMC} with $\numClusters=3$ clusters and $p = [[0.9,0.1,0],[0,0.1,0.9],[0.3,0.7,0]]$.
	The thick arrows visualize to the cluster transition probabilities $p_{k,l}$, while the thin arrows visualize the transitions of a sample path $\process{X_t}{t}$.
	Figure courtesy of \cite{sanders2022singular}.
	}
	\label{fig:Schematic_depiction_of_a_BMC}
\end{figure}

\subsubsection{Other models for experimentation}
\label{sec:Other_models_for_experimentation}

Recall that one of our goals is to develop tools for evaluating whether the \gls{BMC} model is appropriate.
In this setting it is often useful to compare with alternative models.
The models that we have used are collected here for easy reference.

\paragraph{\texorpdfstring{$0$}{0}th-order \texorpdfstring{\glspl{BMC}}{BMCs}}

Let $\numClusters \in [n]$ and consider an arbitrary probability distribution $\eta : [\numClusters] \to [0,1]$. A \emph{$0$th-order \gls{BMC}} is then a \gls{BMC} with cluster transition matrix $p_{k,l}:= \eta_l$ for all $k,l \in [\numClusters]$.
The $0$th-order \gls{BMC} will serve as a benchmark to assert whether the structures we find are due to the sequential nature of the process and do not admit a simpler explanation.

Namely, observe that in a $0$th-order \gls{BMC} each next sample $X_{t+1}$ is independent of the previous sample $X_t$.
A $0$th-order \gls{BMC} therefore generates sequences of independent and identically distributed random variables.
This is contrary to a $1$st-order \gls{BMC}, which generates a sequence of dependent random variables.
The probability of a specific observation does depend on the cluster of the observation, and specifically is identical for every observation within that cluster.

\paragraph{\texorpdfstring{$r$}{r}th-order \texorpdfstring{\glspl{MC}}{MC}}

Conversely, it could occur that sequential dependencies are not limited to the single previous observation.
We hence also consider models with higher-order dependencies.

Consider a discrete-time stochastic process $\{ Y_t \}_{t=1}^\ell$ (not necessarily a \gls{MC}) that satisfies $Y_t \in [n]$ for some $n \in \naturalNumbersPlus$.
We say that $\{ Y_t \}_{t \geq 1}$ is an \emph{$r$th-order \gls{MC}} if and only if for all $t \in [\ell - r]$, all $i^r = (i_1, \ldots, i_r) \in [n]^r$ and $j \in [n]$,
\begin{align}
	\probability{
		Y_{t+1}
		=
		j
	 & \mid
		Y_t = i_r
		,
		Y_{t-1} = i_{r-1}
		,
		\ldots,
		Y_{t-r+1} = i_1,
		Y_{t-r} = s_{t-r},
		,
		\ldots,
		Y_{1} = s_{1}
	}
	\nonumber \\
	=
	\probability{
		Y_{t+1}
		=
	 & j
		\mid
		Y_t = i_r
		,
		Y_{t-1} = i_{r-1}
		,
		\ldots,
		Y_{t-r+1} = i_1
	}
	=:
	P^r_{i^r,j}
	\label{eqn:Definition_of_transition_matrix_Pr}
\end{align}
for some transition matrix $P^r \in [0,1]^{n^r \times n}$.
By imposing that the entry $P^r_{i^r,j}$ may only depend on the cluster assignments $\sigma_n(j),\sigma_n(i_1),\ldots,\sigma_n(i_r)$ one gets a model with longer dependencies which still has a ground-truth notion of clusters, called an \emph{$r$th order \gls{BMC}}.

Given such cluster assignments, Section \ref{section:summary_methods_for_eval} provides methods to evaluate what order is the best fit for provided sequential data.
So, in practice, these methods do require the identification of such cluster assignments first.
If one would simply apply the clustering algorithm for $1$st-order \glspl{BMC} to a \gls{BMC} of much higher order (a task for which the algorithm was not explicitly designed), then one must be aware of a few limitations.
Specifically, if $n$ is large and $r > 1$, then the spectral step can become computationally infeasible in practice as the empirical frequency matrix has size $n^r \times n$.
Further, even after clustering, the number of parameter grows exponentially with $r$, so choosing a model with large time dependence risks overfitting the data if its amount does not scale accordingly.
Nonetheless, if one is mainly concerned with goodness--of--fit and not necessarily with interpretability, then a moderately higher order $r$ can be suitable: see Section \ref{sec:Detected_orders_within_the_data} for our findings with real-world data.

\paragraph{Perturbed \texorpdfstring{\glspl{BMC}}{BMCs}}
Finally, we consider an alternative model which concerns the scenario where a \gls{BMC} captures the dynamics only partially.
Specifically, a \emph{perturbed \gls{BMC}} mixes a $1$st-order \gls{BMC} on $[n]$ that has transition matrix $P_{\text{BMC}}$ with a generic $1$st-order \gls{MC} on $[n]$ that has transition matrix $\Delta$ by consideration of the \gls{MC} with transition matrix
\begin{equation}
	P_{\text{Perturbed}}
	:=
	(1-\varepsilon)P_{\text{BMC}}
	+
	\varepsilon \Delta
	.
\end{equation}
The parameter $\varepsilon \in [0,1]$ measures how much the dynamics are affected by the non-\gls{BMC} part $\Delta$.
Whenever we use a perturbed \gls{BMC}, we specify $\Delta$ on the spot.

\newpage 
\subsubsection{Concerning model misspecification}

In practice, it is unlikely that the complex process $\{ X_t \}_t$ is exactly a \gls{BMC}.
One may hence wonder about the dangers of model misspecification:
\begin{enumerate}[label=(\alph{enumi}),ref=\alph{enumi}]
	\item Is the clustering algorithm robust to violations of the model assumption? \label{q: ClusteringRobust?}
	\item When concerned with a downstream task, does the \gls{BMC} model provide any benefit when compared to models with fewer assumptions? \label{q: RelyOnFewerAssumptions?}
\end{enumerate}
In this regard we would like to point out that the data which we consider is not only complex but oftentimes also sparse.
Let us illustrate the principle by a numerical experiment on synthetically generated datasets.

To model a violation of the model assumptions while retaining a sensible notion of ground-truth communities we considered the perturbed \gls{BMC} model as defined in Section \ref{sec:Other_models_for_experimentation}.
The precise setup can be found in Appendix \ref{sec:Robustness_of_the_clustering_procedure}.

Concerning (\ref{q: ClusteringRobust?}), we find that for small perturbation levels $\varepsilon$ it is still possible to exactly recover the underlying clusters; see Fig. \ref{fig:misclassified}(a).

Concerning (\ref{q: RelyOnFewerAssumptions?}), we consider the scenario where the goal is to estimate the transition kernel $P$ of the Markov chain given a sample path of length $\ell$; see Fig. \ref{fig:misclassified}(b).
We find that clustering worsens performance when $\ell$ is large because a lack of expressivity: the true kernel $P$ is not exactly a \gls{BMC}-kernel.
On the other hand, when $\ell$ is small, clustering improves performance because the simplified model makes the estimator less prone to overfitting.
The answer to (\ref{q: RelyOnFewerAssumptions?}) is thus that it can be advantageous to rely on the \gls{BMC} model assumption when data is sparse.

\begin{figure*}[thbp]
	\centering
	\includegraphics[width=0.49\linewidth]{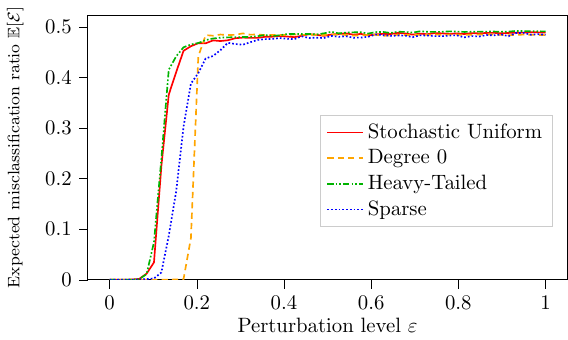}
	\includegraphics[width=0.49\linewidth]{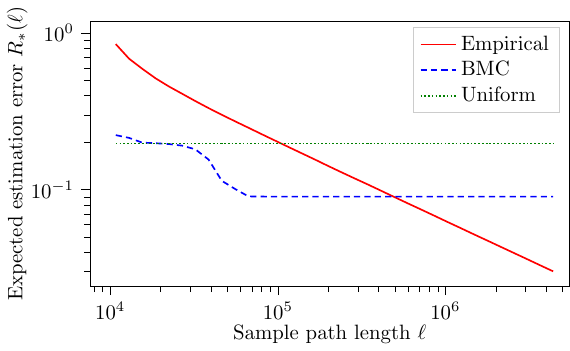}
	\caption{
	(a) The fraction of misclassified states in terms of $\varepsilon$ for various perturbation models $\Delta$. Here,  $\ell_n = \lfloor 30 n\ln(n)\rfloor$ and $n=500$. (b) Estimation error $ R_{*}(\ell):= \mathbb{E}[\pnorm{P - \hat{P}_{*}(\ell)}{}]$ in terms of $\ell$ for three different estimators and data from a perturbed \gls{BMC} with $\varepsilon = 0.05$ and $n=1000$.
    In red: the empirical estimator $\hat{P}_{\text{Empirical}}$ which is the maximum likelihood estimator for a Markov chain with no additional assumptions.
    In blue: the \gls{BMC} estimator $\hat{P}_{\text{BMC}}$.
    In green: the trivial estimator $\hat{P}_{\text{Uniform},ij}:=1/n$  which does not even use the data.
	}
	\label{fig:misclassified}
\end{figure*}

\newpage 
\subsection{Clustering algorithm}
\label{sec:clustering_procedure}

In this section we describe the clustering algorithm from \cite{sanders2020clustering} which was designed to infer the map $\sigma_n$ from the sample path of a \gls{BMC}.
The reason we use this particular clustering algorithm is that it has a mathematical guarantee that it can recover the clusters of \glspl{BMC} accurately \emph{even if} the number of observations $\ell$ is small compared to the number of possible transitions $n^2$.
This is useful for our purposes because observations are generally noisy and few in practice.

The clustering algorithm in \cite{sanders2020clustering} first constructs an \emph{empirical frequency matrix} $\hat{N}$ element-wise from the sequence of observations $X_{1:\ell}$:
for $i,j \in [n]$,
\begin{equation}
	\hat{N}_{ij}
	:=
	\sum_{t=1}^{\ell-1}
	\indicator{ X_t = i, X_{t+1} = j }
	.
	\label{eqn:Definition_of_the_empirical_frequency_matrix}
\end{equation}
Depending on the \emph{sparsity} of the frequency matrix characterized by the ratio $\ell/n^2$, regularization is applied by \emph{trimming}: all entries of rows and columns of $\hat{N}$ corresponding to a desired number of states with the largest degrees, which we denote by $\Gamma$, are set to zero.
The clustering algorithm then executes two steps on the resulting \emph{trimmed frequency matrix} $\hat{N}_\Gamma$:
\begin{itemize}[leftmargin=0.5em]
	\item[Step 1.] Use a spectral algorithm to find an initial approximate cluster assignment.
	\item[Step 2.] Iteratively improve the assignment with a cluster improvement algorithm.
\end{itemize}
We provide pseudocode for these algorithms in Appendix \ref{sec:Pseudo_code_describing_the_clustering_procedure}.

Given some initial guess, here provided by a spectral algorithm, the cluster improvement algorithm consists of local optimization of a log-likelihood function by a hill climbing procedure.
The state space $[n]$ and the number of clusters $\numClusters$ are kept fixed which means that the free parameters are
the cluster transition matrix $p \in \{ q \in [0,1]^{\numClusters \times \numClusters} : \forall k, \sum_l q_{k,l} = 1 \}$ and the cluster assignment map
$\sigma_n : [n] \to [\numClusters]$.
Given an observation sequence $X_{1:\ell}$,
the log-likelihood of the \gls{BMC} model is given by
\begin{equation}
	\hat{\mathcal{L}}
	(
	X_{1:\ell}
	\mid
	p, \sigma_n
	)
	:=
	\sum_{t=1}^{\ell-1}
	\ln{
	\frac{ p_{X_t,X_{t+1}} }{ \# \mathcal{V}_{\sigma_n(X_{t+1})} }
	}.
	\label{eqn:Log_likelihood_of_the_cluster_improvement_algorithm}
\end{equation}

The reason to use this two-step procedure instead of direct likelihood maximization is that finding the global maximizer of \eqref{eqn:Log_likelihood_of_the_cluster_improvement_algorithm} is numerically infeasible.

That hill climbing, which is computationally tractable, succeeds at exactly (resp.\ accurately) recovering the true parameters when initialized with a spectral clustering is formally established in \cite{sanders2020clustering} in the asymptotic regime where $\ell = \omega(n\log n)$ (resp.\ $\ell = \omega(n)$).

\subsection{Methods for evaluating clusters and models}
\label{section:summary_methods_for_eval}

To interpret clustering results and assess model adequacy in the absence of a known ground truth clustering, we require principled evaluation methods tailored to sequential data. 
We use multiple methods and here provide short summaries; the details are given in Appendix \ref{sec:Evaluating_clusters_and_models}.

\paragraph{Performance on a downstream task.}
Clustering can serve as a means of dimensionality reduction when applying computational methods to sequences of observations $X_{1:\ell}$ with a large number of distinct states $n$. 
A clustering $\sigma_n : [n] \to [\numClusters]$ reduces the effective size of the state space to $\numClusters \ll n$, enabling more efficient or more robust downstream computations. 
To evaluate whether the clustering preserves relevant information, we consider a downstream task $T(X_{1:\ell})$ with an associated quality measure $Q$, such as prediction accuracy. 
Let $Q_{\textnormal{pre-reduction}} := Q(T(X_{1:\ell}))$ and $T(\sigma_n(X_{1:\ell}))$ denote the task output after clustering, with quality $Q_{\textnormal{reduced}} := Q(T(\sigma_n(X_{1:\ell})))$. 
The comparison between $Q_{\textnormal{pre-reduction}}$ and $Q_{\textnormal{reduced}}$ provides a concrete proxy for how much useful information is retained through clustering. 
In some cases, $Q_{\textnormal{reduced}}$ may even exceed $Q_{\textnormal{pre-reduction}}$ due to noise reduction in the clustered sequence. 
This method enables the empirical comparison of different clusterings and motivates clustering when the downstream task is numerically intensive or sensitive to overfitting.

\paragraph{Model selection with validation data.}
To compare two candidate models $\mathbb{P}$ and $\mathbb{Q}$ for an observed sequence $x_{1:\ell}$,
we consider a rescaled log-likelihood ratio
\begin{equation}
	\hat{D}(x_{1:\ell}; \mathbb{P}, \mathbb{Q})
	:=
	\frac{1}{\ell} \ln\frac{\mathbb{P}[X_{1:\ell}= x_{1:\ell}] }{\mathbb{Q}[X_{1:\ell} = x_{1:\ell}]}
    .
\end{equation}
This ratio estimates the \gls{KL} divergence rate difference and quantifies how much more likely an observed path is under model $\mathbb{P}$ than model $\mathbb{Q}$.
To reduce the bias, we use a holdout method.
Specifically, we will split the trajectory into two parts: the first half
$
	x_{1:\lfloor \ell/2 \rfloor}
$
will be used for training,
and the second half
$
	x_{\lfloor \ell/2 \rfloor+1:\ell}
$
for validation.
We then use the holdout-based estimate
\begin{equation}
	\hat{D}(
	x_{\lfloor \ell/2 \rfloor+1:\ell}
	;
	\hat{\mathbb{P}}^{X_{1:\lfloor \ell/2 \rfloor}}
	,
	\hat{\mathbb{Q}}^{X_{1:\lfloor \ell/2 \rfloor}}
	),\label{eqn:Less_biased_estimator_of_the_KL_divergence_rate_difference1}
\end{equation}
which reduces the amount of bias when compared to the estimator a standard \gls{KL} divergence estimator. 

\paragraph{Model selection with only training data.}
When validation data is unavailable or data is sparse, we assess model complexity using information criteria rather than held-out performance. 
Specifically, we estimate the order $r$ of a $\numClusters$-state \gls{BMC} (recall Section~\ref{sec:Other_models_for_experimentation}) from the clustered sequence $Y_{1:\ell} = \sigma_n(X_{1:\ell})$, and compare $r$th-order models via the \gls{CAIC} \cite{bozdogan1987model}:
for model $\hat{\mathbb{Q}}^{r,\mathrm{MLE}}$,
\begin{align}
	\mathrm{CAIC}(\hat{Q}^{r,\mathrm{MLE}})
	:=
	-2 \ln{
		\bigl(
		\mathcal{L}( Y_{1:\ell} \mid \hat{Q}^{r,\mathrm{MLE}} )
		\bigr)
	}
	+
	2 \mathrm{DF}(\numClusters,r)
	\bigl(
	1
	+
	\ln{
			(
			\ell - r
			)
		}
	\bigr)
	;
	\label{eqn:CAIC}
\end{align}
see Appendix \ref{sec: OrderSelection} for the details.
Here, $\mathrm{DF}(\numClusters,r)$ denotes the degrees of freedom in an $r$th-order \gls{MC} constrained to have fixed parameters $\numClusters$ and $r$.
Each candidate model $\hat{\mathbb{Q}}^{r,\mathrm{MLE}}$ is fit by maximum likelihood to obtain a transition matrix $\hat{Q}^{r,\mathrm{MLE}}$, and evaluated using a penalized log-likelihood that accounts for model complexity via $\mathrm{DF}(\numClusters, r) = \numClusters^r(\numClusters - 1)$. 
The selected order $r^{\mathrm{CAIC}}$ minimizes the CAIC and balances goodness--of--fit with parsimony. 
This approach allows us to detect under- or overfitting while avoiding bias due to overparameterization in the absence of explicit data splitting.

\paragraph{The shape of spectral noise for identification of alternative models.}
Theory in the \gls{BMC} model predicts that the leading $\numClusters$ singular values of the empirical frequency matrix $\hat{N}$ reflect the signal, while the remaining $n-K$ singular values can be interpreted as noise \cite{sanders2021spectral, sanders2022singular}.
The dependence of this noise profile on the structure of the \gls{BMC} is characterized in \cite{sanders2022singular}. 
We can use this as a model evaluation tool: we can visualize the empirical spectral noise as a histogram and compare with theory.

However, we found that the spectrum of $\hat{N}$ can be misleading as it tends to be dominated by the effect of an inhomogeneous equilibrium distribution which is common in real-world data. 
To address this, we instead examine the \emph{empirical normalized Laplacian} $\hat{L}$, defined element-wise by
\begin{align}
	\hat{L}_{ij}
	:=
	\begin{cases}
		\frac{\hat{N}_{ij}}{\sqrt{\sum_{k=1}^n \hat{N}_{ik}}\sqrt{\sum_{k=1}^n \hat{N}_{kj}}} & \textnormal{if } \hat{N}_{ij} \neq 0, \\
		0                                                                                     & \textnormal{otherwise.}               \\
	\end{cases}
	\label{eqn:Definition_of_the_empirical_Laplacian}
\end{align}
We characterize the spectral noise profile for this matrix in Proposition \ref{prop:Limiting_singular_value_distribution_of_the_Laplacian} and expect it to be more robust to equilibrium imbalances. 
This provides a complementary, unsupervised tool for diagnosing model mismatch and identifying that richer structures may be present without an explicit alternative model.

\section{Experimental setup}

\subsection{Data sets and preprocessing}
\label{sec:Data_sets_and_preprocessing}

We here introduce the data sets and our preprocessing; see Table \ref{tab: Data} for a summary.
The empirical frequency matrices resulting from this preprocessing, and examples of preprocessed trajectories are made available in the supplementary materials.

\paragraph{Sequence of animal positional data}

We use data from the ``Dunn Ranch Bison Tracking Project'' \cite[\#{8019591}]{movebank} that provides \gls{GPS} animal movement data as a sequences of latitude-longitude coordinates; recall Fig. \ref{fig:movebank_screenshot}.
For example, the data of one animal starts as follows:
\begin{equation*}
	(  40.4749 , -94.1129) \to 
	( 40.4748 , -94.1130) \to
	(  40.4749, -94.1129) \to \ 
	\emph{ et cetera}
	.
\end{equation*}
The study provides data from $24$ animals which we concatenated to a single observation sequence.
As preprocessing, we also excluded some outlier \gls{GPS} coordinates outside rectangular $\SI{3.2}{\kilo\meter} \times \SI{1.7}{\kilo\meter}$ caused by malfunctions of the tracking device.

If we assume that every \gls{GPS} coordinate yields a distinct state of a \gls{BMC}, then clustering would be infeasible because there would be as many states as observations.
We therefore combine \gls{GPS} coordinates by binning over a grid of squares with width $\SI{0.04}{\kilo\meter}$, chosen by ad-hoc parameter tuning; see Appendix \ref{sec:appendix_GPS_to_states} for details.
After preprocessing and binning, the sequence becomes
\begin{equation*}
	X_1 = \textnormal{Bin 0}
	\to
	X_2 = \textnormal{Bin 1}
	\to
	X_3 = \textnormal{Bin 0}
	\to
	X_4 = \textnormal{Bin 0}
	\to
	\emph{et cetera}.
	\nonumber
\end{equation*}
We finally eliminated self-jumps such that resting animals do not disturb the findings.
We end up with $n=3155$ states and a sequence of length $\ell=193134$.

\paragraph{Sequence of codons in \texorpdfstring{\gls{DNA}}{DNA}}

A string of DNA can be viewed as a sequence composed of four possible nucleotides, denoted A, T, C, and G. 
These are processed in protein synthesis in three-letter words called \emph{codons}. 
For instance, the codon ACG corresponds to addition of the amino acid threonine as the next building block of a protein. 
Given a sequence of nucleotides like
\begin{equation}
	\textnormal{TTTGTAGTTAGATCTCCTCTATCC}
	\emph{ et cetera}, \nonumber
\end{equation}
it is hence natural to focus on the associated sequence of codons:
\begin{equation}
	X_1 = \textnormal{TTT}
	\to X_2 = \textnormal{GTA} \to
	\cdots
	\to
	X_8 = \textnormal{TCC}
	\to
	\emph{et cetera}.
	\nonumber
\end{equation}
We consider data from the OCA2 gene in human \gls{DNA} \cite{dna}.
The specific gene is merely illustrative: the clustering algorithms can be applied to any gene, and we expect similar results.
We find $\ell = 16 \times 10^4$ transitions and a state space of size $n = 64$.

\paragraph{Sequence of words in texts}
A cleaned corpus based on the Wikipedia datadump of October 2013 was downloaded from \cite{wiki}.
Further preprocessing was standard: we removed all punctuation and numbers, reduced to a root word with the \gls{NLTK}'s \texttt{PorterStemmer.stem()} \cite[Section 3.6]{bird2009natural}, and pruned the $100$ most used words and words with fewer than $1000$ occurrences.
For example, a paragraph such as
\begin{equation}
	\text{\emph{Clustering observations can be very useful!}}\nonumber 
\end{equation}
is converted into the sequence
\begin{equation}
	X_1 = \textnormal{cluster}
	\to
	X_2 = \textnormal{observ}
	\to
	\cdots
	\to
	X_{6} = \textnormal{use}
	.
	\nonumber
\end{equation}

Each $s$th Wikipedia page results in a sequence that is relatively short. 
The corresponding frequency matrix $\hat{N}^{s}$, recall \eqref{eqn:Definition_of_the_empirical_frequency_matrix}, is hence excessively sparse.
We therefore compute and work instead with
$
	\hat{N}
	:=
	\sum_s \hat{N}^{s}
	.
	\label{eqn:Words__Summed_empirical_frequency_matrix}
$
The diagonal of the matrix is further set to zero because self-transitions are common and not particularly informative for the purpose of clustering.
Pruning these removes a potential bias towards homophilic clusters.
We end up with a vocabulary of $n=16994$ words and $\ell \approx 2 \cdot 10^8$ transitions.

\paragraph{Sequence of companies with the highest daily return}

Daily pricing data for every company in the S\&P500 index was downloaded from \cite{alphavantage}.
The data did not span the same time range, so we only retained the $300$ companies with the most complete data.
We determined the times $t_-^{i}$ and $t_i^+$ of the first and final data entry of each constituent consider the time range from
$
	t_0 := \max_{i \leq 300}t_-^{i}$
to
$t_0 +\ell := \min_{i \leq 300} t_+^{i}$.
It turned out that $t_0 = \textnormal{2001--07--26}$ and $t_0 + \ell = \textnormal{2021--10--22}$.
Days without data, such as weekends when the market is closed, were ignored.

Let $O_t^{i}$ and $C_t^i$ denote the opening price and closing price of company $i$'s stock on day $t$, respectively.
We considered the company with the highest daily return:
\begin{equation}
	X_t
	\in
	\operatorname{argmax}\limits_{i \leq 300} C_t^i / O_t^i
	\label{eqn:Definition_of_Xt_for_the_stock_market}
\end{equation}
The resulting sequence of company tickers starts with
\begin{gather}
	X_{t_0} = \textnormal{ADI} \to
	X_{t_0+1} = \textnormal{AES} \to X_{t_0 + 2} = \textnormal{PVH} \to
	\cdots
	.\nonumber
\end{gather}
We again eliminate self-jumps and end up with $\ell\approx 24\times 10^2$ transitions on a state space of size $n=300$.
\begin{table}[ht]
    \caption{Summary of the used datasets. The final two columns are only approximations showing the order of magnitude.}
    \label{tab: Data}
    \centering 
    \footnotesize
    \renewcommand{\arraystretch}{1.1}
    \begin{tabular}{ccccc} 
        \toprule 
        Dataset&\#States $n$&\#Transitions $\ell$&Visits per state $\ell/n$&Sparsity $\ell/n^2$ 
        \\ 
        \midrule
        Codons in DNA&$64$&$16\times 10^4$&$2500$&$40$
        \\   
        Animal movements&$3155$&$19\times 10^4$&$60$&$0.02$\\ 
        Words in text&$16994$&$2\times10^8$&$10000$&$0.7$\\
        Companies S\&P500&$300$&$2\times 10^3$&$8$&$0.3$\\ 
        \bottomrule
    \end{tabular}
\end{table}

\subsection{Implementation description for BMCToolkit}

To tackle large sequences of observations, we programmed a \gls{DLL} in \texttt{C++} that can simulate and analyze trajectories of \glspl{BMC}.
Among other functionalities, the \gls{DLL} is able
to calculate both cluster and state variants of the equilibrium distribution, frequency matrix, and transition matrix of a \gls{BMC};
to compute the difference between two clusters and the spectral norm;
to estimate the parameters of a \gls{BMC} from a sample path;
to execute the spectral clustering algorithm and the cluster improvement algorithm;
to generate sample paths and trimmed frequency matrices;
and to relabel clusters according to the size or the equilibrium probability of a cluster.

The \gls{DLL} utilizes \emph{Eigen}, a high-level \gls{DLL} for linear algebra, matrix, and vector operations; and the \emph{\gls{SPECTRA}}, a \gls{DLL} for large-scale eigenvalue problems built on top of Eigen.
The mathematical components of BMCToolkit were validated through functional testing using Microsoft's Native Unit Test Framework.
The performance of the numerical components of BMCToolkit were finally benchmarked using \emph{Benchmark}, Google's microbenchmark support library.
Our source code can be found at \url{https://gitlab.tue.nl/acss/public/detection-and-evaluation-of-clusters-within-sequential-data}.

We also created a Python module called \emph{BMCToolkit}, and made it available at \url{https://pypi.org/project/BMCToolkit/}.
This Python module distributes the \gls{DLL} mentioned above and includes an easy-to-use Python interface.
When compiling BMCToolkit, we made sure to instruct the \gls{MSVC} compiler to activate the \emph{OpenMP} extension to parallelize the simulation across \glspl{CPU} and so that Eigen could parallelize matrix multiplications (/openmp); to apply maximum optimization (/O2); to enable enhanced \gls{CPU} instruction sets (/arch:AVX2); and to explicitly target 64-bit x64 hardware.

\section{Results}
We now evaluate how well the \gls{BMC} model can capture the structure of the sequential data introduced in Section \ref{sec:Data_sets_and_preprocessing} and if it can yield useful insights.
Specifically, the detected clusters and our findings are discussed in Section \ref{sec:Detected_clusters_within_the_data}, and we study what order of the \gls{MC} best fits the data in Section \ref{sec:Detected_orders_within_the_data}.  

\subsection{Detected clusters within the data}
\label{sec:Detected_clusters_within_the_data}

\subsubsection{Animal movement data}
\label{sec:AnimalMovementData}

We here investigate the \gls{GPS} animal movement data from the \emph{Dunn Ranch Bison Tracking Project}; recall Section \ref{sec:Data_sets_and_preprocessing}.

\paragraph{Subjective evaluation}

The results of the clustering algorithm are depicted in Fig. \ref{fig:bison_clustering}.
It is subjectively evident that the clusters give more insight than the scatter plot in Fig. \ref{fig:movebank_screenshot}.

Observe that the clustering algorithm picks up on geographical features: all clusters are connected regions, except for the largest two clusters $1$ (black dots) and cluster $2$ (orange $c$'s).
Clusters $1$ and $2$ contain the low degree states which explains their geographical spread.
For the other clusters geographical boundaries are visible.
For example, cluster $3$ is bounded from below by creeks and cluster $4$ lies between two creeks.
On satellite imagery one can see a fence north of $7$ and the part of $2$ that is bordering $7$ and in fact, the northern border of these two clusters follows that line.

Let us emphasize that the fact that the clusters respect the underlying geography and barriers is a nontrivial observation: the clustering algorithm identifies states by numbers and does \emph{not} use geographical information on the state labeling.
The labels of the states are in fact arbitrary to the algorithm, states labeled e.g.\ $10$ and $11$ need not be close to each other geographically.
Hence, geographically mixed clusters would also have been a valid outcome of the algorithm.

Let us note that the average rate of transitions within each cluster is $0.79$.
The transitions shown on the map thus do not represent the majority of transitions, but only the transitions between different clusters that occur with probability of at least $0.01$.
The cluster transitions matrix is given in Appendix \ref{sec:appendix_bison_cluster_transition}.

\paragraph{Comparing the histogram of singular values to the limiting distribution of singular values of the inferred \texorpdfstring{\gls{BMC}}{BMC}}
Fig.~\ref{fig:L_animal} next compares the spectral noise of \eqref{eqn:Definition_of_the_empirical_frequency_matrix} and \eqref{eqn:Definition_of_the_empirical_Laplacian} to the theoretical predictions for \glspl{BMC} (see Proposition \ref{prop:Limiting_singular_value_distribution_of_the_Laplacian}).
Observe that with $\numClusters = 15$ clusters, the theoretical prediction captures the general shape of the distribution, but is inaccurate for the smallest and largest singular values especially.
With more clusters, $\numClusters=100$, the theoretical prediction for the distribution of singular values is found to predict the distribution of singular values better across the entire range.
The prediction however remains imperfect.
The peak at zero is probably linked to the fact that there are many states with a low degree.

\begin{figure}[htbp]
	\centering
	\includegraphics[width= 0.99\linewidth]{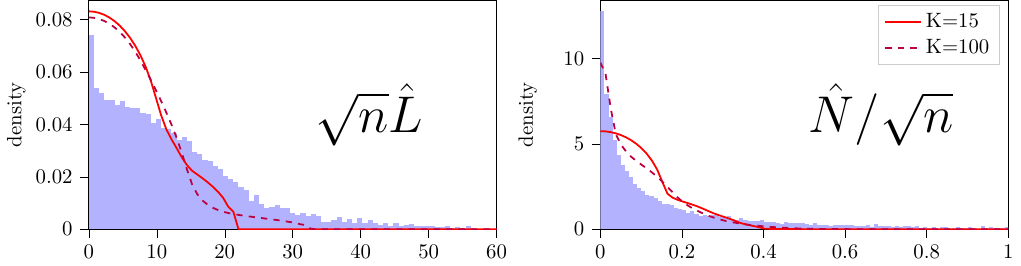}
	\caption{Density-based histogram of singular values for $\sqrt{n}\hat{L}$ and $\hat{N}/\sqrt{n}$ for the animal movement data in blue bars and the theoretical predictions associated with the improvement clustering with $\numClusters = 10$ as the red line and with $\numClusters=100$ as the purple dashed line.}
	\label{fig:L_animal}
\end{figure}

\paragraph{Conclusion}
We conclude that a \gls{BMC} is a useful model for describing animal movement data.
In fact, surprisingly, the clustering algorithm manages to deduce underlying geographical information (such as regions, barriers, and movement patterns) from the mere time dependency within the observation sequence.
Because of this visuo-spatial ability, the algorithm may have a potential use as a tool for spatial recognition.

At the same time however, we also conclude that a \gls{BMC} does not describe the underlying complex process in its entirety.
For example, the distribution of singular values depicted in Fig. \ref{fig:L_animal} is not predicted perfectly.
This is likely caused by the symmetry assumption between states within a \gls{BMC}, which is at odds with the geographical structure of the data.
Indeed, if we cut the region into more but smaller clusters and thus reduce the amount of symmetry within the \gls{BMC} modeling the observation sequence, the \gls{BMC}'s prediction of the distribution of singular values improves.

\subsubsection{Sequence of codons in \texorpdfstring{\gls{DNA}}{DNA}}
\label{sec:detected_clusters__seq_of_codons}

We consider the sequence of codons occurring in the gene OCA2 in human \gls{DNA}.
The detected clusters are displayed in Table \ref{tab: DNA}
\begin{table}[htbp]
	\caption{ The detected clusters of codons in a short sequence of human \gls{DNA}. Observe that many codons in $k = 2$ end with $C$, and that all codons in $k = 3$ start with $G$.}
	\label{tab: DNA}
	\setlength{\tabcolsep}{3pt}
	{
		\scriptsize
		\def\arraystretch{1.5}
		\begin{tabular}{cp{35em}}
			\toprule
			$k$ & Codons within detected cluster $k$                                                                                                     \\
			\midrule
			1   & AAA, AAG, TGT, AGT, CCT, TCT, ACT, CAG, ATT, ATG, CAT, TAT, AAT, TTG, CTT, TGA, CTG, CAA, TGG, ATA, TTA,  AGG, TAA, ACA, TCA, CCA, AGA
			\\
			2   & CAC, GCC, CCC, TCC, ACC, GTC, CTC, TTC, ATC, TGC, AGC, TAC, AAC, GGC, TAG, CTA, GAC                                                    \\
			3   & GTG, GAG, GGT, GCA, GAA, GTA, GGA, GAT, GGG, GTT, GCT                                                                                  \\
			4   & CGA, CGC, ACG, TCG, CCG, GCG, CGT, CGG                                                                                                 \\
			5   & TTT                                                                                                                                    \\
			\bottomrule
			\vspace{0pt}
		\end{tabular}}
\end{table}
\paragraph{Possible detection of codon--pair bias}
The frequency matrix, displayed after clustering, reveals an interesting pattern; see Fig. \ref{fig:DNA}(a).
We observe that all rows and columns associated with the second-to-last cluster $\mathcal{V}_4$ have low density.
This means that the states in $\mathcal{V}_4$ have small equilibrium distribution.
More interesting is the low-density block in the rows and columns corresponding to the transitions from $\mathcal{V}_2$ to $\mathcal{V}_3$.
It appears we have rediscovered a phenomenon known as \emph{codon--pair bias} in biology \cite{gutman1989nonrandom,coleman2008virus, kunec2016codon}. 

There is some evidence that codon--pair bias is nothing more than a consequence of \emph{dinucleotide bias} \cite{kunec2016codon}.
Here, the term dinucleotide bias refers to the fact that the two-letter pair $\textnormal{CG}$ is used infrequently regardless of its position.
This dinucleotide bias can also explain the clusters observed in Fig. \ref{fig:DNA}(a).
Indeed, inspection of the clusters $\mathcal{V}_1, \ldots, \mathcal{V}_5$ reveals that nearly all codons in $\mathcal{V}_2$ end with the nucleotide $\textnormal{C}$ whereas all codons in community $\mathcal{V}_3$ begin with nucleotide $\textnormal{G}$.
There are a few exceptions, the codons TAG and CTA in $\mathcal{V}_2$, but visual inspection of $\hat{N}$ suggests that these may have been misclassified.
Thus, transitions from $\mathcal{V}_2$ to $\mathcal{V}_3$ would give rise to the two nucleotides $\textnormal{CG}$ on the interface.
Also remark that the two leftmost vertical low-density streaks in the block associated with $\mathcal{V}_2$ correspond to codons GCC and GTC which simultaneously begin with a $\textnormal{G}$ and end with a $\textnormal{C}$.
Finally, all codons in $\mathcal{V}_4$ contain the two nucleotides $\textnormal{CG}$.
It thus appears that all low-density regions in the figure could be explained through dinucleotide bias.
We refer to \cite{alexaki2019codon} and the references therein for further discussions of codon--pair bias, dinucleotide bias and their applications.

\paragraph{Comparing the histogram of singular values to the limiting distribution of singular values of the inferred \texorpdfstring{\gls{BMC}}{BMC}}

\begin{figure*}[t]
	\centering
	\subfloat[]{\includegraphics[height=0.35\linewidth]{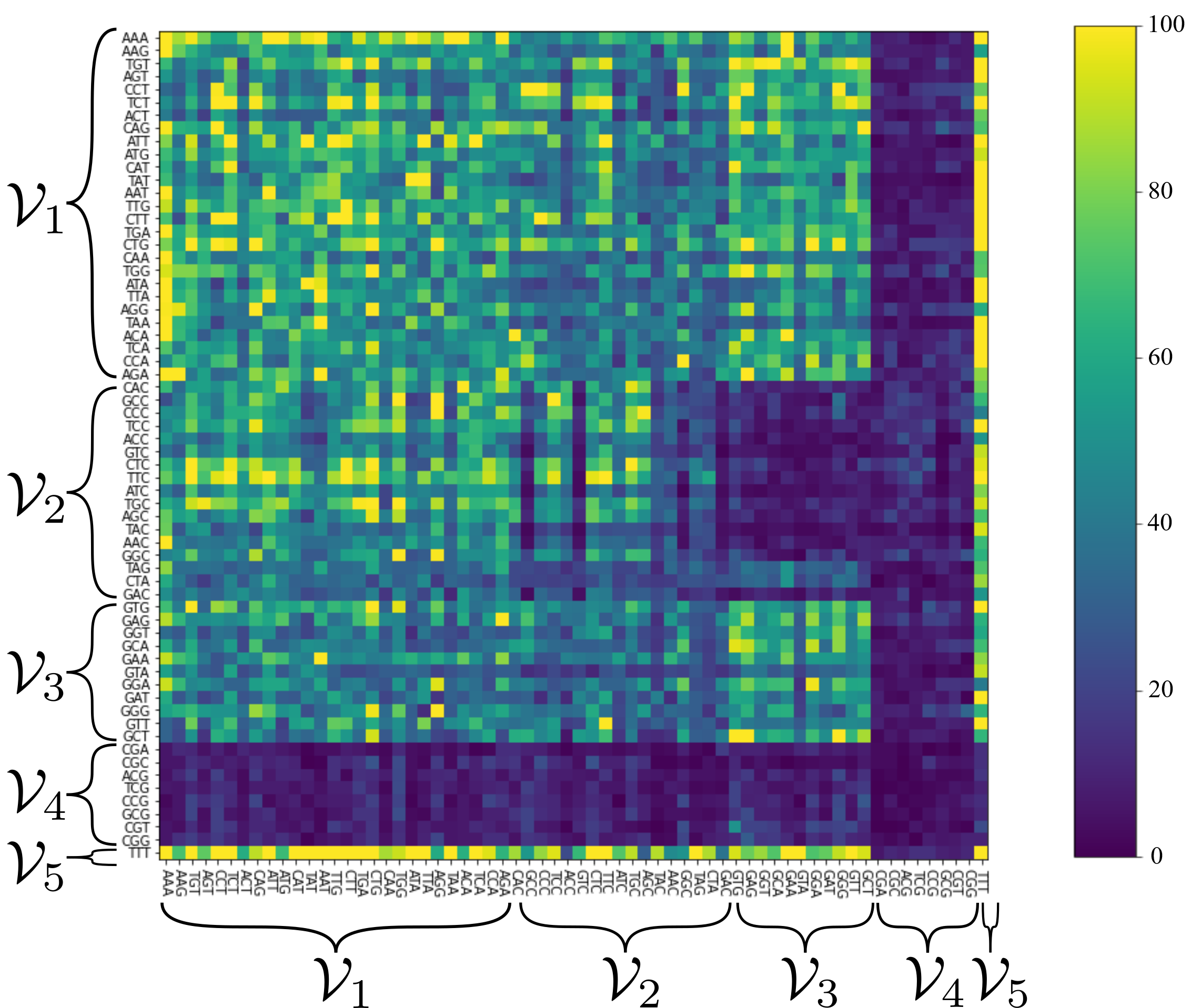}}
	\hspace{20pt}
	\subfloat[]{\includegraphics[height=0.33\linewidth]{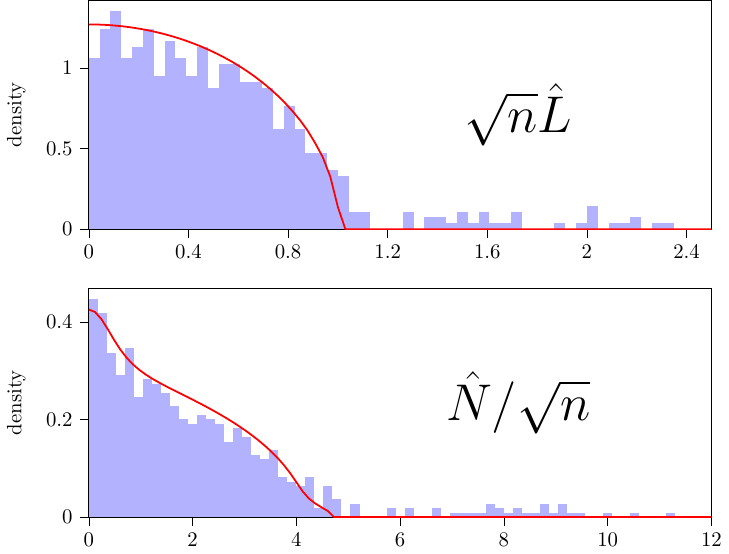}}
	\caption{
		(a)
		The frequency matrix $\hat{N}$ when the codons are sorted by the five detected clusters.
		(b)
		Average density-based histogram of singular values for $\sqrt{n}\hat{L}$ and $\hat{N}/\sqrt{n}$ for the \gls{DNA} sequential data in blue bars and the theoretical predictions associated with the improved clustering as the red line.
		Not displayed is that each observation of $\hat{N}/\sqrt{n}$ also has a single singular value near $40$ and each observation of $\sqrt{n}\hat{L}$ has a single singular value near $8$.
		These extremal singular values are considered to be part of the signal, and consequently not relevant for measuring the spectral noise.
	}
	\label{fig:DNA}
\end{figure*}

It appears from  the reasonable clusters in Fig. \ref{fig:DNA}(a) that a \gls{BMC} could be an appropriate model for this dataset.
Let us now additionally verify whether the shape of the spectral noise is consistent with a \gls{BMC}.
Note that the matrices $\hat{N}$ and $\hat{L}$ are only $64 \times 64$.
Consequently, they only have $64$ singular values.
To get a clearer picture we split the observation sequence into ten equally sized pieces and for each subpath we compute the singular values.
The averaged histogram over these ten observations is compared to the theoretical \gls{BMC}-prediction associated to the clusters in Fig. \ref{fig:DNA}(b).

We observe a good match to the theory for both $\hat{N}$ and $\hat{L}$.
Particularly interesting is the peak near zero and the triangular tail in the interval $[4,5]$.
The theoretical there matches the observed distribution for $\hat{N}/\sqrt{n}$.
Such features would not be predicted in a simpler model without communities such as a matrix with i.i.d.\ entries.
One would then instead expect a quarter-circular law with density proportional to $\indicator{x\in (0,c)}\sqrt{c^2 - x^2}$ for some $c>0$.
This quarter-circular law is observed in the empirical Laplacian $\hat{L}$ suggesting that the main feature in the spectral noise of $\hat{N}$ is due to the equilibrium distribution.
There are also some singular values which escape the support of the limiting singular value distribution.
These are most-likely associated to the signal $\mathbb{E}[\hat{N}]$ and should consequently not be viewed as a part of the spectral noise.

\paragraph{Conclusion}

It appears that the clustering algorithm was able to detect the phenomena of dinucleotide bias in \gls{DNA}.
The spectral noise is consistent with a \gls{BMC} and a simpler model generating a random matrix with independent and identically distributed entries would not have sufficed to predict $\hat{N}$'s singular values.

\subsubsection{Sequence of words on Wikipedia}
\label{sec: Wikipedia}

The clustering algorithm discussed in Section \ref{sec:clustering_procedure} was executed for $\numClusters = 50, 100, 200, 400$, both with and without the improvement algorithm.
Ten improvement iterations were done whenever we used the latter algorithm.
A complete list of the clusters for $\numClusters=200$ with improvement is given in Appendix \ref{sec: DetectedClusters}.

\paragraph{Subjective evaluation}
At a first glance, the found clusters appear meaningful.
For instance, a small cluster with six elements has a distinctly football-related theme: $\mathcal{V}_{125}$ contains the words \emph{champion}, \emph{cup}, \emph{premier}, \emph{coach}, \emph{footbal} and \emph{championship}.
The medium-sized clusters $\mathcal{V}_{50}$, $\mathcal{V}_{51}$, and $\mathcal{V}_{52}$ respectively contain words related to public professions, units, and warfare.
That is, $\mathcal{V}_{50}$ includes stemmed words such as \emph{founder}, \emph{deputi}, \emph{formeli}, \emph{mayor}, \emph{bishop}, \emph{meanwhil}, \emph{successor}, $\mathcal{V}_{51}$ includes \emph{tonn}, \emph{usd}, \emph{capita}, \emph{lb}, and $\mathcal{V}_{52}$ includes \emph{cavalri}, \emph{jet}, \emph{helicoptr}, \emph{rifl}, \emph{warfar}, \emph{battalion}, and \emph{raid}.
The second-largest cluster $\mathcal{V}_2$ predominantly contains names, including \emph{alexandr}, \emph{albrecht}, \emph{gideon}, and \emph{jarrett}.

We further observe that the improvement algorithm yields more balanced clusters: before the improvement algorithm the largest three clusters have sizes $9192$, $1279$ and $1126$, respectively, while after improvement the sizes are $2848$, $1943$ and $1600$.

\paragraph{Performance on a downstream task}
To evaluate the quality of the clusters more objectively, we investigate the performance achieved on a downstream task as discussed in Section \ref{section:summary_methods_for_eval}.

We specifically consider a document classification task where the goal is to predict the label $l(d)$ of a document $d$ given some training dataset.
The considered datasets are described in Appendix \ref{sec: DocClass}.
For instance, the AG News dataset contains news articles with four possible labels: \emph{World}, \emph{Sports}, \emph{Business}, and \emph{Sci/Tech}.

Given a clustering, one can translate each document into a $\numClusters$-dimensional vector by counting the number of occurrences of each cluster in the document; see Appendix \ref{sec:Appendix__CDIDF}.
Thereafter, a logistic regression model is trained to learn a mapping from the $\numClusters$-dimensional vectors to the labels.
Aside from spectral and improvement clusters we also consider a random clustering in which every word is assigned a cluster uniformly at random.
There were some datasets in which neither spectral nor improvement clustering significantly outperformed the random clustering.
We consider these tests inconclusive, but report on them in Appendix \ref{sec: DocClass} for completeness.
The performance on the remaining datasets is displayed in Table \ref{tab: perf}.

Observe that improvement clustering typically outperforms plain spectral clustering.
Further, in the \emph{AG News}, \emph{Yahoo!} and \emph{Wiki} datasets the performance increases with the dimensionality.
The gain in performance from spectral and improvement clustering as opposed to random clustering is there comparable with an increase of dimensionality by a factor 4.
On the other hand, for \emph{Books} and \emph{CMU} it appears that the performance decreases with the dimensionality, although this pattern is less clear.
A possible explanation is that \emph{Books} and \emph{CMU} have less training data so that overfitting may occur when the dimensionality is large.

\begin{table}[htbp]
	\caption{ Performance of clustering before and after improvement as measured by accuracy in the downstream task of document classification as compared to a random clustering. Bold added for the best-performing method.}
	\label{tab: perf}
	\centering
	\setlength{\tabcolsep}{10pt}
	{
		\scriptsize
		\begin{tabular}{cccccccc}
			\toprule
			$\numClusters$ & {Algorithm} & {AG News}    & {Yahoo!}     & {Wiki}       & {Book}       & {CMU}        \\
			\midrule
			$50$           & Random      & 48.3\%       & 27.4\%       & 56.9\%       & 31.0\%       & 67.4\%       \\
			$50$           & Spectral    & 66.0\%       & 39.8\%       & 71.1\%       & 44.4\%       & 69.5\%       \\
			$50$           & Improved    & {\bf 68.5\%} & {\bf 40.1\%} & {\bf 71.5\%} & {\bf 44.7\%} & {\bf 71.8\%} \\[0.5em]
			$100$          & Random      & 55.5\%       & 33.3\%       & 68.4\%       & 30.0\%       & 67.4\%       \\
			$100$          & Spectral    & 72.7\%       & 47.2\%       & {\bf 81.6\%} & 45.2\%       & 70.0\%       \\
			$100$          & Improved    & {\bf 76.8\%} & {\bf 49.0\%} & 80.1\%       & {\bf 46.3\%} & {\bf 70.7\%} \\[0.5em]
			$200$          & Random      & 64.0\%       & 41.7\%       & 80.8\%       & 28.2\%       & 66.8\%       \\
			$200$          & Spectral    & 78.2\%       & 51.7\%       & 85.6\%       & {\bf 44.4\%} & 68.7\%       \\
			$200$          & Improved    & {\bf 80.7\%} & {\bf 54.7\%} & {\bf 86.5\%} & 43.4\%       & {\bf 69.0\%}
			\\[0.5em]
			$400$          & Random      & 72.8\%       & 49.4\%       & 87.8\%       & 28.9\%       & 66.8\%       \\
			$400$          & Spectral    & 81.5\%       & 56.3\%       & 88.0\%       & 42.1\%       & 67.9\%       \\
			$400$          & Improved    & {\bf 83.1\%} & {\bf 58.6\%} & {\bf 89.0\%} & {\bf 44.4\%} & {\bf 68.4\%} \\
			\bottomrule
			\vspace{0pt}
		\end{tabular}
	}
\end{table}

\paragraph{Comparing the histogram of singular values to the limiting distribution of singular values of the inferred \texorpdfstring{\gls{BMC}}{BMC}}\label{sec: NLP_Histograms}

One may be tempted to deduce from the reasonable clusters and the performance in Table \ref{tab: perf} that the \gls{BMC} model is appropriate for this dataset.
The structure in the spectral noise is however not as one would expect.
Consider Fig. \ref{fig:WordsLhat} for a comparison of the empirical singular value distribution with the theoretical predictions.
Observe that there is a good match for $\hat{N}$ but a discrepancy for $\hat{L}$.

The fact that $\hat{N}$ yields a good match can be explained as being due to a strongly inhomogeneous equilibrium distribution from Zipf's law.
The empirical Laplacian $\hat{L}$ removes this dominant effect after which it may be observed that the empirical distribution has a heavy tail which is not present in the \gls{BMC}-based prediction.
In Appendix \ref{sec: Simulation} we demonstrate by a numerical example that the discrepancy which is observed in Fig. \ref{fig:WordsLhat} agrees precisely with the type of discrepancy which is observed for a heavy-tailed perturbation of the \gls{BMC}.
The fact that the entries of the matrices $\hat{N}$ and $\hat{L}$ are heavy-tailed may also be verified by direct inspection.

\begin{figure}[t]
	\centering
	\includegraphics[ width= 0.99\linewidth]{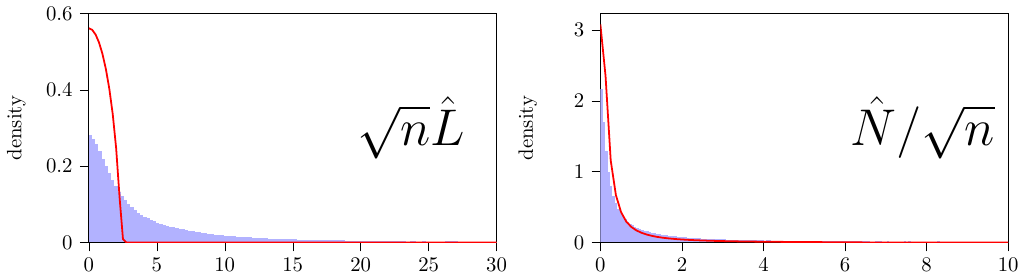}
	\caption{Density-based histogram of singular values for $\sqrt{n}\hat{L}$ for the words sequential data in blue bars and the theoretical predictions associated with the improvement clustering with $\numClusters = 200$ as the red line. Not visible in this figure is that both empirical distributions have long tails. Still $9\%$ of the singular values of $\hat{N}/\sqrt{n}$ exceed 10 and $1\%$ of the singular values of $\sqrt{n}\hat{L}$ exceed $30$.}
	\label{fig:WordsLhat}
\end{figure}

\paragraph{Conclusion}
The clustering algorithm found clusters that we judge to be meaningful. 
The performance on a downstream document classification task further indicated that the improvement algorithm based on the \gls{BMC}-assumption improved the quality of the clusters.
The spectral noise indicated that there is some heavy-tailed component in which can not be accounted for by \glspl{BMC}.
It is hence conceivable that a different model could incorporate the heavy-tailedness and extract even better clusters.

\subsubsection{Companies with the highest daily returns}
\label{sec:Companies_with_the_highest_daily_returns}

We finally turn to the sequence of companies with the highest daily returns.
This analysis was particularly delicate to conduct and we ultimately arrive at the conclusion that a $0$th-order \gls{BMC} could already be sufficient to explain the found clusters.

This conclusion may appear disappointing: it means that the clusters may not encode order-dependent dynamics.
It is however important for a practitioner to be able to arrive at this conclusion when appropriate.
The fact that the evaluation methods from Section \ref{section:summary_methods_for_eval} are able to suggest a $0$th-order \gls{BMC} is correspondingly a good feature: the method would not be informative in the alternative scenario where one always concludes in favor of the $1$st-order \gls{BMC}.
The main goal of this section is hence to demonstrate how the methods can be used in a difficult, sparse, regime.

There are two main reasons why this dataset is difficult to analyze. 
First, the data is sparse; recall from Table \ref{tab: Data} that $\ell / n^2 \approx 0.03$ and $\ell/n \approx 8$.
This sparsity makes recovery of the clusters a hard problem, even if the data-generating-process is truly a \gls{BMC}, and moreover makes evaluation of the found clusters more difficult since the associated confidence bounds are large.
Second, it turns out that the data contains a strong $0$th-order component which could potentially serve as a nuisance factor, concealing a $1$st-order \gls{BMC} component even if it exists.

\paragraph{Subjective evaluation of the clusters}

After some \emph{ad hoc} experimentation, we fix $\numClusters = 3$.
The \gls{SP500}'s factsheet labels every constituent with a sector; see Appendix \ref{sec:Appendix__raw_data__Companies_with_the_highest_daily_returns}.
We can use this labeling to obtain ``fingerprints'' of clusters.

The black bars in Fig. \ref{fig:Stock_market__Aligned_timespans__Sectors_after_the_cluster_improvement_algorithm_of_the_three_largest_clusters}  show the relative percentages of constituents in each sector for the clusters found after the improvement algorithm.
Observe the absence of most utilities constituents within the $2$nd and $3$rd cluster; more than twice as many are assigned to the $1$st cluster than may be expected in a random assignment.
Industrial and health care constituents are also mostly absent within the $3$rd cluster.
Similarly, note the negligible number of consumer discretionary constituents within the $1$st cluster; most are assigned to the $2$nd and $3$rd cluster.
Finally, consider that the $3$rd cluster consists for $29\%$ out of information technology constituents.
These contents suggest that the clusters are not entirely random.
The subsequent experimentation aims to determine what type of information has been encoded in the clusters.

\begin{figure*}
	\centering
	\includegraphics[width=0.99\linewidth]{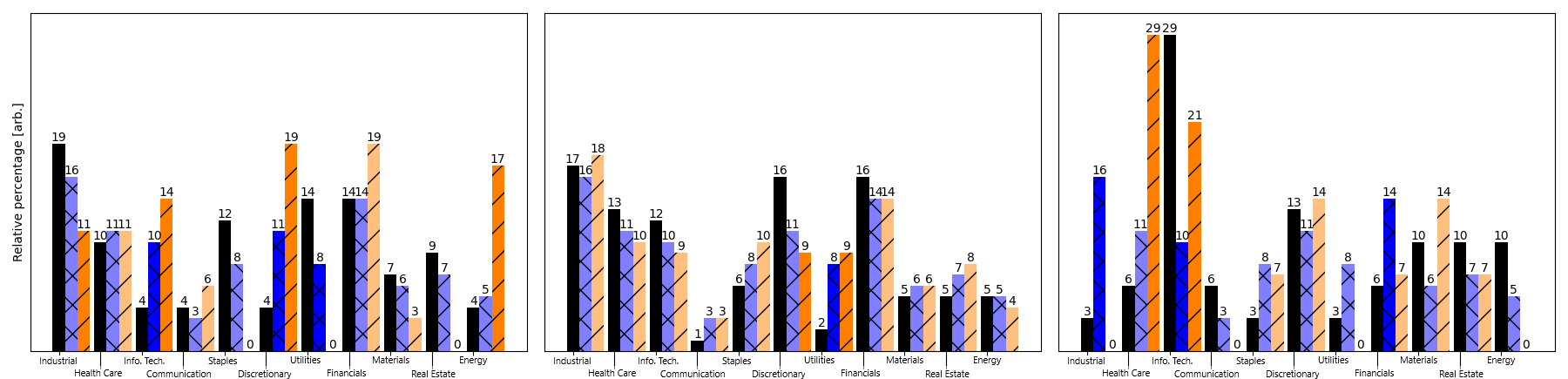}
	\caption{
		The fraction of constituents in each sector for models $\hat{\mathbb{P}}, \hat{\mathbb{Q}}_1$ and $\hat{\mathbb{Q}}_2$ as the black, blue and orange bars respectively.
		The left, middle, and right plots correspond to the $1$st largest, $2$nd largest, and $3$rd largest detected cluster, respectively.
		A bar's color is saturated when the difference in relative percentage exceeds $5\%$ when compared to the black bars.
	}
	\label{fig:Stock_market__Aligned_timespans__Sectors_after_the_cluster_improvement_algorithm_of_the_three_largest_clusters}
\end{figure*}

As a subjective way to evaluate the meaning of the clusters, let us inspect the relative cluster sizes $\hat{\alpha}_k :=
	\# \hat{\mathcal{V}}_k/n$, cluster equilibrium distribution $\hat{\pi}$, and cluster transition matrix $\hat{p}$ of the associated \gls{BMC}:
\begin{equation}
	\hat{\alpha}^{\mathrm{T}}
	\approx
	\begin{pmatrix}
		0.45 \\
		0.45 \\
		0.10 \\
	\end{pmatrix}, \;
	\hat{\pi}^{\mathrm{T}}
	\approx
	\begin{pmatrix}
		0.49 \\
		0.10 \\
		0.41 \\
	\end{pmatrix}, \;
	\hat{p}
	\approx
	\begin{pmatrix}
		0.50 & 0.10 & 0.40 \\
		0.54 & 0.11 & 0.35 \\
		0.46 & 0.10 & 0.44 \\
	\end{pmatrix}\nonumber
	.
	\label{eqn:Estimated_sizes_equilibrium_transition_matrix_for_stock_market}
\end{equation}
Note that the rows of $\hat{p}$ are close to but not quite equal; it namely holds that $\hat{p}_{kl} \approx \hat{\pi}_l$ for every $k, l$.
This observation may suggest a strong $0$th-order \gls{BMC} component.
One can however not immediately conclude that all the deviations from constant columns are due to noise: the data is sparse relative to $n^2$ but not when compared to $\numClusters^2 = 9$.

\paragraph{Comparing against alternative models}

\begin{figure*}
	\centering
	\subfloat[]{\includegraphics[height=0.25\linewidth]{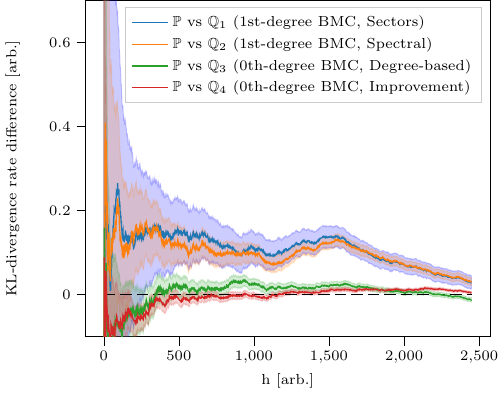}}
	\subfloat[]{\includegraphics[height=0.25\linewidth]{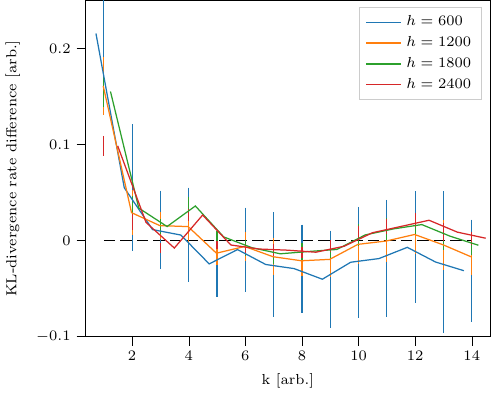}}
	\subfloat[]{\includegraphics[height=0.25\linewidth]{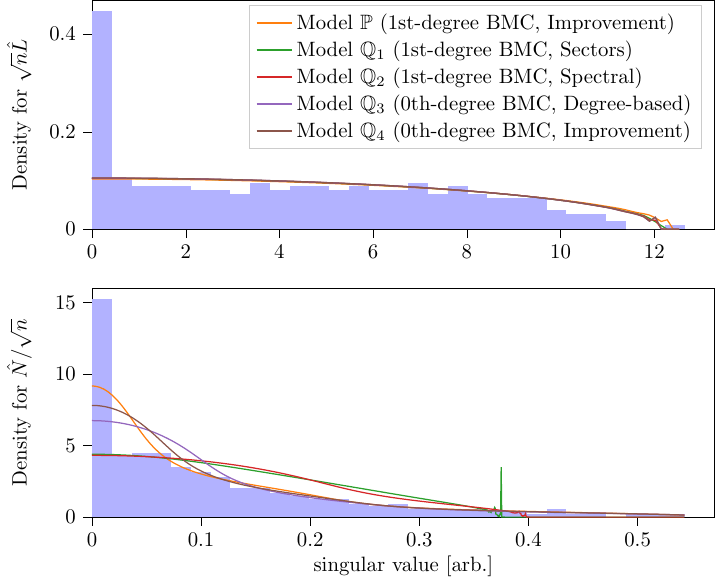}}
	\caption{
	(a)
	The \gls{KL} divergence rate difference estimator
	$
		D(X_{ \lfloor \ell / 2 \rfloor + 1 : \lfloor \ell / 2 \rfloor + h }; \hat{\mathbb{P}}^{X_{1:\lfloor \ell / 2 \rfloor}}, \hat{\mathbb{Q}}_i^{X_{1:\lfloor \ell / 2 \rfloor}})
	$
	on the validation data with 95\% confidence bounds estimated using \eqref{eq: ConfidenceInterval} from Appendix \ref{sec:Confidence_bounds_when_estimating_DRPQ} with mixing time (arbitrarily) guessed to be $20$ days.
	(b) The \gls{KL} divergence rate difference estimator
	$\hat{D}(X_{\lfloor \ell / 2 \rfloor + 1:\lfloor \ell / 2 \rfloor + h}, \hat{\mathbb{P}}^{X_{1:\lfloor \ell / 2 \rfloor}}, \hat{\mathbb{Q}}_{3,k}^{X_{1:\lfloor \ell / 2 \rfloor}})$
	for different sample path lengths $h \in \naturalNumbersPlus$,
	and as a function of $k$ with 95\% confidence bounds using \eqref{eq: ConfidenceInterval}.
	(c)
	The top and bottom figures display the singular values of $\sqrt{n}\hat{L}$ and $\hat{N}/\sqrt{n}$ respectively.
	Both figures exclude the $\numClusters=3$ leading singular values.
	}
	\label{fig:Stock_market}
\end{figure*}

Recall that, using validation data, we can compare the performance of different models by the \gls{KL} divergence rate difference estimator \eqref{eqn:Less_biased_estimator_of_the_KL_divergence_rate_difference1}.
Consider the following models:
\begin{itemize}[noitemsep]
	\item[$\hat{\mathbb{P}}$:] A $1$st-order \gls{BMC} with $K = 3$ clusters found by the spectral algorithm followed by the improvement algorithm.
	\item[$\hat{\mathbb{Q}}_1$:] A $1$st-order \gls{BMC} with $K = 11$ clusters given by the sector labels.
	\item[$\hat{\mathbb{Q}}_2$:] A $1$st-order \gls{BMC} with $K = 3$ clusters, found by the spectral algorithm.
	\item[$\hat{\mathbb{Q}}_3$:] A $0$th-order \gls{BMC} with $K = 3$ clusters, found by sorting along the state's sample equilibrium distribution and determining clusters of equal probability mass.
	\item[$\hat{\mathbb{Q}}_4$:] A $0$th-order \gls{BMC} with $K = 3$ clusters, found by the spectral algorithm followed by an improvement algorithm (modified for a $0$th-order \gls{BMC}).
\end{itemize}

One may also wonder about the effect of the number of parameters.
By keeping the number of clusters fixed, it namely follows that the $0$th-degree models $\hat{\mathbb{Q}}_2,\hat{\mathbb{Q}}_3$ have fewer parameters than the $1$st-degree models $\hat{\mathbb{P}},\hat{\mathbb{Q}}_4$.
Hence, consider the following model for any $k\geq 1$:
\begin{itemize}[noitemsep]
	\item[$\hat{\mathbb{Q}}_{3,k}$:] A $0$th-order \gls{BMC} with $k$ clusters, found by sorting according the state's sample equilibrium distribution and determining $k$ clusters of equal probability mass.
\end{itemize}
Note that the degrees of freedom $\mathrm{DF}_1(n, \numClusters)$ within a $1$st-order \gls{BMC} with fixed parameters $(n, \numClusters)$ equals $\mathrm{DF}_1(n,\numClusters) = n + \numClusters(\numClusters-1)$, whereas the degrees of freedom $\mathrm{DF}_0(n, \numClusters)$ within a $0$th-order \gls{BMC} constrained with fixed parameters $(n,\numClusters)$ equals $\mathrm{DF}_0(n,\numClusters) = n + \numClusters - 1$.
The model $\hat{\mathbb{P}}$ therefore has $n+6$ degrees of freedom whereas $\hat{\mathbb{Q}}_{3,k}$ has $n + k-1$ degrees of freedom.
In particular, the degrees of freedom for $\hat{\mathbb{P}}$ and $\hat{\mathbb{Q}}_{3,7}$ are comparable.
The remaining difference is that $\hat{\mathbb{P}}$ allows for more inhomogeneity within the columns of the transition matrix and less in the rows, whereas $\hat{\mathbb{Q}}_{3,7}$ allows no inhomogeneity within the columns but more in the rows.

Observe in Fig. \ref{fig:Stock_market}(a) that the difference in \gls{KL} divergence rate on the validation data is positive when comparing $\hat{\mathbb{P}}$ against $\hat{\mathbb{Q}}_1$, $\hat{\mathbb{Q}}_2$, barely positive when comparing against $\hat{\mathbb{Q}}_3$, and near-zero when comparing against $\hat{\mathbb{Q}}_4$.
The $0$th-degree models $\hat{\mathbb{Q}}_3$, $\hat{\mathbb{Q}}_4$ perform comparable to the $1$st-degree model $\mathbb{P}$.

Regarding the comparison with $\hat{\mathbb{Q}}_{3,k}$ we may observe in Fig. \ref{fig:Stock_market}(b) that the sign of the \gls{KL} divergence rate difference is probably positive for $k = 1, 2, 4$, possibly positive for $k = 3, 11, 12$ but not much, possibly negative for $k = 6, 7, 8$ but not much, and inconclusive for $k = 5, 9, 10$.
The downward trend for small $k$ suggests that a strictly positive number of free parameters are necessary to accurately represent the data.
Judging from the case $k \approx 7$, in which case the number of degrees of freedom in both models are equal, it appears that the specific freedoms allowed in $\hat{\mathbb{P}}$ give a performance comparable to that attained by the freedoms allowed in $\hat{\mathbb{Q}}_{3,7}$.

\paragraph{Comparing the histogram of singular values to the limiting distribution of singular values of the inferred \texorpdfstring{\gls{BMC}}{BMC}}

Fig. \ref{fig:Stock_market}(c) depicts histograms of singular values and theoretical predictions for the models $\hat{\mathbb{P}}$, $\hat{\mathbb{Q}}_1$, $\hat{\mathbb{Q}}_2$, $\hat{\mathbb{Q}}_3$, $\hat{\mathbb{Q}}_4$'s.
All theoretical predictions were calculated from training data, while the histograms were calculated from validation data.

All theoretical predictions give a fair description of the laws.
Models $\hat{\mathbb{P}}$, $\hat{\mathbb{Q}}_4$ outperform models $\hat{\mathbb{Q}}_1$, $\hat{\mathbb{Q}}_2$, $\hat{\mathbb{Q}}_3$ when it comes to describing the distribution of singular values of $\hat{N}_{\textnormal{validation}} / \sqrt{n}$.
Observe that the empirical observations for $\sqrt{n}\hat{L}_{\textnormal{validation}}$ as well as the predictions associated to $\hat{\mathbb{P}}$, $\hat{\mathbb{Q}}_1$, $\hat{\mathbb{Q}}_2$, $\hat{\mathbb{Q}}_3$, $\hat{\mathbb{Q}}_4$ all appear to be quarter-circular.
This quarter-circular law is consistent with our suspicion of a strong $0$th-degree model component: in a $0th$-degree \gls{BMC}, the limiting law of $\sqrt{n}\hat{L}$ is known to be quarter-circular.
The peak at zero in the empirical observations is likely due to the sparsity.

\paragraph{Conclusion}

In all considered performance measures we saw that the $1$st-degree \gls{BMC} model $\mathbb{P}$ performed approximately equally well as the $0$th-degree models $\hat{\mathbb{Q}}_3$, $\hat{\mathbb{Q}}_4$.
The consideration of the models $\hat{\mathbb{Q}}_{3,k}$ suggested that one further requires a certain number of parameters to achieve sufficient model expressivity.

The sparsity of the data makes it difficult to come to a definitive conclusion.
Still, one generally prefers models with fewer parameters.
Hence, in our opinion, a 0th-order \gls{BMC} would be a suitable model for this dataset.

\subsection{Detected orders within the data}
\label{sec:Detected_orders_within_the_data}
We investigate what order of \gls{BMC} best fits the clustered data $Y_t = \sigma_n(X_t)$ using the information criteria described briefly in Section \ref{section:summary_methods_for_eval} (and in detail in Appendix \ref{sec:Evaluating_clusters_and_models}).
We focus on the \gls{DNA}, \gls{GPS}, and the \gls{SP500} dataset.
The Wikipedia data is omitted due to its impractical size, and because it does not consist of a single sample path but rather a number of small sample paths.

\subsubsection{Results}
We compute  \eqref{eqn:CAIC} for $r = 0,1,2,3,4$ of the following models:

\begin{itemize}[noitemsep]
	\item[$\hat{\mathbb{Q}}^{r, \mathrm{MLE}}$]:
	      The \gls{MLE} of an $r$th-order \gls{MC} estimated from the observation sequence $Y_{1:\ell}$.
\end{itemize}

The result are in Table \ref{tab:CAIC}. We see  that the magnitude of the \gls{CAIC} in Table \ref{tab:CAIC} depends strongly on the observation sequence and the number of clusters.
For the \gls{GPS} coordinates, the differences are notable for most orders due to the large number of clusters $\numClusters=15$, where higher orders become highly penalized.
For \gls{DNA}, the criterion suggests that orders $r \in \{ 1, 2\}$ are optimal.
For the \gls{SP500}, on the other hand, orders $r \in \{0,1\}$ appear to be the best.
We expect a large variance in Table \ref{tab:CAIC} and some over or underfitting the order is possible.
The criterion indicates nonetheless that the transitions of the found clusters, except maybe for the \gls{SP500} dataset, can be better approximated by a nonzero order Markovian process. We will now support this conclusion empirically with the error models for the \gls{DNA} and \gls{SP500} datasets.

\begin{table}[hbtp]
	\caption{
		The \gls{CAIC} in \eqref{eqn:CAIC} for the different datasets.
		Note that the relative difference between the values pertaining to different orders is often small.
		For example, the differences are less than $0.1\%$ between orders $1$, $2$ for the \gls{DNA} data, and between orders $0$, $1$ for the stock market data. This is not the case, however, with the animal data.
	}
	\label{tab:CAIC}
	\centering
	{
		\begingroup
		\setlength{\tabcolsep}{10pt}
		\footnotesize
		\begin{tabular}{ccccccc}
			\toprule
			$r$ & \gls{DNA}       & incr.\ ($\%$)  & \gls{GPS}($\times 10^3$) & incr.\ ($\%$)  & \gls{SP500}   & incr.\ ($\%$) \\
			\midrule
			0   & 432650          & n.a.           & 960.63                   & n.a.           & \textbf{9853} & \textbf{n.a.} \\
			1   & 431502          & -0.27          & 626.54                   & -34.8          & 9860          & +0.07         \\
			2   & \textbf{431263} & \textbf{-0.06} & \textbf{571.49}          & \textbf{-40.5} & 9940          & +0.81         \\
			3   & 435228          & +0.69          & 1121.90                  & +16.8          & 10253         & +3.1          \\
			4   & 458512          & +5.3           & 9789.27                  & +1019          & 11162         & +8.9          \\
			\bottomrule
			\vspace{0pt}
		\end{tabular}
		\endgroup
	}
\end{table}

To further support the accuracy of the \gls{CAIC}, an accuracy study of the criterion is conducted in Appendix \ref{sec:CaicEval} under a generative model based on the datasets. This study suggests that the criterion is less prone to overfit or choose a model with more parameters. Hence, the nonzero orders estimated in Table~\ref{tab:CAIC} hint that a high-order Markov structure in the data exists that a model such as the \gls{BMC} can approximate.

We finally remark that using information criteria for the unclustered observation sequences $X_{1:\ell}$ provides no useful insights due to the large dimensionality of the models.
In particular, the \gls{CAIC} criteria for the unclustered observation sequences for order $r \in \{0, 1\}$ can be seen in Table \ref{tab:CAIC_all}.
As the data shows, the \gls{CAIC} criteria just picks the model with smallest number of parameters.
This is even more extreme in the \gls{GPS} and \gls{SP500} datasets, where on top of large model dimension we have sparse data.

\begin{table}[hbtp]
	\caption{The \gls{CAIC} in \eqref{eqn:CAIC} for the sequence $X_1$, $\ldots$, $X_{\ell}$ for different datasets.}
	\label{tab:CAIC_all}
	\centering
	{
		\footnotesize
        \setlength{\tabcolsep}{20pt}
		\begin{tabular}{cccc}
			\toprule
			$r$ & \gls{DNA}                     & \gls{GPS}                   & \gls{SP500}                  \\
			\midrule
			0   & \textbf{1339.5} $\times 10^3$ & \textbf{2943} $\times 10^3$ & \textbf{54.27} $\times 10^3$ \\
			1   & 1361.9 $\times 10^3$          & $\approx$ 1 $\times 10^8$   & 882 $\times 10^3$            \\
			\bottomrule
			\vspace{0pt}
		\end{tabular}
	}
\end{table}

\subsubsection{Conclusion}
We found that model selection is feasible if we use the clustered sequence $Y_{1:\ell} = \sigma_n(X_{1:\ell})$ obtained after the clustering algorithm, because this reduces the amount of free parameters of the models considerably.

For the \gls{DNA} and \gls{GPS} datasets, the \gls{CAIC} selects a nonzero order \glspl{MC}.
For the \gls{SP500}, the data was too sparse for selecting a specific order with certainty.
However, there are indications that the values obtained in the \gls{CAIC} for the \gls{SP500} dataset are consistent with a $1$st-order \gls{BMC} model with a strong $0$th-order \gls{MC} baseline.

\section{Conclusions}
\label{sec:Conclusions}

We have found that using a \gls{BMC} model for exploratory data analysis in unlabeled observation sequences does in fact produce useful insights.
Although there is no guarantee that there are clusters or that a cluster structure is actually revealing of a ground truth model we can still evaluate the clusters and associated models.
The animal movement example uncovered features which could not have been extracted from only the \gls{GPS} coordinates.
The \gls{DNA} example uncovered known, nontrivial and biologically relevant structure.
In the text-based example, the improvement algorithm enhanced performance on down-stream tasks and the spectral noise identified the heavy-tailed nature of some model violations.
For the daily best performing stocks in the \gls{SP500}, we saw that a $0$th-order \gls{BMC} can describe its statistical aspects, but there are indications that a $1$st-order \gls{BMC} is also a suitable model.

\section*{Acknowledgments}
This publication is part of the project \emph{Clustering and Spectral Concentration in Markov Chains} (with project number OCENW.KLEIN.324) of the research programme \emph{Open Competition Domain Science -- M} which is (partly) financed by the Dutch Research Council (NWO).

The authors also acknowledge support by the \emph{European Union’s Horizon 2020 research and innovation programme} under the Marie Sk\l{}odowska--Curie grant agreement no.\ 945045, and by the \emph{NWO Gravitation project} NETWORKS under grant no.\ 024.002.003.

We thank the Nature Conservancy and Dr. Stephen Blake of Saint Louis University for permission to use the bison movement data. The Max Planck Institute for Animal Behavior and the National Geographic Society Committee for Research and Exploration (Grant:
9385-13) funded the bison GPS collars.

We finally thank Mike van Santvoort for useful discussions while writing this paper.

\bibliography{biblio}

\newpage 
\appendix
\section*{Overview of the appendices}
The appendices below provides supporting information, such as algorithmic descriptions, experimental details, and certain proofs.

Appendix \ref{sec:Pseudo_code_describing_the_clustering_procedure} provides pseudocode for the clustering algorithm of \cite{sanders2020clustering} that we implemented.
Appendix \ref{sec:Robustness_of_the_clustering_procedure} describes perturbed \gls{BMC} models and provides details regarding the simulation experiment which was described in Section \ref{sec:Other_models_for_experimentation}.
The concentration inequality which was used to construct the confidence intervals in Fig. \ref{fig:Stock_market}, is described in Appendix \ref{sec:Confidence_bounds_when_estimating_DRPQ}.
Appendix \ref{sec:Shape_of_the_spectral_noise} a limiting law for the singular value distribution of the Laplacian.
Tools which were used in the experiments and preprocessing are described in Appendix \ref{sec:Tools}.
Finally, some raw data and extra material describing our findings is provided in Appendix \ref{sec:Appendix__Raw_data}.

\section{Pseudo-code describing the clustering procedure}
\label{sec:Pseudo_code_describing_the_clustering_procedure}

The following pseudocode summarizes the clustering procedure that we have implemented.
This pseudocode appeared first in \cite{sanders2020clustering}, and we repeat it here for clarity and your convenience.

\begin{algorithm}
	\caption{Spectral clustering algorithm, courtesy of \cite{sanders2020clustering}.}
	\label{alg:spectral_clustering}
	\hspace*{\algorithmicindent} \textbf{Input:} $n,  \numClusters$ and $\hat{N}$\\
	\hspace*{\algorithmicindent} \textbf{Output:} New cluster assignment $\hat{\mathcal{V}}^{\prime}_1, \ldots,\hat{\mathcal{V}}^{\prime}_{\numClusters}$
	\begin{algorithmic}[1]
	\State $\hat{N}_{\Gamma} \gets \mathrm{Trim}(\hat{N})$\;
	\State $\hat{R} \gets$ $\numClusters$-rank approximation of $\hat{N}_{\Gamma}$\;
	\State $\hat{\mathcal{V}}_1, \ldots,\hat{\mathcal{V}}_{\numClusters} \gets \numClusters\text{-means}([\hat{R}, \hat{R}^{T}])$\;
\end{algorithmic}
\end{algorithm}

The spectral clustering in Algorithm~\ref{alg:spectral_clustering} is used to obtain a good initial estimate for the clusters.
A k-means algorithm is used along a $\numClusters$-rank approximation of $\hat{N}$ (or $\hat{N}_{\Gamma}$ for the trimmed version of $\hat{N}$) to yield an initial guess for the clusters.
It can be proved, however, that this step yields a number of misclassified states that is sublinear in $n$ but not of constant order \cite{sanders2020clustering_sm}.
A second step is then required to attain exact recovery.
In Algorithm~\ref{alg:Cluster_improvement_algorithm}, we see that a procedure similar to a likelihood ratio maximization is used to improve the cluster assignment.
With this extra step it can be proven that the misclassified states will be order constant in expectation.

\begin{algorithm}
	\caption{Cluster improvement algorithm (for $1$st-order \glspl{BMC}), courtesy of \cite{sanders2020clustering}}
	\label{alg:Cluster_improvement_algorithm}
	\hspace*{\algorithmicindent} \textbf{Input:} $n$, $\numClusters$, $\ell$, $\hat{N}$ and initial cluster assignment guess $\hat{\mathcal{V}}_1, \ldots,\hat{\mathcal{V}}_{\numClusters}$.\\
	\hspace*{\algorithmicindent} \textbf{Output:} New cluster assignment $\hat{\mathcal{V}}^{\prime}_1, \ldots,\hat{\mathcal{V}}^{\prime}_{\numClusters}$
	\begin{algorithmic}[1]
	\For{$a \gets 1$ to $\numClusters$}
		\State $\hat{\pi}_{a} \gets \hat{N}_{\hat{\mathcal{V}}_a, [n]}/\ell$, $\hat{\alpha}_{a} \gets \# \hat{\mathcal{V}}_a/n$\;
		\State $\hat{\mathcal{V}}^{\prime}_a \gets \emptyset$\;
		\For{$b \gets 1$ to $\numClusters$}
 			\State $\hat{p}_{a, b} \gets \hat{N}_{\hat{\mathcal{V}}_a, \hat{\mathcal{V}}_b}/ \hat{N}_{\hat{\mathcal{V}}_a, [n]}$\;
		\EndFor
	\EndFor
	\For{$x \gets 1$ to $n$}
		\State	$c \gets \mathrm{argmax}_{l \in [\numClusters]} \sum_{k=1}^\numClusters \bigl( \hat{N}_{x, \hat{\mathcal{V}}_k}\ln(\hat{p}_{l, k}) + \hat{N}_{\hat{\mathcal{V}}_k, x}\ln(\hat{p}_{k, l}/\hat{\alpha}_{l}) \bigr) - \frac{\ell}{n}\frac{\hat{\pi}_l}{\hat{\alpha}_{l}}$\;
		\State $\hat{\mathcal{V}}^{\prime}_c \gets \hat{\mathcal{V}}^{\prime}_c \cup \{x\}$\;
	\EndFor
\end{algorithmic}
\end{algorithm}

\section{Robustness of the clustering procedure to model violations}
\label{sec:Robustness_of_the_clustering_procedure}

Recall that the asymptotic consistency of the clustering procedure has been theoretically studied in \cite{sanders2020clustering_sm} under the assumption that the data-generating process is a \gls{BMC}.
In this section we aim to study the robustness of the clustering procedure to violations of this model assumption.
That is, we investigate the performance of the clustering procedure when the data-generating process is not actually a \gls{BMC}.
We study two main measures of performance.
First, in Appendix \ref{sec: MisclassificationRatio}, we consider the number of misclassified states.
Second, in Appendix \ref{sec:The_time_dependent_misspecification_error_when_approximating_a_perturbed_BMC_by_a_BMC}, we consider the approximation error in a parameter estimation problem where the objective is to estimate the true transition matrix $P$ of a Markovian data-generating process which need-not be a \gls{BMC}.

The first measure of performance requires that the notion of misclassification is sensible even though the data-generating process is not a \gls{BMC}.
To this end we restrict ourselves to models where communities are still well-defined.
More precisely, we consider the perturbed \gls{BMC} model which was defined in Section \ref{sec:Other_models_for_experimentation} and assign as ground-truth communities those of the \gls{BMC}-kernel which was used to construct the perturbed model.
Recall that the definition of a perturbed \gls{BMC} requires to specify the nature of the perturbation kernel $\Delta$.
The following kernels are used for this purpose to model different types of model violations:
\begin{itemize}

	\item[(i)]
	\emph{Uniform Stochastic}: The matrix $\Delta$ is sampled uniformly at random in the set of stochastic matrices. This is accomplished by sampling each row independently from a\\
	$\mathrm{Dirichlet}(1/n, \ldots, 1/n)$ distribution.

	\item[(ii)]
	\emph{Degree $0$}: Fix some $\pi_1,\ldots,\pi_n>0$ with $\sum_{i=1}^n \pi_i = 1$ and let $\Delta_{ij} = \pi_j$ for all $i,j \in [n]$. We construct the $\pi_i$ by sampling independent exponential random variables $e_1,\ldots,e_n\sim \operatorname{Exponential}(1)$ and normalizing $\pi_i = e_i/(\sum_{j=1}^n e_j)$.

	\item[(iii)]
	\emph{Heavy Tailed}: Let $X$ be a random matrix whose entries $X_{ij}$ are i.i.d.\ positive random variables with a heavy-tailed distribution.
	The kernel $\Delta$ is then found by normalizing the rows in order to achieve a stochastic matrix $\Delta := \operatorname{diag}\bigl((\sum_{j}X_{ij})^{-1}\bigr)_{i=1}^n X$.
	We sample the heavy-tailed entries $X_{ij}$ from a Zipf distribution with exponent $s=3/2$.

	\item[(iv)]
	\emph{Sparse}: Consider constants $d > 0$ and $c > 0$ and construct a random matrix $X = A + cJ$ where $A$ is the adjacency matrix from a directed Erd\"os--R\'{e}nyi random graph with average outgoing degree $d$ and $J$ is a constant matrix $J_{ij} = 1/n$.
	The kernel $\Delta$ is then found by rescaling the rows in order to achieve a stochastic matrix $\Delta = \operatorname{diag}\bigl((\sum_{j}X_{ij})^{-1}\bigr)_{i=1}^n X$.
	We take $d = 5$ and $c = 0.1$.
\end{itemize}
In our subsequent experimentation we take $n=2m$ to be an even integer. The \gls{BMC} which is perturbed is chosen to have two equally-sized clusters ($\numClusters = 2$) and cluster transition matrix given by
\begin{align*}
	p = \begin{pmatrix}
		0.6 & 0.4 \\
		0.4 & 0.6
	\end{pmatrix}.
\end{align*}
\subsection{Misclassification ratio for perturbed \texorpdfstring{\glspl{BMC}}{BMCs}}\label{sec: MisclassificationRatio}
This section concerns the number of misclassified states when clustering on a perturbed \gls{BMC} model.
Recall that we chose the \gls{BMC} model to have two equally-sized clusters which means that we may pick the cluster assignment map to be given by $\sigma_n(i) =1 + \indicator{i > n/2}$.
Let $\hat{\sigma_n}:[n]\to \{1,2 \}$ be an estimated cluster assignment which is output by the clustering procedure.
Then, the misclassification ratio $\mathcal{E}$ is defined as
\begin{equation}
	\mathcal{E}
	:=
	\frac{1}{n}\min_{\rho \in S_2} \# \{v\in [n]: \sigma_n(v) \neq (\rho \circ \hat{\sigma_n})(v) \} .
	\label{eqn:def_ratio_misclassified_states}
\end{equation}
Here $S_2$ denotes the set of permutations of $\{1,2 \}$.

Recall from Section \ref{sec:Other_models_for_experimentation} that the parameter $\varepsilon$ of the perturbed \gls{BMC} measures the fraction of transitions which are affected by the perturbation.
In other words, $\varepsilon$ measures the strength of the perturbation.
The estimated expected misclassification ratio $\mathbb{E}[\mathcal{E}]$ is displayed as a function of the perturbation level $\varepsilon$ for a numerical experiment in Fig. \ref{fig:misclassified}.
Up to $\varepsilon \approx 0.1$ the algorithm succeeds in recovering the exact cluster assignment for all four models.
The exact number will naturally depend on the parameters of the \gls{BMC} which was perturbed and will consequently be different in different contexts.
At any rate, we conclude from this experiment that the algorithm appears to be robust with regards to small to medium-sized model violations.

The observation that some model violations can be tolerated may be understood theoretically in terms of the construction of the algorithm.
This robustness is namely natural at the level of the spectral step of the algorithm.
Consider that in a perturbed \gls{BMC} one has the following decomposition:
\begin{align*}
	\hat{N}_{\text{Perturbed}} & = \mathbb{E}[\hat{N}_{\text{BMC}}] + (\mathbb{E}[\hat{N}_{\text{Perturbed}}] - \mathbb{E}[\hat{N}_{\text{BMC}}]) + (\hat{N}_{\text{Perturbed}} - \mathbb{E}[\hat{N}_{\text{Perturbed}}]) \\
	& =: \mathbb{E}[\hat{N}_{\text{BMC}}] + E_{\text{Perturbation}} + E_{\text{Noise}}.
\end{align*}
The sampling noise $E_{\text{Noise}}$ is small in operator norm relative to $\mathbb{E}[\hat{N}_{\text{BMC}}]$ when the sample path is sufficiently long.
It may further be expected that $E_{\text{Perturbation}}$ is small in operator norm whenever the perturbation level $\varepsilon$ is small.
Now recall that spectral step in the algorithm relies on singular value decomposition to compute a rank-$\numClusters$ approximation.
The purpose of this rank-$\numClusters$ approximation, when the process is truly a \gls{BMC}, is to separate the sampling noise $E_{\text{Noise}}$ from the low-rank signal $\mathbb{E}[\hat{N}_{\text{BMC}}]$.
In a small perturbation of a \gls{BMC} the singular value decomposition will however also regard $E_{\text{Perturbation}}$ as an error term.
Consequently, for small perturbations, the spectral step has the beneficial effect that it separates the perturbative error $E_{\text{Perturbation}}$ from the low-rank signal $\mathbb{E}[\hat{N}_{\text{BMC}}]$.

\subsection{Bias--variance tradeoff for parameter estimation in a perturbed \texorpdfstring{\gls{BMC}}{BMC}}
\label{sec:The_time_dependent_misspecification_error_when_approximating_a_perturbed_BMC_by_a_BMC}
It may occur in some cases that one is not interested in the clusterings themselves but rather views them as a means to an end.
Consider the scenario where one desires to estimate the transition kernel of a Markovian process which need not be a \gls{BMC}.
Assume that one has prior reason to suspect that there could be some underlying clusters in the data but also that there could be parts of the dynamics which do not respect the clusters.
In such a case a perturbed \gls{BMC} would be a suitable model for the data.
Let us emphasize that one is here not intrinsically interested in the \gls{BMC}-component $P_{\text{BMC}}$ but rather desires to estimate the ground-truth $P_{\text{True}} := (1-\varepsilon)P_{\text{BMC}} + \varepsilon \Delta$.
It could however be the case that one can exploit the underlying clusters to improve the performance of estimation.

Assume that one knows the number of underlying clusters $\numClusters$ and has access to a sample path $X_0^\varepsilon,\ldots,X_\ell^\varepsilon$ of length $\ell$ of a perturbed \gls{BMC}. Let $\hat{N}$ also denote the associated empirical frequency matrix.
A natural general-purpose estimator for the transition matrix, which does not rely on the existence of clusters, is given by the empirical transition matrix $\hat{P}(\ell)$.
The entries of the empirical transition matrix are given by
\begin{equation}
	\hat{P}_{\text{Empirical}}(\ell)_{ij} := \begin{cases}
		\frac{\hat{N}_{ij}}{\sum_{k=1}^n \hat{N}_{ik}}, \quad & \text{ if }\hat{N}_{ij}\neq 0 \\
		0, \quad                                                      & \text{ if }\hat{N}_{ij}=0.
	\end{cases}
\end{equation}
Another estimator may be found by first computing a clustering $\hat{\mathcal{V}}_1,\ldots,\hat{\mathcal{V}}_\numClusters$.
One can then hope that, since $P_{\mathrm{True}} \approx P_{\text{BMC}}$ for $\varepsilon \approx 0$, it would be sufficient to consider an estimator $\hat{P}_{\text{BMC}}$ for $P_{\text{BMC}}$ whose entries are given by
\begin{equation}
	\hat{P}_{\text{BMC}}(\ell)_{ij} := \begin{cases}
		\frac{1}{\#\hat{\mathcal{V}}_{\hat{\sigma_n}(j)}}\frac{\sum_{x\in \hat{\mathcal{V}}_{\hat{\sigma_n}(i)} , y\in \hat{\mathcal{V}}_{\hat{\sigma_n}(j)}} \hat{N}_{x,y}}{\sum_{m =1}^\numClusters\sum_{x\in \hat{\mathcal{V}}_{\hat{\sigma_n}(i)} , y\in \hat{\mathcal{V}}_{m}} \hat{N}_{x,y} },\quad & \text{ if }\sum_{x\in \hat{\mathcal{V}}_{\hat{\sigma_n}(i)} , y\in \hat{\mathcal{V}}_{\hat{\sigma_n}(j)}} \hat{N}_{x,y}\neq 0 \\
		0,\quad                                                                                                                                                                                                                                                                                            & \text{ if } \sum_{x\in \hat{\mathcal{V}}_{\hat{\sigma_n}(i)} , y\in \hat{\mathcal{V}}_{\hat{\sigma_n}(j)}}\hat{N}_{x,y} = 0.
	\end{cases}
\end{equation}
Finally, for comparison we also consider the following trivial estimator which does not even use the data
$$\hat{P}_{\text{Uniform}}(\ell)_{ij} = \frac{1}{n}.$$
We measure the performance of these estimators as a function of the length of the sample path using the expected estimation error:
\begin{equation}
	R_{*}(\ell)
	% &
	:= \mathbb{E}[\pnorm{P_{\text{True}} - \hat{P}_{*}(\ell)}{}] \quad \text{where } * \in \{\text{Empirical},\text{BMC},\text{Uniform}\}.
	\label{eqn:def_spectral_norm}
\end{equation}
Here, $\pnorm{\cdot}{}$ denotes the operator norm $\Vert M \Vert = \sup_{\Vert v \Vert_2 = 1} \Vert M v \Vert_2$.

We conduct a numerical experiment with a state space of size $n =1000$ and a heavy-tailed perturbation model of perturbation strength $\varepsilon = 0.05$.
Estimated values of the expected estimation error $R_{*}(\cdot)$ as a function of the length $\ell$ of the sample path are displayed in Fig. \ref{fig:misclassified}.
A number of different regimes may be identified.
First, the regime where the sample path is very short meaning that $\ell \approx 10^4$.
Here the empirical estimator $\hat{P}_{\text{Empirical}}$ and the \gls{BMC} estimator $\hat{P}_{\text{BMC}}$ are both unable to outperform the trivial estimator $\hat{P}_{\text{Uniform}}$.
The empirical estimator even performs significantly worse than the trivial estimator in this regime.
Second, the regime where sample path is medium-sized  meaning that $\ell \approx 10^5$.
Here the clustering procedure succeeds and $\hat{P}_{\text{BMC}}$ becomes the best-performing estimator.
Finally, the regime where the sample path grows long meaning that $\ell > 10^6$.
Here the empirical estimator becomes the best-performing estimator.
These different regimes can be understood in terms of a bias--variance tradeoff.
Namely, consider that for short to medium-sized sample paths the \gls{BMC} estimator $\hat{P}_{\text{BMC}}$ has significantly less variance than the empirical estimator $\hat{P}_{\text{Empirical}}$ due to depending on fewer parameters.
This decreased variance is the dominant consideration for the approximation error in this regime.
On the other hand, for long sample paths both estimators $\hat{P}_{\text{BCM}}$ and $\hat{P}$ have low variance and the bias incurred by the approximation $P_{\text{True}} \approx P_{\text{BMC}}$ becomes dominant.

\section{Methods for evaluating clusters and models}
\label{sec:Evaluating_clusters_and_models}

We next discuss methods which can aid in evaluating clusters and models for sequential data obtained from real-world processes.
These methods have to account for the fact that, since we are dealing with real-world nonsynthetic data, we do not know the true process which generated the data.
In particular, we do not have access to a ground-truth clustering.

\subsection{Performance on a downstream task}
\label{sec:Performance_on_a_downstream_task}

One reason to cluster observations of sequential data, is that the clusters provide a tool for dimensionality reduction in subsequent statistical analyzes or optimization procedures.
For instance, the running time of a numerical method which aims to execute some computational task on a sequence of observations may grow considerably with the number of distinct observations $n$.
In such a case it is clear that one has to reduce $n$ or otherwise use a different algorithm.
Reducing $n$ can also help to reduce overfitting, and aid in interpretability.

On the other hand, clustering naturally removes some information from the dataset.
Thus, in a good clustering, the data should retain as much useful information as is possible.
The meaning of ``amount of useful information'' is here ambiguous and depends on the context.
There are cases, however, where the notion can be made concrete.
For instance, suppose that one has a measure of quality $Q_{\textnormal{pre-reduction}} := Q(T)$, evaluating performance of a downstream task $T := T(X_{1:\ell})$ applied to the sequence of observations.
For example, if the algorithm is estimating parameters of some parametric model, then $Q_{\textnormal{pre-reduction}}$ may be the accuracy of prediction on a validation dataset.
One can now use this measure of quality $Q_{\textnormal{pre-reduction}}$ as a proxy for the notion of useful information in a clustering.
Given a clustering $\sigma_n : [n] \to [\numClusters]$ that reduces the number of distinct observations to some $1 \leq \numClusters \ll n$, one can apply the numerical solution method to obtain a solution $\tilde{T} := T( \sigma_n(X_{1:\ell}) )$.
The quantity $Q_{\text{reduced}} := Q(\tilde{T})$ then allows us to determine the quality of the clusters.

Using $Q$ to determine the amount of useful information in clusters can help compare the quality of a number of different clusters which are output by different clustering algorithms.
It can also happen that $Q_{\textnormal{reduced}} > Q_{\textnormal{pre-reduction}}$ due to the reduction of noise within the sequence of grouped observations.
This effect may occur regardless of whether the task is numerically challenging.
When the task is numerically challenging, then the dimension reduction (from $n$ to $\numClusters$) by the map $\sigma_n$ means that we can expect improved performance over methods that do not cluster data when fixing the computational budget.

In the following Sections \ref{sec:Model_selection_with_validation_data} to \ref{sec:The_shape_of_spectral_noise_for_identification_of_alternative_models} we discuss methods which can also reveal whether the \gls{BMC} model is appropriate, and do not require some data-specific measure of quality.

\subsection{Model selection with validation data}
\label{sec:Model_selection_with_validation_data}

Section \ref{sec:Performance_on_a_downstream_task} mentioned that prediction of validation data can serve as a measure of quality $Q$.
We now expand on this idea.

\subsubsection{Rescaled log-likelihood ratio}
Assume we observe sequential data $X_{1:\ell}$ generated by some ground-truth probability distribution $\mathbb{T}$ on $[n]^{\ell +1}$.
The law $\mathbb{T}$ can in principle be arbitrarily complex; for example, the Markov property need not be satisfied.
Note that, for nonsynthetic data, we typically do not have access to the ground-truth $\mathbb{T}$.
Suppose however that we do have two candidate models $\mathbb{P}$ and $\mathbb{Q}$ which are also defined on $[n]^{\ell +1}$.
We then want to determine whether $\mathbb{P}$ or $\mathbb{Q}$ is a better model based on the observed sequential data $X_{1:\ell}$.

For this purpose, we consider a log-likelihood ratio.
Namely, given $x_{1:\ell} \in [n]^{\ell +1}$, consider the quantity
\begin{equation}
	\hat{D}(x_{1:\ell}; \mathbb{P}, \mathbb{Q})
	:=
	\frac{1}{\ell} \ln\frac{\mathbb{P}[X_{1,\ell}= x_{1:\ell}] }{\mathbb{Q}[X_{1:\ell} = x_{1:\ell}]}
	\label{eqn:Biased_estimator_of_the_KL_divergence_rate_difference}
\end{equation}
and its expectation
\begin{equation}
	D(\mathbb{T}; \mathbb{P},\mathbb{Q}):= \mathbb{E}_{\mathbb{T}}[
	\hat{D}(X_{1:\ell}; \mathbb{P}, \mathbb{Q})].\label{eqn:ExpectedLogLikely}
\end{equation}
Then, if $D(\mathbb{T}; \mathbb{P}, \mathbb{Q})>0$ we consider $\mathbb{P}$ to be a better approximation of the ground truth $\mathbb{T}$ and if $D(\mathbb{T}; \mathbb{P}, \mathbb{Q})<0$ we consider $\mathbb{Q}$ to be a better approximation.
In practice we can not compute the expectation $\mathbb{E}_{\mathbb{T}}$ and instead consider the sign of the empirical estimator $\hat{D}(X_{1:\ell}; \mathbb{P},\mathbb{Q})$.

In our experiments it is often the case that $\mathbb{P}$ and $\mathbb{Q}$ are \glspl{MC} on $[n]$
whose transition matrices $P, Q \in [0,1]^{n \times n}$ are known.
In this case one can alternatively express \eqref{eqn:Biased_estimator_of_the_KL_divergence_rate_difference} as
\begin{equation}
	\hat{D}(x_{1:\ell}; \mathbb{P}, \mathbb{Q})
	=
	\frac{1}{\ell}
	\sum_{t=1}^{\ell-1}
	\ln{ \frac{P_{x_t,x_{t+1}}}{Q_{x_t,x_{t+1}}} }
	.
	\label{eqn:Biased_estimator_of_the_KL_divergence_rate_difference_MC}
\end{equation}
Confidence bounds for the estimation of $D(\mathbb{T}; \mathbb{P}, \mathbb{Q})$ by $\hat{D}(X_{1:\ell}; \mathbb{P}, \mathbb{Q})$ in this \gls{MC}-setting are provided in Section \ref{sec:Confidence_bounds_when_estimating_DRPQ}.
It is there additionally assumed that $\mathbb{T}$ is a \gls{MC}, possibly time-inhomogeneous, whose mixing time is known.

\subsubsection{Information-theoretic interpretation for \texorpdfstring{$D(\mathbb{T}; \mathbb{P},\mathbb{Q})$}{D(T,P,Q)}}

Let us briefly note that \eqref{eqn:ExpectedLogLikely} has an information-theoretic interpretation.
Namely, observe that
\begin{equation}
	D(\mathbb{T};\mathbb{P}, \mathbb{Q})
	=\frac{1}{\ell}(
	\mathrm{KL}( \mathbb{T}; \mathbb{Q} )
	-
	\mathrm{KL}( \mathbb{T}; \mathbb{P} ) )
	\label{eqn:LogKL}
\end{equation}
where $\mathrm{KL}$ denotes the \gls{KL} divergence
\begin{align}
	\mathrm{KL}( \mathbb{T}; \mathbb{P} )
	:=
	\mathbb{E}_{Z_{1:\ell} \sim \mathbb{T}}\Bigl[\ln\Bigl(\frac{\mathbb{T}(X_{1:\ell} = Z_{1:\ell})}{\mathbb{P}(X_{1:\ell} = Z_{1:\ell})}\Bigr)\Bigr]. \label{eq: KL_in_terms_of_ell}
\end{align}
One can interpret the quantity $\mathrm{KL}( \mathbb{T}; \mathbb{P} )$ as the expected amount of discriminatory information revealing that $\mathbb{P}$ is not quite the ground-truth probability distribution underlying the sample path $X_{1:\ell}$; see \cite{kullback1951information}.
In many cases, such as when the ground-truth $\mathbb{T}$ is an ergodic Markov chain, it further holds that \eqref{eq: KL_in_terms_of_ell} grows linearly in terms of amount of data $\ell$.

Correspondingly, by \eqref{eqn:LogKL}, one can view $D(\mathbb{T};\mathbb{P},\mathbb{Q})$ as measuring the rate of growth for discriminatory information revealing that $\mathbb{P}$ is a better approximation for the ground truth $\mathbb{T}$ than $\mathbb{Q}$.
To emphasize this perspective we may refer to $\hat{D}(X_{1:\ell}; \mathbb{P},\mathbb{Q})$ as the \emph{\gls{KL} divergence rate difference estimator}.

\subsubsection{Estimation when the models are inferred from the data}
Our experiments routinely determine two different candidate models that we wish to compare, from the same one sample sequence available to us.
Let us emphasize this fact by referring to these candidate models as
\begin{equation}
	\hat{\mathbb{P}}^{X_{1:\ell}}
	\quad
	\textnormal{and}
	\quad
	\hat{\mathbb{Q}}^{X_{1:\ell}}
	.
	\label{eqn:Sample_sequence_dependent_candidate_transition_matrices}
\end{equation}
Observe now that these two candidate models are a function of the observed data $X_{1:\ell}$.
Substituting \eqref{eqn:Sample_sequence_dependent_candidate_transition_matrices} into \eqref{eqn:Biased_estimator_of_the_KL_divergence_rate_difference} could consequently result in a biased estimator and typically favor models with many parameters;
the ``optimal'' model would be the degenerate probability distribution assigning probability 1 to the observed $X_{1:\ell}$.

To reduce the bias, we use a holdout method.
Specifically, we will split the trajectory into two parts: the first half
$
	X_{1:\lfloor \ell/2 \rfloor}
$
will be used for training,
and the second half
$
	X_{\lfloor \ell/2 \rfloor+1:\ell}
$
for validation.
The estimator
\begin{equation}
	\hat{D}(
	X_{\lfloor \ell/2 \rfloor+1:\ell}
	;
	\hat{\mathbb{P}}^{X_{1:\lfloor \ell/2 \rfloor}}
	,
	\hat{\mathbb{Q}}^{X_{1:\lfloor \ell/2 \rfloor}}
	)
	\label{eqn:Less_biased_estimator_of_the_KL_divergence_rate_difference}
\end{equation}
will then significantly reduce the amount of bias when compared to the estimator obtained by substituting \eqref{eqn:Sample_sequence_dependent_candidate_transition_matrices} into \eqref{eqn:Biased_estimator_of_the_KL_divergence_rate_difference}.

Note that \eqref{eqn:Less_biased_estimator_of_the_KL_divergence_rate_difference} can be viewed as a measure of quality in the language of Section \ref{sec:Performance_on_a_downstream_task}.
Eq.\ \eqref{eqn:Less_biased_estimator_of_the_KL_divergence_rate_difference} namely compares whether $\hat{\mathbb{P}}^{X_{1:\lfloor \ell /2\rfloor}}$ or $\hat{\mathbb{Q}}^{X_{1:\lfloor \ell /2\rfloor}}$ better predicted the validation data.

\subsection{Model selection with only training data}
\label{sec:Hypothesis_testing_for_the_order}
As discussed in Section \ref{sec:Model_selection_with_validation_data}, the \gls{KL}-divergence rates can provide a good rule--of--thumb for assessing what models are most interesting but are biased towards models with more parameters if one does not split the data into training and validation data.
Splitting the data is however sometimes undesirable. Namely, if the data is sparse, the estimated models will become even less accurate.
In order to overcome this issue, we will use information criteria that compensate the bias incurred and use it to assess the order of the cluster process.

\subsubsection{Problem setting: order of a \texorpdfstring{\gls{BMC}}{BMC}}

Suppose that a sequence $X_{1:\ell}$ was in fact generated by some $r$th-order \gls{BMC}, but that the order $r \in \{ 0, 1, \ldots \}$ is unknown.
We will use techniques for model selection to try and determine $r$ from the cluster sequence $Y_{1:\ell} = \sigma_n(X_{1:\ell})$.

There are two reasons for using $Y_{1:\ell}$ instead of $X_{1:\ell}$.
First, the parametric models for higher order \glspl{MC} without clusters have a comparable number of free parameters as the sequence length $\ell$ itself, so estimators for the order will behave poorly.
If we look at the cluster chain instead, the number of degrees of freedom will depend on the cluster number $\numClusters$ instead of the number of states $n$, and fortunately $\numClusters \ll n$.
Secondly, we can also study the robustness of the model selection procedure depending on the clustering algorithm.

\subsubsection{Order selection by minimizing an information criterion}\label{sec: OrderSelection}
The parameter that determines the $r$th-order \gls{BMC} model for $Y_{1:\ell}$ is a transition matrix $Q^r$; recall \eqref{eqn:Definition_of_transition_matrix_Pr}.  Note here that the chain $Y_{1:\ell-r}^{r}$ will be constructed from the chain of clusters $Y_{1:\ell} = \sigma_n(X_{1:\ell})$ for a fixed cluster assignment $\sigma_n$.

To estimate $Q^r$ one can consider the log-likelihood
\begin{equation}
	\mathcal{L}( Y_{1:\ell} \mid Q^r )
	:=
	\sum_{t=r}^{\ell-r-1}
	\ln{
	Q^r_{Y_{t-r+1:t},Y_{t+1}}
	}.
	\label{eqn:Log_likelihood_of_rth_order_MC_with_fixed_n_r}
\end{equation}
The maximum-likelihood estimator associated with \eqref{eqn:Log_likelihood_of_rth_order_MC_with_fixed_n_r} is namely given by
\begin{align}
	( \hat{Q}^{r, \mathrm{MLE}} )_{i^r, j}
	:=
	\label{eqn:Empirical_transition_matrices_hatQr}
	\begin{cases}
		0
		 &
		\textnormal{if } \sum\limits_{t=r}^{\ell-r-1}
		\indicator{ Y_{t-r+1:t} = i^r } =0,
		\\
		\frac{
			\sum\limits_{t=r}^{\ell-r-1}
			\indicator{ Y_{t-r+1:t} = i^r, Y_{t+1} = j }
		}
		{
			\sum_{t=r}^{\ell-r-1}
			\indicator{ Y_{t-r+1:t} = i^r }
		}
		 &
		\textnormal{otherwise.}
	\end{cases}
\end{align}
Here $i^r, j$ run over all possible sequences in $[\numClusters]^r$ and $[\numClusters]$ respectively.
We denote $\hat{\mathbb{Q}}^{r,\mathrm{MLE}}$ for the law of an $r$th-order \gls{MC} with $\numClusters$ states and transition matrix $\hat{Q}^{r,\mathrm{MLE}}$.

To determine what order $r$ is the true underlying order of the data one would like to compare $\hat{\mathbb{Q}}^{r, \mathrm{MLE}}$ and $\hat{\mathbb{Q}}^{s, \mathrm{MLE}}$ for some $s \neq r$.
As has been remarked in Section \ref{section:summary_methods_for_eval}, using \eqref{eqn:Biased_estimator_of_the_KL_divergence_rate_difference} for this purpose would give a biased estimator.
Problems with bias in model selection are well-known in the statistics literature and to avoid this issue, the so-called \emph{information criteria} were developed \cite{akaike1974new,bozdogan1987model,anderson2004model,ding2018model}, where to a log-likelihood a penalty term is added to correct the bias.

In our setting we need a penalty term that is sensitive to sparse data and is also consistent. For this purpose, we have chosen the \gls{CAIC} \cite{bozdogan1987model}:
for model $\hat{\mathbb{Q}}^{r,\mathrm{MLE}}$,
\begin{align}
	\mathrm{CAIC}(\hat{Q}^{r,\mathrm{MLE}})
	:=
	-2 \ln{
		\bigl(
		\mathcal{L}( Y_{1:\ell} \mid \hat{Q}^{r,\mathrm{MLE}} )
		\bigr)
	}
	+
	2 \mathrm{DF}(\numClusters,r)
	\bigl(
	1
	+
	\ln{
			(
			\ell - r
			)
		}
	\bigr)
	.
	\label{eqn:CAIC2}
\end{align}
Here, $\mathrm{DF}(\numClusters,r)$ the degrees of freedom in an $r$th-order \gls{MC} constrained to have fixed parameters $\numClusters$ and $r$.
Specifically,
\begin{equation}
	\mathrm{DF}(\numClusters,r)
	=
	\numClusters^r (\numClusters-1)
\end{equation}
where the factor $(\numClusters-1)$ is due to the fact that the rows of $Q^r$ are constrained to add up to one.
We will utilize the \gls{CAIC} to select the right order as follows.
From the collection of models $\hat{\mathbb{Q}}^{0,\mathrm{MLE}}$, $\hat{\mathbb{Q}}^{1,\mathrm{MLE}}$, $\hat{\mathbb{Q}}^{2,\mathrm{MLE}}$, $\ldots$, we may determine the order $r^{\mathrm{CAIC}}$ that minimizes the \gls{CAIC}:
\begin{gather}
	r^{\mathrm{CAIC}}
	:=
	\operatorname{argmin}\limits_{ r \in \{ 0, 1, 2, \ldots \} }
	\mathrm{CAIC}(\hat{Q}^{r,\mathrm{MLE}})
	.
\end{gather}
Note that lower-dimensional models are favored since the degrees of freedom $\mathrm{DF}(\numClusters,r)$, and thus the penalty terms in \eqref{eqn:CAIC2}, increase exponentially in $\numClusters, r$.

In order to evaluate how robust the \gls{CAIC} criterion is, we will estimate the over- and underfit error probabilities with error models and draw conclusions on the selected orders.

\subsection{The shape of spectral noise for identification of alternative models}
\label{sec:The_shape_of_spectral_noise_for_identification_of_alternative_models}

The methods in Sections \ref{sec:Performance_on_a_downstream_task} to \ref{sec:Model_selection_with_validation_data} allow us to compare a \gls{BMC} to alternative models.
The selection of a good alternative can however be difficult when a more complex model than a \gls{BMC} is desirable.
The method described here can aid in the selection of an alternative model.

The method is based on a result from \cite{sanders2022singular} which describes the histogram of the singular values of $\hat{N}$ in the asymptotic regime $n\to \infty$ under the condition that $\ell = \Theta(n^2)$.
The results in \cite{sanders2021spectral} can further be interpreted as the statement that the $\numClusters$ nonzero singular values of $\mathbb{E}[\hat{N}]$ correspond to the $\numClusters$ largest singular values of $\hat{N}$.
In other words, all singular values except these leading few may be interpreted as being due to the noise $\hat{N} - \mathbb{E}[\hat{N}]$.
The histogram of the nonleading singular values may thus be interpreted as the \emph{shape of the spectral noise}.

These results and their interpretation can guide the selection of a good model.
One can namely identify clusters in the data and visually compare the associated \gls{BMC}-prediction with the observed histogram.
If there is a good match, then this may indicate that a \gls{BMC} suits the data well.
If there is a discrepancy, then the nature of the discrepancy can be informative of the properties that the alternative model should have.
It will for instance be shown in Section \ref{sec: Simulation} that a long tail can sometimes be explained using a heavy-tailed perturbation.

We have, however, found that a strongly inhomogeneous equilibrium distribution in the data can dominate the spectral noise in $\hat{N}$.
So long as the clustering respects the equilibrium distribution it then follows that the observations will indeed resemble the theory.
This is an issue since it follows that, in the case of an inhomogeneous equilibrium distribution, the spectral noise of $\hat{N}$ may not be particularly informative.
In such a case one can consider a different random matrix.

The \emph{empirical normalized Laplacian} $\hat{L}$ associated to the observation sequence is element-wise given by
\begin{align}
	\hat{L}_{ij}
	:=
	\begin{cases}
		\frac{\hat{N}_{ij}}{\sqrt{\sum_{k=1}^n \hat{N}_{ik}}\sqrt{\sum_{k=1}^n \hat{N}_{kj}}} & \textnormal{if } \hat{N}_{ij} \neq 0, \\
		0                                                                                     & \textnormal{otherwise.}               \\
	\end{cases}
	\label{eqn:Definition_of_the_empirical_Laplacian2}
\end{align}
We argue in Section \ref{sec:Degree_corrected_BMC} that the variance of the entries of $\hat{L}$ is approximately independent of the equilibrium distribution. Consequently, we expect that the spectral noise of $\hat{L}$ will not be dominated by a possibly inhomogeneous equilibrium distribution.
A proposition describing the limiting histogram of singular values is proved in Section \ref{sec: LaplaceTheory}.
The precise statement is technical but a summary may be found in Proposition \ref{result: Short_Version_Limiting_singular_value_distribution_of_the_Laplacian}.
  
\begin{proposition}
	\label{result: Short_Version_Limiting_singular_value_distribution_of_the_Laplacian}

	Let $X_{1:\ell}$ be a sample path of a \gls{BMC}.
	If $\ell = \Theta(n^2)$, then for almost every $a,b \in \mathbb{R}$ the fraction of singular values in $[a,b]$, i.e., $n^{-1}\#\{i:s_i(\sqrt{n}\hat{L}) \in [a,b]\}$ converges in probability as $n\to \infty$.
	The limit may be computed explicitly in terms of the parameters of the \gls{BMC}.
\end{proposition}

With \refProposition{result: Short_Version_Limiting_singular_value_distribution_of_the_Laplacian}, we can characterize the spectral noise of $\hat{L}$ in a \gls{BMC} and use the spectrum as a tool for data exploration, expectedly even in the presence of an inhomogeneous equilibrium distribution.

\section{Confidence bounds when estimating \texorpdfstring{$D(\mathbb{T}; \mathbb{P}, \mathbb{Q})$}{D(R;P,Q)}}
\label{sec:Confidence_bounds_when_estimating_DRPQ}

We here state a concentration inequality from which we deduce the confidence interval in \eqref{eq: ConfidenceInterval}. Recall that these confidence intervals are used in Fig. \ref{fig:Stock_market}.
The proof is based on a result from \cite{paulin2015concentration} whose assumptions we first verify.

Assume that the true process $\process{X_t}{t \geq 0}$ generating the sequential data $X_1,\ldots,X_\ell$ is a \gls{MC}, which need not be time-homogeneous.
Let us refer to $\process{X_t}{t \geq 0}$'s law as $\mathbb{T}$.
The mixing time of $\process{X_t}{t \geq 0}$ is defined as
\begin{align}
	\tau_{\mathrm{mix}}
	:=
	\min
	\bigl\{
	t\geq 1
	:
	\overline{d}(t)
	\leq
	\tfrac{1}{2}
	\bigr\}
	,
\end{align}
where
\begin{align}
	\overline{d}(t)
	:=
	\max_{1\leq i \leq \ell-t}
	\sup_{x,y\in [n] }
	\mathrm{d}_{\mathrm{TV}}
	\bigl(
	\mathbb{T}[X_{i+t} = \cdot \mid X_i = x],
	\mathbb{T}[X_{i+t} = \cdot \mid X_i = y]
	\bigr)
	.
\end{align}
Here, $\mathrm{d}_{\mathrm{TV}}$ denotes the \emph{total variation distance}:
\begin{align}
&
	\mathrm{d}_{\mathrm{TV}}
	\bigl(
	\mathbb{T}[X_{i+t} = \cdot \mid X_i = x],
	\mathbb{T}[X_{i+t} = \cdot \mid X_i = y]
	\bigr)
\\
	:= &
	\tfrac{1}{2}
	\sum_{z \in [n]}
	\bigl\lvert
	\mathbb{T}[X_{i+t} = z \mid X_i = x]
	-
	\mathbb{T}[X_{i+t} = \cdot \mid X_i = y]
	\bigr\rvert
	.
\end{align}

We claim that the \gls{MC} of transitions $\process{ E_{X,t} }{t \geq 0}$, where $E_{X,t} := (X_{t},X_{t+1})$, then has mixing time at most $\tau_{\mathrm{mix}} + 1$.
Indeed, observe that for any $t\geq \tau_{\mathrm{mix}}+1$, $x_1,x_2,y_1,y_2 \in [n]$ and $1\leq i \leq \ell - t-1$,
\begin{align}
	&
	\tfrac{1}{2}
	\sum_{z_1, z_2 \in [n]}
	\lvert
	\mathbb{P}[
	E_{X,i+t} = (z_1,z_2)\mid E_{X,i} = (x_1,x_2)
	]
	-
	\mathbb{P}[
	E_{X,i+t} = (z_1,z_2)\mid E_{X,i} = (y_1,y_2)
	]
	\rvert
	\nonumber
	\\
	&
	=
	\tfrac{1}{2}
	\sum_{z_1, z_2 \in [n]}
	\mathbb{P}[
	X_{i+t+1} = z_2\mid X_{i+t} =z_1] \label{eq: StepMarkov}\\
	& \hphantom{=
	\tfrac{1}{2}
	\sum_{z_1, z_2 \in [n]}
	\mathbb{P}} \times
	\bigl\lvert
	\mathbb{P}[
	X_{i+t} = z_1 \mid X_{i+1} = x_2]
	-
	\mathbb{P}[
	X_{i+t} = z_1 \mid X_{i+1} = y_2] \bigr \rvert \nonumber
	\\
	&
	=
	\tfrac{1}{2}
	\sum_{z_1\in [n]]}
	\bigl\lvert \mathbb{P}[
	X_{i+t} = z_1 \mid X_{i+1} = x_2]
	-
	\mathbb{P}[
	X_{i+t} = z_1 \mid X_{i+1} = y_2] \bigr \rvert
	\leq
	\tfrac{1}{2}
	.
	\label{eq: StepMix}
\end{align}
Here, the Markov property was used to conclude \eqref{eq: StepMarkov}.
The fact that $\mathbb{P}(X_{i+t=1} = \cdot \mid X_{i+t} = z_1)$ defines a probability distribution, together with the assumption that $t \geq \tau_{\mathrm{mix}}+1$ and the property that $\overline{d}(t)$ is nonincreasing in $t$, was used to arrive at \eqref{eq: StepMix}.

Now suppose that we are given two \glspl{MC} with fixed transition matrices $P$ and $Q$, whose laws we will refer to as $\mathbb{P}$ and $\mathbb{Q}$, respectively.
Assume furthermore that $\max_{i,j \in [n]}\lvert \ln(P_{i,j}/Q_{i,j}) \rvert \leq \delta$ for some $\delta >0$.
For any two sample paths $X_1,\ldots,X_\ell$ and $Y_1,\ldots,Y_\ell$, it then holds that
\begin{align}
	\lvert \hat{D}(X_1,\ldots,X_\ell; P,Q) -   \hat{D}(Y_1,\ldots,Y_\ell; P,Q) \rvert \leq \frac{2\delta}{\ell} \sum_{t=1}^{\ell - 1} \indicator{E_{X,t} \neq E_{Y,t}}.
\end{align}
Consequently, \cite[Corollary 2.10]{paulin2015concentration} applied to the \gls{MC} $\process{E_{X,t}}{t\geq 0}$ yields the desired concentration inequality:
\begin{align}
	\mathbb{P}\bigl(\lvert \hat{D}(X_0, \ldots, X_\ell; \mathbb{P}, \mathbb{Q}) - D(\mathbb{T}; \mathbb{P},\mathbb{Q}) \rvert >t\bigr) \leq 2\exp\Big(\frac{-t^2 \ell^2}{18 \delta^2 (\tau_{\mathrm{mix}}+1)}\Big).
\end{align}

In conclusion:
if we are given two \glspl{MC} with fixed transition matrices $P$ and $Q$ for which $\max_{i,j \in [n]}\lvert \ln P_{i,j} / Q_{i,j} \rvert > 0$, together with an estimate for $\tau_{\mathrm{mix}}$, we can then construct for $z \in [0,1]$ a $100(1-z)\%$ confidence intervals of size
\begin{equation}
	c_z
	:=
	\frac{1}{\ell}
	\max_{i,j \in [n]}
	\Bigl\lvert
	\ln{ \frac{P_{i,j}}{Q_{i,j}} }
	\Bigr\rvert
	\sqrt{ 18 (\tau_{\mathrm{mix}}+1) \ln{\frac{2}{z}} }
	.
	\label{eqn:Confidence_interval_for_KL_divergence_rate_difference}
\end{equation}
This is to say that
\begin{align}
	\mathbb{P}
	\Bigl[
	D(\mathbb{T}; \mathbb{P},\mathbb{Q})
	\in
	\bigl[
	\hat{D}(X_0, \ldots, X_\ell; \mathbb{P}, \mathbb{Q}) - c_z
	,
	\hat{D}(X_0, \ldots, X_\ell; \mathbb{P}, \mathbb{Q}) + c_z
	\bigr]
	\Bigr]
	\geq
	1 - z
	.
	\label{eq: ConfidenceInterval}
\end{align}

\section{Shape of the spectral noise}
\label{sec:Shape_of_the_spectral_noise}

Recall that it was stated in Section \ref{section:summary_methods_for_eval} that the spectral noise in $\hat{N}$ can be dominated by an inhomogeneous equilibrium distribution.
It was further claimed that the Laplacian $\hat{L}$ does not suffer from this issue.
The main goal in this section is to argue that this claim is true.

Some preliminary notation and concepts are introduced in Section \ref{sec: Prelim} after which a theoretical result concerning the limiting singular value distribution of $\hat{L}$ is established in Section \ref{sec: LaplaceTheory}.
A model with an inhomogeneous equilibrium distribution is introduced in Section \ref{sec:Degree_corrected_BMC}.
The claim that $\hat{L}$ can also detect violations to the model assumptions in the presence of an inhomogeneous equilibrium distribution is verified in Section \ref{sec: Simulation} by a simulation experiment.

\subsection{Preliminaries}
\label{sec: Prelim}

The \emph{empirical singular value distribution} $\nu_M$ of a matrix $M\in \mathbb{R}^{n\times n}$ with singular values $s_1(M) \geq \ldots \geq s_n(M)$ is the probability measure on $\mathbb{R}_{\geq 0}$ defined by
\begin{align}
	\nu_{M}(A)
	:=
	\frac{1}{n}
	\#\{i\in [n]:s_i(M) \in A \}
\end{align}
for every measurable set $A\subseteq \mathbb{R}$.
A sequence of random probability measures $\process{\mu_n}{n\geq 1}$ on the real line is are said to \emph{converge weakly in probability } to a probability measure $\mu$ if for every continuous bounded function $f:\mathbb{R} \to \mathbb{R}$ it holds that $\int f {\mathrm d}\mu_n$ converges weakly in probability to $\int f {\mathrm d}\mu$.
The \emph{symmetrization} of a probability measure $\mu$ on the positive real line $\mathbb{R}_{\geq 0}$ is the probability measure $\mu_{\operatorname{sym}}$ on $\mathbb{R}$ given by
\begin{align}
	\mu_{\operatorname{sym}}(A)
	:=
	\tfrac{1}{2}
	\bigl(
	\mu(\{a: a\in A, a\geq 0 \}) + \mu(\{-a:a\in A, a\leq 0\})
	\bigr)
\end{align}
for any measurable $A\subseteq \mathbb{R}$.
Note that $\mu$ can be recovered from its symmetrization since for any measurable $A\subseteq \mathbb{R}_{\geq 0}$ it holds that
\begin{align}
	\mu(A) = 2\mu_{\operatorname{sym}}(A\setminus \{0 \})  + \mu_{\operatorname{sym}}(\{0\}).
\end{align}

The \emph{Stieltjes transform} of a probability measure $\mu$ is the analytic function $s:\mathbb{C}^+ \to \mathbb{C}^-$ given by $s(z) = \int 1/(z-x) {\mathrm d}\mu(x)$.
Here, $\mathbb{C}^+ := \{z \in \mathbb{C}: \operatorname{Im}(z)>0 \}$ denotes the upper half-plane and $\mathbb{C}^-:= \{z \in \mathbb{C}: \operatorname{Im}(z)<0 \}$ denotes the lower half-plane.
The Stieltjes inversion formula \cite[Theorem B.8]{bai2010spectral} allows one to recover $\mu$ from its Stieltjes transform: for any continuity points $a<b$ of $\mu$,
\begin{align}
	\mu([a,b]) = -\frac{1}{\pi} \lim_{\varepsilon \to 0^+} \int_{a}^b \operatorname{Im}(s(x + \sqrt{-1}\varepsilon)){\mathrm d}x.
\end{align}

\subsection{Limiting law of singular value distribution of the Laplacian \texorpdfstring{$\hat{L}$}{L}}
\label{sec: LaplaceTheory}

Fix some positive integer $\numClusters\geq 1$ and a transition matrix $p\in \mathbb{R}^{\numClusters\times \numClusters}$ of an ergodic \gls{MC} on $[\numClusters]$.
Denote $\pi \in [0,1]^\numClusters$ for the equilibrium distribution of the \gls{MC} associated to $p$.
For every $n\geq 1$ consider a partition $\mathcal{V}_1 \cup\ldots \cup \mathcal{V}_\numClusters = [n]$ of the state space into $\numClusters$ nonempty groups $\mathcal{V}_i$.
The subsequent results are concerned with the asymptotic regime where $n\to \infty$.
We here assume that there are $\alpha_1,\ldots,\alpha_\numClusters >0$ such that $\# \mathcal{V}_i = \alpha_i n + o(n)$ and $\sum_{i=1}^\numClusters \alpha_i = 1$.

\begin{proposition}
	\label{prop:Limiting_singular_value_distribution_of_the_Laplacian}

	Let $\hat{L}$ be the empirical normalized Laplacian associated to a sample path $X_1, \ldots, X_\ell$ of the above \gls{BMC}.
	Assume that as $n$ tends to infinity it holds that $\ell = \lambda n^2 + o(n^2)$.
	Then, the empirical singular value distribution $\nu_{\sqrt{n}\hat{L}}$ converges weakly in probability to a compactly supported probability measure $\nu$ on $\mathbb{R}_{\geq 0}$.
	Moreover, the symmetrization $\nu_{\operatorname{sym}}$ has Stieltjes transform $s(z) = \sum_{i=1}^\numClusters \alpha_i (a_i(z) + a_{\numClusters+i}(z))/2$ where $a_1,\ldots,a_{2\numClusters}$ are the unique analytic function from $\mathbb{C}^+$ to $\mathbb{C}^-$ such that the following system of equations is satisfied
	\begin{align}
		a_i(z)^{-1}
		&
		=
		z - \sum_{j=1}^\numClusters \lambda^{-1} \pi(j)^{-1} \alpha_j p_{ij} a_{\numClusters+j}(z)
		,
		\\
		a_{i+\numClusters}(z)^{-1}
		&
		=
		z - \sum_{j=1}^\numClusters \lambda^{-1} \pi(i)^{-1} \alpha_j p_{j,i}a_j(z)
	\end{align}
	for $i = 1,\ldots,\numClusters$.
\end{proposition}

The proof of Proposition \ref{prop:Limiting_singular_value_distribution_of_the_Laplacian} is similar to the proof of \cite[Theorem 1.2]{sanders2022singular_sm} which is there given below \cite[Proposition 4.7]{sanders2022singular_sm}.
The intermediate \cite[Lemma 4.4(ii)]{sanders2022singular_sm} should however be replaced by Lemma \ref{lem: Q} below, and the role of \cite[Equation (22)]{sanders2022singular_sm} is taken over by Lemma \ref{lem: VarProfile} below.
\begin{lemma}\label{lem: Q}
	Let $\Pi_X \in [0,1]^{n}$ denote the equilibrium distribution of the \gls{BMC}, and define
	\begin{equation}
		\hat{Q}
		:=
		\operatorname{diag}((\ell+1) \Pi_X)^{-1/2}(\hat{N} -\mathbb{E}[\hat{N}])\operatorname{diag}((\ell+1) \Pi_X)^{-1/2}
		.
	\end{equation}
	Assume that $\nu_{\sqrt{n} \hat{Q}}$ converges weakly in probability to some probability measure $\nu$ on $\mathbb{R}_{\geq 0}$.
	Under the assumptions of Proposition \ref{prop:Limiting_singular_value_distribution_of_the_Laplacian}, it then holds that $\nu_{\sqrt{n} \hat{L}}$ converges weakly in probability to $\nu$.
\end{lemma}

\begin{proof}
	Consider the following notation:
	\begin{align}
		C_n
		&
		:=
		\operatorname{diag}((\ell+1) \Pi_X)^{-1/2}\mathbb{E}[\hat{N}]\operatorname{diag}((\ell+1) \Pi_X)^{-1/2}
		,
		\nonumber \\
		D_{n,l}
		&
		:=
		\operatorname{diag}\Bigl(\Bigl(\sum_{k=1}^n\hat{N}_{ik}\Bigr)_{i=1}^n\Bigr)^{-1/2}\operatorname{diag}((\ell+1) \Pi_X)^{1/2}
		,
		\nonumber \\
		D_{n,r}
		&
		:=
		\operatorname{diag}((\ell+1) \Pi_X)^{1/2}\operatorname{diag}\Bigl(\Bigl(\sum_{k=1}^n\hat{N}_{kj}\Bigr)_{j=1}^n\Bigr)^{-1/2}
		.
	\end{align}

	Observe that
	$
	\hat{L}
	=
	D_{n,l} \hat{Q} D_{n,r} + C_n
	.
	$
	Furthermore, $\max_{i=1}^n \lvert (\ell+1)^{-1}\Pi_{X,i}^{-1}\sum_{k=1}^n\hat{N}_{ik} - 1 \rvert$ converges to zero in probability by \cite[Corollary 6.11]{sanders2022singular_sm}.
	Since $x \mapsto 1/\sqrt{x}$ is continuous in the neighborhood of $1$ and the operator norm of a diagonal matrix is the maximal value on its diagonal, it follows that $\Vert D_{n,l} - \operatorname{Id} \Vert_{\operatorname{op}}$ converges to zero in probability.

	Note that transitions coming into state $i$ are almost in bijection with the outgoing transitions out of state $i$.
	The only possible exceptions occur when $i = X_1$ or $i = X_\ell$.
	This is to say that for every $i$
	\begin{align}
		\Bigl\lvert
		\sum_{k=1}^n\hat{N}_{ik}
		-
		\sum_{k=1}^n\hat{N}_{kj}
		\Bigr\rvert
		\leq
		2
		.
	\end{align}
	Hence, using that $(\ell + 1)\Pi_{X,i} = \Theta(n)$ and the fact that we already know that $\max_{i=1}^n \lvert (\ell+1)^{-1}\Pi_{X,i}^{-1}\sum_{k=1}^n\hat{N}_{ik} - 1 \rvert$ converges to zero in probability,
	it follows that\\
	$\max_{i=1}^n \lvert (\ell+1)^{-1}\Pi_{X,i}^{-1}\sum_{k=1}^n\hat{N}_{ki} - 1 \rvert$ converges to zero in probability.
	By the continuity of $1/\sqrt{x}$ near $1$ we may now also conclude that $\Vert D_{n,r} - \operatorname{Id} \Vert_{\operatorname{op}}$ converges to zero in probability.

	By two applications of \cite[Lemma 6.8.(iii)]{sanders2022singular_sm} we conclude that $\nu_{\sqrt{n}D_{n,l} \hat{Q} D_{n,r}}$ converges weakly in probability to $\nu$.

	Further, by the fact that the \gls{BMC} starts in equilibrium it holds that $\operatorname{rank}(\mathbb{E}[\hat{N}]) \leq \numClusters$.
	Hence, using the general fact that $\operatorname{rank}(AB) \leq \operatorname{rank}(A)$ for any two matrices $A,B$ of compatible size, we find that
	\begin{align}
		\operatorname{rank}(\sqrt{n}C_n) \leq \operatorname{rank}(\mathbb{E}[\hat{N}]) \leq \numClusters.
	\end{align}
	An application of \cite[Lemma 6.8.(ii)]{sanders2022singular_sm} now yields the desired result, since\\
	$
	\nu_{\sqrt{n} \hat{L}}
	=
	\nu_{\sqrt{n} ( D_{n,l} \hat{Q} D_{n,r} + C_n) }
	.
	$
\end{proof}

\begin{lemma}
	\label{lem: VarProfile}
	Under the assumptions of Proposition \ref{prop:Limiting_singular_value_distribution_of_the_Laplacian} and with notation as in Lemma \ref{lem: Q} it holds that as $n$ tends to infinity
	\begin{align}
		\max_{ij=1,\ldots,n}
		\bigl\lvert
		\operatorname{Var}[ \hat{Q}_{ij} ]
		-
		\lambda^{-1} \pi(\sigma_n(j))^{-1}p_{\sigma_n(i)\sigma_n(j)}
		\bigr\rvert
		=
		o(1).
	\end{align}
\end{lemma}

\begin{proof}
	This is immediate from \cite[Corollary 4.6]{sanders2022singular_sm} using the fact that $\operatorname{Var}[cX] = c^2 \operatorname{Var}[X]$ for any real random variable $X$ and scalar $c\in \mathbb{R}$.
\end{proof}

\subsection{Inhomogeneous equilibrium distribution: \texorpdfstring{\gls{DC-BMC}}{DC-BMC}}
\label{sec:Degree_corrected_BMC}

In order to allow for an inhomogeneous equilibrium distribution we consider the following model which is inspired by the analogous degree-corrected stochastic block model for communities in graphs with inhomogeneous degrees.
Let $\numClusters\geq 1$ be a positive integer, consider a transition matrix $p\in \mathbb{R}^{\numClusters\times \numClusters}$ for an ergodic \gls{MC} on $[\numClusters]$ and equip the state-space with a group-assignment map $\sigma_n: [n] \to [\numClusters]$.
As was the case for \glspl{BMC} we define the groups $\mathcal{V}_{1},\ldots,\mathcal{V}_\numClusters$ by $\mathcal{V}_i = \{v \in [n]: \sigma_n(v) = i \}$.
Assume moreover that every group $\mathcal{V}_i$ is equipped with a probability distribution $\mu_i:\mathcal{V}_i \to [0,1]$.
Then, a \gls{MC} $X_t$ on $[n]$ is called a \gls{DC-BMC} if
\begin{align}
	\mathbb{P}(X_{t+1}=j\mid X_t = i) = p_{\sigma_n(i)\sigma_n(j)} \mu_{\sigma_n(j)}(j).
\end{align}
Recall that in a \gls{BMC} it holds that conditional on $\sigma_n(X_t) = k$ for $t>1$ the observation $X_t$ is chosen uniformly at random in the cluster $\mathcal{V}_k$.
In a \gls{DC-BMC} it instead holds that conditional on $\sigma_n(X_t) = k$ the observation $X_t$ is chosen from the cluster $\mathcal{V}_k$ according to the probability measure $\mu_i$.

Note that the usual \gls{BMC} is recovered when all $\mu_{i}$ are taken to be the uniform measures on their respective groups $\mathcal{V}_i$.
Furthermore, by taking a larger number of groups $\widetilde{\numClusters} =  M\numClusters$ one can still approximate a \gls{DC-BMC} model by a BMC-model.
This is to say that one can use the additional clusters to separate each true group $\mathcal{V}_i$ of the \gls{DC-BMC} model into $M$ subgroups $\widetilde{\mathcal{V}}_{i,1}, \cdots , \widetilde{\mathcal{V}}_{i,M}$ such that $\mathcal{\mu}_i$ is approximately constant on every $\widetilde{\mathcal{V}}_{i,j}$.

We expect that the limiting measure for $\nu_{\sqrt{n} \hat{L}}$ in a \gls{DC-BMC} is equal to the limiting measure of a \gls{BMC} with the same cluster transition matrix $p$ and the same cluster ratios $\alpha_i$ provided that $\max_{i = 1,\ldots,n} \mu_{\sigma_n(i)}(i) = \Theta(1/n)$ and $\min_{i=1,\ldots,n} \mu_{\sigma_n(i)}(i) = \Theta(1/n)$.
If this conjecture is true then the limiting measure does not depend at all on the $\mu_i$ since these do not occur in Proposition \ref{prop:Limiting_singular_value_distribution_of_the_Laplacian}.
The insensitivity to the $\mu_i$ allows to ensure that the spectral noise in $\hat{L}$ is not dominated by an inhomogeneous equilibrium distribution.
The main reason for this conjecture is that the proof of Proposition \ref{prop:Limiting_singular_value_distribution_of_the_Laplacian} implicitly relies on a universality principle of \cite{sanders2022singular_sm} which states that the limiting singular value distribution in a (sufficiently well-behaved) random matrix only depends on the variance of its entries.
We will namely subsequently argue that the variance profile of $\hat{L}$ is approximately independent of distributions $\mu_{k}$; see \eqref{eq: Step_VarLDegreeFree}.

Denote $\pi$ for the cluster equilibrium distribution of the Markov chain associated to $p$ and note that the state equilibrium distribution of a \gls{DC-BMC} is then given by $\Pi_{X,i} = \pi(\sigma_n(i)) \mu_{\sigma_n(i)}$.
Correspondingly, up to approximation errors on the order of $\sqrt{\ell}$,
\begin{align}
	\sum_{k=1}^n\hat{N}_{i,k} \approx \# \{t=1,\ldots,\ell : X_t = i\} \approx \ell \Pi_{X,i} = \ell \pi(\sigma_n(i)) \mu_{\sigma_n(i)}.
\end{align}
Therefore, by the continuity of $x\mapsto \sqrt{x}$ it may be expected that
\begin{align}
	\sqrt{\sum_{k=1}^n \hat{N}_{ik}} \approx \sqrt{\ell \pi(\sigma_n(i)) \mu_{\sigma_n(i)}(i)}\qquad\text{  and } \qquad \sqrt{\sum_{k=1}^n \hat{N}_{kj}}\approx \sqrt{\ell \pi(\sigma_n(j))\mu_{\sigma_n(j)}(j)}. \label{eq: Step_ApproxSqrt}
\end{align}
The variance of a sum of independent random variables is equal to the sum of the variances.
If we write $\hat{N}_{i,j} = \sum_{t=1}^{\ell - 1} \indicator{X_t = i, X_{t+1} = j}$ then these summands are not independent but nonetheless we do expect the variance to approximately distribute over the sum.
Therefore, it is expected that
\begin{align}
	\operatorname{Var}[\hat{N}_{i,j}] \approx (\ell-1) \operatorname{Var}[\indicator{X_t = i, X_{t+1} = j}] = (\ell -1) \pi(\sigma_n(i)) \mu_{\sigma_n(i)}(i) p_{\sigma_n(i)\sigma_n(j)} \mu_{\sigma_n(j)}(j).  \label{eq: Step_VarSplit}
\end{align}
By combining \eqref{eq: Step_ApproxSqrt} and \eqref{eq: Step_VarSplit} it follows that
\begin{align}
	\operatorname{Var}[\hat{L}_{ij}] & \approx \operatorname{Var}\biggl[\frac{\hat{N}_{ij}}{\sqrt{\ell \pi(\sigma_n(i)) \mu_{\sigma_n(i)}(i)}\sqrt{\ell \pi(\sigma_n(j)) \mu_{\sigma_n(j)}(j)}}\biggr]                                   \\
	& \approx \frac{\ell \pi(\sigma_n(i)) p_{\sigma_n(i)\sigma_n(j)}\mu_{\sigma_n(i)}(i) \mu_{\sigma_n(j)}(j)}{(\ell \pi(\sigma_n(i)) \mu_{\sigma_n(i)}(i))(\ell \pi(\sigma_n(j)) \mu_{\sigma_n(j)}(j))} \\
	& = \ell^{-1} \pi(\sigma_n(j))^{-1} p_{\sigma_n(i) \sigma_n(j)}                                                                                                                 \\
	& \approx (\lambda n^2)^{-1} \pi(\sigma_n(j))^{-1} p_{\sigma_n(i)\sigma_n(j)}\label{eq: Step_VarLDegreeFree}
\end{align}
Observe that this agrees with the variance profile which was used in Lemma \ref{lem: VarProfile}.

\subsection{Simulation experiment}
\label{sec: Simulation}

We here measure the sensitivity of the spectral noise in $\hat{L}$ and $\hat{N}$ to violations of the model assumptions in the presence of an inhomogeneous equilibrium distribution by means of a perturbation to a \gls{DC-BMC} model,  defined in Section \ref{sec:Degree_corrected_BMC}.
The experiment is done by means of a simulation.

For the \gls{DC-BMC} model we take $\numClusters = 2$ and we consider clusters of size $\#\mathcal{V}_1 = \#\mathcal{V}_2 = 1000$.
The cluster transition matrix $p$ is defined by $p_{11} = p_{22} = 0.8$ and $p_{12} = p_{21} = 0.2$.
The probability measures $\mu_i$ are found for $i=1,2$ by sampling a vector of i.i.d. exponentially distributed random variables of rate $1$ and normalizing this vector to have $L^1$-norm equal to 1.

We may further consider a perturbation of this \gls{DC-BMC}.
Let $\Delta$ be a heavy-tailed transition matrix as defined in Section \ref{sec:Other_models_for_experimentation} and denote $P_{\operatorname{perturbed}} := 0.95 P_{\operatorname{DC-BMC}} + 0.05 \Delta$.
Recall that the \gls{DC-BMC} component $P_{\operatorname{DC-BMC}}$ can be approximated with a \gls{BMC} with more groups but note that such an approximation is not possible for $\Delta$.
Consequently, we may think of the decomposition for $P_{\operatorname{perturbed}}$ as splitting the ground truth model into a main part which can be approximated with a \gls{BMC} and a second part which requires a different explanation.

In the subsequent experiment we consider observation sequences $\process{X_t}{t=1,\ldots,\ell}$ and $\process{Y_t}{t= 0,\ldots,\ell}$ with length $\ell = 2000^2$ from the \gls{DC-BMC}-model and the perturbed model respectively.
The singular value densities of the $\hat{N}$-matrix constructed from $X$ and $Y$ are displayed in Fig. \ref{fig:simulation_bulks} (a).
Also displayed in Fig. \ref{fig:simulation_bulks} is the theoretical prediction corresponding to a \gls{BMC} found by executing the clustering algorithm with $\widetilde{\numClusters} = 4$ clusters.
Recall that taking $\widetilde{\numClusters} > \numClusters$ allows for the algorithm to split the groups to ensure that $\mu_i$ is roughly constant.
We observe that the empirical densities associated to the \gls{DC-BMC}-model and the perturbed model look quite similar apart from the fact that the perturbed model has a longer tail.
The theoretical prediction associated to the \gls{BMC} further provides an acceptable match for the DC-BMC model but there is also some small part of the tail of the \gls{DC-BMC} model which escapes the theoretical prediction.
Here the issue regarding the sensitivity of $\hat{N}$ becomes apparent:
there are at least two plausible explanations why in empirical data some part of the tail may escape the support of the theoretical density.
A first explanation is the presence of a perturbation $\Delta$ which we view as a violation of the model assumptions.
A second explanation is that the ground truth is a \gls{DC-BMC} and one should take $\widetilde{\numClusters}$ to be larger.
These two explanations are difficult to distinguish from the spectral noise in $\hat{N}$.
In the current example one may argue that the amount of the tail which escapes the theoretical density is larger in the perturbed model.
Such a judgement regarding the size of the tail is however undesirable since it is vague and subjective.

The singular value density of $\hat{L}$ for the two sample paths $X_{1:\ell}$ and $Y_{1:\ell}$ is displayed in Fig. \ref{fig:simulation_bulks} (b).
Here we observe that the empirical densities of the \gls{DC-BMC} model and the perturbed model are severely different.
The theoretical prediction associated to the \gls{BMC} moreover provides a good match to the \gls{DC-BMC} model as was expected by the conjecture of Section \ref{sec:Degree_corrected_BMC}.
We conclude that the spectral noise in the Laplacian $\hat{L}$ is more sensitive to violations of the model assumptions than the spectral noise in $\hat{N}$, particularly in the presence of an inhomogeneous equilibrium distribution.

\begin{figure}[t]
	\centering
	\subfloat[]{\includegraphics[width = 0.45\textwidth]{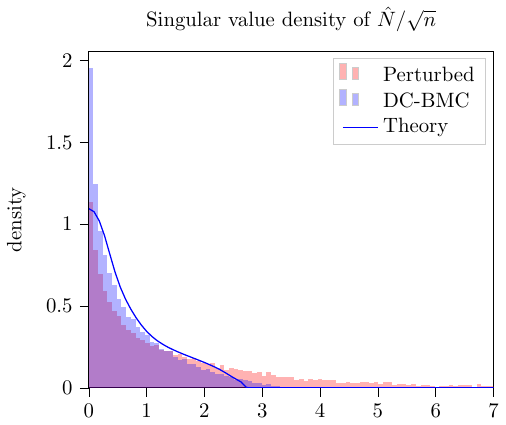}}
	\subfloat[]{\includegraphics[width = 0.45\textwidth]{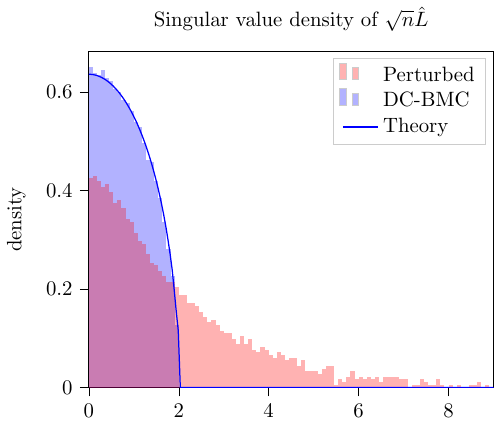}}
	\caption{
		(left)
		The singular value density of $\hat{N}$ for a simulated \gls{DC-BMC} (blue bars) as compared to the theory (blue line) and a perturbed model (red bars).
		(right)
		The singular value density of $\hat{L}$.
	}
	\label{fig:simulation_bulks}
\end{figure}

\section{Evaluation of the \texorpdfstring{\gls{CAIC}}{CAIC} criterion}\label{sec:CaicEval}

The \gls{CAIC} was used to yield estimators for the order of the data-generating process for the DNA, Animal movement and Stock market datasets, respectively. 
In order to probe how accurate these estimators are, we consider a generative model for data with similar empirical transition laws as those of the datasets. Under several perturbation levels, we study both the over- and underfit probabilities of the order, and inspect how robust the estimator is under model perturbation.

We consider the empirical transition law $\hat{\mathbb{P}}^{r, \mathrm{MLE}}$ from the original data $X_{1:\ell}$ on the full state space $[n]$.
With $\hat{\mathbb{P}}^{r, \mathrm{MLE}}$  for $r \in \{0,1\}$ we consider two perturbed data-generating models and investigate the clustered process $Y^{r}_{\varepsilon}$.
The models are:
\begin{itemize}[noitemsep]
	\item[$\mathbb{W}_{\varepsilon}^1$:] A perturbed $1$st-order \gls{BMC} with probability distribution $\hat{\mathbb{P}}^{1, \mathrm{MLE}}$ and a heavy-tailed $0$th-order perturbation. 
    In contrast to the general perturbed models described in Section \ref{sec:Other_models_for_experimentation}, the perturbation here is a $0$th-order \gls{MC}.
	\item[$\mathbb{W}_{\varepsilon}^0$:] A perturbed $0$th-order \gls{BMC} with probability distribution $\hat{\mathbb{P}}^{0, \mathrm{MLE}}$ and a heavy-tailed 1st-order perturbation.
\end{itemize}

Denote $Y^{r,\varepsilon}_{1:\ell} = \sigma_n(X^{\varepsilon}_{1:\ell})$ for the cluster process with $X^{\varepsilon}_{1:\ell} \sim \mathbb{W}_{\varepsilon}^r$ for $r\in \{0,1 \}$.
We will study the robustness of the \gls{CAIC} criterion by examining how often it over-and underfits when selecting $s \in \{0,1 \}$ for the models $\hat{\mathbb{Q}}^{s, \mathrm{MLE}}$ with the clustered sequence $Y^{r,\varepsilon}_{1:\ell}$.
The overfit error probability is the probability that the criterion selects a $1$st-order process when the underlying generating process is $\mathbb{W}_{\varepsilon}^0$: 
\begin{equation}
	e_{\mathrm{over}}(\varepsilon) := \mathbb{P}_{X^{\varepsilon}_{1:\ell} \sim \mathbb{W}_{\varepsilon}^0}(\textrm{argmin}_{r \in \{0,1\}} \mathrm{CAIC}(Y^{r,\varepsilon}_{1:\ell}) = 1),	\label{eqn:CAIC_Overfit_Selection_Error_Probability}
\end{equation}
One reason why this can occur is that the perturbation may not respect the cluster structure, which can cause the clustered process $Y_{1:\ell}^{r,\varepsilon}$ to have higher-order dependencies even if $X_{1:\ell}^\varepsilon$ is $0$th order. 
The underfit error probability is defined as
\begin{equation}
	e_{\mathrm{under}}(\varepsilon) := \mathbb{P}_{X^{\varepsilon}_{1:\ell} \sim \mathbb{W}_{\varepsilon}^1}(\textrm{argmin}_{r \in \{0,1\}} \mathrm{CAIC}(Y^{\varepsilon}_{1:\ell}) = 0),	\label{eqn:CAIC_Underfit_Selection_Error_Probability}
\end{equation}
that is, the probability we select a $0$th-order process while the actual underlying data-generating process $\mathbb{W}_{\varepsilon}^1$ is $1$st-order. 

We focus on the \gls{DNA} and \gls{SP500} datasets. 
Because the \gls{SP500} dataset is the least clear dataset, we also consider a synthetic observation sequence.
This synthetic observation sequence is generated using the same model $\mathbb{W}_{\varepsilon}^r$ as is obtained for the stock market, but will be five times as long: $5\ell$ with $\ell$ the length of original path of the \gls{SP500} dataset.
This ``extended stock market model'' gives a synthetic proxy to study the effect of sparsity on the criterion robustness \emph{as if} we have access to more data.

\begin{figure*}[t]
	\centering
	\subfloat[]{\includegraphics[height=0.26\linewidth]{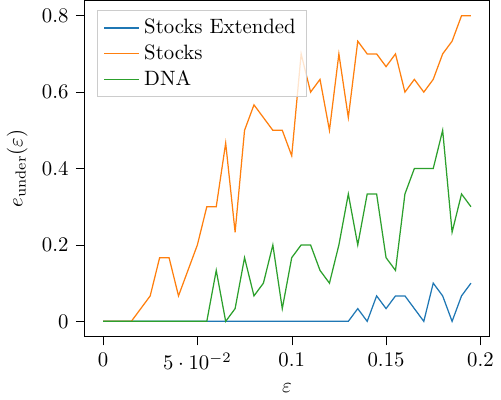}}
	\subfloat[]{\includegraphics[height=0.26\linewidth]{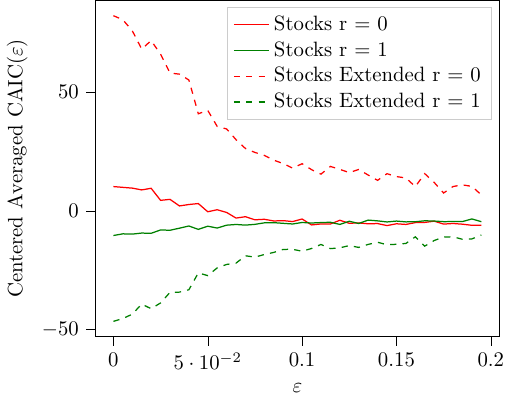}}
	\subfloat[]{\includegraphics[height=0.26\linewidth]{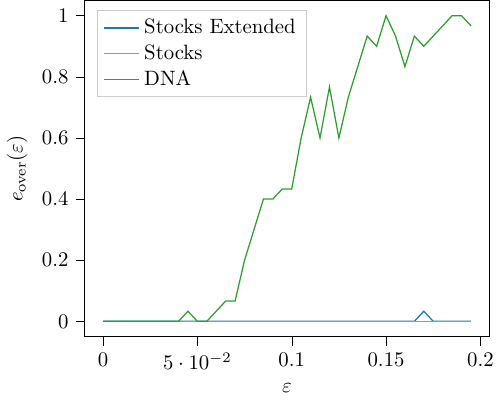}}
	\caption{
	(a) Underfit probability $e_{\mathrm{under}}(\varepsilon)$  as function of $\varepsilon$ for the \gls{DNA}, \gls{SP500}, and extended \gls{SP500} datasets assuming that the data-generating process is $\mathbb{W}^1_{\varepsilon}$. 
	(b) Centered average of \gls{CAIC} for datasets assuming the data-generating process is $\mathbb{W}^1_{\varepsilon}$.
	We remark that the empirical variance is an order of magnitude too large to be represented in the plot ($\mathrm{Var}(\mathrm{CAIC}(Y_{1:\ell}(\varepsilon)) \simeq O(10^2)$).
	Despite this very large variance, the selection process is robust for small error $\varepsilon$.
	(c)
	Overfit error probability $e_{\mathrm{over}}(\varepsilon)$ as function of $\varepsilon$ assuming the data-generating process is $\mathbb{W}^0_{\varepsilon}$.
    In all tests the number of repetitions was $R=30$.
	}
	\label{fig:CAIC_error_probability_and_value_comparison}
\end{figure*}

Fig. \ref{fig:CAIC_error_probability_and_value_comparison} displays the error probabilities and centered \gls{CAIC} values.
We see that both the underfit $e_1(\varepsilon)$ and overfit error $e_2(\varepsilon)$ are small for small $\varepsilon$.
The overfit error is, however, considerable larger for the \gls{DNA} dataset than for the \gls{SP500} dataset.
This supports the claim that the \gls{CAIC} chooses the model with fewest parameters for the same amount of information and is hence less prone to overfit when the data is sparse.
The underfit error is on the contrary small for the \gls{DNA} dataset, also for $\varepsilon \in [0.1,0.2]$.
Taken together, this all suggests that order selection via information criteria is robust to small perturbations.

The case of the \gls{SP500} dataset is especially interesting.
In the results in Table \ref{tab:CAIC_all}, the criterion chooses $r=0$ whereas in the $\mathbb{W}_{\varepsilon}^{1}$ model in Fig. \ref{fig:CAIC_error_probability_and_value_comparison}(a)--(b), the criterion selects $r=1$ up to $\varepsilon \sim 0.1$.
Afterwards, deviating from the \gls{BMC} model by just $1$ out of $10$ jumps in the \gls{SP500} dataset will make the criterion behave similarly as in Table \ref{tab:CAIC_all}.
This is also supported by Fig. \ref{fig:CAIC_error_probability_and_value_comparison}(b), where the difference between the criterion for $r=0$ and $r=1$ in the \gls{SP500} dataset takes values in $[0,10]$, which we coincidentally also see in Table \ref{tab:CAIC_all}.
This suggests that there may be a $1$st-order Markovian structure in the \gls{SP500} dataset but also a strong $0$th-order component.
Alternatively, the data may simply be too sparse for order selection.
The latter hypothesis is also supported by the synthetically extended dataset, where model selection has fewer problems.

\section{Extra tools for some of the different data sets}\label{sec:Tools}

\subsection{The cf-idf vectorization method}
\label{sec:Appendix__CDIDF}
Let $\mathcal{V}$ denote the vocabulary, which is a set of words, and fix a clustering $\sigma_n:\mathcal{V} \to [\numClusters]$.
In order to turn documents into vectors we make use of a straightforward modification of the common \emph{term frequency-inverse document frequency} document vectorization method, we refer to the modification as \emph{cluster frequency-inverse document frequency} (\gls{cf-idf}).
Let $\mathcal{D}$ denote a collection of documents.
Every document $d\in \mathcal{D}$ is here viewed as a sequence of words meaning that $d\in \prod_{t=1}^{\ell_d} \mathcal{V}$ for some $\ell_d >0$.
For every cluster $k\in [\numClusters]$ and document $d\in \mathcal{D}$ we define
\begin{align*}
	\operatorname{cf}(k,d)
	&
	:= \ln\big(1 + \#\{t =1,\ldots,\ell_d: \sigma_n(d_i) = k\} \big),                                                                                                   \\
	\operatorname{idf}(k,\mathcal{D})
	&
	= \ln\Big(\frac{\sum_{k=1}^\numClusters( 1+\sum_{d\in \mathcal{D}}\#\{t =1,\ldots,\ell_d: \sigma_n(d_i) = k\} }{1 + \sum_{d\in \mathcal{D}}\#\{t =1,\ldots,\ell_d: \sigma_n(d_i) = k\}  }\Big), \\
	\operatorname{cf-idf}(k,d)
	&
	:= \operatorname{cf}(k,d)\cdot \operatorname{idf}(k,\mathcal{D}).
\end{align*}
Observe that $\operatorname{cf-idf}(\cdot, d)$ assigns a $\numClusters$-dimensional vector to any document $d\in \mathcal{D}$.
A word which contributes multiple times in a document also yields a higher contribution of the corresponding cluster to $\operatorname{cf}(k,d)$. Finally, words which are not in the vocabulary of the clustering are not counted at all in our processing.

\subsection{Algorithm for creating a grid for animal movement data}
\label{sec:appendix_GPS_to_states}

Given a desired grid side length $x$ (in kilometers), we can calculate the (regional) latitudinal and longitudinal degree corresponding to that distance $x$ (assuming the earth to be a near perfect sphere).
Latitudinal differences amount to the same distance in kilometers.
One degree of latitude is 1/36$0$th of the earths circumference ($\SI{40075}{\kilo\meter}$) so one degree of latitude is equivalent to $\SI{110.574}{\kilo\meter}$. Accordingly, $x$ km are represented by ${x}/{110.574}$ degrees of latitude.
One degree of longitude however represents a different amount of kilometers, depending on the latitude: one degree of longitude is $|111.320 \cdot \cos_\text{dd}(\text{latitude})|$ km and $x$ km are represented by $|{x}/{111.320 \cdot \cos_\text{dd}(\text{latitude})}|$ degrees of longitude, at a specific latitude.
Here $\cos_\text{dd}$ is the cosine acting on decimal degrees. We are making use of a small angle approximation, which breaks down near the poles.
The process of creating squares and assigning the units $y_i$ to them is done in the following Algorithm \ref{alg:gps_to_gridstate}.

\begin{algorithm}
	\caption{\gls{GPS} data to transitions between squares of side length $x$ km}
	\label{alg:gps_to_gridstate}
	\hspace*{\algorithmicindent} \textbf{Input:} Grid size $x$, \gls{GPS} data $(y_{i1},y_{i2},y_{i3})_{i=1,\dots, n}$.\\
	\hspace*{\algorithmicindent} \textbf{Output:} Sequence of transitions between states $s=(s_1,\dots,s_n)$
	\begin{algorithmic}[1]
		\State $j_\mathrm{lat} = \lfloor y_{11}\cdot 110.574 / x \rfloor$
		\State $j_\mathrm{long} = \lfloor y_{12}\cdot 111.320 \cdot | \cos_\mathrm{dd}( j_\mathrm{lat} \cdot 110.574 / x ) | / x \rfloor$
		\State $S_{1}:= \big[ j_\text{lat} \cdot \frac{x}{110.574}, (j_\text{lat}+1) \cdot \frac{x}{110.574} \big)\times \big[ j_\text{long}  \cdot |\frac{x}{111.320 \cdot \cos_\text{dd}(j_\mathrm{lat}\cdot 110.574 / x)}| , (j_\text{long}+1)  \cdot |\frac{x}{111.320 \cdot \cos_\text{dd}(j_\mathrm{lat}\cdot 110.574 / x)}| \big) $
		\State $\mathcal{S} \gets \{S_1\}$
		\State $s\gets (1)$
		\For{$\mathrm{index}_\mathrm{coord} \gets 1$ to $n$}
		\For{$\mathrm{index}_\mathrm{square}\in \{1,2,\dots\#\mathcal{S}\}$ }
			\State $\mathrm{added}\gets \mathrm{False}$
 			\If{$(y_{\mathrm{index}_\mathrm{coord}1},y_{\mathrm{index}_\mathrm{coord}2}) \in S_{\mathrm{index}_\mathrm{square}}$}
				\State $s\gets (s_1,\dots , s_{\mathrm{index}_\mathrm{coord}-1},\mathrm{index}_\mathrm{square})$
				\State $\mathrm{added}\gets \mathrm{True}$
			\EndIf
			\If{$\mathrm{added}=\mathrm{False}$}
				\State $j_\mathrm{lat} = \lfloor y_{\mathrm{index}_\mathrm{coord}1}\cdot 110.574 / x \rfloor$
				\State $j_\mathrm{long} = \lfloor y_{\mathrm{index}_\mathrm{coord}2}\cdot 111.320 \cdot | \cos_\mathrm{dd}( j_\mathrm{lat} \cdot 110.574 / x) | / x \rfloor$
				\State $S_{\#\mathcal{S}}:= \big[ j_\text{lat} \cdot \frac{x}{110.574}, (j_\text{lat}+1) \cdot \frac{x}{110.574} \big)
				\times \big[ j_\text{long}  \cdot |\frac{x}{111.320 \cdot \cos_\text{dd}(j_\mathrm{lat}\cdot 110.574 / x)}| , (j_\text{long}+1)  \cdot |\frac{x}{111.320 \cdot \cos_\text{dd}(j_\mathrm{lat}\cdot 110.574 / x)}| \big) $
				\State $\mathcal{S} \gets \{S_1,\dots, S_{\#\mathcal{S}}\}$
				\State $s\gets (s_1,\dots , s_{\mathrm{index}_\mathrm{coord}-1}, \#\mathcal{S})$
			\EndIf
		\EndFor
		\EndFor
\end{algorithmic}
\end{algorithm}

The algorithm assigns all the \gls{GPS} points to squares. This procedure may not work well if the animal is moving close to the poles, because the small angle approximation is not justified anymore. This procedure is also not particularly well suited for regions where the earth is not behaving like a sphere, for example if the animal in moving on mountains.

\section{Raw data of, and extra material on, some of the datasets}
\label{sec:Appendix__Raw_data}

\subsection{Transition matrix for bison clusters}
\label{sec:appendix_bison_cluster_transition}

Below is the cluster transition matrix for the improvement clustering depicted in Fig. \ref{fig:bison_clustering}, the numbers are rounded to the second decimal place.
\setcounter{MaxMatrixCols}{15}
\[
\small\begin{pmatrix}
	0.82 & 0.05 & 0 & 0.01 & 0.02 & 0.01 & 0.01 & 0.01 & 0.02 & 0.02 & 0 & 0.01 & 0.01 & 0 & 0 \\
	0.07 & 0.76 & 0.02 & 0.03 & 0.05 & 0 & 01 & 0 & 0.01 & 0 & 0 & 0.02 & 0.02 & 0 & 0.01 \\
	0.01 & 0.05 & 0.87 & 0.03 & 0 & 0 & 0 & 0 & 0 & 0 & 0.03 & 0 & 0 & 0 & 0 \\
	0.02 & 0.06 & 0.03 & 0.85 & 0.03 & 0 & 0 & 0 & 0 & 0 & 0 & 0 & 0 & 0 & 0.01 \\
	0.04 & 0.07 & 0 & 0.02 & 0.77 & 0 & 0 & 0 & 0 & 0 & 0 & 0 & 0 & 0 & 0.1 \\
	0.04 & 0 & 0 & 0 & 0 & 0.88 & 0.02 & 0 & 0 & 0 & 0 & 0 & 0.03 & 0.02 & 0 \\
	0.03 & 0.01 & 0 & 0 & 0 & 0.01 & 0.86 & 0 & 0.01 & 0 & 0 & 0 & 0.04 & 0.03 & 0 \\
	0.06 & 0 & 0 & 0 & 0 & 0 & 0 & 0.84 & 0.07 & 0 & 0 & 0.01 & 0 & 0 & 0 \\
	0.09 & 0.01 & 0 & 0 & 0 & 0 & 0.01 & 0.06 & 0.72 & 0.01 & 0 & 0.1 & 0 & 0 & 0 \\
	0.06 & 0 & 0 & 0 & 0 & 0 & 0 & 0 & 0.01 & 0.79 & 0.1 & 0.03 & 0 & 0 & 0 \\
	0.01 & 0 & 0.04 & 0 & 0 & 0 & 0 & 0 & 0 & 0.09 & 0.83 & 0.02 & 0 & 0 & 0 \\
	0.03 & 0.06 & 0 & 0 & 0 & 0 & 0 & 0.01 & 0.1 & 0.04 & 0.02 & 0.73 & 0 & 0 & 0 \\
	0.03 & 0.06 & 0 & 0 & 0 & 0.04 & 0.05 & 0 & 0 & 0 & 0 & 0 & 0.77 & 0.05 & 0 \\
	0 & 0 & 0 & 0 & 0 & 0.05 & 0.07 & 0 & 0 & 0 & 0 & 0 & 0.14 & 0.72 & 0 \\
	0.02 & 0.03 & 0 & 0.01 & 0.31 & 0 & 0 & 0 & 0 & 0 & 0 & 0 & 0 & 0 & 0.63
\end{pmatrix}
\]

\subsection{Groups of words for improvement with 200 groups}
\label{sec:Appendix_raw_data__Words__Improvement200}

\subsubsection{Document classification datasets}
\label{sec: DocClass}

We here describe the datasets which are used to construct Table \ref{tab: perf} and report on some other datasets where our findings are inconclusive in Table \ref{tab: perf2}.

\paragraph{AG News.}
This dataset provided by \cite{zhang2015character} consists of tuples $(x,y,z)$ where $x$ is the title of a news article, $y$ is a description of the news article and $z$ is an assigned class.
There are here four possible classes which $z$ can take as values namely \emph{World}, \emph{Sports}, \emph{Business} and \emph{Sci/Tech}.
For each such class the dataset contains precisely 30 000 training samples and 1 900 testing samples.
In our processing we concatenated $x$ and $y$ into a single string and the task is to predict the class label $z$ based on this string.

\paragraph{Yahoo!.}
This dataset provided by \cite{zhang2015character} contains questions and answers from \emph{Yahoo! answers}.
The dataset consists of tuples $(x,y_1,y_2,z)$ where $x$ is a question, $y_1, y_2$ are answers to this question and $z$ is category to which the question belongs.
It can here also occur that the question has fewer than two answers in which case $y_1$ or $y_2$ is the empty string.
There are ten possible classes which $z$ can take as values namely \emph{Society \& Culture}, \emph{Science \& Mathematics}, \emph{Health}, \emph{Education \& Reference}, \emph{Computers \& Internet}, \emph{Sports}, \emph{Business \& Finance}, \emph{Entertainment \& Music}, \emph{Family \& Relationships} and \emph{Politics \& Government}.
For each such class the dataset contains precisely 140 000 training samples and 5 000 testing samples.
In our processing we concatenated $x$, $y_1$ and $y_2$ into a single string and the task is to predict the class label $z$ based on this string.

\paragraph{Wiki.}
This dataset comes from the DBPedia ontology project \cite{lehmann2015dbpedia} and the precise version used here is constructed by \cite{zhang2015character}.
The dataset consists of tuples $(x,y,z)$ where $x$ is a title of a Wikipedia page, $y$ is the abstract of the page and $z$ is the category to which the page belongs.
There are 14 possible classes which $z$ can take as values.
For each such class the dataset contains precisely 40 000 training samples and 5 000 testing samples.
In our processing we did not use the title $x$ so the task is to predict the class label $z$ based on the abstract $y$.

\paragraph{Book.}
This dataset is constructed based on books from Project Gutenberg and their genres are assigned on GoodReads, the dataset was obtained from \cite{books}.
The dataset contains tuples $(x,y,z)$ where $x$ is the tile of a book, $y$ is the full text of this book and $z$ contains a set of genres.
We only retained those data points for which $z$ is a set with a single element from one of the following six categories: \emph{cookbooks}, \emph{fantasy}, \emph{horror}, \emph{politics}, \emph{religion} or \emph{science-fiction}.
We further randomly selected 2 000 training samples which left 387 samples for testing.
In our processing we did not use the title $x$ so the task is to predict the genre $z$ based on the text $y$.

\paragraph{CMU.}
The CMU Book Summary Dataset contains plot summaries for books which are extracted from Wikipedia by \cite{bamman2013new}.
The dataset contains tuples $(x,z)$ with $x$ a plot summary and $z$ the category to which the book belongs.
We retain all datapoints whose category $z$ occurs at least 50 times which leaves us with two genres namely \emph{Fantasy} and \emph{Science-Fiction}.
We randomly select 1 138 datapoints for training which leaves us with 380 testing samples.
The task is to predict the genre $z$ based on the summary $x$.

\paragraph{20news.}
This dataset contains newsgroup postings for 20 different newsgroups which are collected by \cite{Lang95}.
The dataset is accessed using the function \texttt{fetch\_20newsgroups} from\\
\texttt{sklearn.datasets}.
The dataset contains tuples $(x,z)$ where $x$ is a message sent to the newsgroup and $z$ is the label of the newsgroup.
There are 20 possible classes which $z$ can take as values.
There are 11 314 training samples and 7 532 testing samples.
The task is to predict the newsgroup $z$ given the message $x$.

\paragraph{Spam.}
This dataset contains text messages which are either legitimate or spam, collected by \cite{gomez2006content,cormack2007feature,cormack2007spam}.
The dataset is accessed from \cite{UCI} and contains tuples $(x,z)$ where $x$ is a text message and $z$ is a label indicating if the message is spam.
The possible values for $z$ are spam or ham.
The task is to predict $z$ given the message $x$.
Unfortunately, due to a mistake, the experiment is executed without splitting the dataset in training and testing samples.
This means that the 4 179 available samples are used both during training and testing.
Splitting in testing and training would however not change the inconclusive conclusion from Table \ref{tab: perf2} and another experiment with a test-train split is not executed.

\paragraph{Reuters.}
The Reuters RCV1 corpus \cite{lewis2004rcv1} consists of a collection of news stories and was accessed using \texttt{nltk.download(\'reuters\')}.
The dataset contains tuples $(x,z)$ with $x$ a news article and $z$ the category to which it belongs.
There are 58 possible values for $z$.
There are 6 577 training samples and 2 570 testing samples.
The task is to predict the category $z$ based on the text in $x$.

\begin{table}[hbtp]
	\caption{Results for performance on document classification where neither method significantly outperformed a random clustering. }
	\label{tab: perf2}
	\centering
	\begin{tabular}{c c c c }
		\toprule
		{Algorithm} & {20news} & {Spam} & {Reuters} \\
		\midrule
		Random $K = 50$& 23.2\% &86.4\%& 65.4\%  \\
		Spectral $K = 50$& 23.0\% &86.1\%& 63.0\% \\
		Improved $K = 50$& 25.2\%& 86.1\%& 63.3\%
		\\[0.5em]
		Random $K = 100$&  31.0 \%& 86.7\%& 67.7\%  \\
		Spectral $K = 100$& 31.1\%& 86.9\%& 66.2\%  \\
		Improved $K = 100$& 33.6\%& 87.2\%& 68.9\%
		\\[0.5em]
		Random $K=200$& 38.0\%& 87.2\%& 68.4\%  \\
		Spectral $K=200$&  36.2\%& 87.0\%& 69.0\% \\
		Improved $K=200$& 40.2\%& 87.2\%& 70.8\%
		\\[0.5em]
		Random $K=400$& 44.7\%& 87.6\%& 87.8\%  \\
		Spectral $K=400$& 41.4\%& 87.7\%& 88.0\% \\
		Improved $K=400$&43.9\%& 87.7\%& 89.0\% \\
		\bottomrule
		\vspace{0pt}
	\end{tabular}
\end{table}

\subsubsection{Detected groups}\label{sec: DetectedClusters}

Here are the detected groups when using the cluster improvement algorithm:

\begin{itemize}[noitemsep]

	\item[$\mathcal{V}_{1} = $]
	{
		\tiny
		mauser, blackout, wari, yak, sprite, puff, nightlif, capitalis, vhf, shroud, athena, featurelength, workflow, fright, grasshopp, misunderstood, aeroplan, farreach, prequel, ascii, veterinarian, heyday, metalwork, timeout, nod, cavern, nf, utilitarian, chevi, ting, aphrodit, unsatisfactori, pieti, inund, heist, cl, fullfledg, autopsi, intang, deregul, hyphen, hdtv, gild, majest, nasti, discretionari, computer, ambival, invinc, ide, rt, radiohead, brainwash, slur, teaser, pl, indepth, outag, pak, rebirth, pun, barksdal, crypt, outtak, crosscultur, blaster, spit, simplist, amt, bn, bd, bikini, dea, miscarriag, reenact, makeshift, synergi, uninterrupt, surrealist, toolkit, pervas, shini, dismal, dizzi, spire, dm, lancia, hadrian, viacom, foal, hippi, bonnet, subplot, cfa, poseidon, inhuman, ecstasi, drawer, subaru, diminut, til, amiga, beggar, yoke, twinengin, redefin, stomp, giraff, elisa, preambl, servitud, ridership, thoroughbr, miser, lingua, medusa, unreal, gl, delinqu, garuda, equit, earmark, tesco, {\textsection}, nb, bogi, dod, sx, impromptu, balconi, fastbal, ingam, coerciv, adjunct, carib, embezzl, disrespect, smuggler, bitch, freestand, slipper, netscap, textual, vp, glam, highestr, falsifi, facetofac, shadi, yamaha, cradl, sceptic, londonbas, weari, utopian, sigmund, contenti, counterfeit, medley, vigilant, weakest, superimpos, fg, retel, solstic, vibrant, tapestri, martian, illustri, lander, reevalu, kneel, involuntari, jug, hl, dingo, eel, mn, wealthier, earner, affection, learnt, mermaid, tempest, nz, rhapsodi, astrophys, wiki, vaudevil, wager, leaflet, dazzl, approx, bloodsh, ode, rung, shovel, lorri, indiscrimin, proudli, xs, ata, pastim, bane, dar, unidentifi, usda, climber, idiom, sabl, gogo, purs, klan, threeday, withhold, remington, cling, shouldnt, agoni, delin, applaus, plagiar, toughest, meta, mingl, profan, tc, interdict, requiem, shutdown, participl, synth, jab, meme, misl, gaia, hightech, buoy, contriv, despis, sceneri, mimick, labyrinth, larval, fsa, panelist, tangibl, ns, geophys, simplif, geforc, selfhelp, mv, amulet, hurdl, midsiz, delimit, underag, animos, hegel, goofi, afi, preoccupi, sinner, cheetah, ict, waiver, panda, timet, ate, wand, kar, levit, foreshadow, voucher, batsmen, strabo, interscop, styliz, charisma, exploratori, hyundai, kitten, reaper, redress, tabul, terrif, vindic, warranti, hitch, manx, blink, hiro, boomerang, fantasia, hieroglyph, gloss, galley, jumper, remit, industrialis, idiot, safari, crunch, linger, wrc, macaqu, emplac, biker, xxx, retort, uh, fao, carp, endus, playground, michelin, viewership, meteor, fw, sf, toon, greenpeac, flamingo, sleepi, ordinarili, overton, fd, gamer, pluto, shill, primal, rattl, tg, kawasaki, northrop, ser, unseen, cola, backfir, tricki, confound, slick, cfr, purportedli, knuckl, erp, amish, shrew, cad, deterr, sco, ju, spaceship, godzilla, bitten, supervillain, intrud, thunderbird, boogi, mammoth, hg, tt, annot, sway, quip, splinter, videotap, vj, namepl, scarf, mf, teleport, panorama, treacher, downplay, aircrew, pedigre, brahman, loki, gmt, cordial, gal, souvenir, asiat, limp, virtuou, rebuk, barrag, unfavor, extort, exhort, kde, vr, manhunt, hiss, memorabilia, rsa, conglomer, twodoor, woodwork, expend, gauntlet, leech, acp, stallion, elv, appendix, artefact, dime, prophesi, rambler, anteced, expropri, overtur, curli, bm, richer, wellreceiv, devour, oc, expressli, glyph, gull, reemerg, swimsuit, puppi, nu, stumbl, overrid, ub, narcissist, selfdestruct, ontolog, geni, tekken, nec, mya, payabl, threequart, thug, utterli, latenight, fallout, worthwhil, cpm, superpow, swiftli, eco, jellyfish, usn, joystick, spoon, grim, gimmick, irc, clarif, figurin, caregiv, hesh, nearer, dsm, laughter, hum, slider, hannabarbera, easiest, sg, caricatur, honesti, humankind, calypso, constru, irrespect, utmost, deceit, dislodg, antic, finalis, subconsci, surfer, alterc, maze, azur, ska, packard, sari, conduc, sar, reconnect, stagnat, undead, starscream, sculpt,  , poach, psp, geopolit, seaplan, fallaci, untitl, unbalanc, sticker, fiasco, dentistri, intimaci, stormi, wc, ui, outcrop, friez, triplet, bald, vertigo, bot, courtship, superb, rl, ansi, sash, footwear, interrel, gemston, doorway, cynic, evoc, kodak, insecur, alchemi, disobey, aac, devalu, neo, cohort, longestrun, hairstyl, jtwc, stasi, ipo, bh, descriptor, oratori, atf, porcelain, foolish, briberi, vase, fireplac, nourish, electronica, institution, widest, cisco, allnew, cute, seduct, dummi, msdo, sweater, rss, witti, mallet, mazda, unfit, snare, oem, anoint, raspberri, scorn, highaltitud, debtor, habitu, peril, rockstar, reckless, sturgeon, combo, theyv, unravel, townspeopl, bf, wildli, offlin, bleak, chameleon, dogg, lsd, videogam, pon, fret, spaciou, spaniel, allround, garnet, livelihood, info, savior, kelvin, wither, arson, twa, lte, lowbudget, taker, ova, attende, lamborghini, stalem, sikorski, funerari, harem, pillow, forfeit, middleearth, embroid, cv, interdepend, underestim, lawless, seamless, sag, shrapnel, thee, sci, plummet, rollsroyc, pythagorean, olympu, expound, incens, centralis, ipcc, airtoair, rake, pantomim, queer, atv, vaniti, nsa, repercuss, asp, transpos, internship, domesday, hypnosi, encycl, lifeboat, hibern, cloak, hyena, lager, sic, unsolv, worldli, chomski, vulcan, stealth, pep, blasphemi, entrepreneuri, msc, paw, burlesqu, nuisanc, eta, backyard, dumb, freemason, clearer, dreamcast, icbm, heartbreak, lego, sewn, shun, decrypt, miscellan, feral, sarcast, toru, gunmen, omiss, lad, forefront, timbr, cdrom, dread, authorit, brink, feign, plutarch, oct, squat, gunfir, grecoroman, naiv, stagger, toast, derid, antler, apa, loophol, misdemeanor, tunic, tnt, cpr, issuanc, magnolia, fearless, preposit, nitro, hid, ua, rioter, nr, chiropract, dagger, notoc, mist, pixar, kosher, enron, encor, carnat, thirdperson, gunman, clich{\'e}, minion, unreason, palladium, ticker, overtli, republish, notif, vacanc, punit, kraft, bosch, salsa, tracker, diy, geek, adjourn, supercomput, harp, goblin, hatchback, tub, whichev, eyebrow, octagon, unspecifi, folli, bombardi, questionnair, pf, undevelop, bop, hardtop, regal, haplogroup, salesman, harrier, quak, willingli, okay, crumbl, salient, infrequ, bailout, britannica, choreographi, shit, npc, phobia, reborn, groundwork, octopu, rum, fric, airship, gr, hone, luger, humil, homo, onesid, echelon, viet, wardrob, gazel, nonn, opel, crossbord, unlicens, locust, catchi, revolution, cherish, rh, fa{\c{c}}ad, inconclus, exalt, hex, towel, mummi, hostess, prowess, monograph, centaur, mozilla, bidder, deepwat, moog, impass, lifes, partak, aristotelian, unheard, fad, gw, ornat, seren, rhino, rove, awak, motorist, chilli, underdog, downgrad, illiter, apprenticeship, gadget, californian, sportsman, makeov, dissatisfact, eman, chopper, isa, wikileak, coalesc, deepest, fullscal, tiein, wig, pathologist, melodrama, empower, centerpiec, telepath, lite, protract, deu, minigam, robber, twentytwo, jeopardi, buff, sodomi, unsur, valiant, mute, powerless, amg, stump, pygmi, volley, lavish, naa, energ, penc, afterlif, ax, dogmat, weasel, playlist, apparel, moos, ascertain, bi, oss, deliver, epitaph, , newslett, klingon, snatch, raccoon, unbroken, iliad, trivia, disillus, ol, flick, dude, twoweek, solitud, montag, uncut, howl, everlast, tl, phonograph, beforehand, rp, aol, proverb, crucifixion, audiovisu, hb, topless, sutra, overweight, retro, toad, spar, distinctli, minimalist, jerk, laps, sear, license, synthesis, goon, ntsc, arbitrarili, drape, wrought, borderlin, selector, necklac, mileag, reap, lick, wrapper, nymph, orc, peerreview, pip, destabil, hurri, courtesi, biospher, fax, reassur, surreal, therein, bra, cock, statur, handi, sentri, upsid, intermediari, sensual, iucn, publicis, acdc, tranquil, glanc, biplan, eloqu, backlash, focuss, dismount, coloss, extracurricular, widescreen, topdown, roar, technicolor, pictori, quiz, hyster, neapolitan, oneman, rebroadcast, polygami, underscor, vedanta, funki, shorthand, interdisciplinari, tamper, spelt, pdp, perch, reexamin, jingl, voc, subpoena, exhum, kettl, elektra, culprit, hallway, ital, silo, lovecraft, divest, flamenco, budgetari, fuzzi, almighti, assail, decoy, aptitud, septuagint, turntabl, impal, underwrit, prepaid, intro, leve, dice, glare, nurtur, nirvana, gunfight, readership, indigo, maniac, concoct, wow, minaret, pelt, miracul, excurs, transnat, behold, witchcraft, sleeper, mirag, mercedesbenz, fest, personifi, lastminut, daimler, stud, dada, limousin, baffl, platon, confuciu, symposium, objectori, cx, whiski, parrot, troublesom, swarm, biometr, seabird, preclud, lunat, speedi, sunglass, attic, merlin, closeup, sober, sha, resuppli, categoris, blight, maximis, rg, inr, unfamiliar, payoff, devoid, bonus, acrobat, mash, rook, nyse, utensil, coercion, soyuz, twig, champagn, enigmat, imparti, colossu, mindset, allig, detractor, preproduct, hump, deadliest, immor, skinni, tuner, irrespons, odysseu, vo, hitter, cybertron, geo, ara, backstori, sh, righteous, tester, disorgan, kingship, ipa, explanatori, thrash, selfcontain, oo, nk, annuiti, af, handson, resel, themat, electrif, twentyon, scam, superhuman, compuls, dissatisfi, clinician, onslaught, indign, disdain, sublim, shove, handler, kippur, payout, panoram, mana, habea, hedgehog, joker, pg, totem, myriad, straighten, allegor, urgenc, calf, outgo, quilt, acacia, apprais, subgenr, curtail, pol, pup, unicorn, decri, misti, leibniz, mardi, autist, galact, vibe, collag, junip, symbolis, breakout, embellish, propens, dire, shortfal, rudimentari, paralyz, meticul, riddl, machinegun, snp, fi, lupin, quorum, disparag, fetish, masquerad, monologu, oar, copul, dl, theyd, spec, highdefinit, unnot, bernoulli, lowercas, dx, slump, countercultur, barbecu, impregn, inaccess, intellect, hoop, ppp, windmil, junk, anu, catchphras, wilt, isoiec, hideout, {\textyen}, lash, curtiss, terra, tighter, someday, trebl, brawl, pandora, behindthescen, stag, pejor, existenti, debit, kc, cyclop, firefox, nvidia, grudg, epithet, fuell, mela, ancillari, wrongli, {\textdagger}, lastli, guillotin, firstgener, sadli, gtr, solemn, roadsid, ur, shopper, oneshot, finder, slew, tramp, sl, mutil, ri, linnaeu, hoist, gorgeou, esteem, boar, cider, sled, whistleblow, lowpow, handmad, surrog, electro, whereupon, peugeot, abstain, lavend, endear, instil, verizon, priestli, embroil, horsesho, tenyear, monolith, obvers, outset, limbo, feroci, capcom, cowl, shaker, wreckag, implicitli, bracelet, homeown, optimu, twentythre, tripod, inocul, goodi, uplift, rejoic, sideway, overdub, jackpot, aclu, ee, chore, permeat, sevenyear, psychiatri, tranc, immacul, codex, bsa, cbi, mono, newfound, sparrow, nightli, cessna, coca, hospic, priestess, advert, nil, monstrou, ridden, lute, anthrax, yam, unhealthi, luthor, cosmopolitan, directv, superstit, mink, catfish, captor, imax, dwindl, sprung, veda, grotto, indispens, tauru, uhf, dt, grail, parabl, gopher, bayonet, longev, enix, pepsi, honorif, barter, centauri, inquest, dy, maharishi, aton, pud, pulsar, mic, outcast, poincar{\'e}, reinvent, bk, voiceless, voiceov, pumpkin, britney, underlin, nike, zodiac, nostalg, fokker, rejuven, magneto, tabloid, standoff, fang, authorship, toddler, acl, rapist, clumsi, glossi, goliath, avi, deconstruct, rhinocero, subterranean, gpl, dlc, {\textperiodcentered}, stride, cgi, entangl, firefli, tor, humanoid, postproduct, vane, blew, bootleg, bluegrass, hawker, kickstart, flea, tortois, kr, helper, prerequisit, entrylevel, thinli, sorcer, philanthrop, pegasu, rein, hacker, jive, binocular, ovid, interlac, kin, enigma, dat, transpir, suv, punctuat, parlor, pinch, preset, alchemist, wouldb, pharmacist, piti, sentient, gangsta, spectacl, mattress, plough, melancholi, sender, adida, subchannel, stigma, watcher, valor, trespass, accru, windi, infanc, konami, passer, whiskey, billiard, interraci, hord, yeah, applaud, hera, medallion, antelop, pax, paranorm, starbuck, elf, tlc, kepler, hysteria, dusk, landbas, passov, nostalgia, snoop, tr, issuer, bl, anxiou, kite, fairchild, scifi, raisin, entrepreneurship, protestor, mahogani, womb, perl, hk, vegan, frantic, bastard, covet, turboprop, urn, corvett, stela, sentinel, agnost, winfrey, youngster, pelican, requisit, coop, odyssey, tacitu, unharm, hypocrisi, wellestablish, upanishad, bodywork, parenthes, rj, puberti, misspel, reckon, auspici, sapien, curfew, cymbal, twopart, loom, blacksmith, pancak, multilingu, egalitarian, dp, unambigu, viper, assemblag, turnaround, rn, cull, weep, peach, autobot, impati, participatori, volvo, awesom, ericsson, derail, mayhem, atrium, recuper, hen, oprah, leftov, omnibu, saab, frenzi, oedipu, circumv, marconi, xmm, unfairli, summaris, puck, reassembl, paranoid, whatsoev, intertwin, department, stargat, mx, sip, namco, unbeliev, sig, facelift, hype, relentless, chakra, jewelleri, supplementari, chime, markup, torrent, generos, shortcom, deciph, chaser, tartan, guevara, carriageway, extraordinarili, spaceflight, disengag, impetu, farc, comprehend, shipwreck, galileo, poorer, bw, execution, unmark, compliant, winemak, din, horseback, clockwis, stupa, creas, tulip, strangl, conscienti, outperform, lexu, glamor, tougher, stricter, sandal, inexperienc, com, hourli, payrol, jag, secondhand, suffic, reconfigur, diaper, groundbreak, neat, secondgener, rko, effigi, tweak, tonnag, cryptic, fetch, hsv, supergroup, affidavit, bg, cun, grumman, uncanni, unleash, disallow, jg, braveri, pervers, folio, ratchet, impedi, um, orat, mnemon, holist, disgrac, earnest, gall, ati, newt, tata, webbas, rune, hog, brandi, bandag, baccalaur, gallup, fowl, priorit, plank, ku, misfit, eo, csa, rooftop, bonfir, draught, quirki, hydra, vm, paranoia, loco, quan, reliant, wd, pitchfork, sprinkl, lakh, coda, heap, hasten, harshli, xl, materialist, miseri, orderli, mundan, sprang, dreamwork, emu, iaea, patienc, heron, twoday, lax, undoubtedli, unison, stutter, barcod, hooligan, longlast, hitherto, ass, sadist, staircas, exoner, dupont, saucer, nou, dolbi, psychopath, cellar, mele, volta, preorder, bodhisattva, pap, voltair, splash, techno, avert, gallop, fieri, char, amazoncom, smear, aura, wildfir, suitcas, discord, worshipp, gnome, dissimilar, inconveni, substant, scrambl, detour, overcrowd, mismanag, gin, treacheri, quadrupl, cambrian, extravag, cookbook, splendid, derelict, masteri, breastfeed, sixmonth, scari, stubborn, triton, domino, tempt, resuscit, standardis, mca, fn, consumm, thunderbolt, apron, cctv, diphthong, cajun, nuditi, esperanto, monochrom, decidedli, reboot, flop, commonplac, overs, nra, masturb, shank, josephu, barren, kinship, acquitt, rosari, mediocr, assort, unprotect, lien, apparit, cg, liar, multitud, gpa, factual, proprietor, typifi, rx, pimp, emphat, transgress, outdat, cauldron, esquir, astound, tug, biscuit, handicraft, rescuer, novic, ls, sacrifici, broom, extraterrestri, cinemat, primu, contemporan, lightheart, nontradit, ddt, heroism, slant, handtohand, deepen, fictiti, yamato, bale, mug, metaanalysi, papyru, incapacit, concis, rewritten, subvert, blueprint, oneoff, taglin, engulf, atc, attrit, pretens, loath, recordbreak, typewrit, bun, interceptor, selfish, shack, indec, bharat, peta, upstair, mace, louder, fy, coupon, genitalia, malic, gambit, salamand, untouch, hottest, retitl, fisherman, tyrannosauru, flamboy, handwrit, appal, matador, rendezv, solari, coproduct, sapphir, astra, reclassifi, hindustani, aoc, uri, progenitor, threemonth, apocalypt, untru, astral, rv, twoseat, stricken, scotch, bullion, improperli, heracl, outburst, tp, onehalf, hasnt, antagon, ruse, icao, neptun, rad, sparkl, nypd, mart, asa, ero, discriminatori, grotesqu, javelin, beginn, macroeconom, housew, etho, clover, ramadan, iss, invalu, disprov, afloat, typefac, rebat, worldview, freighter, gees, hermit, iaf, lame, haze, dynamit, eclect, lingeri, interlud, tout, careless, precari, discover, reinterpret, peacetim, equinox, unplug, inflight, monik, racket, oldfashion, reputedli, tack, orb, unnatur, troll, nam, interspers, burglari, digger, wallet, tame, uneth, galactica, shard, ingeni, misinterpret, apocryph, cannonbal, pn, precept, shaken, thale, compassion, iec, prologu, epistemolog, shuffl, buyout, stare, drinker, pinnacl, booklet, mta, reintroduct, dilig, selfconsci, gs, extrapol, slack, ebook, citro{\"e}n, refund, stardom, turban, gorilla, upbeat, spinner, cautiou, bellow, decca, viennes, hardcov, stave, scooter, stretcher, euthanasia, anti, recast, operat, pli, buick, remad, eb, tesla, acorn, conjur, forgeri, onetim, pleistocen, needi, resurfac, onehour, cinematographi, undetect, ama, solidifi, fabul, stat, backpack, epitom, motorola, mouthpiec, diversif, swastika, wheelchair, brightest, yom, phalanx, multidisciplinari, lore, transpond, profoundli, pd, anarchi, internation, chinook, albatross, mu, westinghous, tombston, timeless, scoop, exquisit, gunneri, funniest, contextu, midday, nighttim, pluck, {\texttimes}mm, tuck, lampoon, concur, orchid, sire, worthless, kia, tore, foreclosur, mainstay, corset, calligraphi, antidot, cryptographi, disarma, xerox, hastili, resumpt, castor, fugu, cuckoo, retribut, glorifi, apc, peg, foray, birch, lieu, introspect, mau, chequ, wreath, hitchhik, pew, spreadsheet, dropout, bulldoz, iu, abrupt, loft, lucif, oa, caption, pe, beret, uneasi, penthous, l{\"u}, supplant, dh, entic, watchdog, neanderth, americana, taint, ot, jargon, pu, reclus, pinki, eater, silhouett, nov, f{\'e}in, rampag, snapshot, gass, pri, flatter, carousel, msa, forprofit, deadlock, seclud, fiddl, brokerag, skit, dualiti, mahal, ture, derogatori, stout, odin, chai, evas, fleetwood, adulter, cyber, rehab, muppet, lex, elus, nonverb, raptur, martyrdom, aug, unjust, falun, keyston, nuanc, ordeal, headphon, chaotic, brillianc, penanc, saffron, hh, gentli, edific, blaze, plung, slipperi, lexicon, standpoint, dubiou, unintent, gunshot, forgot, ro, keynot, underdevelop, preemptiv, futil, succe, aqua, burgeon, fourwheel, reprimand, rye, eboni, asham, cadenc, append, multi, dionysu, glimps, blizzard, stripper, stylu, misconcept, {\'e}ireann, rarer, roost, unsign, hasbro, gc, shook, lust, priu, orion, megatron, enchant, rem, lg, fanfar, dike, keynesian, polem, sciencefict, eyesight, mag, irrat, var, opengl, unorthodox, ito, rampant, downturn, coerc, popularis, bohr, selfproclaim, ohm, eid, artemi, sonnet, glitter, cocktail, disassembl, outcri, jeep, trash, promiscu, hypnot, gown, layoff, reconsid, recollect, netflix, misfortun, ubuntu, melon, dinar, fukushima, tabernacl, hertz, brute, codec, d{\'a}il, kingfish, psych, renegad, infidel, parisian, elucid, picnic, paraphras, freemasonri, booti, firstperson, ak, paperwork, foo, subvers, vc, wrongdo, steak, insolv, rocker, interplay, scholast, looney, amic, rancher, cracker, tn, thaw, yearlong, colli, barbar, obelisk, scoreboard, overt, plow, loudli, vf, corona, illfat, gita, moonlight, twoway, firebal, ichigo, couch, abod, gundam, symphon, chute, lush, crate, kd, acm, sed, psa, magellan, tyrant, sow, pedest, gambler, craze, embroideri, vt, karaok, apprehens, adept, lr, botani, hindustan, cu, unexplain, transsexu, grit, yahweh, youll, orangutan, slander, valkyri, diner, alt, calv, meander, preempt, tutori, rustic, mover, polari, sloth, ponder, introductori, reshap, kierkegaard, tm, gigant, archeri, crisp, zedong, stairway, uav, skid, rariti, zeta, interestingli, checker, eschew, htc, ramayana, kiwi, standbi, glaciat, reindeer, notforprofit, causat, ecommerc, repaint, pilat, juror, anvil, abridg, sauron, trident, quad, magpi, catapult, franca, indetermin, viz, footnot, shortcut, horrif, meaningless, pinpoint, neoliber, werewolf, audiobook, accordion, tith, disclaim, wiener, backer, triumphant, biomed, thistl, showroom, curricula, manslaught, crossbow, roadster, postcard, cockroach, blackandwhit, glitch, drm, psychotherapi, cactu, klux, kanji, kiosk, npr, moratorium, breez, overdr, dentist, bystand, workout, underpin, psychoanalysi, dreadnought, taj, delphi, woe, woodpeck, syncop, ib, dialog, counterpoint, expressionist, adulteri, cohabit, overpow, lancer, dunlop, elud, abyss, adverb, retroact, guis, whenc, conflat, dab, fragranc, childbirth, handbook, firsthand, unrestrict, smoothli, disrepair, barefoot, impressionist, ici, poke, contradictori, glad, quarantin, castrat, lest, br, decepticon, phish, overtak, omen, crucifi, tangl, shapeshift, surmount, qa, ig, zenith, refit, tango, cyborg, widerang, talon, subprim, bland, bugl, chernobyl, alias, phenomen, gmc, stylish, unrealist, silli, unearth, fanbas, sling, healer, accentu, remors, ww, mania, gratitud, vignett, tantra, maraud, atheism, bac, signag, sunken, pti, corrobor, epilogu, mil, mbc, sidebysid, downsiz, kangaroo, dormant, midi, slid, hoc, kt, backstag, inton, stringent, delici, scribe, nikon, comma, pr, bhakti, spectr, rediscov, promo, pixi, helpless, hive, sociopolit, stray, apt, fanat, trapper, boxoffic, vultur, infal, mantra, subcultur, mime, desol, proactiv, fullsiz, skype, autograph, keel, overcam, idealist, fourdoor, heartbeat, mattel, pasteur, emerald, csi, hopeless, elaps, envisag, rabi, mahabharata, hug, vulgar, delus, outsourc, satisfactori, scuttl, noncommerci, boomer, maru, remiss, lm, nexu, mane, esp, psychoanalyt, ale, suffoc, pois, almanac, etiquett, transcendent, euclid, restat, bondag, tintin, morph, inaccuraci, unauthor, noncombat, falter, earthli, repaid, pda, henceforth, obliter, cadillac, swirl, dd, fledg, epoch, manoeuvr, displeas, vehicular, forese, allegori, oat, widget, maximu, miocen, aegi, fend, kabbalah, manic, skunk, gaze, piraci, chevron, stifl, shabbat, improb, clap, unconvent, np, bmg, nasdaq, sucker, apprehend, sae, molotov, withheld, poss, fsb, alevel, zebra, vip, horu, tusk, hs, messerschmitt, elk, goos, harlequin, reced, hourlong, drank, usbas, invoc, ebay, dichotomi, potion, subsum, electra, teas, cardboard, ironi, gcse, awe, wutang, envi, insofar, jackal, grappl, incest, hobbyist, sampler, vet, starship, soc, cliqu, typolog, casket, isp, bmi, cumbersom, schemat, insignific, delicaci, beech, grung, misrepres, paramed, overshadow, closet, torment, asu, laplac, stork, panason, rag, gamecub, phenomenolog, expon, policymak, mojo, esa, che, cybernet, afp, nestl, walmart, exemplari, twohour, xx, handwritten, nerd, polka, ethnograph, antisubmarin, portmanteau, notwithstand, freshli, toc, courier, referr, stewardship, kei, worldclass, foreground, disqualif, loneli, torchwood, enquiri, havoc, unpublish, unsettl, hobbit, swung, selfsuffici, beagl, gra, fab, accustom, druid, vodka, foundri, crave, abound, rubl, preschool, malici, righteou, obligatori, reimburs, condon, millennia, yesterday, hoard, centurion, sinist, junker, personif, sli, decapit, adc, voodoo, herm, waterg, kindl, attir, abelian, dowri, greed, yearbook, endgam, throwback, msn, camper, nag, cappella, ia, hardest, adjud, glamour, undo, hearth, semicircular, monopol, dew, ks, rubbl, utopia, greedi, dsp, ccc, futurist, courtroom, rr, medicaid, blacklist, telegram, lazi, kink, sinn, symbiot, hade, php, outing, nazareth, scissor, minstrel, unresolv, subsect, belliger, naughti, siren, uranu, dissuad, veer, prometheu, restitut, stateoftheart, sprinter, prerecord, cleanli, maimonid, exceedingli, overtaken, coup{\'e}, ddr, barricad, anthropomorph, ope, empathi, neuter, notebook, cessat, nullifi, ox, shortwav, suitor, bandai, scorpion, startl, richli, underwear, bae, daredevil, horsemen, tumbl, doomsday, cong, arisen, pinbal, visionari, kayak, thirst, peertop, graveyard, diacrit, oldsmobil, spaghetti, bitterli, poppi, displeasur, manli, tardi, ei,
		\par
	}

	\item[$\mathcal{V}_{2} = $]
	{
		\tiny

		cher, bj, constanc, dani, bartlett, rene, melvil, rowl, barth, ryder, stephenson, hitchcock, kendal, brahma, elisabeth, burt, polli, foss, craven, manni, kerr, berni, klau, alexandr, benoit, petri, bernstein, upton, rei, parri, maci, russ, townshend, fei, elop, baptis, gabe, melvin, osbourn, osman, sita, meng, stepmoth, dusti, ein, wheeler, beckett, elain, becki, jai, descart, slade, khrushchev, bigg, weir, kaplan, bingham, elena, pahlavi, kissing, adolph, guthri, dramatist, griev, mayfield, thornton, eliot, kang, sisterinlaw, voldemort, susi, mcgrath, cynthia, rahul, ame, xiang, chong, lazaru, mahmoud, priscilla, berg, paterson, fatima, silverman, cale, smokey, jakob, eno, sampson, cassidi, baum, elmer, partridg, tong, mildr, guan, harlan, mcbride, cabaret, mccoy, karim, serg, eastwood, reilli, gee, jacobi, croft, neumann, sgt, joann, meg, reggi, zane, bongo, cullen, vito, sax, trotski, bhatt, gilmour, nan, pam, lamont, braxton, ryu, vaughan, kobe, wesson, tyson, mcintyr, lea, hillman, capt, roth, vicki, boyl, emeri, brandt, marquess, earnhardt, zappa, olaf, sutton, kaiser, astor, nikolai, ringo, ashok, minogu, bard, stacey, mcclellan, calvert, kramer, hagen, cartwright, modi, farrel, walden, bai, ella, anand, ulyss, exchequ, sargent, kamal, fianc{\'e}, jacquelin, godfrey, custer, jacobson, brezhnev, childless, louie, madden, jawaharl, busch, humphri, jonni, anwar, donaldson, filmographi, draper, goddard, kirbi, huey, grimm, sherri, cbss, flair, nath, aerosmith, begum, lowel, milo, zhu, emmanuel, abbott, jamess, fatherinlaw, patterson, linu, seaman, deacon, woodward, shortlist, heartbroken, abram, rori, vera, horowitz, lorrain, joli, indra, mehm, reluctantli, kabir, rosenth, ste, mckenna, coleridg, boo, scarlett, aguilera, jinnah, bellami, dent, remarri, pei, lam, carla, meredith, wilbur, dodd, courtney, rishi, sanford, newsweek, ell, henrik, annul, gregg, mcgovern, lott, poe, bridget, hale, joplin, giuliani, gerhard, ramon, heidegg, biopic, serena, ulrich, herod, illegitim, tiffani, braun, dorian, concubin, ridley, dre, robbin, trombon, bundi, hanson, tun, epstein, sylvia, hooker, josef, titu, obituari, nme, dariu, vanc, miln, ty, rudolph, exclaim, sadi, thatcher, dion, hodg, prima, rajah, henrietta, davenport, dowag, dorsey, rousseau, manu, taunt, kubrick, mahmud, radha, herzog, cecilia, sigismund, tobago, hayn, maha, djokov, reza, daryl, kathleen, johan, beal, obo, valentino, pandava, armand, whitlam, arjun, orton, hui, ewe, lulu, heinz, sardar, shakira, rani, tinker, trier, pia, kern, jp, hai, reev, unita, lal, sumner, ke, bint, coltran, cabot, weston, hine, fullback, anastasia, jefferi, anita, aloud, britten, gordi, whereabout, woolf, quintet, halen, robson, picker, tilli, stafford, olivi, vikram, ek, yadav, aurangzeb, eisner, maher, puri, honeymoon, cartman, tal, woodrow, lindsey, leann, eliza, flynn, durant, netanyahu, asimov, shamrock, keenan, shaun, happili, konstantin, hendrix, namesak, hume, courtier, diva, orson, himmler, welch, becker, yate, cinderella, valeri, mae, zack, naomi, cromwel, andretti, corneliu, sidekick, sanjay, cutler, burnett, teuton, barnard, davey, vick, clapton, brigham, winthrop, khalid, gradi, handel, shapiro, norma, brodi, ritter, connor, bequeath, schneider, hammerstein, collier, samantha, marlon, whitehead, barkley, eyr, yeat, ritchi, mastermind, waugh, rosi, olson, madelein, desmond, alicia, tolstoy, bei, gail, adler, buster, daw, josiah, chandra, abigail, dmitri, shelbi, ramsay, michelangelo, harmonica, schmidt, mcgraw, liszt, wolff, lev, gideon, holliday, finch, gomez, sheldon, englishman, tung, evelyn, magdalen, cowrit, jun, howel, galen, infatu, mckenzi, priya, muller, corsair, lyricist, halfsist, lindbergh, edna, shea, goldberg, germain, streisand, theodosiu, christen, raphael, rutherford, austen, prescott, senna, maclean, alban, uncredit, hain, buckley, bianca, organist, kung, reginald, ramsey, nightingal, macfarlan, boyz, entourag, guo, ripper, sj, peyton, favr, slept, horton, landi, chun, blanch, zachari, timur, cello, offbroadway, barrymor, trey, bain, lu, gough, ping, menzi, gladston, menon, muir, barlow, nguyen, ganesha, murad, adi, cedric, bentley, ing, richter, grayson, pearc, nana, bree, cassandra, wilk, brigg, mullen, varma, helmut, aj, doherti, olli, ruskin, hubbard, moran, dicken, stonewal, nemesi, hershey, stoog, snyder, hendrick, bate, hari, napier, yin, ingram, duff, staffer, prot{\'e}g{\'e}, peck, mcdonnel, palin, sergei, nakamura, ja, jenna, hansen, lau, romano, papa, ric, slater, leonid, winger, rockwel, hollyoak, nevil, duan, albrecht, cinematograph, spector, cantor, irwin, gaiu, sweeney, hutchinson, harley, kellogg, choi, dow, nikola, stein, maureen, narayana, sylvest, spielberg, hartman, ander, arya, leah, lucil, siegfri, clemen, geffen, blackwel, tanner, jing, ayer, igor, melani, bartend, jolli, saxophonist, howe, taft, claudia, nat, picard, dobson, carmichael, monti, mulder, carver, duran, grover, flo, moodi, natalia, nathaniel, gabl, brando, kimbal, wainwright, maynard, pj, dunham, alfa, gilli, parton, tendulkar, coowner, baird, blanchard, jang, springsteen, sati, markham, miriam, berat, thierri, rous, hernandez, sharif, patsi, carolyn, anjou, ang, dyer, houghton, pauli, oppenheim, underwood, novella, nader, clarinet, jb, damian, waltz, tennant, cohn, og, mustafa, kemal, saul, beyonc{\'e}, omalley, freder, dutt, beaumont, mckinley, minh, greatgrandfath, ayr, gan, malon, oti, tao, mcpherson, rabin, donovan, huffington, agatha, gueststar, cobain, dun, rollin, pir, rae, benton, clau, kyli, karan, gaga, relent, linden, fulton, jj, marcel, cato, tutelag, salvator, orr, compton, canning, ruben, nolan, sila, mcgregor, bernhard, sinatra, chaplin, ao, hector, engel, priestley, gibbon, forsyth, mugab, mcdowel, melinda, pamela, burr, merl, ashton, lawler, virtuoso, ripley, yamamoto, chu, erni, prasad, dalton, paisley, narayan, brutu, cara, housem, arden, dil, vasili, barrow, ala, nicki, bandlead, luciu, hick, cicero, ellison, steiner, hayek, cbe, hubert, reuter, odonnel, compatriot, molest, violinist, andersen, tomlinson, footstep, famer, phoeb, obe, foreword, kuhn, pollard, eusebiu, akira, teller, iqbal, duffi, leela, katharin, kaufman, thorp, iyer, dhabi, bea, shu, ozzi, pickett, gottfri, bender, orl{\'e}an, carlyl, bono, alexi, g{\"o}ring, fisk, kean, dustin, schumann, lister, cass, oconnor, snl, donni, keegan, ail, benefactor, letterman, kamen, unmarri, darryl, bonham, syke, stefani, sham, madoff, kala, layton, konrad, dixi, yusuf, aur, erich, popper, garrett, merri, philanthropist, mansfield, acharya, rowan, brennan, luka, loretta, jeremiah, tj, cush, darrel, babu, skipper, lacey, hester, kimberli, kazan, bryce, hepburn, mercer, sinha, jovi, graf, asher, burgess, om, fielder, tudor, zelda, lori, zimmerman, greenwood, xu, ballard, terrel, addam, ballerina, putnam, rai, kobayashi, martini, fowler, wiley, brock, alec, massey, kitt, cunningham, julien, loeb, bourn, villeneuv, rubin, slain, squir, gorbachev, bhai, schwarzenegg, harald, mara, persh, romney, simeon, connolli, alf, frazier, rolf, ich, guido, bertrand, doesn, juda, metallica, waitress, foley, spade, mather, oconnel, playbyplay, mohan, abd, cena, hallow, blyth, atkin, tanya, louisa, dolli, surya, zu, yao, gertrud, mandir, rigg, yan, seinfeld, georgi, gareth, chow, inferno, ava, merton, forster, bede, brenda, annett, shakur, larsen, huang, mai, mahler, butch, stefan, skye, smiley, gale, metcalf, ezekiel, bradman, claudiu, hobb, tex, denis, plini, mcguir, dickinson, baxter, vern, mandi, edda, pavel, maximilian, keaton, rhi, chloe, coppola, lillian, liddel, khanna, amar, bachchan, flanagan, jedi, payn, mcqueen, sasha, damon, goldsmith, marian, mccormick, alain, bess, garland, accomplic, {\'e}mile, caldwel, cosbi, sheen, skinner, eduard, shakti, nair, supper, mosley, raoul, cowork, jamal, burrough, bran, soninlaw, leung, gillard, irvin, megadeth, garth, kendrick, pryor, dandi, majorgener, moffat, booker, derrida, hadley, sheppard, marlow, behest, mariann, nawab, ajay, tweed, beatric, laurent, yd, bhutto, mcgee, phylli, guggenheim, ravi, dirk, curt, wittgenstein, mage, gogh, bliss, allman, prem, derrick, unbeknownst, keller, ching, antoinett, keat, kart, epistl, crockett, ellington, taco, luciano, siva, trudeau, mabel, reuben, boyer, socialit, axel, terenc, hayley, hanna, saviour, eastend, exhusband, holt, barr, baro, babe, dong, larkin, mehta, walton, playmat, scroog, dem, beckham, stravinski, conway, hussain, minni, foreman, peng, matti, godfath, chaucer, rashid, warden, shin, crowley, brewster, messi, fitch, andrei, olsen, knox, simm, frankenstein, barbi, montagu, maa, clayton, lilli, godwin, rees, rosemari, galloway, sweetheart, royc, mori, editorinchief, darci, salim, payton, deborah, kingsley, zia, sharma, brabham, hooper, infuri, ida, flirt, oneal, gillett, unborn, gonzalez, dietrich, nur, {\'o}, mors, finley, sal, waller, higgin, bandmat, foucault, archduk, chand, osullivan, genghi, reddi, roxi, coe, hark, emil, sadler, duma, paddi, edith, slay, addison, jessi, mandela, angi, didn, theresa, leno, duli, devin, clarkson, gerard, perez, jericho, brent, quinci, strauss, davidson, nilsson, rawl, sheila, nico, ignatiu, nelli, cochran, yuri, musa, yoko, clifford, mackenzi, lola, friedman, johansson, mogul, bose, peacock, gemma, tiberiu, hewitt, mariu, housekeep, haig, jona, hess, mirza, moe, eminem, christensen, tobi, gong, larson, prost, lyle, ike, mahesh, vaughn, luc, drummond, lawson, bloch, yong, eastman, hearst, mansel, kwan, oreilli, brisco, gerrard, frenchman, everett, ariel, kathryn, bauer, sexiest, hemingway, octavian, hilari, zheng, publicist, yun, stewi, manson, ghulam, hoffman, padma, mbe, demetriu, qc, erasmu, graem, wordsworth, viscount, ursula, shackleton, selena, berger, vettel, maud, bogart, boa, barnaba, garfield, abbi, grandmast, murdoch, swann, mcconnel, costello, rascal, tyron, middleton, dina, policeman, kat, berman, tak, sandman, ying, gaston, puja, gina, tess, mobster, hoffmann, caleb, schwartz, cassi, oakley, opin, ji, sculli, manfr, beatti, cheng, dyke, lana, liang, eaton, greer, faraday, spitzer, bowman, mei, timberlak, clifton, arti, hoyt, jensen, schulz, chappel, spock, audrey, bullock, titular, trajan, roberta, buffett, dewey, rana, paig, salman, hathaway, lai, erwin, whitak, tristan, salomon, damien, shepard, crouch, arun, merril, virgil, marcia, ezra, gifford, bismarck, shearer, rosen, wynn, arjuna, alumnu, aubrey, beau, colbi, goeth, patton, mackay, aziz, cleopatra, angu, sutcliff, stringer, stevenson, mukherje, oswald, dai, debra, vijay, eug{\`e}n, edi, zoe, ohara, bey, mayer, gama, goodwin, heali, ramakrishna, wren, ness, schultz, faust, spenc, betsi, dahl, olympian, aldin, abe, waldo, jare, gilmor, rowland, hopper, morley, wonderland, dane, albright, stoner, camil, wendel, bene, traver, vinci, fabian, parvati, wilkin, bett, hu, penelop, ahmadinejad, russo, mendelssohn, goebbel, burnham, loren, coolidg, lehman, fleme, distraught, brecht, mein, xiao, maguir, abi, osama, zach, tweet, milli, hindenburg, williamson, nero, lear, josephin, alvin, newel, munro, dominiqu, jock, joachim, bloomberg, warhol, joanna, tagor, kidd, sabrina, saunder, watkin, lowri, tammi, baudelair, nietzsch, hayden, seward, ada, rohan, paleontologist, sherwood, cobb, kai, walther, kara, meek, erica, rudi, holloway, yogi, f{\"u}r, atkinson, henley, yvonn, lara, fiona, deng, haydn, springer, js, merritt, veronica, adel, lerner, joshi, hawthorn, sima, theo, lew, mia, kri, charley, rajiv, samson, juliet, archangel, hayward, meteorologist, weiss, sloan, eileen, alistair, randal, shan, farley, dimitri, hasan, pott, twain, wyatt, mariah, matthia, conqueror, roe, weinstein, hilda, ginsberg, colbert, bakr, heidi, gage, nicholson, overhear, heller, grandpar, marguerit, nugent, jonah, woo, cori, wee, cassel, poirot, esther, alam, willard, durga, hurley, putin, cheney, pollock, leroy, singleton, clint, altman, behead, provost, eulog, mira, norri, raju, dreamer, lakshmi, mir, daley, evangelist, libretto, mitt, betroth, hanuman, barton, maharaj, regina, webber, jasmin, roach, jokingli, katz, nikita, erin, bahadur, mathew, xfile, faber, hippo, lennox, rommel, steinberg, suleiman, jarvi, martel, shen, cavendish, kristen, kaishek, polk, b{\'e}la, dev, angrili, cho, pooja, viola, bert, marjori, mott, richi, cj, nora, gladi, henson, jayz, simmon, tanaka, dali, anil, wr, gunn, firth, blackston, rep, jasper, capo, cathi, nate, emanuel, gamespot, faulkner, chopin, hodgson, speer, darl, hazrat, alison, grossman, carlton, tchaikovski, mosh, cornerback, shanti, zhao, tarzan, leland, pandit, iain, cheung, hazel, jagger, patel, staci, wharton, groen, viktor, yve, monet, mack, wanda, ogden, schubert, elijah, harriet, chung, garri, keyn, cari, libbi, macbeth, aga, cheryl, lieberman, phelp, vivian, fairbank, talbot, abel, apologis, lucia, peirc, alexei, mclean, kimmel, smyth, copeland, coward, lar, tha, bradshaw, roommat, savil, lena, cowel, dempsey, earp, wilkinson, def, chopra, marr, averi, carmin, ethel, hon, jude, ledger, hotspur, kristin, hua, daphn, tobia, tian, stow, suzi, jiang, templar, mortim, sutherland, bariton, bergman, sai, hatfield, hartley, janic, gao, denton, bradburi, morrow, shivaji, jameson, carli, wilcox, amelia, ricci, mcnamara, pell, conn, ono, arlen, ismail, himach, ej, tam, morrissey, feng, housewif, pundit, exwif, heme, sahib, matilda, gallagh, scorses, cha, jani, harrington, brig, decker, yeltsin, lauri, percussionist, haa, townsend, olga, soloist, ist, sach, gershwin, grime, gardin, sawyer, malik, lim, bret, jodi, hamid, gotti, dori, thom, dreyfu, francesca, amir, waiter, barbarossa, bunt, cindi, keyboardist, lizzi, tracey, siegel, cy, marta, sabin, botanist, amadeu, macleod, orwel, stanton, grandchildren, xv, kahn, kelley, easton, felic, frasier, macmillan, gile, nikki, wasn, nanni, conni, umar, m{\"u}ller, hahn, dickson, barron, shankar, gopal, jimi, estrang, sonia, rankin, elliot, jarrett, headmast, amo, fonda, lamar, megan, middleag, novak, ambros, vinni, marley, nehru, jeanbaptist, deathb, prodigi, maxi, getti, dyson, macpherson, cocreat, druri, stepfath, swanson, harman, psycho, shelton, pasha, greenberg, dawkin, ganesh, tycoon, mona, henchmen, douglass, rasmussen, herbi, judd, countess, diaz, whistler, ling, comedydrama, yue, msnbc, dillon, gillian, huxley, ren, ree, haley, duval, mccall, bunni, jahan, fallon, bowen, rusti, qi, iren, horrifi, jung, rudd, codi, goldstein, minaj, natasha, bartholomew, kemp, dora, whitman, biden, swore, asha, bragg, knowl, radcliff, xian, lydia, granni, mimi, wen, dunbar, jermain, rosenberg, hem, darbi, archibald, prakash, calhoun, fay, uthman, elia, blain, regi, dumont, nadia, gretzki, mckay, enoch, petra, edmond, isaiah, congratul, winslow, aquitain, mister, burrel, clanci, rooney, carlson, clare, jax, sammi, gillespi, chr{\'e}tien, rufu, gwen, deva, erickson, faisal, vader, kipl, jc, gleason, banjo, amr, billionair, suzann, yew, jeann, bower, hutton, graci, dole, weinberg, aquina, barnett, corbett, mandolin, scotti, glover, saraswati, lenni, feldman, earldom, falk, corey, heiress, gavin, mclaughlin, pai, marlen, xviii, huston, rihanna, werner, mitch, fran,
		\par
	}

	\item[$\mathcal{V}_{3} = $]
	{
		\tiny

		predic, ellipt, ganglion, gpu, stockpil, taller, asynchron, edema, unpleas, circumfer, ultrasound, placenta, interlock, nanotechnolog, highqual, planar, gestat, blocker, pentium, glide, myocardi, disinfect, modulu, polygon, gnu, fingerprint, plutonium, lubric, conspicu, overdos, uptak, gaussian, vastli, seismic, sedimentari, decompress, truncat, conif, agonist, plum, rotari, xy, cleaner, blister, sq, silt, deterg, alkaloid, broth, jelli, boson, minimis, lettuc, circuitri, fasten, norepinephrin, leach, conduit, barley, subspac, aft, herbivor, testicl, sinu, mismatch, carbohydr, phosphat, lymphocyt, sewag, paralysi, coolant, modular, dementia, hivaid, actin, monom, perpendicular, tangent, longitudin, asphalt, trough, dung, lid, superconduct, fractal, swollen, phosphoru, blackberri, cantilev, disconnect, enamel, discomfort, foam, hydrid, lifecycl, parametr, cretac, smelt, mustard, benzodiazepin, valuat, heater, nicotin, sync, entropi, seam, igneou, avian, airspe, retrofit, coli, shrink, slit, pdf, refract, electromechan, hydroxyl, unintend, starch, longrang, uneven, pineappl, phosphor, tray, quantifi, permeabl, incis, basalt, facet, pancrea, {\fontsize{6pt}{6pt}\selectfont \ensuremath{\leq}}, phenol, thruster, graft, lisp, permut, biomass, silic, unman, neurosci, carburetor, incendiari, nonhuman, mussel, inward, recharg, lowcost, venou, pediatr, trajectori, ct, inlet, micro, taxabl, dilat, compressor, cytokin, pvc, masonri, shoal, optimum, grower, highend, lagrangian, granul, predetermin, cranial, inerti, neurotransmitt, inhibitori, alkyl, bowel, gum, pesticid, nitrat, pickl, stamina, headlight, eukaryot, fascia, cytoplasm, verif, porou, pendulum, lupu, retent, hubbl, lumbar, {\ensuremath{\delta}}, creep, halogen, torsion, recombin, evergreen, adhes, nt, knit, meteorit, methamphetamin, selfesteem, steroid, stellar, harden, coke, subsist, unicod, {\ensuremath{\varphi}}, voip, cellulos, congenit, sweeten, covent, taxonomi, abdomen, manur, fructos, measl, depreci, rodent, firmwar, gsm, pancreat, burner, herbal, yaw, aftermarket, improp, pacemak, hue, handheld, hallucin, intracellular, cmo, alga, slug, mucu, crank, cramp, inflow, megawatt, submachin, hemp, cereal, prognosi, refractori, zoom, subsurfac, insomnia, adob, reflector, cholera, snout, vortex, xenon, platelet, woven, penicillin, fermi, fungi, rust, muddi, graze, polymeras, duplex, evenli, garbag, fungu, tau, horsepow, forehead, tradeoff, sine, debug, flammabl, amnesia, buckl, decidu, nonzero, convect, clog, sulphur, prefront, contraind, finer, syring, peat, atrophi, incub, reload, baggag, auditori, spectromet, lactat, mucosa, sulfid, inorgan, calori, parabol, aircondit, hepat, bruis, decomposit, metamorph, spong, khz, pollin, notch, lichen, mtdna, scalp, null, peni, bleach, lifespan, uv, scarciti, wastewat, filesystem, latenc, thicker, countermeasur, phenotyp, foliag, router, rippl, trivial, soak, psychot, braid, itch, kinas, neonat, ipv, canin, morphin, constrict, amp, bloodstream, airbag, phonolog, inductor, deflat, primordi, macroscop, pastur, matric, proportion, lumber, yeast, pv, spleen, rot, thorac, undul, reddish, gait, microprocessor, dorsal, acet, infus, locu, ballast, tuna, benign, analges, k{\"o}ppen, multiplex, analogu, hotter, gravi, nozzl, tomographi, loosen, morbid, coval, malwar, anomal, quotient, ultrason, hilbert, regimen, pollen, resin, insensit, wifi, workstat, appetit, onboard, beryllium, disson, precaut, infest, inert, builtin, logarithm, grind, eeg, leukemia, reusabl, millet, camshaft, inhal, lighten, pellet, perfor, phylogenet, encapsul, manmad, withstand, cdma, nausea, flex, sap, smallscal, snowfal, covari, tentacl, asymmetri, pandem, crosssect, coagul, vagin, cinnamon, mole, ach, chill, bottleneck, firewal, boolean, pcr, tread, antidepress, firstord, tecton, formaldehyd, lithium, hn, maiz, refil, acetylcholin, taper, subtyp, perenni, chimpanze, inexpens, psychosi, siphon, refresh, highperform, ditch, anesthesia, particul, breadth, sharpen, microbi, ribosom, litter, strawberri, hf, scalar, url, lifethreaten, ounc, acryl, bait, centrifug, quench, perfum, chimney, pci, baselin, powerpc, dilut, stabilis, epsilon, inertia, boni, kombat, fetal, unequ, shutter, plumb, nucleotid, rangefind, ammonia, magnesium, neutrino, cough, unsaf, razor, filtrat, protrud, pulley, rectangl, dissect, boron, cylindr, milder, retina, thunderstorm, endogen, welldefin, waterproof, mening, addon, overload, eyelid, salin, pore, quadrat, reentri, halflif, traction, undesir, drawback, latex, vertebra, socket, scuba, diod, opaqu, herbicid, graphit, cervic, starvat, bulki, sedat, isom, ef, cn, aspirin, carcinogen, reclam, apoptosi, carnivor, stainless, crankshaft, cholesterol, beak, euclidean, fuze, ovul, twostrok, wedg, geotherm, fluoresc, polymer, waist, mediums, sutur, testosteron, proteas, spp, stool, helix, tremor, oven, backbon, aldehyd, glutam, gearbox, gaseou, gastrointestin, apertur, amphibian, metabolit, audibl, flake, calculu, fiberglass, tint, tether, unsuit, bodili, adren, prune, mould, prenat, crosslink, brine, serotonin, assay, accret, deactiv, pharmacolog, chalk, turbul, elicit, wafer, xml, plankton, euler, termit, antimicrobi, prokaryot, theta, otter, diarrhea, tonic, heterogen, ligament, allel, {\ensuremath{\sigma}}, chromatographi, valenc, headlamp, oyster, ether, mgkg, ioniz, bedrock, taxonom, fibrosi, mammalian, prosthet, crt, noisi, shortest, farmland, parallax, electrochem, gel, ejacul, buoyanc, bmp, seawe, doppler, pelvic, axial, tensil, welldevelop, allergi, arabl, lumen, telephoni, flap, dredg, intox, amd, scaffold, plume, atm, deforest, lc, dipol, tungsten, excis, monoton, hydroelectr, scanner, tightli, rash, bladder, nebula, applianc, seafood, gut, css, plumag, viscos, clariti, {\ensuremath{\lambda}}, bulg, anemia, deceler, antioxid, citru, dehydr, purifi, scrub, amplitud, kbit, chipset, homolog, gunpowd, octav, lessen, costeffect, asymptot, millisecond, footprint, hypothet, cumul, amphetamin, pariet, extinguish, resistor, supercharg, spectral, binomi, opioid, dashboard, viabil, vesicl, benzen, lcd, landfil, buildup, aerosol, barb, allerg, pneumat, spectra, urinari, grassland, tannin, numb, inflect, macrophag, capacitor, cleav, emuls, tcp, mitochondri, monoxid, photosynthesi, placebo, topographi, counteract, intermitt, fission, petal, etiolog, tyrosin, beet, cathet, lug, ghz, inciner, muscular, excret, toxin, ampl, subunit, syphili, tick, damper, safer, duct, endocrin, sweat, steril, epidemiolog, glue, embryon, leakag, mesh, transluc, dosag, methanol, pear, attenu, palett, convolut, latent, semiautomat, shunt, cleavag, fluorid, reactant, continuum, obliqu, hamiltonian, wearer, ht, ellips, errat, ev, hydrolog, dough, centimetr, {\textmu}m, asthma, biopsi, inflammatori, electrolyt, amplif, spoiler, sausag, regen, interperson, cobalt, timer, lath, diffract, lowlevel, fece, highpressur, subduct, washer, impart, munit, handset, autoimmun, powertrain, plywood, vagina, salti, malnutrit, halid, asbesto, microorgan, submerg, relativist, necrosi, debilit, hygien, throughput, manganes, unload, magnifi, smallpox, perturb, lymphoma, iq, avion, {\ensuremath{\pi}}, saliva, pasta, probabilist, schr{\"o}dinger, recurs, methan, hairi, sac, virul, slab, outpati, mangrov, glaze, retard, tactil, tonal, ventral, smoother, pastri, waveform, thermomet, ultraviolet, hotspot, transmembran, cuff, unix, fungal, kwh, triangular, scrape, fixat, anaerob, walnut, isomorph, charcoal, seawat, legum, movabl, exacerb, toplevel, cocoa, contour, appendag, breech, vomit, qualit, pavement, stew, splice, realworld, agil, clad, drier, dataset, cortic, pars, epithelium, seab, impur, immatur, serum, dim, shred, atyp, hydrat, hydropow, gust, thyroid, stochast, urea, ovarian, uteru, medial, endotheli, bacterium, damp, decompos, prism, synaps, primer, thirdparti, mri, scalabl, selfpropel, ganglia, javascript, lorentz, brood, usabl, airflow, coars, sticki, orthogon, germ, polyethylen, opportunist, ester, gastric, yearround, hydrolysi, calibr, reset, zip, ionic, convex, slender, chew, migrain, pebbl, curvatur, groundwat, snack, peptid, rainwat, outweigh, potent, contigu, nmr, harmless, cadmium, brittl, modem, postag, nectar, clade, anal, follicl, generalpurpos, stove, vend, cathod, yogurt, contracept, capacit, ozon, kappa, indistinguish, milki, neon, cation, thicken, theropod, solidst, uterin, howitz, arthriti, electrostat, nylon, {\ensuremath{\varepsilon}}, dissoci, localis, anatom, celsiu, softer, thinner, warhead, hash, dendrit, rectifi, tubular, dimer, ulcer, anesthet, viscou, linearli, highenergi, smoker, hemoglobin, lobster, warmth, carbid, biodivers, occlus, vanilla, diaphragm, {\ensuremath{\alpha}}, microscopi, tar, anod, hover, sideeffect, seedl, brighter, thereof, aircool, cdc, heaviest, topograph, vitro, folder, innerv, ripen, reagent, ailment, eucalyptu, cornea, supernova, illicit, immunolog, quark, nuclei, queue, edibl, mildli, quantiz, concav, innat, bipolar, determinist, cigar, onsit, aromat, smartphon, polymorph, titanium, slr, actuat, interconnect, mixer, outward, centimet, pylon, dohc, zx, resili, riemann, hydro, loader, asexu, lamin, windshield, floral, weakli, latch, emitt, garlic, peroxid, tricycl, ovari, lng, antipsychot, intric, subfamili, compost, bile, bulb, abras, mitochondria, redirect, synapt, tandem, histolog, ventricular, gelatin, olfactori, hippocampu, histon, mpa, transvers, cabbag, alkalin, flang, filament, shrimp, polio, torso, ammonium, condition, {\ensuremath{\beta}}, raft, unaffect, nematod, asymmetr, laundri, swine, denser, planck, cartesian, infinitesim, mach, endpoint, intraven, necessit, ccd, cushion, helium, mimic, exce, maneuver, arthropod, readabl, cyanid, booster, chloroplast, aneurysm, saltwat, chemotherapi, unpredict, petrochem, cutoff, mandibl, carcass, alveolar, fixedw, ppm, microb, fore, fig, kv, hyperbol, increment, {\ensuremath{\theta}}, transistor, suction, situ, perceptu, tumour, impract, carboxyl, truss, {\ensuremath{\mu}}m, fume, fertilis, spacetim, semiarid, ineffici, etch, incandesc, squeez, alluvi, interoper, buttock, motil, runoff, aquif, brightli, pelvi, sew, tighten, lag, greas, lymph, nucleophil, binder, projector, odor, conic, vinegar, aroma, dirac, enzymat, colorless, projectil, bhp, darken, poultri, drip, embryo, bromid, ligand, shrub, ieee, symptomat, scent, pcb, aortic, germin, biodiesel, eigenvalu, manpow, chromat, lh, fern, millimet, collater, funnel, homemad, toxicolog, malign, crystallin, mango, {\ensuremath{\pm}}, brows, irrevers, biochem, migratori, spindl, coaxial, disproportion, semen, distal, moist, constrain, predatori, fibrou, frontal, loaf, syntact, respir, luggag, strata, mosquito, hierarch, warmer, immobil, elast, sturdi, camouflag, kiln, forearm, invertebr, offroad, iodin, lexic, hydrophob, spectroscopi, rom, chlorin, spheric, dimorph, feeder, byproduct, thigh, schema, psi, mbit, faulti, cytochrom, orgasm, isoform, autosom, biotechnolog, highpow, chromium, sprout, dopamin, infarct, basal, flatten, scsi, vat, mpeg, cucumb, rub, lumin, irradi, markov, aberr, insolubl, aquacultur, granular, plasmid, coronari, nomenclatur, ripe, fissur, wheelbas, foodstuff, biofuel, cleanup, pathophysiolog, handgun, apic, relaps, planetari, corneal, sulfat, sheath, hexagon, quadrant, halv, cyclic, menstrual, cataract, sonar, influenza, yarn, conveyor, shingl, malfunct, yellowish, barium, inlin, jpeg, methyl, genotyp, luminos, nonstandard, bayesian, carbonyl, outflow, androgen, glu, overh, loudspeak, lipid, antisoci, poisson, suppressor, {\^a}, shellfish, condom, replenish, fetu, ipad, clot, postur, fourier, superfici, bsd, savanna, lengthen, rudder, lump, carrot, recoil, magnif, encas, catalyt, iodid, taxa, fluorin, preferenti, suck, hamstr, avers, massproduc, hydroxid, weed, hz, capillari, pallet, sanitari, lattic, epitheli, soybean, markedli, anion, ic, imbal, carcinoma, ecm, overlay, carbin, cryptograph, helic, caffein, microwav, alkali, sore, tab, higg, heurist, purif, clamp, plaster, lightli, transloc, http, anomali, twodimension, extracellular, estrogen, malt, linen, geodes, workload, mainfram, indent, firepow, threedimension, scaveng, touchscreen, collagen, pariti, radiant, antiinflammatori, cryogen, hose, kerosen, subtract, nucleic, squid, clam, subclass, nitric, opensourc, pituitari, insecticid, flu, motherboard, macintosh, metadata, ubiquit, {\ensuremath{\gamma}}, glycol, palsi, abstin, canopi, retrograd, colder, affix, airfram, vivo, sickl, dehydrogenas, chunk, potenc, html, predictor, adhd, epilepsi, rgb, canist, phylogeni, linkag, cassava, radial, hardwood, nostril, hemorrhag, ruptur, biosynthesi, aqueou, nodul, flare, pounder, dn, metallurgi, volt, cipher, biochemistri, mrna, complementari, {\ensuremath{\omega}}, crustacean, nippl, solder, wingspan, strut, microbiolog, almond, intrus, vapour, extratrop, magma, mite, spore, fahrenheit, diurnal, soy, forag, uncontrol, cleft, glacial, vertex, disloc, lemur, marrow, bluetooth, photovolta, unwant, transfus, droplet, sludg, retin, fp, sewer, molar, cartilag, sql, galvan, delic, agar, instantan, powerpl, molten, sunflow, subsystem, palat, thorium, untreat, quicker, punctur, transient, uniformli, spici, soften, soda, gui, ethylen, polyest, vascular, {\ensuremath{\mu}}, axon, petrol, somat, superstructur, functor, precess, turbocharg, hydrocarbon, syrup, oneway, sugarcan, tendon, knob, dryer, furnac, hyperact, kb, arid, benchmark, streamlin, pouch, skew, minu, runtim, stationari, macro, fourcylind, apex, cystic, phosphoryl, overflow, infertil, dope, sclerosi, chiral, interstellar, sequenti, tuber, hing, filler, repositori, undersid, ethyl, lambda, transduc, sn, nocturn, nanoparticl, skelet, rearrang, amorph, spam, keyword, inactiv, ethernet, caterpillar, bog, longitud, massag, bio, snail, throttl, ventricl, radiolog, elong, quartz, malform, aerob, weaponri, hypertens, groin, ipod, concuss, redistribut, summat, rivet, cultivar, pheromon, takeoff, catalyz, dielectr, silica, floppi, paddl, amput, rf, mollusc, keton, cyst,
		\par
	}

	\item[$\mathcal{V}_{4} = $]
	{
		\tiny

		anchorag, uci, bremen, metropoli, holstein, terrier, sagar, airstrip, decommiss, rink, pisa, burgundi, showdown, wichita, raleigh, honolulu, playhous, hillsborough, essen, openair, yucat{\'a}n, sooner, careerhigh, cheltenham, augusta, bazaar, suntim, avon, internazional, regatta, awa, luzon, taekwondo, sw, tahiti, hereford, galatasaray, wat, punic, wyom, swindon, stirl, samoa, surrey, boardwalk, goaltend, lynx, zurich, midwest, cypress, hackney, fruition, lineman, pendleton, hampstead, pike, sinai, warwick, paralymp, britannia, lowli, tripoli, eskimo, qs, hom, vicechancellor, durham, chengdu, triest, lsu, barangay, somerset, hermitag, dakar, payperview, baja, metz, silesian, williamsburg, antigua, galway, fillmor, kochi, heathrow, patna, lauderdal, grizzli, jamestown, swat, chattanooga, equestrian, chesapeak, hilton, farmhous, headtohead, arcadia, heidelberg, genoa, sofia, suffolk, dorset, borneo, berkshir, racetrack, fk, tallinn, ghent, auditorium, northbound, utc, calai, canuck, centenari, goali, shortstop, limerick, ut, yearend, volga, granada, atol, thenc, blazer, nugget, chinatown, nxt, ipswich, geelong, southward, cologn, auckland, glastonburi, seoul, xm, galveston, thoroughfar, schleswig, ural, reenter, brunel, usl, salford, iaaf, raptor, deco, palermo, tavern, haifa, turf, infield, dresden, georgetown, reschedul, gamewin, northumbria, unbeaten, outskirt, threepoint, cairn, biennial, wolverhampton, aggi, huski, inver, bodybuild, rotunda, fordham, squash, dormitori, encamp, colspan, somm, qanta, redshirt, mcgill, bahama, nagar, caf, humber, monmouth, stoni, uruguayan, kathmandu, ff, wellesley, bathurst, cove, carlisl, tucson, antwerp, tf, upstat, centenni, everglad, disembark, northerli, myrtl, taluk, durban, kolkata, turin, midget, penang, doneg, langley, schoolboy, sacramento, straddl, yokohama, thessaloniki, indu, viaduct, niagara, backtoback, newport, prom, vermont, greenwich, postcod, aleagu, guangdong, subregion, savoy, eal, seasid, olympiad, twotim, mangalor, z{\"u}rich, nagpur, allamerica, mrt, barbari, northernmost, wta, burbank, argyl, jurass, nj, ventura, freiburg, tmobil, woke, fewest, semi, roh, kennel, madeira, staterun, asean, vicker, siberian, shawne, argonaut, hsbc, precinct, lill, nippon, verona, fresno, monterey, belgrad, jaya, guangzhou, sumatra, redskin, bhopal, redoubt, flint, varanasi, interc, tryout, showtim, norwich, heisman, bogot{\'a}, intercollegi, occident, renumb, anfield, undraft, abscbn, saratoga, snowboard, wiltshir, bois, boulder, cheyenn, dealership, starboard, ute, boutiqu, vale, claremont, tonga, canberra, shropshir, arbor, hinterland, tehran, gladiat, salem, seminol, olympia, sept, cavit, nrl, steamship, coldest, minneapoli, edmonton, turnpik, midwestern, anaheim, taipei, sesam, derri, belfast, mersey, entrant, beacon, stuttgart, salzburg, maui, fujian, ohl, ctv, jaipur, grang, savannah, irb, picket, motel, northumberland, callup, upland, rochest, pyongyang, cnbc, fedex, kiel, roadblock, nagoya, luton, colombo, bungalow, gettysburg, horsedrawn, eastwest, loyola, staten, adriat, mariana, vauxhal, powerhous, potsdam, humboldt, raceway, clemson, coimbator, embank, waterfront, caen, camden, fremantl, hillsid, bere, caledonia, ashram, sarajevo, upn, goalscor, indycar, burnley, bergen, bronx, astro, wildcat, repertori, shipyard, doncast, barnsley, hawkey, wakefield, defenceman, aaa, swansea, strikeout, wb, andalusia, hq, d{\'e}but, fulham, annapoli, storey, mma, sichuan, fjord, firstround, knoxvil, kickbox, pagoda, heineken, baden, basel, longdist, promenad, dunde, panhandl, azor, captainci, dortmund, panchayat, exet, siriu, waterford, perth, condor, croydon, pontiac, nanj, gothenburg, regularseason, mekong, sumo, byu, nit, krak{\'o}w, auburn, islet, ajax, eureka, newark, lisbon, peke, watford, simulcast, zion, himalaya, oiler, barnet, steamboat, nassau, bash, sunda, mainz, godavari, armori, hampton, fai, caledonian, nyc, rt{\'e}, dor, wilmington, soho, hobart, calgari, aspen, sorti, westchest, coeduc, doha, shutout, galile, alsac, windsor, albani, olympiaco, bordeaux, bsc, oncampu, highestgross, lucknow, airbas, westfield, rotterdam, trolley, transvaal, vanderbilt, regroup, nyu, spruce, jakarta, landfal, superson, oricon, allegheni, disus, saracen, skylin, nautic, westbound, wharf, nl, casablanca, mesa, fiesta, antarctica, hammersmith, louisvil, pacer, wadi, fuji, twentysix, pba, judo, quezon, freshmen, amphitheatr, sevilla, albuquerqu, halfhour, stepp, gloucestershir, avalanch, southampton, raffl, panathinaiko, badminton, dhaka, novi, aegean, li{\`e}g, collieri, cfl, dinamo, barclay, cheshir, bali, colchest, wessex, akron, crosscountri, toledo, blackpool, parkland, juventu, augsburg, massif, calcutta, allaround, macau, parma, midatlant, caspian, postworld, eindhoven, hanov, pune, leyt, cayman, wba, marriott, fl, airlift, peripheri, merseysid, meridian, papua, halifax, blitz, cbd, seaport, glendal, kota, cyclist, stalingrad, kingston, malibu, ballpark, outfield, purposebuilt, leipzig, worcest, expo, lufthansa, eri, guantanamo, euphrat, fir, oriol, golfer, copenhagen, barg, amherst, uptown, bison, carpathian, aba, dartmouth, gloucest, peterborough, eastbound, aero, vodafon, voivodeship, leinster, sprawl, smoki, loch, steamer, peninsular, sinojapanes, nave, dunk, caldera, palisad, concordia, secondlargest, diamondback, prep, collingwood, ghat, skier, gala, fiba, loir, wellington, kyoto, sixday, bermuda, golan, valencia, cumberland, southernmost, blackhawk, aerodrom, walkway, fremont, danzig, streetcar, tacoma, midseason, highris, spree, scarborough, oclock, omaha, highschool, timor, allireland, albion, zee, lafayett, matricul, bayer, elgin, wrexham, piccadilli, aqueduct, caraca, homecom, tasmania, caravan, bournemouth, sill, stoppag, dweller, yeshiva, transcontinent, yangtz, timeslot, unc, distilleri, sellout, weeknight, midsumm, avro, poli, whitehal, catchment, sk, cavali, saskatchewan, worcestershir, electrifi, suez, warwickshir, oslo, piedmont, middlesbrough, hilli, belmont, watersh, manitoba, gotham, shetland, antil, selangor, rodeo, hilltop, sabr, bucharest, bethlehem, stockholm, midtown, fb, tundra, aberdeen, severn, potomac, churchyard, pomerania, kany, concours, oceania, norwood, comcast, wildcard, alta, ballroom, pennant, stamford, grandstand, subdistrict, realign, hove, sevil, dockyard, clipper, overtook, firstyear, norfolk, hostel, chatham, pdc, lancast, sioux, dorm, ovat, daytona, semifinalist, baylor, ave, odessa, bastion, bandar, nowdefunct, fairfield, condominium, uninhabit, avalon, johannesburg, faro, euroleagu, benfica, beirut, guam, buckinghamshir, lazio, singleseason, shandong, alaskan, cod, portsmouth, phra, millwal, concord, nairobi, buccan, napoli, utrecht, boca, threetim, moat, jaffa, wbc, bookstor, pretoria, fairfax, eton, monorail, riversid, aurora, uppsala, jutland, spokan, bukit, grenadi, shenzhen, kimberley, lockout, greenfield, platt, greyhound, ipl, himalayan, upscal, skyscrap, shrewsburi, wineri, shorelin, slough, monza, bluff, saigon, helsinki, grassi, warmup, pyrene, unincorpor, arlington, feb, d{\"u}sseldorf, charleston, overland, waterloo, corinthian, stockton, twentyfour, mannheim, nohitt, penultim, madurai, algier, handbal, borussia, soldout, bobcat, canari, cork, yosemit, mare, coliseum, equatori, snooker, corinth, lexington, bridgehead, yukon, anzac, nw, addi, stratford, zagreb, dover, danub, isthmu, ravin, tiebreak, guernsey, chesterfield, quay, pregam, okinawa, inducte, cumbria, tuft, kandahar, winnipeg, portico, brandenburg, afb, verd, palo, chestnut, guildford, cska, northampton, wesleyan, lookout, argo, downhil, causeway, lans, natal, undisput, stockport, elm, porto, bonn, naia, concacaf, parramatta, coyot, bologna, strasbourg, knick, fargo, italia, tyne, breaker, spitfir, gma, az, badger, falkland, tvb, budapest, harlem, enfield, beaufort, threeway, anglia, devon, dynamo, canadien, oasi, triplea, christchurch, rada, tehsil, raaf, aachen, racecours, carmel, tbilisi, comanch, dayton, nottingham, sunris, rampart, bilbao, yellowston, agra, chichest, gaa, rochdal, marlborough, cougar, snowi, asiapacif, arrondiss, champ, mcc, brighton, psv, purdu, oxfordshir, inelig, constructor, disneyland, maroon, sein, fenway, rutger, pembrok, fia, bromwich, carleton, tko, roundabout, aisl, bedford, eurasia, tx, foothil, usaaf, staffordshir, redwood, riyadh, salisburi, transatlant, citadel, calder, derbyshir, darlington, huntington, deccan, busiest, rerun, tianjin, ahl, kindergarten, mexicanamerican, cheerlead, essex, gmbh, outli, harrow, trafalgar, allstat, amtrak, charlton, woodstock, elb, pasadena, pullman, amman, yunnan, professorship, scoreless, courthous, marquett, burlington, everest, hertfordshir, toulous, munster, ecoregion, kensington, ymca, sv, jfk, nant, bye, marlin, andean, cdp, bloomfield, malm{\"o}, havana, allahabad, confluenc, m{\'a}laga, tulsa, nsw, escarp, buena, sussex, circumnavig, reno, tramway, ganga, bakeri, topten, stamped, chittagong, montevideo, agglomer, standout, hartford, buckey, coventri, crossroad, clubhous, bangkok, westwood, idaho, surat, appalachian, lima, multipurpos, firstev, leiden, southbound, middlesex, nouveau, bari, glee, picturesqu, siemen, phnom, waiv, tuni, gymnasium, oldham, shire, kremlin, australasia, lago, westernmost, syracus, hanoi, cebu, lincolnshir, tyrol, paddington, tee, mercia, whaler, welterweight, northsouth, rooster, hc, trenton, penitentiari, lausann, longhorn, huddersfield, ahmedabad, tuscani, guadalcan, slum, zoolog, marseil, cafeteria, orkney, wnba, mainlin, lighthous, kilkenni, roanok, osaka, sidewalk, lancashir, macquari, ashor, rerout, mindanao, feyenoord, hurl, homestead, brunswick, trafford, sahara, alto, w{\"u}rttemberg, raze, equalis, puma, fife, ecw, plymouth, seneca, mohawk, huron, porch, oneday, stagecoach, govt, matchup, scrimmag,
		\par
	}

	\item[$\mathcal{V}_{5} = $]
	{
		\tiny

		vain, misconduct, glider, practis, solitari, traitor, vengeanc, vanish, pertain, safeguard, remembr, subdu, axiom, salut, nobodi, bump, judgement, gentl, deaf, restart, runaway, eighti, smuggl, heroic, sneak, trooper, timelin, smile, bulli, exemplifi, allud, sheer, bb, thiev, breeder, immers, bp, veto, noon, sympathi, realism, lifelong, phantom, thor, cosmolog, karat, sermon, tailor, hesit, anecdot, ch, wive, rational, liken, defi, straightforward, deadli, reincarn, perpetr, risen, creditor, curb, achil, imper, multimedia, token, vandal, penni, sb, tremend, dilemma, rhyme, terribl, spoof, bail, loot, forgiv, redempt, arbitr, lan, banish, spartan, repuls, conjectur, genealog, pill, blur, semin, donkey, ci, awaken, regret, veil, taxpay, landslid, ark, bizarr, newest, ge, exagger, pentagon, pal, hr, groom, inhous, masculin, atroc, wholesal, midst, entireti, distract, unexpectedli, purg, embarrass, graffiti, isnt, homag, undertook, revert, onscreen, rand, prop, acquit, stab, mess, summon, slayer, pre, nonexist, jade, contempt, dissemin, twothird, reallif, cannib, racist, bred, clown, probat, caution, wick, brilliant, ll, nude, evok, firework, roam, inact, messeng, gotten, weird, rm, revisit, theyr, somebodi, incap, hallmark, unwil, tori, ineffect, homeless, calm, pretend, bias, diversifi, extant, problemat, unpreced, ts, wouldnt, motto, allot, transcend, pragmat, sincer, remodel, spectacular, guilt, funni, cheat, obey, undercov, banknot, flock, notifi, quotat, somehow, vener, repent, colloqui, mayb, auster, mega, anymor, grate, dalek, nowaday, sharpli, trump, annihil, epa, amin, ware, unto, novelti, melod, atari, wield, invis, courag, saloon, atheist, onethird, decim, surprisingli, disastr, starv, swear, eager, countdown, inadvert, outrag, belov, cosmo, disclosur, ki, dividend, highprofil, agit, disapprov, promptli, talmud, await, halo, proxi, justif, inmat, onair, astonish, deduc, drunk, cp, parcel, thrill, reassign, dull, satan, feloni, plausibl, mural, ec, eccentr, espionag, sensibl, transgend, yell, twentyf, yahoo, rever, unfair, unus, wherein, connot, eras, deserv, erad, storytel, coven, couldnt, sabotag, optimist, taboo, restless, ve, ni, elderli, vow, shame, supposedli, intuit, countless, jealou, sunshin, inappropri, nurseri, sorri, rage, hamper, slaughter, pharaoh, samurai, psychic, comed, hadith, plaintiff, katrina, clue, fcc, scarc, firstli, paus, thou, excerpt, incompet, unaccept, demolit, thoroughli, hijack, tomorrow, annoy, bisexu, marketplac, faint, misunderstand, burger, {\textellipsis}, profici, erot, spotlight, gunner, indulg, cr, transcrib, secreci, stir, avatar, dear, humbl, signifi, psychedel, salon, feat, lure, secretli, insult, curios, newborn, infin, peculiar, worthi, fx, suspici, rhythmic, sexi, wherev, ascent, ka, startup, endeavour, excus, flashback, viciou, foe, propheci, chronolog, mistaken, gag, sacrif, fundrais, hello, injustic, popularli, treason, blunt, revamp, immort, stakehold, bloodi, clever, awkward, liabl, puriti, whisper, compliment, {\textsterling}m, reconcili, imperson, heighten, unpopular, pascal, zeu, oversaw, supersed, ridicul, drone, plc, skateboard, abid, mock, werent, denial, yoga, famous, seldom, socrat, slash, est, rc, disposit, em, gi, motown, ugli, stylist, disqualifi, weaker, charm, ari, metaphys, makeup, messiah, dell, es, wont, allah, commend, societ, triad, endless, metacrit, thwart, rpg, au, thief, ra, serpent, decisionmak, disintegr, mitsubishi, whoever, viewpoint, receipt, seduc, credenti, merci, cruel, subsid, indict, afflict, ambit, altogeth, occult, frighten, roleplay, ah, scrutini, spars, va, mistakenli, highlevel, ai, adulthood, workforc, bust, fanci, contractu, hadnt, arguabl, overturn, skater, exposit, broaden, foremost, ambiti, arent, dec, incorrectli, gentlemen, interf, banquet, scarlet, porn, abduct, amid, imperfect, batch, slim, acronym, obscen, bunch, soror, perish, plato, notori, everybodi, blackmail, fist, brothel, satisfact, torah, perfectli, ae, meantim, si, chariot, nc, prehistor, buffi, landlord, forthcom, amc, persona, insan, bounti, deceiv, fiat, verdict, various, amidst, evad, tran, pleasant, reappear, lucki, reintroduc, prelud, archetyp, discern, enumer, marg, hercul, rework, afraid, lefthand, loser, quietli, menac, fresco, lgbt, crippl, unnam, criterion, coma, stupid, oneself, bestow, interrog, youv, strive, gestur, weve, beneficiari, proud, holocaust, paradox, sporad, amass, neglig, postmodern, arrog, ransom, hinder, hack, devote, loud, marit, persuas, nice, inflict, bedroom, noteworthi, daddi, terrifi, compass, hungri, decept, femin, impend, adversari, suspicion, litig, destini, unreli, scare, harder, medicar, evidenc, chiefli, freak, lunch, repel, robberi, stipul, fragil, entail, amnesti, fraudul, ruthless, alik, hatr, mod, reclaim, freud, mourn, vigil, asleep, shout, reconcil, babylon, gangster, kant, importantli, rumour, therapist, demis, orphan, riski, despair, bt, shi, drunken, sad, jealousi, bowler, python, geniu, obvious, daylight, password, aforement, sub, indefinit, standalon, resent, aristotl, undergon, casual, handicap, implicit, injunct, urgent, goodby, simplic, keeper, attest, intimid, nonsens, pend, unconstitut, wipe, quarrel, rfc, guess, swimmer, jehovah, outlook, millionair, useless, tolkien, righthand, vintag, negat, hisher, disgust, tsunami, umpir, skeptic, provoc, bargain, mythic, rewrit, pivot, relic, unconsci, eyewit, rid, backdrop, atlanti, intensifi, dislik, clarifi, apocalyps, humili, dispens, interpol, noir, sacrament, erron, turnov, sp, masterpiec, aggrav, seventi, prematur, halloween, cremat, bearer, incit, parol, versatil, akin, retrospect, opium, climax, discrep, quota, willing, broker, vest, enrag, fond, haunt, samsung, vagu, magnific, psalm, grasp, digniti, ng, dracula, wane, sympathet, overli, nowher, ostens, allus, slang, ufo, homicid, formid, boycott, instantli, intim, xmen, karma, comrad, unawar, conscienc, slap, detriment, offspr, energet, astrolog, anyway, havent, asylum, repay, lucr, purana, obsolet, hardship, setback, incident, certainti, voluntarili, bold, unlaw, ar, plead, reminisc, resurg, memorandum, hereditari, deceas, hare, customari, myself, nineteen, confisc, mom, ok, cocacola, mileston, vivid, overnight, isi, fabl, tragic, obedi, autobiograph, magician, rude, dad, circumcis, narcot, min, leverag, unfold, foul, hoax, decent, preexist, ought, honest, obsess, wise, dissent, invok, drown, inscrib, postul, ex, vigor, irrelev, infiltr, ninja, nonstop, discredit, famin, brutal, imaginari, sting, lender, characteris, fascin, memo, adolesc, temptat, plea, ultra, luck, cs, underway, bodyguard, ourselv, onstag, bankrupt, ir, muse, defer, breakup, salvag, shooter, mao, relianc, refut, incarcer, salvat, lengthi, firmli, envis, rogu, redeem, slogan, shatter, ea, reluct, everywher, eleg, underworld, theoriz, succumb, misus, immin, celesti, blockbust, ancestr, notorieti, damn, bandit, re, utilis, learner, covert, nineti, costli, sage, undisclos, outright, coincident, cruelti, hedg, usher, hed, dharma, mighti, cohes, mankind, allegi, unfinish, mutant, gon, anybodi, rode, spous, yearli, pawn, gentleman, rubi, retali, furnish, caller, gypsi, sorrow, panic, absurd, finest, nonviol, inaccur, secondli, unavail, seeker, unlock, confidenti, stranger, doom, hound, forbid, wan, farewel, reel, instinct, mislead, outspoken, invalid, purport, ape, hardli, summar, yourself, shaman, stanza, oracl, sake, dare, heterosexu, lament, tempo, forcibl, tu, vegetarian, phonet, buzz, uncov, incred, outlaw, furiou, maya, musket, ego, idl, iri, pornographi, dismay, sd, delight, curiou, unlimit, deed, baptism, crook, ti, indiffer, censor, archaic, gossip, relaunch, heavenli, empow, corps, ptolemi, poetic, wellb, ascrib, abruptli, codenam, restrain, inquir, randomli, cobra, hint, intoler, shelv, cop, batsman, keen, handsom, stuff, constel, restraint, bribe, fool, mob, ador, marijuana, commodor, prejudic, endeavor, wrath, reinstat, longstand, emphasis, prestig, tb, enthusiasm, miracl, horribl, grief, bother, stun, lent, merced, coexist, feminin, stole, corpu, refrain, illus, oh, contempl, skip, umbrella, fuck, hunger, authoris, pardon, incompat, intrigu, libel, gym, forgotten, soar, audi, verg, crise, supernatur, porsch, par, aint, pray, humour, suzuki, utter, paragraph, deter, censorship, firefight, disregard, prefac, nightmar, exodu, unexpect, ko, tempera, uncomfort, pointer, forget, conspir, amen, furi, unhappi, evict, beg, ya, swap, coffin, thorough, recogniz, solicit, overhaul, hobbi, aw, spoil, reiter, infam, stunt, ta, prank, disguis,
		\par
	}

	\item[$\mathcal{V}_{6} = $]
	{
		\tiny

		cartel, guerilla, huntergather, nazism, sparta, stronghold, baptiz, cleans, sloven, planner, detaine, quo, bah{\'a}{\'\i}, jat, kurdish, unpaid, parthian, swede, orphanag, authoritarian, colonis, gestapo, extermin, moravian, militar, abolit, sicilian, unicef, annal, assyria, nepali, signatori, hama, safavid, royalist, academia, reorganis, cornish, islamabad, sabah, acced, seleucid, chaplain, xiongnu, alqaeda, yoruba, anatolia, demarc, grassroot, vehement, nagasaki, expeditionari, abolitionist, cree, epiru, aleppo, nkvd, indoaryan, plunder, kazakh, hispania, nationalis, d{\'e}tat, vicar, upbring, breakaway, ghanaian, mesopotamian, angloamerican, suffrag, anarch, bicamer, bourgeoi, goodwil, aristocraci, gentri, rajasthan, iroquoi, priori, orthodoxi, auditor, paratroop, kgb, malacca, amerindian, galician, extremist, charismat, milo{\v{s}}evi{\'c}, joseon, automak, rwandan, paraguay, curia, huguenot, xinjiang, oriya, uzbekistan, wehrmacht, chechnya, bantu, chairperson, celt, offshoot, montenegrin, dissid, zulu, hezbollah, mesoamerican, unrest, samaritan, maori, lakota, fianna, entrench, druze, eucharist, scandinavia, privatis, parliamentarian, frisian, prerog, gupta, mayan, uncondit, johor, hannib, austriahungari, gibraltar, freedmen, downfal, extradit, insular, codifi, arafat, protestant, burgundian, tasmanian, rhineland, carthag, transylvania, silla, caretak, ankara, franciscan, brethren, herodotu, palatin, guiana, truce, yiddish, confucian, mestizo, francophon, intergovernment, hegemoni, enshrin, senatori, somali, naacp, hutu, nizam, eurasian, liberia, peerag, oman, imf, antioch, politburo, staunch, kashmiri, iit, espous, peso, akkadian, arian, uzbek, bedouin, oust, algerian, banu, aeronaut, estonian, gnostic, junta, outnumb, jiujitsu, seced, auschwitz, frankish, germanspeak, guyana, namibia, conven, esoter, puritan, pinyin, unitarian, briton, overthrew, bureaucrat, dacia, basra, crackdown, precolumbian, overrul, azerbaijani, hittit, heraldri, meiji, modernis, priesthood, mutini, kuomintang, launder, outreach, seljuk, aryan, romani, transliter, meteorolog, caucu, assad, heret, tantric, jurisprud, rescind, sumerian, ascens, bohemian, claimant, barbarian, athenian, pretext, apartheid, secess, turmoil, sinhales, multiparti, privi, hakka, strife, achaemenid, weimar, lahor, bolshevik, gallic, benin, appeas, elizabethan, acadian, canaan, gentil, sizeabl, habsburg, seneg, creol, defam, carthaginian, moldavia, burgh, unilater, ombudsman, sukarno, maltes, recaptur, bohemia, burmes, masjid, ngo, regenc, capitul, edo, peacekeep, cadr, singaporean, pali, demot, tunisian, preemin, edict, monast, manchu, czechoslovak, catalonia, diplomaci, englishspeak, rhodesia, ensign, spaniard, knesset, defianc, intermarriag, schism, unoppos, subjug, hagu, andalusian, policemen, nordic, subcommitte, goth, barbado, belarusian, fledgl, consular, majesti, seafar, croat, coinag, noncommiss, iberia, malawi, khmer, sharia, haryana, servicemen, ceylon, archeolog, eunuch, selfgovern, antislaveri, subcontin, counterterror, bihar, artisan, riviera, plight, ceasefir, catalan, roug, malaya, hardlin, devout, zen, ultimatum, mistreat, odisha, dragoon, airmen, nobleman, turkmenistan, solidar, bylaw, slovenia, kabul, dalit, paraguayan, warlord, inuit, darfur, tort, imposit, zoroastrian, rector, bosnian, serf, rupe, seventhday, instig, bavaria, gop, expatri, rumbl, entrust, iconographi, niger, dynast, jacobit, excommun, taoist, unitari, babylonian, benevol, ru, lawmak, latvian, gaddafi, notari, paleolith, censur, khyber, historiographi, austrohungarian, marxism, hokkien, tyranni, postsecondari, populist, fishermen, bolster, bhutan, nuremberg, sarawak, constitution, liturg, reformist, subsaharan, cornerston, individualist, sui, semit, travancor, mamluk, somalia, fiji, taiwanes, byzantium, baroni, separatist, zambia, disarm, ravag, zanzibar, kazakhstan, hillari, multin, civilis, devolv, chalukya, leningrad, roc, seamen, mongolian, nonchristian, heresi, kali, sindh, lowincom, hajj, gubernatori, nicaragua, vilniu, quaker, grievanc, uyghur, swahili, assent, mujahideen, postcoloni, quell, prefect, wealthiest, patrician, mon, belaru, igbo, suharto, tripartit, grenada, incurs, spearhead, governorgener, saxoni, hellen, nco, hussar, silesia, overrun, charlemagn, sindhi, botswana, lombard, revit, baton, gurkha, bipartisan, nonmuslim, sabha, romantic, ssr, viceroy, gregorian, kiev, overseen, mahayana, vassal, pacifist, slav, traditionalist, shogun, manchuria, adventist, togo, ministeri, byelect, thebe, oecd, hun, uboat, leftist, overthrown, highrank, envoy, augustinian, macedon, primaci, taoism, javanes, breton, ioc, forerunn, thrace, loyalist, sworn, jihad, orthographi, planter, prehistori, dictatorship, ugandan, mali, damascu, frenchspeak, selfdefens, dogma, emissari, caucasian, wwii, m{\={a}}ori, calvinist, gloriou, reichstag, angola, bourbon, cypriot, ismaili, guatemalan, basqu, distrust, syncret, tutsi, manageri, pashtun, bangladeshi, imperialist, admiralti, judah, bjp, anglo, pacifi, indoeuropean, sami, fief, justinian, mauritiu, bishopr, prc, kurd, moro, geologist, frontlin, legat, relinquish, prewar, aceh, bosniak, tajikistan, retak, bonapart, oblast, sufi, supervisori, lobbyist, rhodesian, jain, paramilitari, outpost, impoverish, craftsmen, aztec, malayan, anc, rightw, fascism, sikhism, janata, slovak, afrikaan, navajo, choctaw, upperclass, mecca, landown, vichi, populac, gaul, saladin, bern, commanderinchief, mesoamerica, reactionari, despot, environmentalist, carolingian, jordanian, antitrust, depos, sassanid, fugit, pontif, archdioces, resettl, novgorod, spd, boycot, polynesia, flander, romanesqu, faa, appointe, fide, forbad, federalist, armada, ashkenazi, galicia, agrarian, uphold, siames, antiwar, vernacular, cloister, vanguard, enclav, industrialist, cognat, uae, hasid, bourgeoisi, bahrain, guatemala, baath, venezuelan, contra, nonpartisan, zionist, sixyear, scientolog, decentr, mep, sudanes, promulg, constabulari, disobedi, pompey, gunboat, moorish, businessmen, gujarati, encroach, fundamentalist, madagascar, shang, khanat, expuls, friar, eritrea, traine, mesopotamia, shinto, selfdetermin, mysor, papaci, balochistan, emancip, corsica, benedictin, crimea, tatar, karachi, anglosaxon, wto, insurrect, haitian, goguryeo, uttarakhand, skirmish, nicaraguan, brunei, undocu, yemeni, syriac, chechen, latino, sizabl, wight, kyrgyzstan, legitimaci, mandarin, cosponsor, solicitor, muster, piou, eurozon, secretarygener, tokugawa, nara, sephard, dravidian, rebelli, mozambiqu, inquisit, armistic, hmong, hellenist, mauritania, horticultur, florentin, cambodian, diocesan, dday, turnout, nascent, tipu, discont, boer, militarili, enslav, anatolian, naga, bolivia, maoist, bureaucraci, ascet, azad, stateown, aden, turkic, judea, venetian, flemish, surveyor, impeach, rajput, illyrian, moroccan, cham, vedic, repatri, barrist, moldova, judiciari, consecr, conservat, islamist, diaspora, latvia, jharkhand, renounc, medina, baku, sectarian, plebiscit, dal, politi, magnat, maldiv, flotilla, reunif, nepales, phoenician, abdic, pentecost, adjut, etruscan, mubarak, bolivian, candidaci, cochair, embargo, reaffirm, daytoday, anticommunist, poorest, jammu, pla, upheav, libyan, chairmanship, indentur, kmt, repar, checkpoint, siam, ghetto, cochin, goa, israelit, fatah, ratif, workingclass, bnp, maratha, overwhelmingli, counterinsurg, airspac, moot, ecuadorian, castilian, liechtenstein, crimean, repudi, mobilis, supremaci, mercantil, disciplinari, revok, unionist, presumpt, oversight, overlord, bavarian, pillag, jainism, congoles, influx, assyrian, demographi, autocrat, plo, coptic, whig, strategist, mla, affluent, peshawar, surinam, inca, shiit, catech, brahmin, sardinia, osteopath, orissa, ndp, tanzania, writ, staf, magyar, tenet, cantones, chola, polynesian, riga, moravia, hondura, philanthropi, ecumen, nasser, multicultur, kurdistan, gdr, scythian, sikkim, clandestin, leftw, rwanda, tunisia, kenyan, heartland, musharraf, multilater, peasantri, fluent, beliz, idf, pantheon, jamaican, ltte, messian, totalitarian, cameroon, middleclass, genoes, labrador, hiroshima, nongovernment, statewid, thracian, statesman, levant, visigoth, samoan, usurp, qin, desecr, crete, ionian, foothold, ardent, steward, umayyad, pogrom, zionism, reconstitut, cityst, ashoka, jurist, cossack, mausoleum, diocletian, assam, mennonit, auspic, cpc, richest, loanword, scandinavian, dalmatia, dacian, berber, battlecruis, peruvian, governorship, chieftain, iberian, tagalog, secretariat, shipbuild, gazett, abbasid, arama, malabar, myanmar, neolith, liaison, erstwhil, slovenian, pursuant, siberia, statehood, dignitari, indochina, circa, nguyen, registrar, toppl, synod, reassert, interwar, unarm, lordship, sympath,
		\par
	}

	\item[$\mathcal{V}_{7} = $]
	{
		\tiny

		tyre, pigeon, flora, mous, upright, chicken, groov, ant, tap, coconut, hind, knive, rig, curl, hat, corn, tooth, marbl, underwat, cooki, foil, lit, bud, wrap, beneath, snap, feather, cosmic, potato, pie, explod, timber, wolv, burrow, fuse, bread, eject, stuf, rack, plenti, gray, pepper, helmet, blown, deton, shade, calib, grill, owl, worm, straw, granit, badg, collid, deer, chrome, lace, pile, weld, outfit, tile, ribbon, chocol, miniatur, axe, tilt, beard, whip, monkey, haul, candi, sabbath, basement, bamboo, spike, pod, crest, cam, candl, nut, envelop, juvenil, bark, mint, brass, limeston, fog, shirt, powder, chees, cherri, skeleton, saddl, dot, garment, bolt, enclosur, dungeon, sandston, wrist, furnitur, salad, blond, axl, adorn, whistl, lotu, bake, duck, dune, pork, velvet, collar, canva, comb, peanut, lamb, weav, rip, fri, enclos, squirrel, locker, stuck, bracket, butter, click, blast, triangl, exot, camel, cart, robe, clutch, engrav, spray, brick, pyramid, pale, brace, strand, cream, pink, torch, carpet, mast, grip, scratch, chick, knife, aluminum, cab, hollow, underneath, darker, thrown, wasp, vent, jupit, rib, spider, bitter, bite, toss, stripe, dash, slice, flesh, jewel, larva, dwarf, omega, diver, pickup, leather, crab, lean, liveri, toe, rabbit, fauna, pierc, twist, lip, cement, flip, bucket, pencil, mat, shed, honey, dip, tongu, turret, halfway, burst, punch, deflect, coral, forg, brew, bead, stalk, dairi, swan, crocodil, motif, wash, comet, pizza, rainbow, wore, aquat, backward, nail, mosaic, jar, emblem, roast, orchard, piston, pedal, chi, crane, magnum, belli, violet, ceram, wax, arrow, strap, blossom, cow, bumper, beef, shake, flavour, maiden, shotgun, lever, sleev, goat, reef, herd, preciou, debri, alley, screw, bug, pant, tattoo, bent, mortar, shine, vault, pan, ginger, bee, beetl, mule, fuselag, hung, lemon, onion, kitchen, pad, lizard, liquor, rotten, coaster, jewelri, dial, crescent, trim, claw, recip, banana, pot, curtain, smell, crawl, hood, replica, fade, scrap, plaqu, blank, bloom, pulp, flour, fed, mud, butterfli, dust, wheat, purpl, swallow, peel, wool, willow, mapl, poni, chin, spear, tan, noodl, gravel, crust, soup, cockpit, sandwich, scroll, pig, dye, sheep, eclips, botan, leap, fin, lantern, tini, textil, thread, butt, dinosaur, dig, basket, oval, trouser, header, mushroom, lightweight, lightn, illumin, fender, nickel, void, sock, ore, roller, chop, discard, ink, sauc, turtl, swing, stain, atp, frozen, rectangular, shave, diagon, grab, ivori, glove, pour, herb, ornament, shower, insignia, sox, slip, trout, elbow, blanket, exterior, ceil, cigarett, asteroid, brush, gem, venom, pigment, shark, ladder, drift, juic, stair, lime, dump, shaft, horizon, mold, dug, glow, spiral, sour, torn, skate, cheek, pea, silk, leaf, spun, sink, capsul, rim, flush, floyd, dessert, thumb, jacket, heel, accessori, cane, bounc, potteri, dirt, leopard, grape, rope, fur, launcher, microphon, stitch, choke, liner, mantl, bean, moth, tomato, nake, cage, cone, tin, frog, eleph, hammer, grenad, crimson, reptil, thunder, amber, bathroom, chili, tshirt, tight, hut, eaten, needl, batter, cube, har, warp, cutter, throat, hatch, spice, vine, toilet, burnt, plug, hook, salmon, retract, cake, bubbl, skirt, balloon,
		\par
	}

	\item[$\mathcal{V}_{8} = $]
	{
		\tiny

		devi, lester, wrestlemania, lili, calvin, stephani, freddi, akbar, sandra, agn, kathi, humphrey, philipp, mickey, doc, liu, vishnu, ronald, gil, stevi, patern, dant, jami, ned, rudolf, anton, piu, melissa, rao, col, winston, louis, clive, bonni, saddam, andr, gu, seth, wang, aaron, liam, valentin, edmund, patti, isaac, paulin, wu, brotherinlaw, maggi, judi, sue, lionel, brad, doug, shannon, darren, clarenc, randi, tina, jeremi, kyle, ronni, niec, marvin, hermann, joel, cum, stan, trevor, geoffrey, hassan, betti, nina, madhya, jeffrey, gustav, mahatma, pierr, horac, nicol, wong, trent, sharon, lynn, zhou, holli, ludwig, hulk, wolfgang, eleanor, dee, swami, benni, emma, rachel, miranda, tel, leigh, eugen, elton, bo, mo, basil, mohammad, xavier, yu, yi, rama, hannah, sherlock, clement, timothi, halfbroth, bryan, bori, marilyn, erik, edwin, wei, shirley, grandson, nephew, kirk, indira, allison, anni, friedrich, shane, hal, shri, rupert, sidney, kati, chiang, mauric, archi, luci, isabel, travi, paula, helen, omar, alexandra, ernest, tai, karen, buddi, maid, herbert, johann, isabella, qb, sy, gregori, marc, bernard, marion, barack, jennif, andrea, glenn, bon, barney, butcher, katherin, leonard, jenni, vanessa, laurel, dalai, prof, ahmad, kitti, jess, cao, vladimir, jacki, ahm, lesli, marcu, daisi, sophia, lyndon, patricia, lil, dana, jessica, boyfriend, amanda, marti, felix, alfonso, jo, christin, constantin, pratt, laurenc, sonni, wilhelm, debbi, shawn, chen, joan, emili, sheikh, jerom, perci, ethan, sen, conan, mama, edgar, reverend, kurt, befriend, dorothi, diana, fu, dale, vic, lauren, ashley, kicker, bruno, mose, fritz, mick, dwight, sara, alia, noel, augustin, uttar, dexter, ernst, cowritten, heather, byron, franki, robbi, josh, dudley, guru, mikhail, theodor, julia, matern, kumar, heinrich, brett, malcolm, teresa, abdul, clair, vernon, christina, bing, judith, brendan, granddaught, randolph, ellen, jin, glen, benedict, te, petersburg, loi, janet, sebastian, laura, baba, claud, raj, donna, clyde, raymond, liz, wendi, florenc, nichola, lindsay, rev, gerald, woodi, ho, lok, magnu, juliu, ivan, jule, leopold, mel, ren{\'e}, tottenham, romeo, madam, teddi, abdullah, salli, grandmoth, angela, trinidad, cal, carol, cyru, fr, allan, vinc, moham, fran{\c{c}}oi, empress, jacqu, brittani, notr, stella, noah, jake, kenni, congressman, maj, nathan, martha, milton, consort, herman, vincent, joey, seymour, walli, und, nigel, molli, eva, kapoor, nicola, abba, peggi, gerri, zhang, adrian, lt, ibrahim, cecil, mozart, jill, kenneth, colin, rodney, julian, hey, sid, conrad, olivia, krishna, nanci, ricki, brandon, imam, neal, raja, adolf, derek, goldman, joshua, lou, alma, linda, antoni, eli, otto, augustu, gloria, lanc, leonardo, sophi, rita, franz, beth, roland, kuala, ruth, dian, chad, fanni, rex, carolin, andhra, hank, mistress, rebecca, lin, natali, bart, traci, frontman, tara, catherin, geoff, n{\'e}e, ferdinand, helena, maharaja, elvi, yang, aunt,
		\par
	}

	\item[$\mathcal{V}_{9} = $]
	{
		\tiny

		dealer, mortgag, agenda, forbidden, embodi, lab, propon, whenev, medit, compli, behav, credibl, deficit, organiz, gambl, durat, bilater, discrimin, heroin, advocaci, penal, portfolio, scholarli, cheap, rhetor, overview, abort, uncertain, biblic, comprehens, uncommon, pursuit, sociolog, depriv, abstract, strictli, sentiment, perpetu, rehabilit, inclus, proven, inspect, anonym, monetari, self, bigger, identif, pronunci, prevail, strict, pleasur, ordin, escal, sudden, incorrect, formul, implic, firearm, tender, justifi, articul, dictat, judgment, abbrevi, relax, conjunct, liabil, terminolog, proposit, retriev, augment, shorten, overlap, weaken, traffick, lifestyl, statutori, imit, legitim, contradict, reliev, curriculum, bia, monopoli, proce, deem, antisemit, specialti, ideolog, contrari, placement, habit, stanc, conceptu, restructur, albeit, hierarchi, voluntari, specialis, loyalti, humanitarian, theft, copyright, etymolog, ambigu, discours, setup, immens, privaci, inconsist, classroom, metaphor, endur, methodolog, synonym, remedi, authent, silenc, simplifi, taxat, intact, alarm, procur, conspiraci, disclos, feasibl, steadili, vital, incomplet, wholli, verifi, workplac, plagu, maxim, appreci, norm, reward, infring, constraint, dealt, concurr, tough, compel, criteria, assumpt, homosexu, imageri, fratern, critiqu, manifest, omit, endang, racism, inabl, predomin, lineag, adher, anticip, humor, complianc, vocabulari, quran, complement, expenditur, fulfil, correctli, diminish, strongest, harsh, broadli, feminist, peer, profound, mediat, nonetheless, puzzl, eas, accordingli, modest, explicit, openli, flaw, partit, sophist, artifact, cope, practition, unrel, adequ, donor, claus, seemingli, forecast, spite, disagre, irregular, deepli, inher, hypothesi, largescal, chemistri, crucial, confin, fiscal, guidanc, aspir, obscur, realist, convey, frustrat, absent, breach, outlin, buyer, offenc, disagr, wisdom, postal, submiss, conform, royalti, compromis, extraordinari, obviou, merit, broader, healthi, properli, overcom, stereotyp, prioriti, systemat, affirm, quiet, chao, encompass, undertaken, capitalist, logist, aesthet, analyz, rigor, charit, poorli, scenario, healthcar, adject, neglect, provok, repress, astronom, segreg, oppress, verb, essenc, racial, guidelin, explicitli, deterior, fraud, enlarg, distant, collector, deduct, pace, buddha, steadi, autonomi, government, disadvantag, burden, alert, fare, offend, exempt, compulsori, wors, tendenc, trait, enorm, enlighten, noun, discourag, wartim, advent, singular, fault, accent, astronomi, everyday, mandatori, freeli, visa, insight, genuin, harass, assur, harmoni, overwhelm, primit, scope, obstacl, heal, premis, regardless, underw, categor, aros, unclear, verbal, boost, lend, percept, non, plural, wherebi, conscious, likewis, expertis, geolog, tenant, inevit, uniti, sphere, anthropolog, trademark, necess, inventori, incent, undertak, regulatori, assimil, virtu, conceal, moreov, prescrib, profess, consciou, exam, forens, registri, iso, pharmaceut, clone, embrac, devis, consensu, undermin,
		\par
	}

	\item[$\mathcal{V}_{10} = $]
	{
		\tiny

		sol, roo, {\`a}, libertador, rivera, barrio, dauphin, carmen, flore, qu{\'e}bec, revu, javier, alessandro, roi, iglesia, lope, f{\'e}lix, rodr{\'\i}guez, alfredo, gran, avant, je, le{\'o}n, p{\'e}rez, banda, fran{\c{c}}ais, provenc, ain, rancho, willem, pont, argentinian, sarkozi, pe{\~n}a, oro, {\'a}ngel, khomeini, marqu, sul, allend, salazar, davao, silvio, chico, mort, delgado, claudio, blanc, antoin, maestro, ni{\~n}o, salina, cid, ole, international, brasil, universidad, c{\'o}rdoba, enrico, navarro, navarr, varga, val, tito, guadalup, banco, mariano, jaim, vila, paolo, c{\^o}te, benito, guadalajara, nord, garibaldi, bam, vittorio, sergio, castillo, qaeda, {\'e}cole, bravo, jardin, witt, hustl, moreno, molina, catalina, rey, comt, batista, serra, rochel, parc, libr, julio, gael, ferrer, bernardo, mend, dio, ortiz, sant, veracruz, am{\'e}rica, estadio, historia, luna, ernesto, vill, eduardo, campo, angelo, espa{\~n}a, pietro, cerro, teatro, oaxaca, laguna, carrera, emilio, vasco, ignacio, opu, per{\'o}n, haut, toro, toma, lombardi, hern{\'a}ndez, terr, marcello, ricardo, laci, siena, gonz{\'a}lez, ruiz, deportivo, una, casa, coco, puebla, pico, jong, rossi, estrada, chavez, ju{\'a}rez, tarantino, aux, santana, bella, dei, capon, g{\'o}mez, fern{\'a}ndez, loma, grupo, padr, ra{\'u}l, nacion, dal{\'\i}, vita, vizier, gonzaga, lobo, quentin, ramo, roch, m{\'e}xico, temp, ramirez, della, guerrero, paso, ivanov, blanco, alvarez, asturia, vin, mata, s{\'a}nchez, mina, stefano, pueblo, khalifa, boi, laval, mal, r{\"a}ikk{\"o}nen, ciudad, gard, garc{\'\i}a, alonso, yankov, c{\'e}sar, zaragoza, ch{\^a}teau, sur, guillermo, domingo, nuevo, ram{\'o}n, ronaldo, francesco, herrera, ou, sera, dia, martinez, mendoza, joaquin, cort, tijuana, arroyo, ayatollah, giorgio, isla, montoya, leyland, aragon, yo, chevali, saba, fontain, sanchez, ferrara, mart{\'\i}nez, bel, novo, castil, alejandro, piero, para, canto, aquino, arturo, luigi, messina, pinto, marqui, ligu, gore, federico, romero, dino, mateo, gambino, stade, lair, scala, centro, quito, divoir, museo, guillaum, rodrigo, vida, telenovela, salvador, rizal, nueva, mussolini, palazzo, alamo, por, mond, duchess, national, dom, divisi{\'o}n, maccabi, trujillo, santand, dolor, ateneo, borg, vicent, verdi, diablo, fray, amor, rue, sonora, vie, fernand, palai, alba, bol{\'\i}var, samba, aguinaldo, bahia, mayo, primera, femm, felip, hidalgo, cali, cabrera, cort{\'e}, torino, jazeera, soto, coutur, jo{\~a}o, nort, que, viva, tre, gallo, nadal, louvr, como, r{\'\i}o, d{\'\i}az, mart{\'\i}n, monterrey, fernandez, paz, su{\'a}rez, lac, greco, mus{\'e}, massa, cesar, enriqu, rosario, soci{\'e}t{\'e}, renn, vall, ponc, giusepp, l{\'o}pez, fontana, chanel, conquistador, piazza, ch{\'a}vez, cristina, picasso, porta, croix, lux, saud, gonzal, acad{\'e}mi, mora,
		\par
	}

	\item[$\mathcal{V}_{11} = $]
	{
		\tiny

		roadway, dine, luxuri, unveil, excav, travers, grove, fring, countrysid, sedan, subspeci, harbour, convoy, bend, ridg, trench, thrive, closur, builder, ambush, fortress, frigat, java, voyag, meadow, renault, pipelin, tanker, att, pave, escort, coastlin, ski, leisur, strait, steep, highland, fountain, perimet, beaver, sm, aerospac, downstream, shelter, scenic, junction, gorg, trunk, bunker, usaf, rebuilt, cedar, ferri, inland, portal, toll, pedestrian, northward, alpin, marsh, subdivid, loung, torpedo, tent, intersect, uss, detach, expressway, pt, plaza, greenland, bangalor, sunk, hemispher, aboard, tractor, freestyl, terrain, raf, boom, cafe, ambul, antitank, lagoon, swept, sniper, caf{\'e}, wreck, ca, il, boe, breweri, wilder, antiaircraft, interst, refineri, toyota, wetland, canyon, cascad, hm, refurbish, fork, armament, observatori, dwell, smallest, bs, demolish, pier, chennai, cruiser, surg, motorway, taxi, waterway, racer, tram, nokia, zoo, aa, pavilion, volcano, lawn, tributari, paradis, palm, ramp, bypass, hike, vineyard, flew, inn, rebrand, harbor, baltic, mk, thame, anywher, warehous, honda, shelf, nightclub, mig, hamlet, pub, fortif, winchest, oak, upstream, hangar, barrack, telecom, quarri, vista, nissan, refug, beverli, pearl, groceri, somewher, crater, slope, mall, ny, arctic, consortium, ranch, distributor, ab, atla, widen, reconnaiss, rug, forestri, pillar, vicin, supermarket, sandi, redevelop, fisheri, parkway, flank, pine, divert, outlet, overlook, cano, facad, mt, aquarium, eastward, monsoon, marina, corridor, cliff, hudson, atop, flown, rocki, ordnanc, depot, erod, ballist, ag, offshor, auxiliari, capitol, encircl, tornado, parachut, swamp, buse, erect, chevrolet, rift, bike, waterfal, mansion, volkswagen, suburban, pa, sank, airbu, antarct, dock, llc, nearest, glacier, runway, refuel, aerial, apach, po, airfield, neighbourhood, fortifi, maneuv, amalgam, gm, sunset, gateway, cf, panama, woodland, chrysler, lodg, erupt, fenc, airplan, surf, plantat, estuari, boulevard, alp, carriag, warship, interchang, hub, amazon, casino, remnant, amphibi, lowland, mi, endem, nile, redesign, stapl, jungl, rhine, prairi, stall, boast, dismantl, battleship, terminu, hawaiian, sanctuari, luftwaff, terrac, nh, altar, haven, courtyard, cottag, en, subdivis, rental, volcan, subway, plateau, battlefield, fs, cater, adjoin, farther, sweep, freeway, reopen, platoon, typhoon, westward, tow, tallest, bombard, delawar, pond, manor, hamburg, wagon, shipment, garag, cruis, flagship, wildlif, cabin, mound, spa, township,
		\par
	}

	\item[$\mathcal{V}_{12} = $]
	{
		\tiny

		physiolog, byte, degener, dental, insulin, dispar, radioact, nervou, enzym, varianc, aerodynam, lung, recurr, diagnosi, antibiot, virus, obstruct, collis, diagnos, patholog, textur, fractur, infecti, surgic, implant, facial, mice, decay, inadequ, regener, vertebr, cognit, transplant, evolutionari, viabl, lesion, passiv, limb, thermodynam, socioeconom, arteri, pathogen, volatil, abdomin, irrit, insuffici, neural, gamma, neurolog, sensat, reflex, exponenti, tract, mood, reproduc, sensori, viral, feedback, nonlinear, cardiac, chromosom, uncertainti, momentum, neutron, primat, diagnost, duplic, bacteria, arous, viru, inflamm, impuls, renal, liver, schizophrenia, synthet, deviat, oscil, slight, synthesi, psychiatr, simpler, cerebr, analyt, sigma, chronic, substrat, reactiv, subgroup, vein, phonem, defici, quantit, benefici, morpholog, incur, vibrat, polynomi, nucleu, cardiovascular, coher, encod, seizur, dysfunct, focal, acut, cure, reciproc, fusion, accuraci, allevi, prone, react, ankl, instabl, phi, robust, angular, bleed, headach, genom, drastic, phenomena, gravit, genit, microscop, replic, obes, gland, intercours, malaria, lobe, urin, antagonist, semant, impair, templat, pronoun, shortterm, affin, diverg, trauma, congest, stimuli, nutrit, indirect, unstabl, proton, harmon, autism, cellular, apparatu, distress, anatomi, invert, nerv, scar, fatigu, hormon, tumor, lethal, deform, stiff, prolong, epidem, syndrom, likelihood, catalyst, mild, propag, tempor, mitig, unchang, ecosystem, genera, tens, underli, toxic, meaning, development, posterior, unnecessari, infer, traumat, inequ, advers, pulmonari, exert, queri, conson, stimul, recess, conjug, invers, stomach, nich, tensor, inclin, dietari, caviti, pregnanc, beta, indirectli, bandwidth, drought, parasit, hazard, transpar, spectrum, multipli, symmetr, molecular, grammat, simplest, neuron, causal, occurr, reson, imped, onset, catastroph, regress, breast, anxieti, digest, inferior, invari, alzheim, correl, cord, suffix, hypothes, subtl, kidney, intestin, spatial, suscept, transcript, corros, marker, parkinson, therapeut, degrad, reus, iter, rigid, dimension, syntax, cortex, homogen, paradigm, fever, decod, cannabi, spinal, syllabl, bacteri, subset, coeffici, arithmet, efficaci, cue, repetit, worsen, ingest, pneumonia, semiconductor, nasal, stimulu, antibodi, spine, reproduct, swell, durabl, inhibitor, modal, prostat, polym, peripher, tuberculosi, breakdown, redund, antigen, induc, symmetri, geometr, tertiari, cocain, mutat, addict, metabol, electrod, respiratori, intrins, magnitud, spontan, xray, hiv, rna, precursor, prolifer, equilibrium, prescript, inhibit, arbitrari, pest, pathway, vaccin, anterior, pi, analys, cosmet, isotop, distort, diabet, abnorm,
		\par
	}

	\item[$\mathcal{V}_{13} = $]
	{
		\tiny

		gibb, fischer, rodriguez, stalin, goodman, sherman, macdonald, gill, troy, levi, lenin, solomon, pearson, porter, rodger, dunn, casey, cameron, thomson, thompson, stern, lama, berri, boon, bradford, fletcher, gandhi, dame, mater, ferguson, carpent, hawkin, reynold, caesar, wallac, perkin, weaver, barrett, harper, bowi, gould, curri, myer, drake, chapman, byrn, owen, mclaren, hussein, wright, canterburi, cole, forrest, benson, reid, hoover, morrison, sheridan, newton, spencer, bailey, gilbert, fraser, freeman, walsh, emerson, fuller, griffith, carey, jen, starr, morri, scotia, ix, blake, helm, sinclair, livingston, phillip, carrol, levin, quinn, mccarthi, watson, wagner, rahman, xvi, curti, fitzgerald, crosbi, harrison, bro, montana, armstrong, lynch, hammond, elli, xi, webster, walker, lennon, xii, chan, parker, maxwel, archer, tate, potter, edison, dixon, bradi, nichol, osborn, kent, reed, logan, nash, bennett, finn, lang, thorn, allmus, stuart, eisenhow, holden, fisher, whitney, clara, presley, booth, montgomeri, dylan, beck, luther, kay, murray, irv, hogan, cohen, arnold, webb, lyon, lambert, christi, obama, blair, heath, newman, gibson, burk, churchil, shelley, ebert, powel, crow, shaw, eden, carson, truman, watt, wade, jenkin, henderson, butler, harvey, vii, lumpur, mustang, cain, roosevelt, murphi, riley, penn, reagan, sander, mccain, viii, baldwin, monica, boyd, barker, hyde, kane, swift, mann, tyler, doyl, bach, palmer, crawford, coleman, barber, carr, jefferson, schumach, hast, luca, barn, nixon, griffin, md, aka, pitt, mead, sim, koch, macarthur, wesley, oneil, campbel, preston, holm, gardner, dawson, bradley, tucker, sr, vi, meyer, den, fe, hay, hardi, chamberlain, collin, frost, savag, mccartney, mason, morton, robertson, burton, klein, baker, chester, laden, byrd, piper, der, robinson, sterl, norton, einstein, hopkin, sullivan, buchanan, stark, johnston, elliott, duncan, stewart, laud, bacon, hart, moss, hancock, peterson, mitchel, parson, marx, kerri, buck, darwin, mcdonald, turner, hawk, lopez, rockefel, monro, mcmahon, obrien, xiv, richardson, holland, hamilton, xiii, chandler, bryant, cox, weber, joyc, madison,
		\par
	}

	\item[$\mathcal{V}_{14} = $]
	{
		\tiny

		anthem, carniv, spinoff, wii, followup, reissu, sequel, conductor, apollo, cameo, slate, vh, eve, gig, unreleas, smart, directori, contributor, trio, embark, sung, ds, horror, rendit, poster, puppet, lp, pb, sang, arcad, facebook, manga, rca, crazi, hd, recount, broadway, genesi, rap, editori, simpson, seller, vol, herald, myspac, rbhiphop, mc, percuss, jam, smash, diari, jockey, funk, bet, preview, cassett, amaz, latest, hail, bollywood, triumph, protagonist, favourit, airplay, forev, regga, guin, rapper, fantast, mtv, artwork, batman, xbox, hardcor, liveact, duo, cheer, choir, recur, commentari, headlin, quartet, rehears, selftitl, con, choru, vinyl, doll, breakthrough, fulllength, oz, dinner, hymn, dancer, maker, trilog, pen, remix, enthusiast, podcast, referenc, telegraph, remast, saga, sitcom, ace, summari, sketch, ballad, cohost, duet, wizard, remak, villain, marvel, hiphop, reunion, soni, melodi, spawn, idol, theatric, lone, backup, epic, aria, teen, zombi, joy, madonna, nintendo, mini, superhero, cbc, bestknown, itv, spiderman, repris, coproduc, emi, recit, monthli, legendari, glori, flute, chat, bang, breakfast, unoffici, showcas, bonu, midnight, shortliv, upload, crossov, banner, greet, bluray, reprint, superman, screenplay, lesson, memor, chant, improvis, merchandis, rereleas, stereo, dirti, catalogu, satir, platinum, tonight, ne, tribut, ensembl, pok{\'e}mon, thriller, delux, cnn, echo, garner, playabl, instant, numberon, cowrot, parodi, discographi, demo, dj, sonic, imprint, catalog, incarn, scream, hiatu, antholog, disco, upcom, memoir, mgm, trek, itun, circu, dawn, nickelodeon, twitter, espn, photographi, hbo, playboy, footag, autobiographi, clip, riaa, longrun, orchestr, paramount, tragedi, trailer, youtub, mixtap, eurovis, beatl, sega, playstat, miniseri, ap, bside, eponym, blog, repertoir, bestsel, rerecord,
		\par
	}

	\item[$\mathcal{V}_{15} = $]
	{
		\tiny

		sc, wwf, tenni, streak, sec, cub, chess, podium, lap, talli, spectat, seventeen, sheffield, texan, rooki, boxer, allstar, wolverin, afl, fixtur, everton, sat, laker, gt, comeback, springfield, sixteen, vacat, softbal, overtim, bulldog, surpass, blackburn, brewer, pageant, tna, cowboy, brisban, tackl, leed, rivalri, fierc, odi, speedway, trainer, conced, roster, postseason, refere, freshman, sack, barcelona, leicest, av, wicket, striker, raven, poker, dodg, wimbledon, penguin, knockout, bruin, eighteen, {\textendash}present, millennium, acc, sixti, ucla, richmond, scorer, stint, packer, thanksgiv, quarterback, raider, {\^a}, runnersup, dodger, golf, bundesliga, division, hometown, aggreg, stoke, trophi, pac, thirti, semest, thirteen, trojan, flyer, bout, fumbl, {\textonehalf}, sunderland, gymnast, ufc, bristol, nebraska, gator, raini, usc, replay, falcon, nashvil, td, oakland, volleybal, easter, shootout, lacross, maverick, icc, milwauke, brave, hattrick, mlb, marathon, wander, vacant, midfield, nwa, narrowli, vike, pitcher, wcw, bench, ferrari, forti, charger, indianapoli, {\textendash}{\textendash}, wwe, rbi, jacksonvil, indoor, postpon, adelaid, lotteri, premiership, fourteen, memphi, undef, chelsea, colt, lineback, semifin, sophomor, preseason, fifteen, incept, derbi, smackdown, dart, clinch, starter, cincinnati, runner, hornet, fastest, inning, rebound, bronco, panther, kickoff, cardiff, celtic, offseason, rover, ensu, rejoin, punt, philli, quarterfin, goalkeep, remaind, nhl, threw, arsen, varsiti, ivi, steeler, collegi, tier, qualif, beaten, wigan, ml, orlando, dolphin, firstclass, feud, duel, mascot, columbu, berth, er, panzer, deadlin, inter, motorsport, seahawk, europa, halftim, preliminari, newcastl, sidelin, titan, wrestler, bolton, yacht, alltim, midway, runnerup, releg, autumn, nascar, wembley, yanke, spur, intercept, catcher, rematch, jaguar,
		\par
	}

	\item[$\mathcal{V}_{16} = $]
	{
		\tiny

		lao, sailor, ukrainian, kosovo, serbian, loyal, bloc, modernday, slavic, morocco, arabian, ulster, citizenship, guinea, uruguay, dissolut, persia, aristocrat, provision, midland, armenia, colombian, croatian, azerbaijan, arabia, cornwal, taliban, filipino, yugoslav, confederaci, hyderabad, ethiopia, feudal, cuisin, jamaica, refuge, treasuri, ghana, lebanes, cuban, madra, bengal, legion, punjab, ham, istanbul, negro, blockad, balkan, ethiopian, frontier, colonist, malta, concess, ottawa, tang, yuan, nomad, czechoslovakia, ming, gaelic, counterattack, alexandria, vietnames, roma, mercenari, serb, newfoundland, kenya, uganda, gaza, expel, dominion, peasant, gujarat, iraqi, westminst, patriarch, insurg, consul, trader, dominican, settler, migrant, ancestri, mongol, herzegovina, colon, palestin, montenegro, archipelago, syrian, malay, malaysian, mafia, nors, qing, haiti, libya, bombay, pilgrim, papal, hispan, homeland, lithuania, freed, pragu, upris, cede, zimbabw, qatar, yemen, ira, chilean, kerala, detain, isl, commando, macedonian, dakota, besieg, conscript, karnataka, warsaw, pact, canton, kuwait, nepal, hostag, detent, bulgarian, folklor, aborigin, predominantli, protector, normandi, maharashtra, prussian, afghan, victorian, pilgrimag, monarchi, incumb, macedonia, dubai, overthrow, denounc, bosnia, algeria, genocid, militia, cypru, indonesian, sudan, prefectur, albanian, ussr, pakistani, brussel, mongolia, romanian, embassi, congo, napl, catholic, pow, baghdad, vatican, tibet, garrison, cambodia, nobil, saxon, caliph, unif, conting, thai, slovakia, nationalist, lithuanian, constantinopl, luxembourg, passport, yugoslavia, georgian, airway, turk, caucasu, albania, reestablish, takeov, guerrilla, estonia, burma, scot, nigerian, proclam, argentin, emir, partisan, yorkshir, finnish, sicili, fascist, synagogu, deport, flourish, cairo, ecuador, sovereignti, cheroke, napoleon, kashmir, milit, prussia,
		\par
	}

	\item[$\mathcal{V}_{17} = $]
	{
		\tiny

		demon, cri, knew, creatur, asid, creator, devil, blind, constantli, wolf, infant, fate, hitler, wed, revel, stolen, competitor, doubt, crush, pretti, pleas, worri, surpris, exactli, lover, wonder, persuad, childhood, accident, trick, els, quick, badli, thank, lesbian, fun, sword, samesex, sure, ye, na, friendship, rider, prompt, tortur, teenag, testimoni, pride, ma, worst, bare, everyon, pronounc, doesnt, insist, aliv, mad, lifetim, grave, certainli, till, killer, disappoint, alien, desper, hate, realis, troubl, ghost, repeatedli, shadow, anyon, fallen, narrat, bride, homer, su, none, spark, captiv, dialogu, temporarili, kidnap, guilti, scandal, joke, wake, ive, monster, versu, recal, quest, angri, resurrect, bless, welcom, toy, funer, confess, robot, wasnt, prostitut, pregnant, specul, anger, imagin, costum, testifi, sin, comfort, forth, conceiv, unfortun, guardian, companion, rape, hang, hide, sacrific, foster, whatev, fake, reportedli, ms, heaven, shall, devast, id, hunter, grace, sick, storylin, suddenli, presum, audit, witch, fortun, mistak, passion, spare, steal, rumor, survivor, interrupt, wrong, truli, boss, acknowledg, suppos, blame, vampir, kiss, inherit, curs, silent, apolog, evil, coincid, betray, contend, etern, guy, mar, upset, deliber, mate, odd, remind, allegedli, cant, jail, portrait, predecessor, complain, laugh, hell, spoke, serious, confid, kid, custodi, hurt, strang, suicid, repli, imprison, beast, aveng, closest, dozen, innoc, confront, reveng,
		\par
	}

	\item[$\mathcal{V}_{18} = $]
	{
		\tiny

		sugar, mask, habitat, explos, narrow, flavor, cotton, nois, concret, accumul, tea, sight, bag, carv, arc, insert, tail, thin, cloud, tear, cylind, insect, valv, sand, flash, meat, raw, mammal, tall, arch, clay, drill, pocket, galaxi, shoe, rod, ft, beam, tast, suspens, patch, wire, dive, bore, ammunit, dish, outer, meal, rough, teeth, mirror, tip, barrel, shoulder, hull, stroke, fossil, cattl, fold, shorter, cm, beer, discharg, bow, disk, loop, copper, swim, inner, sheet, plastic, bath, thick, locomot, finger, pistol, float, crystal, diamet, bicycl, neck, blade, steam, pack, penetr, empti, exhaust, ash, bed, chest, steer, gaug, snake, blend, sweet, predat, rat, clock, bullet, pipe, flat, bottl, flame, spin, gear, rocket, motorcycl, wet, tone, axi, gap, fat, tobacco, slide, ear, milk, trap, laser, mercuri, inch, smoke, poison, genu, ingredi, roof, curv, boot, horn, fabric, nest, brake, rubber, skull, deck, polar, cluster, circular, barrier, grain, grass, pet, tire, breath, scatter, pit, knee, mouth, stamp, blow, lamp, stick, soft, chip, hidden, specimen, fragment, outdoor, cartridg, shell, propel, log, harvest, lift, sharp, prey, kit, button, batteri, drag, slot, smooth, delta, whale, crack, medium, pin, bright, coffe, pool, vertic, cannon, artifici, faster, cultiv, horizont, dispos, appl, nose, wooden, egg, metr, alpha, dome,
		\par
	}

	\item[$\mathcal{V}_{19} = $]
	{
		\tiny

		mathematician, trumpet, superstar, preacher, sonata, patronag, psychologist, isbn, chancellor, sculptor, encyclopedia, physicist, endow, avantgard, elementari, berkeley, conservatori, rave, beethoven, curat, soprano, modernist, tsar, avid, philharmon, sergeant, archaeologist, reich, comedian, mit, jd, tenor, cyril, brigadi, prolif, unesco, birthplac, manifesto, paperback, apprentic, telugu, malayalam, humanist, deutsch, chef, spokesperson, punjabi, wikipedia, councillor, magna, diploma, bibliographi, magistr, singersongwrit, uc, counselor, truste, biograph, ballet, creed, alumni, newcom, mentor, synopsi, ign, economist, yale, pamphlet, postgradu, englishlanguag, banker, choral, businessman, princeton, smithsonian, math, clerk, coauthor, librarian, sheriff, cornel, thesi, dictionari, superintend, tuition, freelanc, violin, entrepreneur, culinari, seminar, astronaut, urdu, sociologist, forb, screenwrit, petti, emin, troup, vocat, vogu, polytechn, {\textbullet}, pianist, veterinari, discipl, tutor, regent, inspector, yorker, nonfict, biologist, shepherd, concerto, ba, neoclass, rabbi, textbook, abbot, op, preparatori, mba, jointli, filmmak, standup, spokesman, parttim, vicepresid, surgeon, pupil, supervisor, choreograph, pornograph, citat, marathi, bilingu, psychiatrist, playwright, treatis, renown, pseudonym, quarterli, naturalist, bulletin, fellowship, classmat, theorist, kannada, hindi, acquaint, anthropologist, constabl, columnist, baroqu, appel, pharmaci, dissert, shakespear, defunct, saxophon, cartoonist, playback, inventor, grammar, blogger, chemist, instructor, upheld, campus, prose, subtitl, africanamerican, fairi, riff, pp, bengali, thinker, emeritu, technician, stanford, novelist,
		\par
	}

	\item[$\mathcal{V}_{20} = $]
	{
		\tiny

		render, sensit, proper, fatal, perceiv, toler, sole, rapidli, earthquak, necessarili, solv, bound, owe, repair, satisfi, emphas, craft, strengthen, explan, wider, ecolog, fals, excit, somewhat, poverti, framework, wealth, tension, simultan, landscap, imposs, manipul, immun, temporari, massiv, exact, matur, grown, behaviour, elabor, outcom, disrupt, closer, emphasi, flexibl, pose, ordinari, defect, minim, notion, complic, understood, circumst, ration, strain, valuabl, superior, similarli, furthermor, longterm, vowel, reli, wage, stronger, exploit, fairli, passag, unless, comparison, dissolv, rapid, counter, extinct, heavili, shock, routin, oral, ongo, easi, accomplish, ignor, afford, modif, absenc, socal, compens, suppress, expos, preval, perspect, partli, phenomenon, vulner, widespread, slowli, moder, therebi, dramat, attain, mainten, confus, character, circul, displac, ideal, absolut, familiar, destruct, reinforc, theoret, violent, margin, deliveri, recoveri, stabl, suffici, resolv, substanti, strongli, aris, unusu, disabl, accur, sort, neutral, huge, facilit, fewer, accommod, attitud, valid, reconstruct, relev, greatli, harm, suitabl, safe, radic, alloc, disturb, impli, optim, isol, persist, weak, easier, visibl, mere, supplement, broad, proof, diet, reliabl, aggress, trend, interfer, borrow, trigger, align, preced, undergo, gradual, gender, loos, clearli, exclud,
		\par
	}

	\item[$\mathcal{V}_{21} = $]
	{
		\tiny

		forum, hostil, evacu, counsel, guarante, stake, seiz, constitu, voter, ceas, inquiri, referendum, terrorist, servant, friendli, ratifi, halt, interven, ralli, administ, reorgan, prospect, parliamentari, auction, unifi, advisor, prosecut, equiti, liberti, demograph, eu, rebel, submit, leas, permiss, telecommun, pledg, salari, autonom, civic, lawsuit, tribal, faction, tribun, withdraw, fulltim, withdrawn, enact, judici, tenur, commenc, physician, consolid, recipi, mandat, contractor, resum, bureau, begun, consent, nonprofit, accredit, reelect, complaint, bid, clinton, sa, diplomat, elector, charter, withdrew, rent, nasa, unsuccess, volunt, behalf, unanim, nationwid, merger, decre, admiss, nato, intervent, propaganda, cia, petit, landmark, repeal, elit, endors, sovereign, elig, deleg, subordin, abolish, spi, commerc, holder, dispatch, licenc, ss, legislatur, analyst, cadet, terror, auto, aftermath, democraci, culmin, registr, oblig, maritim, enlist, archiv, specialist, criticis, clash, warrant, bankruptci, coalit, pension, opt, welfar, interim, advic, condemn, privileg, recognis, prosecutor, workshop, mutual, regain, rebuild, expir, disband, sharehold, slaveri, casualti, sanction, riot, advisori, publicli, fda, prosper, ballot, lobbi, statut, activist, coup, congression, un, supervis, surrend, urg, fbi, renov, overse, postwar, discontinu,
		\par
	}

	\item[$\mathcal{V}_{22} = $]
	{
		\tiny

		infrar, vapor, fibr, inflat, nitrogen, residu, ambient, knot, lowest, surplu, absorpt, {\ensuremath{-}}, mw, silicon, commod, detector, torqu, subsidi, heavier, temper, offset, distil, solvent, dens, gase, ventil, boil, uranium, crude, ph, bulk, melt, humid, deeper, mph, greenhous, gradient, lava, payload, spill, {\textdegree}c, lighter, deplet, muzzl, mb, overhead, abund, sanit, rpm, fatti, upward, plasma, threshold, kmh, hydraul, eros, shortag, alloy, glucos, evapor, cheaper, compart, refriger, sulfur, tide, wavelength, cooler, shear, median, photon, thrust, enrich, moistur, petroleum, gasolin, buffer, freez, width, zinc, cyclon, radiu, boiler, puls, subtrop, veloc, satur, mg, machineri, potassium, sodium, amino, pollut, thermal, combust, shale, dioxid, kinet, diffus, nm, drain, calcium, dispers, emit, tidal, friction, nutrient, propuls, unemploy, ignit, flux, {\textdegree}f, shallow, contamin, freshwat, recycl, turbin, precipit, tariff, clearanc, ethanol, rainfal, sunlight, slower, rainforest, discount, beverag, reservoir, insul, sperm, lesser, intak, aluminium, chlorid, irrig, {\textdegree}, fluctuat, vacuum, latitud, dissip, livestock, solubl, fertil, jaw, manifold, lunar, coil, drainag, literaci, condens, altitud, mortal, sediment, proxim, dose, electromagnet, ferment, rotor, downward, fraction,
		\par
	}

	\item[$\mathcal{V}_{23} = $]
	{
		\tiny

		realtim, telescop, cpu, protocol, codic, node, interv, compat, diagram, grid, terrestri, finit, server, lens, io, freight, ps, email, cach, googl, stack, navig, automot, proprietari, sensor, compact, plugin, notat, transmit, static, xp, autom, commut, connector, api, subscrib, portabl, pc, broadband, desktop, q, matrix, gb, browser, layout, ibm, linear, consol, array, android, {\texttimes}, turbo, probe, graviti, modul, mac, kernel, relay, hp, hybrid, denot, {\textrightarrow}, prefix, tablet, radar, iphon, ac, quantum, integ, pixel, graph, cc, kw, shuttl, mhz, automobil, embed, supplier, db, vitamin, delet, z, intermedi, vendor, emul, random, simul, gameplay, diesel, chassi, ip, default, gp, laptop, remot, wireless, subscript, newer, bmw, converg, download, topolog, os, prototyp, vector, algebra, scan, encrypt, usb, antenna, cargo, font, transmitt, theorem, leak, surveil, chord, intel, synchron, refin, bundl, amplifi, len, app, readili, menu, interfac, premium, printer, analog, multiplay, reactor, linux, synthes, hardwar, conveni, paramet, dual, infinit, processor, spacecraft, databas, packet, configur, highspe, geometri, discret, binari,
		\par
	}

	\item[$\mathcal{V}_{24} = $]
	{
		\tiny

		danni, leo, justin, neil, maria, ibn, nelson, colleagu, kevin, warren, ted, dean, russel, lisa, nova, eric, billi, tim, pat, dick, steven, princess, longtim, teammat, matthew, kim, fred, benjamin, jean, willi, singh, carter, max, cousin, jon, jan, pete, hugh, carl, bassist, kelli, kate, larri, widow, da, craig, ralph, eldest, harold, ron, susan, abu, eddi, santa, lloyd, terri, nick, charlott, franklin, ross, yearold, jay, costar, grandfath, greg, anna, jane, gari, bruce, jeff, alan, charli, shah, elder, wayn, jacob, li, albert, phil, sibl, michel, bin, christoph, drummer, alic, karl, ed, archbishop, ian, ryan, victor, margaret, leon, bobbi, johnni, tommi, denni, rick, ken, robin, perri, luke, todd, ben, sarah, norman, morgan, anthoni, girlfriend, gordon, matt, sean, brook, andi, gen, jerri, donald, evan, graham, dougla, jason, jonathan, barri, oliv, abraham, uncl, reunit, chuck, alfr, brian, roy, walter, cofound, youngest, baron, ami, mario, muhammad, keith, alex, frederick, jimmi, dave, rob, dan, barbara, samuel,
		\par
	}

	\item[$\mathcal{V}_{25} = $]
	{
		\tiny

		sampl, expens, classif, index, upgrad, innov, strategi, algorithm, otherwis, enhanc, topic, difficulti, wherea, stabil, variabl, equival, input, usag, experiment, automat, evalu, client, visual, context, motion, coordin, fundament, shift, discoveri, graphic, dynam, mode, intens, accid, represent, classifi, segment, util, variat, revers, differenti, variant, modifi, evolut, laboratori, fast, monitor, revis, core, virtual, assess, error, logic, henc, dimens, map, zero, enabl, mathemat, pure, transmiss, delay, sustain, procedur, calcul, essenti, alter, evolv, handl, extern, correct, weather, appropri, composit, bit, packag, orient, check, add, specifi, extra, predict, descript, equat, statist, precis, scheme, manual, balanc, updat, fix, andor, divers, partial, strength, manner,
		\par
	}

	\item[$\mathcal{V}_{26} = $]
	{
		\tiny

		hungari, contin, exil, mumbai, norway, turkey, ukrain, patriot, tokyo, beij, frankfurt, caribbean, ontario, bulgaria, athen, delhi, romania, nigeria, afghanistan, peninsula, lebanon, cuba, taiwan, belgium, cemeteri, austria, iran, iceland, malaysia, munich, finland, hawaii, switzerland, northeastern, greec, vancouv, jerusalem, neighbor, fled, thailand, mainland, alaska, sieg, amsterdam, queensland, geneva, croatia, southwestern, hampshir, venic, glasgow, villa, serbia, peru, netherland, nevada, manila, brazil, emigr, annex, pirat, dublin, indonesia, syria, ambassador, metro, chile, summit, madrid, invad, singapor, orlean, northwestern, southeastern, shanghai, portug, colombia, edinburgh, poland, montreal, alberta, moscow, sweden, presentday, venezuela, bangladesh, denmark, milan, elsewher, vienna, argentina, quebec, abroad, neighbour,
		\par
	}

	\item[$\mathcal{V}_{27} = $]
	{
		\tiny

		parad, injur, slow, trace, nicknam, touch, oppon, caught, pull, ahead, wound, penalti, crowd, chase, broke, induct, vs, journey, fought, straight, bat, sail, besid, row, climb, ram, longest, shut, cap, twin, knock, twelv, bought, disappear, struck, jump, stood, departur, twice, trip, broken, span, driven, substitut, laid, pitch, suspend, ward, throw, twenti, tiger, retreat, lane, hurrican, kick, rush, goe, u, ran, drove, whilst, drawn, eleven, giant, gone, buri, tripl, gang, wait, drew, plu, sit, strip, fell, catch, exit, warrior, lay, push, readi, collaps,
		\par
	}

	\item[$\mathcal{V}_{28} = $]
	{
		\tiny

		coron, worship, mytholog, sultan, abbey, shrine, monk, monasteri, mosqu, wealthi, persian, realm, ce, missionari, han, thcenturi, ruler, shiva, feast, antiqu, bce, ruin, heir, rebellion, myth, renaiss, priest, conquest, chapel, cathedr, revolt, ancestor, mystic, descent, commemor, surnam, goddess, burial, buddhist, gothic, tibetan, byzantin, mediev, onward, throne, rite, nobl, sikh, clan, denomin, alphabet, proclaim, mughal, conquer, prophet, hindu, crusad, buddhism, dynasti, patron, monarch, ascend, sanskrit, deiti, flee, calendar, inscript, massacr, treasur, armenian, sacr, counterpart, archaeolog, monument, saudi, tomb,
		\par
	}

	\item[$\mathcal{V}_{29} = $]
	{
		\tiny

		shape, sequenc, abil, signal, target, integr, qualiti, uniqu, interact, symbol, etc, presenc, detail, directli, devic, focu, equal, eg, principl, fit, resourc, factor, categori, pattern, knowledg, messag, user, definit, advantag, mass, contrast, capabl, characterist, environ, correspond, detect, consum, interpret, matter, distinct, phase, demonstr, sens, reflect, transform, item, kind, impact, techniqu, root, option, deriv, simpl, analysi, solut, consider, content, tool, skill, compon, display, mechan, multipl, basic, restrict, safeti, altern, address, consequ, ie, implement, aspect, electron, instanc, ident,
		\par
	}

	\item[$\mathcal{V}_{30} = $]
	{
		\tiny

		wast, bind, dri, membran, storag, depth, fluid, clean, consumpt, oxid, orbit, substanc, radiat, atmospher, fuel, compress, skin, solar, adjust, fruit, heat, veget, deposit, soil, hydrogen, protein, crop, layer, acceler, decreas, feed, emiss, particl, oxygen, tissu, inject, load, maximum, ion, carbon, atom, filter, pump, dna, fiber, exposur, acid, reduct, solid, bone, muscl, mixtur, angl, tropic, molecul, warm, coal, stem, tube, salt, absorb, cool, miner, receptor, liquid, rotat, fresh, drink, optic, excess, extract, alcohol, constant, minimum,
		\par
	}

	\item[$\mathcal{V}_{31} = $]
	{
		\tiny

		photograph, occasion, possess, cite, ban, enjoy, pilot, request, depict, suit, confirm, unlik, guid, meant, pursu, abandon, rescu, repeat, encount, descend, favor, obtain, watch, chosen, distinguish, incorpor, dedic, paid, respond, choos, fashion, sought, search, warn, explor, invent, convert, preserv, experienc, perman, permit, regist, introduct, convers, encourag, gather, assign, engag, count, ensur, creation, seek, grew, restor, kept, threaten, attribut, buy, recruit, accus, deni, send, deliv, recommend, recov, belong, princip, split, accompani, conclud,
		\par
	}

	\item[$\mathcal{V}_{32} = $]
	{
		\tiny

		currenc, enterpris, membership, certif, household, expans, tourist, sector, loan, interior, subsidiari, insur, net, ltd, payment, export, consult, farmer, rural, disast, visitor, inc, renew, worldwid, destin, partnership, profit, fair, relief, asset, merg, fee, budget, geograph, ticket, viewer, revenu, residenti, exclus, survey, entiti, sponsor, cash, transact, compris, recreat, crisi, estat, censu, investor, trust, employe, vast, ownership, chariti, ministri, illeg, patent, acquisit, infrastructur, ventur, debt, co, donat, tourism, domest, retail, newli, telephon, financ,
		\par
	}

	\item[$\mathcal{V}_{33} = $]
	{
		\tiny

		austin, virginia, maryland, kentucki, connecticut, atlanta, jersey, portland, seattl, boston, oregon, kansa, chicago, manchest, iowa, wisconsin, melbourn, avenu, reloc, pittsburgh, illinoi, baltimor, michigan, ranger, arizona, downtown, miami, liverpool, brooklyn, houston, phoenix, detroit, arena, toronto, dalla, colorado, birmingham, louisiana, denver, philadelphia, pennsylvania, texa, berlin, suburb, tech, minnesota, cardin, manhattan, buffalo, indiana, usa, fc, utah, massachusett, metropolitan, cleveland, georgia, missouri, florida, alabama, sydney, borough, ohio, arkansa, oklahoma, tennesse, mississippi,
		\par
	}

	\item[$\mathcal{V}_{34} = $]
	{
		\tiny

		storm, plate, cycl, boat, chamber, rear, winter, bomb, apart, meter, floor, leg, bottom, steel, frame, burn, flood, ring, door, insid, wave, fli, bar, switch, panel, tabl, block, attach, height, column, parallel, spring, glass, tank, onto, edg, gate, mill, stone, forward, bond, lock, wheel, circl, vessel, crash, deep, chain, mm, stream, seed, shop, path, circuit, pair, wood, garden, tower, feet, fill, mount, foot, factori, plane, truck,
		\par
	}

	\item[$\mathcal{V}_{35} = $]
	{
		\tiny

		medicin, cooper, attract, statu, librari, mainli, declin, formal, prepar, focus, hospit, agent, demand, programm, architectur, whole, organis, recogn, themselv, file, divid, attent, foundat, agenc, purpos, primarili, discuss, mostli, document, legal, teach, benefit, domin, conduct, controversi, employ, regul, basi, effort, aid, exhibit, mission, economi, intellig, custom, job, cours, situat, convent, money, emerg, oppos, aim, branch, progress, expand, opportun, secret, worker, contact, conflict,
		\par
	}

	\item[$\mathcal{V}_{36} = $]
	{
		\tiny

		convict, {\textquoteleft}, someon, answer, hear, divorc, herself, neither, didnt, truth, nor, im, alon, mind, heard, impress, babi, admit, sentenc, promis, fear, dead, punish, perfect, noth, older, anyth, bad, sleep, remark, gift, wish, moment, birth, wit, victim, convinc, spirit, everyth, intent, talent, commit, jesu, notic, inde, suspect, soul, unknown, realiz, mysteri, whi, occas, beauti, coupl, happi, chanc, rememb, gay, holi, dream, listen,
		\par
	}

	\item[$\mathcal{V}_{37} = $]
	{
		\tiny

		sourc, concept, method, problem, protect, appli, rather, object, experi, rel, subject, particular, measur, individu, occur, reason, condit, combin, certain, specif, improv, normal, theori, express, concern, complex, approach, evid, sound, physic, typic, imag, structur, properti, applic, materi, function, formula, itself, defin, signific, element, observ, speci, remov, code, indic, compar, valu, therefor, either, data,
		\par
	}

	\item[$\mathcal{V}_{38} = $]
	{
		\tiny

		account, futur, learn, shown, suggest, seen, recent, face, particularli, achiev, relationship, past, adopt, introduc, hold, propos, full, memori, charg, reveal, surviv, maintain, separ, numer, carri, contribut, today, share, respect, promot, subsequ, regard, select, key, particip, gave, advanc, earn, accept, despit, saw, especi, rais, gain, whose, identifi, except, least, argu, toward, extend,
		\par
	}

	\item[$\mathcal{V}_{39} = $]
	{
		\tiny

		toni, clark, o, miller, tom, bell, jordan, scott, adam, harri, jone, marshal, frank, ford, brown, kennedi, jr, chri, allen, johnson, mike, moor, simon, howard, anderson, knight, don, bush, ray, jack, van, daniel, jim, von, roger, lee, iv, taylor, jackson, lewi, joe, davi, sam, biographi, wilson, ali, steve, smith, khan, bob,
		\par
	}

	\item[$\mathcal{V}_{40} = $]
	{
		\tiny

		phrase, bibl, prayer, legaci, illustr, ritual, poet, tale, liter, gospel, narr, poetri, scholar, influenti, cinema, dub, amongst, reader, hebrew, essay, linguist, devot, romanc, painter, legend, dialect, chapter, icon, philosoph, spoken, poem, manuscript, speaker, vers, heritag, spell, canon, romant, cult, quot, chronicl, literari, earliest, wellknown, sculptur, reviv, pioneer, script, literatur,
		\par
	}

	\item[$\mathcal{V}_{41} = $]
	{
		\tiny

		bird, yellow, sun, wine, resembl, lion, dress, grey, flower, spot, dragon, eat, hair, dog, planet, moon, breed, belt, coin, wild, cloth, colour, magic, orang, hunt, iron, eagl, wear, worn, rain, snow, coat, cat, decor, ride, rose, hors, rich, diamond, bull, rice, hole, ice, dark, uniform, seal, cook, bear, shield,
		\par
	}

	\item[$\mathcal{V}_{42} = $]
	{
		\tiny

		anglican, liturgi, congreg, sunni, clergi, pagan, ld, apostol, oath, pradesh, lutheran, baptist, methodist, pastor, shia, judaism, theologian, scriptur, brotherhood, rabbin, sect, nun, episcop, cleric, apostl, theolog, basilica, dioces, persecut, secular, hinduism, evangel, ecclesiast, communion, parish, triniti, seminari, jesuit, marxist, christ, mormon, presbyterian, orthodox, anarchist, preach, libertarian, ordain, martyr,
		\par
	}

	\item[$\mathcal{V}_{43} = $]
	{
		\tiny

		manuel, juan, rosa, carlo, pedro, s{\~a}o, lorenzo, sierra, di, rafael, jos{\'e}, giovanni, lui, roberto, mont, pablo, andr{\'e}, fernando, marco, jorg, gabriel, alberto, silva, aviv, miguel, hugo, ana, cruz, copa, fidel, mar{\'\i}a, torr, garcia, monaco, paulo, du, polo, marino, castro, antonio, santo, jose, franco, bernardino, santiago,
		\par
	}

	\item[$\mathcal{V}_{44} = $]
	{
		\tiny

		northeast, basin, migrat, cave, corner, creek, resort, southwest, tunnel, railroad, nearbi, southeast, pacif, mediterranean, inhabit, geographi, pole, atlant, coastal, boundari, restaur, municip, canal, dam, desert, km, highway, headquart, ocean, adjac, trail, cape, northwest, hotel, fort, stretch, castl, plain, entranc, beach, shore, mile, underground, neighborhood,
		\par
	}

	\item[$\mathcal{V}_{45} = $]
	{
		\tiny

		sever, manag, chang, base, found, area, provid, although, produc, product, creat, power, intern, complet, report, each, open, line, within, local, act, point, anoth, remain, lead, own, compani, oper, major, addit, accord, continu, receiv, design, set, under, present, build, current, form, hous, same, support,
		\par
	}

	\item[$\mathcal{V}_{46} = $]
	{
		\tiny

		violenc, threat, leadership, belief, opinion, faith, recognit, motiv, scientist, resolut, argument, freedom, divin, speech, intellectu, spiritu, duti, philosophi, alleg, advoc, conclus, ethic, disput, moral, corrupt, instruct, exercis, excel, choic, expert, examin, favour, vision, nevertheless, debat, reput, reject, doctrin, disciplin, statement, creativ, dismiss, assert,
		\par
	}

	\item[$\mathcal{V}_{47} = $]
	{
		\tiny

		natur, those, special, consist, activ, though, limit, repres, engin, bodi, possibl, market, further, involv, test, project, exampl, model, standard, respons, industri, contain, effect, issu, type, land, event, exist, human, period, class, control, case, term, way, great, process, ad, anim, offer, requir,
		\par
	}

	\item[$\mathcal{V}_{48} = $]
	{
		\tiny

		figur, mention, credit, print, pictur, piec, collabor, plot, earlier, label, novel, text, theme, screen, celebr, websit, journal, newspap, adapt, arrang, instrument, press, comment, magazin, scene, audienc, interview, page, letter, volum, articl, voic, paper, background, doctor, inspir, card, edit, fan, mix, paint,
		\par
	}

	\item[$\mathcal{V}_{49} = $]
	{
		\tiny

		peabodi, jubile, daytim, pulitz, bafta, primetim, filmfar, prizewin, awardwin, allamerican, nobel, posthum, telecast, desk, firstteam, accolad, emmi, globe, saturn, finalist, cann, sundanc, oscar, nomine, brit, gemini, prestigi, baseman, medalist, carnegi, nielsen, laureat, guild, juno, dove, mvp, honorari, cw, mellon, grammi,
		\par
	}

	\item[$\mathcal{V}_{50} = $]
	{
		\tiny

		founder, deputi, formerli, mayor, bishop, meanwhil, successor, chose, chairman, admir, lincoln, advis, secretari, fellow, hire, cabinet, crown, invit, inaugur, ceo, scout, editor, politician, resign, lawyer, colonel, journalist, assassin, mp, veteran, lieuten, architect, chair, rival, renam, presidenti, briefli, victoria, commission, attorney,
		\par
	}

	\item[$\mathcal{V}_{51} = $]
	{
		\tiny

		tonn, usd, capita, lb, ton, ago, litr, cubic, cent, trillion, exceed, metric, kilomet, {\texteuro}, pound, gram, annum, dollar, revolv, yen, fifti, euro, kg, kilogram, crore, gallon, kilometr, se, gdp, weigh, acr, gross, hectar, rs,
		\par
	}

	\item[$\mathcal{V}_{52} = $]
	{
		\tiny

		cavalri, jet, helicopt, rifl, warfar, battalion, raid, assault, naval, airborn, artilleri, fleet, bomber, missil, guard, strateg, squadron, submarin, regiment, patrol, expedit, fighter, corp, deploy, infantri, brigad, armour, carrier, combat, tactic, aviat, personnel, armor,
		\par
	}

	\item[$\mathcal{V}_{53} = $]
	{
		\tiny

		egyptian, portugues, polish, welsh, palestinian, vice, austrian, commonwealth, oversea, nazi, puerto, dutch, czech, scottish, confeder, swiss, hungarian, tamil, continent, merchant, sri, turkish, mexican, irish, iranian, danish, isra, belgian, imperi, swedish, norwegian, provinci, brazilian,
		\par
	}

	\item[$\mathcal{V}_{54} = $]
	{
		\tiny

		behavior, mental, cancer, brain, emot, neg, therapi, symptom, psycholog, compound, diseas, abus, clinic, seriou, ill, reaction, depress, gene, treatment, surgeri, biolog, pain, stress, treat, patient, sexual, infect, chemic, disord, failur, genet, risk,
		\par
	}

	\item[$\mathcal{V}_{55} = $]
	{
		\tiny

		via, varieti, link, extens, construct, flight, ground, manufactur, store, transport, equip, ship, access, distribut, suppli, rang, facil, space, avail, format, free, commerci, car, weapon, fire, aircraft, food, comput, plant, vehicl, connect,
		\par
	}

	\item[$\mathcal{V}_{56} = $]
	{
		\tiny

		coverag, regularli, hollywood, comedi, disney, logo, weekli, anchor, amateur, serial, daili, theater, fox, fm, poll, documentari, realiti, bbc, cancel, holiday, venu, affili, franchis, drama, nbc, mail, syndic, sky, cb, abc, cartoon,
		\par
	}

	\item[$\mathcal{V}_{57} = $]
	{
		\tiny

		brief, genr, adventur, tape, disc, photo, compil, biggest, cd, soundtrack, mainstream, dvd, hero, entitl, highlight, signatur, lyric, favorit, ep, christma, audio, certifi, lineup, fantasi, tune, greatest, session, sing, solo, string,
		\par
	}

	\item[$\mathcal{V}_{58} = $]
	{
		\tiny

		civilian, jew, citizen, occup, immigr, egypt, era, affair, regim, pakistan, occupi, tribe, settl, settlement, soldier, republ, coloni, israel, airlin, allianc, invas, slave, flag, revolut, alli, troop, camp, philippin, rome, navi,
		\par
	}

	\item[$\mathcal{V}_{59} = $]
	{
		\tiny

		previous, soon, simpli, refus, heart, frequent, better, yet, longer, might, hope, agre, actual, intend, commonli, {\textemdash}, here, hard, immedi, unabl, expect, ultim, quickli, alway, alreadi, fail, ever, probabl,
		\par
	}

	\item[$\mathcal{V}_{60} = $]
	{
		\tiny

		miss, crew, schedul, youth, challeng, elimin, prior, entri, incid, beat, lose, struggl, driver, strike, compet, owner, crime, offens, enemi, shoot, partner, defend, draw, contest, tie, latter,
		\par
	}

	\item[$\mathcal{V}_{61} = $]
	{
		\tiny

		abl, claim, doe, must, hand, upon, attempt, still, need, action, order, initi, without, onc, instead, find, should, decid, help, never, right, plan, eventu, tri,
		\par
	}

	\item[$\mathcal{V}_{62} = $]
	{
		\tiny

		softwar, advertis, joint, licens, rail, microsoft, transit, onlin, web, instal, digit, mobil, camera, bu, traffic, satellit, platform, window, cabl, phone, passeng, termin, internet,
		\par
	}

	\item[$\mathcal{V}_{63} = $]
	{
		\tiny

		appar, desir, fulli, lot, increasingli, highli, beyond, quit, understand, difficult, prove, perhap, prefer, easili, clear, rare, danger, sex, awar, necessari, extrem, tend,
		\par
	}

	\item[$\mathcal{V}_{64} = $]
	{
		\tiny

		stop, stand, keep, behind, taken, brought, escap, travel, drop, destroy, fall, rest, discov, captur, sent, save, arriv, bring, visit, drive, stay, put,
		\par
	}

	\item[$\mathcal{V}_{65} = $]
	{
		\tiny

		orchestra, hop, acoust, symphoni, hip, guitarist, vocalist, bass, folk, drum, guitar, vocal, keyboard, warner, rb, punk, jazz, pop, rhythm, piano, songwrit,
		\par
	}

	\item[$\mathcal{V}_{66} = $]
	{
		\tiny

		japan, canada, territori, zealand, germani, europ, spain, mexico, england, ireland, kingdom, franc, australia, border, china, capit, russia, provinc, itali, scotland, pari,
		\par
	}

	\item[$\mathcal{V}_{67} = $]
	{
		\tiny

		dont, think, seem, ask, want, told, feel, felt, happen, realli, done, someth, thought, am, let, know, talk, explain, tell, got,
		\par
	}

	\item[$\mathcal{V}_{68} = $]
	{
		\tiny

		concentr, profil, elev, capac, extent, ratio, yield, significantli, domain, incom, output, quantiti, densiti, sum, percentag, slightli, frequenc, effici, voltag, proport,
		\par
	}

	\item[$\mathcal{V}_{69} = $]
	{
		\tiny

		bridg, villag, road, front, hill, site, templ, section, street, middl, outsid, rout, valley, mountain, centr, cross, resid, port, nativ,
		\par
	}

	\item[$\mathcal{V}_{70} = $]
	{
		\tiny

		again, led, return, die, sign, join, leav, replac, meet, attack, togeth, begin, enter, left, lost, reach, kill, held, mark,
		\par
	}

	\item[$\mathcal{V}_{71} = $]
	{
		\tiny

		younger, husband, historian, succeed, alongsid, saint, reign, ladi, lord, queen, emperor, pope, portray, sir, mr, dr, le, captain, princ,
		\par
	}

	\item[$\mathcal{V}_{72} = $]
	{
		\tiny

		boy, king, wife, parent, friend, son, murder, marri, father, met, brother, woman, girl, mother, daughter, marriag, sister, whom, child,
		\par
	}

	\item[$\mathcal{V}_{73} = $]
	{
		\tiny

		soap, bueno, thth, fifteenth, buckingham, sixteenth, tampa, nineteenth, fourteenth, thirteenth, eleventh, midth, twentieth, twentyfirst, twilight, eighteenth, twelfth, pga, seventeenth,
		\par
	}

	\item[$\mathcal{V}_{74} = $]
	{
		\tiny

		sinc, origin, member, success, appear, live, life, them, group, against, histori, peopl, home, famili, seri, show, countri, perform,
		\par
	}

	\item[$\mathcal{V}_{75} = $]
	{
		\tiny

		park, side, along, region, across, river, western, central, london, built, northern, town, eastern, island, near, locat, throughout, southern,
		\par
	}

	\item[$\mathcal{V}_{76} = $]
	{
		\tiny

		undergradu, cambridg, oldest, oxford, lectur, bachelor, faculti, enrol, graduat, phd, professor, taught, teacher, scholarship, nurs, campu, harvard, galleri,
		\par
	}

	\item[$\mathcal{V}_{77} = $]
	{
		\tiny

		legisl, administr, seat, committe, congress, vote, parliament, campaign, opposit, commiss, council, leader, governor, senat, bill, candid, assembl,
		\par
	}

	\item[$\mathcal{V}_{78} = $]
	{
		\tiny

		june, born, februari, novemb, august, januari, career, juli, septemb, forc, york, decemb, march, april, octob, announc, began,
		\par
	}

	\item[$\mathcal{V}_{79} = $]
	{
		\tiny

		given, person, allow, result, becaus, consid, veri, even, refer, among, describ, mean, give, see, make, caus, due,
		\par
	}

	\item[$\mathcal{V}_{80} = $]
	{
		\tiny

		oil, spread, upper, urban, climat, lie, portion, agricultur, ga, zone, farm, forest, surround, bay, wind, flow, mine,
		\par
	}

	\item[$\mathcal{V}_{81} = $]
	{
		\tiny

		sunday, monday, saturday, pm, walt, tuesday, afterward, wednesday, weekend, shortli, friday, thereaft, weekday, morn, thursday, afternoon, newscast,
		\par
	}

	\item[$\mathcal{V}_{82} = $]
	{
		\tiny

		medici, bayern, aston, jure, plata, liga, facto, sall, moin, atl{\'e}tico, rothschild, havilland, janeiro, palma, vega, versail, gaull,
		\par
	}

	\item[$\mathcal{V}_{83} = $]
	{
		\tiny

		b, c, x, v, iii, g, e, r, f, j, d, k, p, l, w, h,
		\par
	}

	\item[$\mathcal{V}_{84} = $]
	{
		\tiny

		next, summer, previou, ten, hour, five, everi, nine, seven, six, eight, big, hall, night, entir, few,
		\par
	}

	\item[$\mathcal{V}_{85} = $]
	{
		\tiny

		research, commun, servic, inform, studi, scienc, educ, organ, econom, center, institut, polit, busi, program, train, social,
		\par
	}

	\item[$\mathcal{V}_{86} = $]
	{
		\tiny

		elizabeth, loui, arthur, alexand, stephen, patrick, martin, mari, philip, joseph, lawrenc, andrew, franci, edward, ann, duke,
		\par
	}

	\item[$\mathcal{V}_{87} = $]
	{
		\tiny

		blood, gun, light, tree, wing, room, eye, sea, machin, surfac, heavi, color, earth, metal, wall, fish,
		\par
	}

	\item[$\mathcal{V}_{88} = $]
	{
		\tiny

		draft, nba, pick, rugbi, nfl, squad, junior, winner, tournament, playoff, qualifi, confer, ncaa, retir, senior,
		\par
	}

	\item[$\mathcal{V}_{89} = $]
	{
		\tiny

		christian, greek, thousand, latin, arab, protest, islam, muslim, hundr, speak, minor, translat, ancient, religion, jewish,
		\par
	}

	\item[$\mathcal{V}_{90} = $]
	{
		\tiny

		exchang, sold, acquir, invest, sale, brand, purchas, sell, transfer, corpor, global, fund, pay, stock, privat,
		\par
	}

	\item[$\mathcal{V}_{91} = $]
	{
		\tiny

		true, thing, littl, idea, your, question, my, whether, how, our, look, fact, god, too, me,
		\par
	}

	\item[$\mathcal{V}_{92} = $]
	{
		\tiny

		famou, note, write, list, short, wrote, danc, classic, compos, collect, notabl, read, written, cover,
		\par
	}

	\item[$\mathcal{V}_{93} = $]
	{
		\tiny

		justic, investig, appeal, judg, crimin, approv, declar, constitut, peac, decis, arrest, prison, trial, grant,
		\par
	}

	\item[$\mathcal{V}_{94} = $]
	{
		\tiny

		avoid, damag, prevent, suffer, affect, grow, potenti, poor, resist, loss, drug, injuri, strong, lack,
		\par
	}

	\item[$\mathcal{V}_{95} = $]
	{
		\tiny

		languag, cultur, practic, word, relat, view, interest, histor, movement, influenc, associ, modern, tradit, societi,
		\par
	}

	\item[$\mathcal{V}_{96} = $]
	{
		\tiny

		russian, indian, french, canadian, german, australian, spanish, royal, english, foreign, italian, japanes, chines,
		\par
	}

	\item[$\mathcal{V}_{97} = $]
	{
		\tiny

		differ, small, larg, main, ani, close, similar, import, popular, common, good, variou, increas,
		\par
	}

	\item[$\mathcal{V}_{98} = $]
	{
		\tiny

		charl, john, georg, robert, michael, jame, paul, david, william, henri, thoma, peter, richard,
		\par
	}

	\item[$\mathcal{V}_{99} = $]
	{
		\tiny

		twoyear, lockhe, rhode, oneyear, rio, virgin, zeppelin, fouryear, threeyear, amus, sponsorship, fiveyear,
		\par
	}

	\item[$\mathcal{V}_{100} = $]
	{
		\tiny

		athlet, pro, cricket, basebal, profession, wrestl, soccer, bowl, basketbal, hockey, super, stadium,
		\par
	}

	\item[$\mathcal{V}_{101} = $]
	{
		\tiny

		final, start, late, second, end, four, earli, until, last, befor, three,
		\par
	}

	\item[$\mathcal{V}_{102} = $]
	{
		\tiny

		price, speed, level, energi, pressur, temperatur, degre, rate, water, cell, cost,
		\par
	}

	\item[$\mathcal{V}_{103} = $]
	{
		\tiny

		stage, fight, defeat, race, club, battl, victori, player, competit, match, win,
		\par
	}

	\item[$\mathcal{V}_{104} = $]
	{
		\tiny

		pass, come, turn, go, came, get, run, date, just, move, went,
		\par
	}

	\item[$\mathcal{V}_{105} = $]
	{
		\tiny

		fa, afc, stanley, middleweight, fifa, intercontinent, sprint, nfc, uefa, heavyweight, costa,
		\par
	}

	\item[$\mathcal{V}_{106} = $]
	{
		\tiny

		charact, book, music, role, featur, stori, star, titl, direct, version,
		\par
	}

	\item[$\mathcal{V}_{107} = $]
	{
		\tiny

		larger, less, below, greater, higher, smaller, abov, reduc, lower, low,
		\par
	}

	\item[$\mathcal{V}_{108} = $]
	{
		\tiny

		so, i, we, could, what, do, did, you, if, like,
		\par
	}

	\item[$\mathcal{V}_{109} = $]
	{
		\tiny

		negoti, firm, provis, proceed, jurisdict, impos, amend, violat, prohibit, enforc,
		\par
	}

	\item[$\mathcal{V}_{110} = $]
	{
		\tiny

		prize, outstand, honor, silver, honour, bronz, ceremoni, medal, golden,
		\par
	}

	\item[$\mathcal{V}_{111} = $]
	{
		\tiny

		channel, movi, launch, host, entertain, news, media, broadcast, sport,
		\par
	}

	\item[$\mathcal{V}_{112} = $]
	{
		\tiny

		consecut, eighth, tenth, ninth, seventh, rd, sixth, fifth, nd,
		\par
	}

	\item[$\mathcal{V}_{113} = $]
	{
		\tiny

		children, himself, young, man, women, age, men, death, old,
		\par
	}

	\item[$\mathcal{V}_{114} = $]
	{
		\tiny

		annual, total, tax, percent, highest, rise, estim, averag, growth,
		\par
	}

	\item[$\mathcal{V}_{115} = $]
	{
		\tiny

		socialist, reform, labour, communist, liber, democrat, republican, labor, conserv,
		\par
	}

	\item[$\mathcal{V}_{116} = $]
	{
		\tiny

		former, attend, chief, assist, elect, board, serv, head, appoint,
		\par
	}

	\item[$\mathcal{V}_{117} = $]
	{
		\tiny

		spent, roughli, almost, squar, spend, approxim, worth, nearli, {\textsterling},
		\par
	}

	\item[$\mathcal{V}_{118} = $]
	{
		\tiny

		academ, environment, scientif, technic, medic, primari, financi, health, secondari,
		\par
	}

	\item[$\mathcal{V}_{119} = $]
	{
		\tiny

		outbreak, gulf, vietnam, cold, iraq, korean, tag, revolutionari,
		\par
	}

	\item[$\mathcal{V}_{120} = $]
	{
		\tiny

		establish, offici, church, rule, offic, author, independ, parti,
		\par
	}

	\item[$\mathcal{V}_{121} = $]
	{
		\tiny

		ball, roll, walk, cut, step, break, shot,
		\par
	}

	\item[$\mathcal{V}_{122} = $]
	{
		\tiny

		weight, size, length, distanc, enough, scale, amount,
		\par
	}

	\item[$\mathcal{V}_{123} = $]
	{
		\tiny

		indi, korea, asian, coast, wale, asia, carolina,
		\par
	}

	\item[$\mathcal{V}_{124} = $]
	{
		\tiny

		doubl, femal, male, promin, adult, guest, cast,
		\par
	}

	\item[$\mathcal{V}_{125} = $]
	{
		\tiny

		champion, cup, premier, coach, footbal, championship,
		\par
	}

	\item[$\mathcal{V}_{126} = $]
	{
		\tiny

		song, top, singl, album, track, band,
		\par
	}

	\item[$\mathcal{V}_{127} = $]
	{
		\tiny

		command, polic, staff, execut, post, box,
		\par
	}

	\item[$\mathcal{V}_{128} = $]
	{
		\tiny

		green, white, red, blue, black, gold,
		\par
	}

	\item[$\mathcal{V}_{129} = $]
	{
		\tiny

		hit, uk, studio, concert, debut, tour,
		\par
	}

	\item[$\mathcal{V}_{130} = $]
	{
		\tiny

		versa, latterday, rico, nadu, rica, rican,
		\par
	}

	\item[$\mathcal{V}_{131} = $]
	{
		\tiny

		y, grand, et, del, el, al,
		\par
	}

	\item[$\mathcal{V}_{132} = $]
	{
		\tiny

		say, said, {\textquoteright}, believ, love,
		\par
	}

	\item[$\mathcal{V}_{133} = $]
	{
		\tiny

		museum, contemporari, fine, master, martial,
		\par
	}

	\item[$\mathcal{V}_{134} = $]
	{
		\tiny

		american, west, east, south, north,
		\par
	}

	\item[$\mathcal{V}_{135} = $]
	{
		\tiny

		america, africa, india, african, bank,
		\par
	}

	\item[$\mathcal{V}_{136} = $]
	{
		\tiny

		singer, musician, actress, writer, actor,
		\par
	}

	\item[$\mathcal{V}_{137} = $]
	{
		\tiny

		secur, task, defens, reserv, defenc,
		\par
	}

	\item[$\mathcal{V}_{138} = $]
	{
		\tiny

		through, back, away, down, off,
		\par
	}

	\item[$\mathcal{V}_{139} = $]
	{
		\tiny

		third, regular, finish, rank, fourth,
		\par
	}

	\item[$\mathcal{V}_{140} = $]
	{
		\tiny

		half, largest, round, decad, quarter,
		\par
	}

	\item[$\mathcal{V}_{141} = $]
	{
		\tiny

		now, becom, best, becam, well,
		\par
	}

	\item[$\mathcal{V}_{142} = $]
	{
		\tiny

		artist, rock, style, director, video,
		\par
	}

	\item[$\mathcal{V}_{143} = $]
	{
		\tiny

		olymp, theatr, festiv, airport, trade,
		\par
	}

	\item[$\mathcal{V}_{144} = $]
	{
		\tiny

		yard, field, score, touchdown, goal,
		\par
	}

	\item[$\mathcal{V}_{145} = $]
	{
		\tiny

		prix, juri, slam, duchi, testament,
		\par
	}

	\item[$\mathcal{V}_{146} = $]
	{
		\tiny

		motor, magnet, real, nuclear, electr,
		\par
	}

	\item[$\mathcal{V}_{147} = $]
	{
		\tiny

		million, around, popul, us,
		\par
	}

	\item[$\mathcal{V}_{148} = $]
	{
		\tiny

		british, union, presid, govern,
		\par
	}

	\item[$\mathcal{V}_{149} = $]
	{
		\tiny

		earl, birthday, grade, anniversari,
		\par
	}

	\item[$\mathcal{V}_{150} = $]
	{
		\tiny

		univers, art, law, student,
		\par
	}

	\item[$\mathcal{V}_{151} = $]
	{
		\tiny

		billboard, hot, peak, chart,
		\par
	}

	\item[$\mathcal{V}_{152} = $]
	{
		\tiny

		usual, often, thu, sometim,
		\par
	}

	\item[$\mathcal{V}_{153} = $]
	{
		\tiny

		season, leagu, war, team,
		\par
	}

	\item[$\mathcal{V}_{154} = $]
	{
		\tiny

		fiction, technolog, polici, care,
		\par
	}

	\item[$\mathcal{V}_{155} = $]
	{
		\tiny

		recept, review, prais, acclaim,
		\par
	}

	\item[$\mathcal{V}_{156} = $]
	{
		\tiny

		francisco, angel, diego, kong,
		\par
	}

	\item[$\mathcal{V}_{157} = $]
	{
		\tiny

		divis, minut, overal,
		\par
	}

	\item[$\mathcal{V}_{158} = $]
	{
		\tiny

		militari, armi, depart,
		\par
	}

	\item[$\mathcal{V}_{159} = $]
	{
		\tiny

		opera, marin, palac,
		\par
	}

	\item[$\mathcal{V}_{160} = $]
	{
		\tiny

		nomin, academi, won,
		\par
	}

	\item[$\mathcal{V}_{161} = $]
	{
		\tiny

		n, m, t,
		\par
	}

	\item[$\mathcal{V}_{162} = $]
	{
		\tiny

		lake, britain, deal,
		\par
	}

	\item[$\mathcal{V}_{163} = $]
	{
		\tiny

		long, much, far,
		\par
	}

	\item[$\mathcal{V}_{164} = $]
	{
		\tiny

		determin, vari, depend,
		\par
	}

	\item[$\mathcal{V}_{165} = $]
	{
		\tiny

		known, took, take,
		\par
	}

	\item[$\mathcal{V}_{166} = $]
	{
		\tiny

		agreement, contract, treati,
		\par
	}

	\item[$\mathcal{V}_{167} = $]
	{
		\tiny

		copi, billion, per,
		\par
	}

	\item[$\mathcal{V}_{168} = $]
	{
		\tiny

		televis, episod, tv,
		\par
	}

	\item[$\mathcal{V}_{169} = $]
	{
		\tiny

		day, month, week,
		\par
	}

	\item[$\mathcal{V}_{170} = $]
	{
		\tiny

		washington, district, suprem,
		\par
	}

	\item[$\mathcal{V}_{171} = $]
	{
		\tiny

		soviet, feder, european,
		\par
	}

	\item[$\mathcal{V}_{172} = $]
	{
		\tiny

		religi, indigen, ethnic,
		\par
	}

	\item[$\mathcal{V}_{173} = $]
	{
		\tiny

		counti, colleg, california,
		\par
	}

	\item[$\mathcal{V}_{174} = $]
	{
		\tiny

		mid, bc,
		\par
	}

	\item[$\mathcal{V}_{175} = $]
	{
		\tiny

		critic, posit,
		\par
	}

	\item[$\mathcal{V}_{176} = $]
	{
		\tiny

		retain, assum,
		\par
	}

	\item[$\mathcal{V}_{177} = $]
	{
		\tiny

		ottoman, roman,
		\par
	}

	\item[$\mathcal{V}_{178} = $]
	{
		\tiny

		minist, prime,
		\par
	}

	\item[$\mathcal{V}_{179} = $]
	{
		\tiny

		ii, civil,
		\par
	}

	\item[$\mathcal{V}_{180} = $]
	{
		\tiny

		radio, railway,
		\par
	}

	\item[$\mathcal{V}_{181} = $]
	{
		\tiny

		de, la,
		\par
	}

	\item[$\mathcal{V}_{182} = $]
	{
		\tiny

		air, arm,
		\par
	}

	\item[$\mathcal{V}_{183} = $]
	{
		\tiny

		comic, publish,
		\par
	}

	\item[$\mathcal{V}_{184} = $]
	{
		\tiny

		columbia, dc,
		\par
	}

	\item[$\mathcal{V}_{185} = $]
	{
		\tiny

		station, network,
		\par
	}

	\item[$\mathcal{V}_{186} = $]
	{
		\tiny

		lanka, lankan,
		\par
	}

	\item[$\mathcal{V}_{187} = $]
	{
		\tiny

		cathol, empir,
		\par
	}

	\item[$\mathcal{V}_{188} = $]
	{
		\tiny

		court, school,
		\par
	}

	\item[$\mathcal{V}_{189} = $]
	{
		\tiny

		high, public,
		\par
	}

	\item[$\mathcal{V}_{190} = $]
	{
		\tiny

		lo, hong,
		\par
	}

	\item[$\mathcal{V}_{191} = $]
	{
		\tiny

		{\textquotedblleft}, {\textquotedblright},
		\par
	}

	\item[$\mathcal{V}_{192} = $]
	{
		\tiny

		st,
		\par
	}

	\item[$\mathcal{V}_{193} = $]
	{
		\tiny

		san,
		\par
	}

	\item[$\mathcal{V}_{194} = $]
	{
		\tiny

		fame,
		\par
	}

	\item[$\mathcal{V}_{195} = $]
	{
		\tiny

		award,
		\par
	}

	\item[$\mathcal{V}_{196} = $]
	{
		\tiny

		place,
		\par
	}

	\item[$\mathcal{V}_{197} = $]
	{
		\tiny

		wide,
		\par
	}

	\item[$\mathcal{V}_{198} = $]
	{
		\tiny

		centuri,
		\par
	}

	\item[$\mathcal{V}_{199} = $]
	{
		\tiny

		th,
		\par
	}

	\item[$\mathcal{V}_{200} = $]
	{
		\tiny

		$\emptyset$
	}

\end{itemize}

\subsection{Companies with the highest daily returns}
\label{sec:Appendix__raw_data__Companies_with_the_highest_daily_returns}

\subsubsection{Sector breakdown within the \texorpdfstring{\gls{SP500}}{SP500} and the dataset}
\label{sec:Appendix_raw_data__Stock_market__Sector_breakdown_within_the_SP500_and_the_dataset}

Table \ref{tab:Sector_breakdown_of_the_SP500} contains the sector breakdown within the dataset as well as the \gls{SP500}.
These numbers are based on \gls{SP500}'s factsheet from 2022.

\begin{table}[hbtp]
	\centering
	\begin{tabular}{p{5cm}p{1.5cm}p{4cm}}
		\toprule
		Sector & Weight & Percentage in dataset \\
		\midrule
		Industrials & 7.8\% & 16.3\% \\
		Health Care & 12.7\% & 11.0\% \\
		Information Technology & 29.3\% & 10.0\% \\
		Consumer Discretionary & 13.2\% & 10.7\% \\
		Communication Services & 10.4 & 3.0 \\
		Consumer Staples & 5.6\% & 10.7\% \\
		Utilities & 2.4\% & 7.7\% \\
		Financials & 10.8\% & 14.0\% \\
		Materials & 2.5\% & 6.3\% \\
		Real Estate & 2.6\% & 7.3\% \\
		Energy & 2.7\% & 5.3\% \\
		\bottomrule
		\vspace{0pt}
	\end{tabular}
	\caption{Sector breakdown of the \gls{SP500} by index weight, together with the relative percentages of each sector within the $300$ constituents considered in the dataset.}
	\label{tab:Sector_breakdown_of_the_SP500}
\end{table}

\subsubsection{Ticker symbols of the \texorpdfstring{$300$}{300} constituents}
\label{sec:Appendix_raw_data__Stock_market__Ticker_symbols_of_the_300_constituents}

Here are the ticker symbols of the $300$ companies that we considered, in order:

\begin{itemize}[noitemsep]
	\item[$\mathcal{V} = $]
	{
		\tiny

		IP, CB, ZBH, AAPL, GS, IBM, AMGN, MMM, CVX, FDX, COST, CMI, UNP, AVB, BLK, SPG, HD, LMT, JNJ, KMB, JPM, GD, MCK, ESS, CI, UNH, CSCO, PXD, MCD, NVDA, INTC, PSA, MTB, HON, BXP, GWW, NOC, TMO, BA, INTU, APD, TRV, RTX, PEP, CAT, AMAT, TXN, ORCL, WHR, BDX, PPG, QCOM, SHW, UPS, PH, LRCX, PFE, NSC, HUM, ECL, DE, ADP, GE, SRE, ROK, WMT, EOG, MLM, PG, RE, DIS, NEE, T, ITW, KLAC, XOM, PNC, RL, AON, EA, LOW, BAC, AXP, VZ, CMCSA, SYK, EBAY, STZ, WFC, HPQ, ROP, AMT, ABT, CLX, BEN, C, LLY, SNA, SWK, MS, CTXS, KSU, MCO, MRK, EL, FRT, KO, HAL, APA, WM, SJM, ADI, DHR, FCX, JCI, VMC, MSI, IFF, SBUX, GILD, CTAS, CVS, ALL, UHS, COO, SLB, MMC, TT, MCHP, NLOK, MDT, HSY, TGT, BMY, TROW, NKE, USB, COP, EFX, XLNX, MO, DRI, ROST, DTE, JNPR, CCI, BBY, NTAP, DUK, OXY, TFX, VLO, LHX, PAYX, FITB, SBAC, ETR, GPS, COF, NEM, MRO, KR, YUM, CL, DD, DGX, WBA, SCHW, MAR, GLW, SO, BK, JBHT, NTRS, PGR, TJX, AIG, ADM, HWM, A, CAG, STT, SYY, HES, ABC, DOV, CSX, EIX, GIS, TFC, WMB, NUE, VFC, ETN, BAX, EMR, EXC, JKHY, CAH, AEP, AFL, XEL, TECH, ATVI, ARE, DVN, HIG, AVY, NWL, WY, OMC, PCAR, FE, D, MAS, POOL, LEN, BBWI, VNO, EQR, NI, TER, CPB, DHI, PEG, K, LUMN, PPL, HAS, MU, MKC, PLD, LNC, ZION, ED, APH, MGM, CNP, PVH, CMA, CTSH, EMN, FAST, TSN, IEX, RSG, AEE, EXPD, TSCO, TXT, ES, CINF, MOS, CHRW, CERN, PBCT, RCL, UDR, CTRA, PEAK, TAP, CCL, SEE, KIM, ALB, KEY, XRAY, RMD, STE, DRE, BWA, WEC, RHI, GPC, FMC, L, J, LEG, OKE, MAA, CMS, PHM, VTR, IRM, PKI, O, ODFL, SWKS, AES, HRL, BLL, AME, AJG, IVZ, RJF, PNR, GL, LUV, IPG, PNW.
		\par
	}
\end{itemize}

\subsubsection{Processed sequence of observations}
\label{sec:Appendix_raw_data__Stock_market__Processed_sequence_of_observations}

Here is the complete, processed sequence of observations:
\begin{itemize}[noitemsep]
	\item[$X_{1:\ell} = $]
	{
		\tiny

		ADI, AES, PVH, HUM, NTAP, AMT, EBAY, NTAP, J, RL, PVH, ROST, ODFL, DVN, EOG, XLNX, ODFL, LOW, A, INTU, CCI, NTAP, ODFL, ATVI, TSCO, EBAY, STE, BLL, MLM, EXPD, CCI, ODFL, LUV, J, AAPL, ZBH, HAS, TGT, ROK, AJG, CTXS, ODFL, NEM, NLOK, ATVI, SWKS, MSI, SWKS, EA, SWKS, MGM, MU, NLOK, JNPR, NEM, JNPR, CTXS, SCHW, SWKS, JNPR, XLNX, CTSH, TAP, USB, FCX, NTAP, NI, JNPR, HPQ, QCOM, SBAC, TSCO, PVH, SBAC, CTSH, SBAC, TSCO, SWKS, MOS, TER, SBAC, MKC, WMB, FCX, RMD, PKI, ODFL, USB, MRO, NTAP, ODFL, NTAP, ODFL, DRI, SWKS, AMT, WMB, EL, TSN, CTSH, ATVI, AES, GLW, ODFL, JNPR, TER, JBHT, HAL, ODFL, HAL, SWKS, HUM, ZBH, NVDA, SWKS, HUM, ABC, STE, HAL, QCOM, RCL, FCX, JCI, LMT, STE, AAPL, HUM, MOS, CCI, GLW, CCI, SBAC, MAR, PVH, PAYX, SJM, STE, TT, ROP, AES, WMB, ATVI, AMT, CCI, JNPR, WMB, SBAC, AES, JNPR, CTXS, ODFL, PNR, ODFL, AES, RSG, FCX, QCOM, PHM, JKHY, FMC, TSCO, JNPR, QCOM, PXD, AES, BMY, ODFL, NVDA, POOL, TSN, MKC, CCI, LRCX, SBAC, AMT, UNH, JKHY, MSI, SBAC, APD, NTAP, STE, NVDA, JCI, CMCSA, JCI, RMD, HPQ, EFX, NTAP, RL, RSG, NTAP, SWKS, CNP, JCI, TECH, CNP, NEM, WMB, SWKS, CNP, NEM, ODFL, NTAP, ROK, SJM, CTSH, ORCL, ODFL, MCHP, JCI, HUM, AES, JCI, AMGN, SBAC, JBHT, AMT, LUV, AMAT, CCI, AES, AMT, JCI, NEM, JNPR, ORCL, LRCX, CNP, CTXS, UPS, JCI, ROP, NTAP, AMT, SBAC, COF, AES, SBAC, JKHY, WMB, CNP, CCI, WMB, NWL, GLW, MO, MGM, IPG, LRCX, XEL, CCI, AON, SWKS, SBAC, CCI, GLW, WMB, AES, TSN, GLW, TSN, AMT, GLW, SBAC, CCI, NLOK, BBY, JKHY, COF, NVDA, AES, MOS, AES, HON, JCI, SWKS, PEP, DUK, WFC, EIX, MCHP, CNP, MCHP, WMB, DD, MU, VZ, GPS, XEL, BK, DGX, SBAC, AES, AMT, TSCO, APH, MU, AES, XEL, IPG, AMT, CCI, SBAC, WM, PKI, SBAC, AES, JNPR, NI, LMT, SWKS, CNP, LLY, WMB, NVDA, SBAC, WMB, SJM, AMT, SBAC, TER, AMT, SBAC, MSI, TER, HAL, AES, DUK, AES, STE, AMT, AES, LRCX, SBAC, AES, WMB, SEE, CCI, DRI, ODFL, NKE, SWKS, AES, SBAC, ATVI, TGT, WMB, SBAC, CERN, SBAC, GLW, TECH, AMT, LHX, WMB, GLW, STE, SBAC, WMB, SBUX, SBAC, AES, SBAC, TER, VLO, WMB, PKI, NUE, SBAC, GLW, AES, A, AES, ADI, WMB, DGX, SBAC, UHS, WMB, CMCSA, MU, CTAS, BBWI, GLW, CCI, WMB, IPG, GILD, TSCO, IPG, AME, ODFL, SBAC, WMB, SBAC, NTAP, SBAC, FCX, CVS, CMS, CTXS, FAST, AES, MCD, SYK, CCI, TGT, SBAC, CMI, CMS, JKHY, WMB, SBAC, ROK, SBAC, BDX, CMS, TAP, AMT, TECH, AES, SBAC, MCD, CCI, ROST, SBAC, HIG, SWKS, MSI, CTSH, AES, SBAC, INTU, MO, NI, SBAC, TXN, GLW, BBWI, SBAC, MU, CTSH, BBY, MCD, SWKS, MU, AMT, ODFL, PGR, BAX, HON, TER, SWKS, GILD, GPS, SBAC, CMS, CTSH, SBAC, CTXS, CCI, SBAC, SWKS, SBAC, TER, HPQ, GILD, LRCX, BLL, CERN, HAL, TT, CTSH, TSCO, POOL, FAST, CCI, SBAC, CI, CMI, XEL, SBAC, EXPD, RL, BBY, NKE, LRCX, TECH, KSU, TSN, TXN, MU, SBAC, NTAP, NLOK, JNPR, FCX, LRCX, AAPL, PKI, COO, J, COO, UHS, SBAC, NVDA, CMS, ROK, FE, RL, AES, CTSH, SCHW, SLB, SBAC, NLOK, MKC, LMT, AES, GLW, AMT, DHI, CCI, FAST, SBAC, SWKS, ODFL, TJX, CTSH, SBAC, IPG, GWW, EFX, SBAC, CERN, RE, POOL, PHM, STE, RCL, TER, CCI, AES, CI, MMC, EXPD, TXN, TJX, FCX, L, ABC, CERN, ABT, PGR, HAL, FCX, UNH, ATVI, RCL, NTAP, MOS, EA, LRCX, MOS, J, MOS, EOG, KSU, SBAC, CCL, MCHP, SWKS, AES, NEM, RSG, IPG, CTAS, CERN, MCD, SBAC, AAPL, NVDA, SBAC, FMC, MCK, SBAC, IRM, GILD, AMT, SBAC, JNPR, SBAC, TMO, TER, JNPR, LHX, TSCO, CAH, TSCO, TSN, AVY, PKI, LEG, RMD, TSN, PBCT, RMD, TER, WMB, EXPD, FCX, TFX, SBUX, MU, FCX, HSY, MCD, ADI, SWKS, PHM, AES, MCHP, CCI, COO, CCI, AAPL, BDX, AAPL, NLOK, TECH, FAST, HUM, ODFL, CTXS, TSN, MCD, SWK, NVDA, TER, PNR, FMC, NVDA, PBCT, BLK, TRV, SBAC, JBHT, PEAK, CERN, VLO, VTR, ABT, GWW, SBAC, ODFL, MCK, MSI, MCHP, BXP, TROW, AES, LUMN, MCK, EOG, JNPR, SWKS, DVN, TSN, FCX, DHI, FITB, NTAP, CMI, NTAP, NUE, ROST, ECL, PVH, PXD, ODFL, GLW, STE, NVDA, EXPD, AAPL, EA, JBHT, SBAC, MGM, EL, CMS, STZ, AME, AAPL, LEG, CNP, MOS, SBAC, JNPR, SCHW, A, COO, TER, HUM, WMB, MOS, SYK, FCX, XLNX, CTSH, TGT, HUM, CAH, UNH, PXD, MCHP, NTAP, AME, LRCX, HWM, PCAR, CCI, PKI, ZBH, SBAC, ADM, OKE, SWKS, ROST, GPS, MOS, TSCO, ADI, TGT, ADI, LUV, GLW, SWKS, INTU, SWKS, PHM, FMC, ATVI, AAPL, BLK, MOS, COO, ROST, COO, PHM, SBAC, NVDA, SBAC, MCHP, CAH, VTR, IEX, CTAS, TER, PXD, FITB, DRI, HAL, MRK, FMC, SBAC, CTSH, CTXS, SWKS, ORCL, DHR, ROP, SBAC, GLW, SBAC, AAPL, SBAC, GPC, AJG, SWK, TER, MCHP, MMC, VFC, OMC, TMO, IRM, PKI, HUM, AON, EXPD, MO, RJF, ROK, PVH, ATVI, PVH, HRL, EA, LHX, MO, AMAT, EOG, AAPL, CTRA, MOS, NEM, SBAC, LRCX, NTAP, SBAC, NTAP, CL, FCX, XLNX, ATVI, MOS, VTR, PHM, MRK, LEG, OXY, KSU, ODFL, EXPD, CAG, CTSH, MU, TER, HRL, EFX, HSY, EXPD, MGM, AAPL, ODFL, ATVI, HAL, ATVI, PXD, ADM, TGT, INTU, SBAC, ETR, J, AME, NVDA, MDT, NUE, VLO, RL, TSCO, RMD, TSCO, VFC, AON, CTSH, NTAP, ATVI, A, WMB, TSN, MRK, NEM, MMC, PHM, SLB, FMC, CTXS, TECH, FCX, JNPR, TAP, PVH, ADI, SBAC, HAL, MGM, VLO, DVN, HAL, NVDA, PAYX, VLO, FMC, AIG, HPQ, TXN, CMS, VLO, MCK, ODFL, VLO, RMD, TSCO, EOG, SEE, LLY, SYK, ABT, VLO, PVH, GLW, JKHY, GILD, ATVI, PKI, GLW, MCHP, HUM, EOG, PVH, EBAY, CCI, ATVI, GLW, SBAC, GLW, TMO, ORCL, LOW, A, YUM, AAPL, JNPR, AES, SWKS, PVH, CCI, STZ, PEAK, GPS, ZBH, ATVI, SWKS, EOG, MGM, MMC, EA, BBY, SBAC, POOL, TECH, IPG, KR, CCI, DVN, SBAC, STZ, PAYX, VTR, MKC, CTRA, ROK, IVZ, SCHW, TER, EA, RHI, CMI, CTAS, SWK, EA, TER, EBAY, TSCO, WHR, PNR, MCHP, IRM, ATVI, TSN, TAP, PBCT, ROST, NVDA, AME, CNP, APA, MCD, AAPL, NTAP, TECH, HPQ, CMS, CAT, SBAC, RMD, ATVI, ABC, HRL, LRCX, VLO, VMC, HES, COF, AAPL, MGM, COO, ATVI, CMI, RE, MU, MCD, RE, AAPL, NSC, AON, PGR, ALL, EOG, RMD, STZ, EBAY, MAR, EIX, CLX, LUV, TJX, ROST, APH, CTRA, TXN, AAPL, SWKS, MO, FITB, MGM, AAPL, SBAC, CTRA, HPQ, IRM, ADM, MSI, EXPD, JBHT, WMB, AON, GPS, SBAC, LRCX, PVH, BMY, TSN, CERN, EOG, INTU, MCHP, SBAC, DE, MMC, TER, COO, FCX, SBAC, TER, ZBH, GLW, NEM, ES, CTRA, ECL, INTU, PFE, JBHT, MO, DRI, MRK, TRV, GILD, HUM, ODFL, STZ, HES, ATVI, HES, ALB, ADM, LHX, ATVI, DHI, FCX, ATVI, NTAP, PEG, MOS, JNPR, APH, SLB, ETN, CHRW, GLW, TXT, CMI, AAPL, SBAC, MLM, JKHY, MOS, MU, RMD, VFC, TSN, AON, BMY, CTRA, ROP, VLO, COO, RHI, PHM, IEX, ROP, PAYX, JNPR, GLW, ADM, SWKS, MAA, CVS, TSN, FCX, ATVI, JNPR, SWKS, EOG, SBAC, RMD, PVH, IPG, PHM, PXD, CHRW, EBAY, WHR, NLOK, JKHY, SWKS, NVDA, PHM, CCI, AMT, MOS, UNH, HON, VFC, PHM, APH, AME, LEG, PEG, ODFL, TSN, SNA, NVDA, MOS, LRCX, EXPD, ATVI, ADM, HUM, TSN, COO, UNH, TFX, HUM, JNPR, KR, TSN, JNPR, WHR, SBAC, ATVI, RMD, OKE, RCL, JNPR, WBA, DRE, TRV, TAP, NSC, CERN, DIS, NTAP, DVN, SBAC, PXD, CCL, AES, SWKS, AAPL, JBHT, SWKS, MRO, NVDA, FCX, SBAC, WMB, POOL, MO, EIX, GPS, KLAC, ESS, CTXS, DHI, MAR, ODFL, RL, GL, CERN, HAS, ALB, XLNX, IRM, NVDA, HUM, VMC, MCO, PHM, LMT, ADM, EBAY, CERN, CHRW, NVDA, KIM, JNPR, NVDA, EBAY, NTAP, FCX, WY, SJM, MDT, SWKS, ODFL, EBAY, LRCX, RCL, SWKS, DVN, PBCT, PHM, EFX, NWL, BBY, BLK, SWKS, MGM, NUE, MMC, SWKS, OXY, BWA, LRCX, HES, EXPD, EBAY, MCO, DHI, JBHT, NVDA, CMI, DD, NVDA, MRO, XLNX, YUM, LRCX, ETN, CAH, TSCO, CHRW, CERN, EBAY, NUE, NSC, MCHP, PKI, NLOK, NVDA, D, ROST, CTRA, JNPR, BA, ADM, NUE, TMO, LRCX, HD, PNR, NTAP, EOG, LRCX, DE, FCX, PPL, CHRW, CSCO, PXD, PHM, VLO, NUE, GLW, MAS, CTRA, PCAR, NLOK, TSN, COO, WY, SWKS, RMD, NUE, HPQ, JNPR, GPS, SWKS, AAPL, A, IVZ, ODFL, GLW, MSI, GPS, AAPL, COO, BWA, MOS, GS, PHM, PEAK, CCI, COO, CSX, EBAY, FCX, KLAC, JNPR, EOG, CHRW, ODFL, POOL, CTSH, GLW, VNO, MLM, MOS, EXPD, CSX, NTAP, MOS, SJM, VMC, HAL, INTU, TAP, MAS, CI, ORCL, TSCO, KLAC, LHX, VTR, PVH, APH, EL, T, JNPR, MCO, CSX, ODFL, MLM, ROP, CSX, EBAY, LEN, ATVI, HAL, AVB, JBHT, CHRW, PXD, NVDA, AES, BBWI, CMCSA, FAST, PNC, WY, MMC, PHM, PPG, POOL, PXD, RMD, WHR, SWKS, MCHP, POOL, XRAY, WHR, IFF, MU, EXC, STE, PVH, GPS, AMGN, AMAT, AMGN, HAL, LEN, SCHW, LEG, CERN, MCO, LRCX, MOS, AVB, MOS, BMY, WMT, CTAS, CMI, COO, DGX, ADM, NEE, CTRA, MOS, PGR, NVDA, HUM, CMI, NVDA, SLB, FMC, TECH, FCX, AME, HES, SCHW, AAPL, DGX, TGT, CMI, HES, SWKS, FRT, TJX, KSU, LRCX, BLK, SHW, CAT, CI, UNP, SNA, ESS, STZ, L, TROW, FRT, HUM, ATVI, TROW, EXPD, VTR, FRT, TFX, LUMN, NVDA, A, CMA, EXC, MOS, NVDA, NUE, BBWI, ROST, LMT, NEM, APH, IVZ, CMI, FCX, RL, WY, MS, ATVI, AMGN, HUM, FCX, DHI, NWL, CMI, TAP, SWKS, EL, AES, IVZ, ALB, MLM, INTU, HAS, LEN, OXY, KSU, MGM, BEN, MOS, TJX, JBHT, NTAP, STT, NVDA, UNP, PVH, PHM, NTAP, PCAR, NEE, MOS, JNPR, MLM, NEM, XRAY, SBAC, BWA, HES, RL, ODFL, PSA, LUV, MOS, LEG, AMAT, SBUX, JCI, OXY, BBWI, TGT, HUM, SPG, ALB, ABC, DHI, IVZ, TECH, ORCL, LEN, PVH, PHM, JNPR, COO, LEN, IVZ, MAR, HES, VFC, MOS, IVZ, PEAK, CMI, PPG, TFX, VTR, NEM, J, NEE, LLY, CAH, OKE, ROST, MOS, LUV, PVH, MMC, XLNX, LOW, PHM, NTAP, ALB, LEN, SBAC, RHI, PHM, NVDA, CMI, IEX, CI, MCO, SBAC, LEN, EQR, ROK, CTRA, HRL, MOS, DHI, POOL, SBAC, MOS, LEN, JNPR, EOG, EA, NTAP, IVZ, TSN, IEX, CMA, PEAK, RJF, HUM, PHM, PEAK, JPM, MS, HUM, MS, BLK, PVH, MOS, PAYX, KIM, MU, LEN, MU, POOL, MRK, COF, CTXS, MU, LOW, POOL, HES, SCHW, KLAC, LEG, SLB, HES, FITB, RJF, TROW, MOS, COF, PVH, CMI, PHM, MRO, PXD, DHI, STE, EOG, ATVI, HWM, STZ, PHM, NVDA, LHX, VMC, HUM, CI, QCOM, HSY, MOS, BBY, CTSH, VLO, MCO, SWKS, CTRA, BAX, APA, GPS, MOS, TER, MS, ZION, ATVI, MS, CCL, KEY, HAL, LUMN, KR, MAS, MOS, STZ, RMD, DTE, ECL, MGM, LEN, PNW, HWM, FITB, PHM, USB, TFC, ZION, KEY, MOS, RCL, DRE, SNA, POOL, RE, BAC, HES, WEC, CTSH, CI, MGM, MLM, ABC, DHI, MGM, ATVI, MOS, PHM, UHS, ES, EOG, VLO, LEN, MGM, POOL, PVH, LEN, PHM, MGM, YUM, PVH, APA, CMA, WHR, AVY, MU, AIG, ZION, JBHT, AIG, NEM, EQR, RJF, NEM, VMC, LEN, TSN, JPM, RHI, STT, FITB, MU, FMC, AIG, AES, STT, GLW, LNC, UNH, KEY, WFC, PPG, SBAC, AES, KEY, TRV, MRO, PNC, ODFL, IPG, MGM, RSG, AIG, HIG, J, MGM, DHI, AES, HRL, GS, AES, ALL, HIG, PLD, STZ, GLW, IPG, KIM, LNC, LEN, LNC, CTAS, PLD, MS, HIG, DRE, NVDA, PEAK, UNH, DRE, NEM, PLD, KSU, MU, RCL, TECH, ARE, MU, PCAR, HES, IEX, IRM, MGM, MU, LNC, ATVI, DHI, GPS, COO, LNC, VFC, DRE, MU, PH, PNC, STT, HIG, VLO, ZION, FITB, ODFL, BA, ROK, DHI, APH, CI, FITB, GE, EXPD, IPG, BWA, DHI, IVZ, COF, MO, KEY, LEN, COF, FITB, USB, FITB, LEN, SWKS, GLW, SWKS, UNH, USB, DRE, MU, COF, AIG, IP, AIG, HWM, NEM, LNC, IP, PLD, NWL, LNC, HUM, HIG, TXT, BWA, MGM, TXT, UNH, ROK, DRE, FITB, ZBH, AXP, DRE, MGM, ORCL, STT, SNA, IVZ, MAS, UHS, TECH, CERN, NWL, PKI, FITB, MGM, ZION, CI, ZION, MGM, MOS, TFC, HIG, ODFL, ZION, MU, UHS, CTXS, BLK, UDR, TER, CI, CSX, PCAR, HWM, SBUX, KEY, IPG, LEN, BLK, IP, KEY, VMC, RL, CI, TJX, HUM, CI, CNP, IP, LEN, JKHY, LNC, MAA, RCL, MTB, KIM, CI, ROST, MU, AIG, IPG, AXP, WHR, AIG, VNO, IPG, KEY, HIG, COF, ZION, TXT, C, FMC, ODFL, KEY, ARE, AIG, BBWI, MAR, WMB, NSC, HIG, FITB, TFC, HUM, DE, AIG, EXPD, RMD, KIM, AIG, KSU, FITB, MOS, AIG, BBY, CMCSA, SPG, AIG, MGM, MSI, AIG, ZION, CMA, VLO, LEG, AIG, RJF, EA, CTAS, MGM, HIG, MCO, YUM, HSY, AIG, BBWI, SCHW, MU, LEN, MU, BWA, LEN, KIM, VLO, LHX, MSI, FITB, DE, LEN, TROW, ROP, UNH, AIG, IVZ, NEM, COF, AIG, PHM, AIG, SWKS, PLD, TXT, ODFL, CTAS, MAS, EL, MOS, DRE, PGR, MOS, CI, COO, BBWI, MGM, EBAY, SWKS, MOS, KLAC, RHI, CMCSA, ODFL, CMA, UNH, LUV, MGM, WHR, NVDA, ODFL, DRI, IVZ, AIG, MGM, FAST, AIG, ODFL, A, WFC, KEY, HIG, ZION, CMI, TFX, MGM, FITB, TXT, VLO, C, ZION, RJF, MGM, HAL, UHS, IVZ, CTXS, PBCT, NVDA, STE, AJG, COO, NEM, LEN, AIG, MGM, MSI, IP, GILD, TXT, JNPR, AIG, TMO, KEY, SRE, IPG, FITB, PCAR, AIG, PVH, AIG, C, PCAR, LNC, MGM, MLM, FDX, ROST, MGM, RCL, KEY, AIG, MGM, AES, MCK, PXD, MGM, ZION, AIG, MGM, J, AIG, VNO, ZION, VMC, MLM, C, ZION, CHRW, TXT, PHM, KSU, UHS, HUM, LHX, RMD, MLM, UDR, GILD, STE, AON, MAS, AIG, TXT, EA, UHS, SEE, BAX, RL, ROST, PXD, BAX, CMCSA, TXT, NTAP, CAH, HSY, HAL, SBAC, PBCT, NEM, ZION, ODFL, LEN, NTAP, AIG, TER, MCO, MTB, KR, NEE, SJM, MCO, MO, ODFL, TXT, BBWI, JNPR, JCI, AES, MCO, AES, MSI, IVZ, PGR, GS, MOS, WY, STT, ABT, FITB, TFX, AIG, ECL, CVS, CTXS, SWK, MOS, UHS, MOS, IFF, PKI, AES, CL, BWA, EL, AEE, CSCO, SWK, TGT, NTAP, INTU, MOS, LEN, LNC, DRI, LNC, ABC, HIG, APH, PVH, GS, NTAP, NLOK, RCL, TER, MU, BLK, MCK, MOS, ORCL, WHR, NVDA, HWM, KLAC, MU, LUV, NEM, CTRA, TSN, BBWI, EA, JNPR, STZ, CI, MGM, MCO, FITB, WMB, CHRW, GPC, WFC, ZION, WFC, FITB, NVDA, TSCO, AME, ODFL, FMC, HAL, MCK, VMC, MGM, DHI, BK, HAL, EOG, HAL, JNPR, CAT, TSN, RCL, HAL, CTRA, HUM, MU, SWKS, HRL, WBA, LOW, MAS, LEN, MGM, TER, JBHT, LNC, AIG, MSI, CTRA, AIG, STZ, MSI, CERN, DRE, RCL, XLNX, TER, AIG, KEY, MOS, GPS, NVDA, LNC, MSI, AIG, NVDA, IPG, NVDA, HES, NVDA, UHS, TER, NOC, LMT, LOW, WY, SEE, SHW, HAL, NWL, HES, AME, SEE, NVDA, IPG, A, HAS, LEN, RL, JNPR, CLX, MRO, IVZ, TER, NVDA, INTU, PEG, CTRA, PKI, VLO, EIX, AES, XLNX, VLO, A, EXC, LUV, NLOK, SBUX, WY, CTRA, F
		\par
	}
\end{itemize}

This sequence of observations constitutes the following sparse frequency matrix in Fig. \ref{fig:Stock_market__Aligned_timespans}.

\begin{figure}[hbtp]
	\centering
	\includegraphics[width=0.618\linewidth]{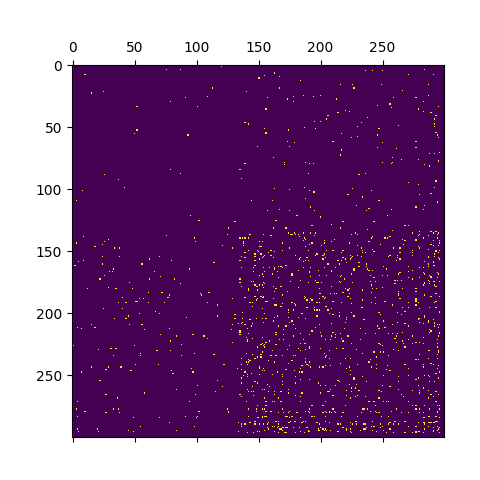}
	\caption{
		A plot of the matrix $\{ \indicator{\hat{F}_{ij} > 0} \}_{i,j}$, where the rows and columns are sorted according to the improved clustering. We plotted the matrix like such because $\hat{F}$ is quite sparse due to the trajectory's length $\ell = 2451$ being quite short: the minimum, median, mean, and maximum of the entries of the matrix $\{ \hat{F}_{i,j} \}_{i,j}$ are $0$, $0$, $\ell / n^2 \approx 0.027$, and $14$, respectively.
	}
	\label{fig:Stock_market__Aligned_timespans}
\end{figure}

\subsubsection{Detected groups}
\label{sec:Appendix_raw_data__Stock_market__Detected_groups}

Here are the detected groups in model $\hat{\mathbb{P}}$:

\begin{itemize}[noitemsep]
	\item[$\mathcal{V}_1 = $]
	{
		\tiny
		CB (Financials), GS (Financials), IBM (Information Technology), AMGN (Health Care), MMM (Industrials), CVX (Energy), FDX (Industrials), COST (Consumer Staples), UNP (Industrials), AVB (Real Estate), SPG (Real Estate), HD (Consumer Discretionary), JNJ (Health Care), KMB (Consumer Staples), JPM (Financials), GD (Industrials), ESS (Real Estate), CSCO (Information Technology), INTC (Information Technology), PSA (Real Estate), MTB (Financials), HON (Industrials), BXP (Real Estate), GWW (Industrials), NOC (Industrials), BA (Industrials), APD (Materials), TRV (Financials), RTX (Industrials), PEP (Consumer Staples), CAT (Industrials), AMAT (Information Technology), BDX (Health Care), PPG (Materials), SHW (Materials), UPS (Industrials), PH (Industrials), PFE (Health Care), NSC (Industrials), ECL (Materials), DE (Industrials), ADP (Information Technology), GE (Industrials), SRE (Utilities), WMT (Consumer Staples), PG (Consumer Staples), DIS (Communication Services), NEE (Utilities), T (Communication Services), ITW (Industrials), XOM (Energy), PNC (Financials), BAC (Financials), AXP (Financials), VZ (Communication Services), SYK (Health Care), CLX (Consumer Staples), BEN (Financials), C (Financials), LLY (Health Care), SNA (Industrials), FRT (Real Estate), KO (Consumer Staples), APA (Energy), WM (Industrials), DHR (Health Care), IFF (Materials), CVS (Health Care), ALL (Financials), TT (Industrials), MDT (Health Care), BMY (Health Care), NKE (Consumer Discretionary), COP (Energy), EFX (Industrials), DTE (Utilities), DUK (Utilities), OXY (Energy), ETR (Utilities), KR (Consumer Staples), YUM (Consumer Discretionary), CL (Consumer Staples), DD (Materials), DGX (Health Care), WBA (Consumer Staples), SO (Utilities), BK (Financials), NTRS (Financials), CAG (Consumer Staples), SYY (Consumer Staples), DOV (Industrials), EIX (Utilities), GIS (Consumer Staples), TFC (Financials), ETN (Industrials), BAX (Health Care), EMR (Industrials), EXC (Utilities), AEP (Utilities), AFL (Financials), ARE (Real Estate), AVY (Materials), OMC (Communication Services), FE (Utilities), D (Utilities), VNO (Real Estate), EQR (Real Estate), NI (Utilities), CPB (Consumer Staples), PEG (Utilities), K (Consumer Staples), LUMN (Communication Services), PPL (Utilities), HAS (Consumer Discretionary), MKC (Consumer Staples), ED (Utilities), EMN (Materials), RSG (Industrials), AEE (Utilities), ES (Utilities), CINF (Financials), UDR (Real Estate), CCL (Consumer Discretionary), XRAY (Health Care), WEC (Utilities), GPC (Consumer Discretionary), L (Financials), OKE (Energy), MAA (Real Estate), O (Real Estate), BLL (Materials), AJG (Financials), PNR (Industrials), GL (Financials), PNW (Utilities)
		\par
	}

	\item[$\mathcal{V}_2 = $]
	{
		\tiny
		IP (Materials), ZBH (Health Care), CMI (Industrials), BLK (Financials), LMT (Industrials), MCK (Health Care), CI (Health Care), UNH (Health Care), PXD (Energy), MCD (Consumer Discretionary), TMO (Health Care), INTU (Information Technology), TXN (Information Technology), ORCL (Information Technology), WHR (Consumer Discretionary), QCOM (Information Technology), ROK (Industrials), EOG (Energy), MLM (Materials), RE (Financials), KLAC (Information Technology), RL (Consumer Discretionary), AON (Financials), EA (Communication Services), LOW (Consumer Discretionary), CMCSA (Communication Services), EBAY (Consumer Discretionary), STZ (Consumer Staples), WFC (Financials), HPQ (Information Technology), ROP (Industrials), ABT (Health Care), SWK (Industrials), MS (Financials), CTXS (Information Technology), KSU (Industrials), MCO (Financials), MRK (Health Care), EL (Consumer Staples), SJM (Consumer Staples), ADI (Information Technology), JCI (Industrials), VMC (Materials), MSI (Information Technology), SBUX (Consumer Discretionary), GILD (Health Care), CTAS (Industrials), UHS (Health Care), SLB (Energy), MMC (Financials), MCHP (Information Technology), NLOK (Information Technology), HSY (Consumer Staples), TGT (Consumer Discretionary), TROW (Financials), USB (Financials), XLNX (Information Technology), MO (Consumer Staples), DRI (Consumer Discretionary), ROST (Consumer Discretionary), BBY (Consumer Discretionary), TFX (Health Care), LHX (Industrials), PAYX (Information Technology), GPS (Consumer Discretionary), COF (Financials), MRO (Energy), SCHW (Financials), MAR (Consumer Discretionary), JBHT (Industrials), PGR (Financials), TJX (Consumer Discretionary), ADM (Consumer Staples), HWM (Industrials), A (Health Care), STT (Financials), HES (Energy), ABC (Health Care), CSX (Industrials), NUE (Materials), VFC (Consumer Discretionary), JKHY (Information Technology), CAH (Health Care), XEL (Utilities), TECH (Health Care), DVN (Energy), HIG (Financials), NWL (Consumer Discretionary), WY (Real Estate), PCAR (Industrials), MAS (Industrials), POOL (Consumer Discretionary), BBWI (Consumer Discretionary), DHI (Consumer Discretionary), PLD (Real Estate), LNC (Financials), ZION (Financials), APH (Information Technology), CNP (Utilities), CMA (Financials), CTSH (Information Technology), FAST (Industrials), IEX (Industrials), EXPD (Industrials), TSCO (Consumer Discretionary), TXT (Industrials), CHRW (Industrials), CERN (Health Care), PBCT (Financials), RCL (Consumer Discretionary), CTRA (Energy), PEAK (Real Estate), TAP (Consumer Staples), SEE (Materials), KIM (Real Estate), ALB (Materials), KEY (Financials), RMD (Health Care), STE (Health Care), DRE (Real Estate), BWA (Consumer Discretionary), RHI (Industrials), FMC (Materials), J (Industrials), LEG (Consumer Discretionary), CMS (Utilities), VTR (Real Estate), IRM (Real Estate), PKI (Health Care), HRL (Consumer Staples), AME (Industrials), IVZ (Financials), RJF (Financials), LUV (Industrials)
		\par
	}

	\item[$\mathcal{V}_3 = $]
	{
		\tiny
		AAPL (Information Technology), NVDA (Information Technology), LRCX (Information Technology), HUM (Health Care), AMT (Real Estate), HAL (Energy), FCX (Materials), COO (Health Care), JNPR (Information Technology), CCI (Real Estate), NTAP (Information Technology), VLO (Energy), FITB (Financials), SBAC (Real Estate), NEM (Materials), GLW (Information Technology), AIG (Financials), WMB (Energy), ATVI (Communication Services), LEN (Consumer Discretionary), TER (Information Technology), MU (Information Technology), MGM (Consumer Discretionary), PVH (Consumer Discretionary), TSN (Consumer Staples), MOS (Materials), PHM (Consumer Discretionary), ODFL (Industrials), SWKS (Information Technology), AES (Utilities), IPG (Communication Services)
		\par
	}

\end{itemize}

\end{document}